\documentclass{article}
\usepackage{arxiv}

\usepackage{amsmath,amssymb,amsthm}

\usepackage{newpxtext,newpxmath}
\usepackage[round,authoryear]{natbib}
\usepackage{float}
\usepackage{microtype}
\usepackage{verbatim}
\usepackage{mathtools}
\usepackage{bm}
\usepackage{enumitem}
\usepackage{xcolor}
\definecolor{darkblue}{rgb}{0.0, 0.0, 0.55}
\usepackage{algorithm}
\usepackage{algorithmic}
\usepackage{bbm}
\usepackage{thmtools,thm-restate}
\usepackage{xspace}
\usepackage{xargs}
\usepackage{stmaryrd}
\usepackage{hyperref}
\hypersetup{
    colorlinks=true,
    citecolor=darkblue,
    linkcolor=darkblue,
    urlcolor=darkblue
}
\usepackage{refcount}
\usepackage{url}
\usepackage[capitalize,noabbrev]{cleveref}
\usepackage{graphicx}
\usepackage{subcaption}
\usepackage{booktabs}
\usepackage{array}
\usepackage{etoc} 
\usepackage{placeins}

\usepackage[para]{footmisc}

\DeclareCaptionLabelFormat{figandsub}{(\thefigure#2)}
\DeclareMathOperator*{\argmax}{arg\,max}
\DeclareMathOperator*{\argmin}{arg\,min}

\setlist[itemize]{itemsep=0pt, topsep=2pt, leftmargin=*}
\setlist[enumerate]{noitemsep, topsep=2pt}

\crefname{assumption}{Assumption}{Assumptions}
\Crefname{assumption}{Assumption}{Assumptions}
\crefname{corollary}{Corollary}{Corollaries}
\Crefname{corollary}{Corollary}{Corollaries}
\crefname{convention}{Convention}{Conventions}
\Crefname{convention}{Convention}{Conventions}
\crefname{note}{Note}{Notes}
\Crefname{note}{Note}{Notes}
\crefname{algorithm}{Algorithm}{Algorithms}
\Crefname{algorithm}{Algorithm}{Algorithms}
\crefname{subsection}{Subsection}{Subsections}
\Crefname{subsection}{Subsection}{Subsections}
\crefname{table}{Table}{Tables}
\Crefname{table}{Table}{Tables}

\newcommand{\pdfref}[2]{%
  \texorpdfstring{\Cref{#2}}{#1 \getrefnumber{#2}}%
}


\theoremstyle{definition}
\newtheorem{assumption}{Assumption}
\newtheorem{lemma}{Lemma}
\newtheorem{theorem}{Theorem}
\newtheorem{remark}{Remark}

\newtheorem{definition}{Definition}
\newtheorem{corollary}{Corollary}
\newtheorem{proposition}{Proposition}
\newtheorem{example}{Example}
\newtheorem{convention}{Convention}

\newcommand{\bE}{\mathbb{E}}
\newcommand{\bP}{\mathbb{P}}
\newcommand{\bR}{\mathbb{R}}

\newcommand{\bN}{\mathbb{N}}

\newcommand{\mM}{\mathcal{M}}
\newcommand{\mS}{\mathcal{S}}
\newcommand{\mA}{\mathcal{A}}
\newcommand{\mB}{\mathcal{B}}
\newcommand{\mF}{\mathcal{F}}
\newcommand{\mG}{\mathcal{G}}
\newcommand{\mX}{\mathcal{X}}
\newcommand{\mY}{\mathcal{Y}}
\newcommand{\mZ}{\mathcal{Z}}

\newcommand{\spn}{\mathrm{sp}}
\newcommand{\asp}{\mathrm{asp}}
\newcommand{\op}{\mathrm{op}}
\newcommand{\dom}{\mathrm{dom}}

\newcommand{\mT}{\mathcal{T}}

\newcommand{\Str}{\mS_{\mathrm{tr}}}

\newcommand{\KL}{\text{KL}}

\newcommand{\ones}{\mathbf{1}}

\newcommand{\indicator}{\mathbbm{1}}
\newcommand{\pluseq}{\mathrel{+}=}

\setlist[enumerate,1]{leftmargin=2.0em}

\title{
Mixing Regularization for Average-Reward $Q$-Learning and $Q$-Planning
}

\author{
  Julien Roger \\
  \'Ecole polytechnique \\
  \texttt{julien.roger@polytechnique.org} \\
  \And
  Sam Bodmer \\
  New York University \\
  \texttt{sgb329@nyu.edu} \\
  \AND
  Jian Qian \\
  The University of Hong Kong \\
  \texttt{jianqian@hku.hk} \\
  \And
  Benjamin Riviere \\
  New York University \\
  \texttt{bpr3966@nyu.edu} \\
}

\renewcommand{\undertitle}{}
\renewcommand{\headeright}{}
\date{}

\begin{document}
\maketitle
\begin{abstract}
    The average-reward setting for reinforcement learning is challenging because, without discounting, the span-contraction of the Bellman operator depends on the mixing properties of the underlying Markov dynamics. Whereas existing work typically assumes contraction, for instance through environment mixing, we study regularization toward a given reward-agnostic, $m$-step mixing policy to control the mixing properties of the learned policy. Our central mixing lemma quantifies how the reference policy's mixing properties transfer to the resulting learned policies. In particular, a single regularizer-dependent quantity simultaneously controls the contraction of the regularized Bellman operator and the exploration of the resulting policies. A second regularizer-dependent quantity controls the bias introduced by regularization, revealing a trade-off between convergence speed and approximation accuracy. Building on this result, we establish finite-time guarantees for regularized $Q$-iteration, synchronous and single-trajectory $Q$-learning, and $Q$-planning. Under one-step mixing, our bounds for approximating the regularized fixed point match the classical model-query bound for relative value iteration and, up to logarithmic factors, an existing sample-complexity bound for synchronous $Q$-learning. We also establish finite-time guarantees under multi-step mixing. For $Q$-planning, we characterize how the planning horizon trades off boundary bias against sample complexity. Finally, we present numerical studies validating the analysis: (i) knowledge of a strong mixing policy improves convergence and exploration, with larger benefits under more severe bottlenecks; and (ii) the observed rankings of $\alpha$-divergence regularizers by bias and convergence speed are consistent with our theoretical analysis.
\end{abstract}

\etocdepthtag.toc{mtmain}

\section{Introduction}
\label{sec:intro}
The average-reward setting of reinforcement learning is challenging because, without a discount factor, the Bellman operator need not be contractive. Contraction can instead be established through mixing properties of the closed-loop dynamics induced by the transition kernel and the policy, where mixing is the property of a Markov chain that describes how quickly it forgets its initial state over time. A common approach to finite-time analysis is to assume quantitative mixing conditions, such as uniformly bounded mixing times over a class of policies~\citep{jin2021towards}, or to directly assume span contraction of a multi-step Bellman operator~\citep{zhang2021finite}. The resulting guarantees then depend on the corresponding mixing or contraction parameters.

However, the design of algorithms that exploit knowledge of a strongly mixing policy to improve their own convergence and exploration properties remains relatively understudied. We regularize toward a single reward-agnostic, $m$-step mixing reference policy to control the mixing properties of the policies generated during learning. Access to an existing policy is motivated by policy fine-tuning~\citep{xie2021policy}, including robotic manipulation, where pretrained policies can be adapted to new tasks through online reinforcement learning~\citep{yang2024robot}.

Our approach also connects to entropy-regularized algorithms such as Soft Actor-Critic~\citep{haarnoja2018soft} and variants of Proximal Policy Optimization with an entropy bonus~\citep[Eq.~(9)]{schulman2017proximal}. On finite action spaces, an entropy bonus is equivalent, up to an additive constant, to a KL penalty toward the uniform policy. We extend this principle from regularizing toward a uniform random policy to a reference policy that mixes the state chain.

This setting raises two questions: 
\begin{center}
    \textit{
    How does regularization toward a mixing policy control contraction, exploration, and bias, and ultimately the sample complexity of average-reward $Q$-learning? 
    How do these insights inform the design of regularizers?
    }
\end{center}

We answer these by providing a framework for mixing regularization in average-reward reinforcement learning, and we derive its algorithmic consequences. Our main contributions are:
\begin{itemize}
    \item \textbf{A framework for mixing regularization.} Our central result is a mixing lemma that quantifies how the mixing properties of the reference policy transfer to the dynamics induced by regularized greedy policies. The key technical step is to express the difference between two regularized Bellman updates using the transition kernel of an average of greedy policies along the segment between their input $Q$-functions. This average inherits the same lower bound on action probabilities. The resulting domination bound yields both multi-step span contraction and state--action coverage along learning trajectories. We bound the bias relative to the unregularized solution using a second quantity that depends on the regularizer.
    \item \textbf{Finite-time guarantees for learning and planning.} We use our framework to analyze four algorithms that highlight the roles of contraction, statistical error, exploration, and localization: regularized $Q$-iteration, synchronous $Q$-learning, epoch $Q$-learning from a single trajectory, and $Q$-planning. Our guarantees measure accuracy in span relative to the regularized fixed point and hold under multi-step mixing of the reference policy. Under one-step minorization, our $Q$-iteration bound recovers the classical model-query complexity of relative value iteration, and our synchronous $Q$-learning bound matches the worst-case sample-complexity scaling of \citet{zhang2021finite} up to logarithmic factors. For epoch $Q$-planning, we quantify the rollout depth and the number of sampled transitions sufficient to approximate the regularized fixed point at the root to a given accuracy.
    \item \textbf{Empirical validation.} We empirically find that knowledge of a strongly mixing reference policy can improve both convergence and exploration, with larger gains for more severe bottlenecks. In a second experiment, larger $\alpha$ generally yields lower bias at a common temperature, while smaller $\alpha$ yields faster error decay toward the regularized fixed point at a matched bias level. Both trends are consistent with our analysis of $\alpha$-divergence regularizers.
\end{itemize}

\section{Background on Regularized Average-Reward MDPs}
\label{sec:background}

\subsection{Average-Reward MDPs}
\label{subsec:average_reward_mdps}
We consider an infinite-horizon average-reward MDP $\mM=(\mS,\mA,p,r)$ with finite state and action spaces $\mS$ and $\mA$, transition kernel $p(\cdot|s,a)$, and deterministic reward function $r:\mS\times\mA\to\bR$. A stationary policy is a map $\pi:\mS\to\Delta(\mA)$, and an action-value function is a map $Q:\mS\times\mA\to\bR$.

The gain of a policy $\pi$ is its long-run expected reward per step from state $s$, $\rho^\pi(s):=\liminf_{N\to\infty}\frac{1}{N}\bE_\pi\!\left[\sum_{n=0}^{N-1}r(s_n,a_n)\,\middle|\,s_0=s\right]$, and the optimal gain is $\rho^*(s):=\sup_\pi\rho^\pi(s)$. 

For communicating MDPs, the optimal gain $\rho^*$ is uniform across states and the average-reward optimality equation admits a solution $Q^*$~\citep{schweitzer1978functional}:
\begin{equation}
    \label{eq:optimal-eq}
    \mT(Q)(s,a):=r(s,a)+\bE_{s'\sim p(\cdot|s,a)}\!\left[\max_{a'}Q(s',a')\right],
    \qquad 
    \mT(Q^*)=Q^*+\rho^*\ones,    
\end{equation}
where $\mT$ is the Bellman optimality operator.

Adding a constant to $Q^*$ preserves the optimality equation.
We therefore measure errors using the span seminorm. For a function
$f$ on a nonempty finite set $\mX$, let $\|f\|_{\mX,\spn}:= \max_{x\in\mX} f(x) - \min_{x\in\mX} f(x)$. We omit $\mX$ when the domain is clear.

A classical sufficient condition for span contraction is the one-step Doeblin minorization. It requires every next-state distribution to contain the same component $\nu\in\Delta(\mS)$ with weight $\beta\in(0,1]$:
\begin{align}
    \label{eq:b_doeblin}
    p(s'|s,a)\geq\beta\nu(s') \qquad \text{for all } (s,a,s').
\end{align}
The operator $\mT$ is then a span contraction with factor at most $1-\beta$~\citep[Section~8.5.2]{puterman1994markov}.

\subsection{Regularized MDPs}
\label{subsec:regularized}
We use the regularized MDP framework of \citet{geist2019theory}. Let $\tau>0$ be the temperature and, for each $s\in\mS$, let $\Omega_s:\Delta(\mA)\to\bR\cup\{+\infty\}$ be a strongly convex function. Examples of regularizers $\Omega$ include $\alpha$-divergences to a reference distribution such as the forward and reverse Kullback--Leibler (KL) divergences.

Using the shorthand $Q_s=Q(s,\cdot)$ and $\pi_s=\pi(\cdot|s)$, the regularized value at state $s$ is $\Omega^*(Q)(s):=\max_{\pi_s\in\Delta(\mA)}\left\{\langle\pi_s,Q_s\rangle-\tau\Omega_s(\pi_s)\right\}$. The maximizing distribution defines the regularized greedy policy $\pi_{\Omega, Q}(\cdot|s) := \argmax_{\pi_s\in\Delta(\mA)}\left\{\langle\pi_s,Q_s\rangle-\tau\Omega_s(\pi_s)\right\}$. 

The regularized Bellman operator uses this value at the next state,
\begin{align*}
    \mT_\Omega(Q)(s,a):=r(s,a)+\bE_{s'\sim p(\cdot|s,a)}\!\left[\Omega^*(Q)(s')\right].
\end{align*}

The regularized value $\Omega^*(Q)(s)$ is the convex conjugate of $\tau\Omega_s$ evaluated at $Q_s$. Under \Cref{assumption:convex-analytic-regularity},  the maximizer is unique and equals $\pi_{\Omega,Q}(\cdot|s)=\nabla_{Q_s}\Omega^*(Q)(s)$. Regularization also introduces bias by changing the optimal solution from $Q^*$ to $Q^*_\Omega$, and we quantify this bias in later sections.

\section{Mixing-Regularized MDPs}
\label{sec:mixing_regularized}

We study regularization toward a reference policy $\mu$ whose induced Markov chain mixes in the sense of \Cref{assumption:m-step-mixing} below. Our key observation is that regularization can transfer this mixing property to the dynamics induced by the regularized greedy policies. Our mixing lemma formalizes this mechanism by showing that a single regularizer-dependent quantity controls both Bellman contraction and state--action coverage, while a second quantity enters our bound on the regularization bias. Proofs are deferred to \Cref{app:proofs_mixing}.

\subsection{Reference policy and regularity assumptions}
\label{subsec:mixing_assumptions}

We assume access to a stationary, full-support reference policy $\mu$. For a policy $\pi$, we define the state--action transition kernel
\begin{equation*}
    K_\pi(s',a'|s,a) = p(s'|s,a)\pi(a'|s')    
\end{equation*}
and let $K_\pi^m$ denote its $m$-step composition.

\begin{assumption}[Reference policy is $m$-step Doeblin mixing]
\label{assumption:m-step-mixing}
    There exist an integer $m\geq 1$, a parameter $\beta_m\in(0,1]$, and a full-support distribution $\nu\in\Delta(\mS)$ such that
    \begin{equation*}
        \sum_{a'}K_\mu^m(s',a'|s,a) \geq \beta_m\nu(s'), \quad \forall (s,a,s') \in \mS\times\mA\times\mS.
    \end{equation*}
    Furthermore, let $\eta(s,a)=\nu(s)\mu(a|s)$ and $\eta_{\min}=\min_{(s,a)}\eta(s,a)$. 
\end{assumption}

Under \Cref{assumption:m-step-mixing}, the law of $s_m$ after taking action $a$ in state $s$ and following $\mu$ thereafter contains a common component $\beta_m\nu$ for every initial pair $(s,a)$. Any dependence on $(s,a)$ is therefore confined to the remaining mass $1-\beta_m$. For $m=1$, the next-state law is $p(\cdot\mid s,a)$, so the minorization reduces to the policy-independent one-step Doeblin condition~\eqref{eq:b_doeblin}. For $m>1$, the intermediate actions are drawn from $\mu$, and the minorization constant $\beta_m$ generally depends on $\mu$. 

Since $\nu$ has full support, every state is reachable with positive probability from any initial pair in $m$ transitions. Hence the MDP is communicating and the unregularized optimality equation~\eqref{eq:optimal-eq} admits a solution $Q^*$, which need not be unique up to an additive constant.

Additionally, we impose the following regularity conditions on the regularizer.

\begin{assumption}[Regularity of the regularizer]
    \label{assumption:convex-analytic-regularity}
    The temperature $\tau$ is positive and, for every $s\in\mS$, $\Omega_s$ is strongly convex, closed, and proper.\footnote{Strict convexity suffices for our main results, as explained in \Cref{rem:strict_convexity}. Strong convexity yields the explicit Lipschitz bound for the greedy policy map in \Cref{lem:omega_star_props}.(i).}
\end{assumption}

\begin{assumption}[Admissibility of the reference policy]
    \label{assumption:reference_policy}
    For every $s\in\mS$, $\mu(\cdot|s)\in\dom(\Omega_s)$.
\end{assumption}

\Cref{assumption:reference_policy} ensures that evaluating the regularized objective at the reference policy gives a finite value in every state.

\subsection{Mixing Properties of Regularization}
\label{subsec:mixing_properties}

Our main result shows how regularizing toward a single mixing reference policy yields both span contraction of the regularized Bellman operator and state--action coverage along learning trajectories.

First, we establish a useful representation of differences between regularized Bellman updates. Given $Q,Q'\in\bR^{\mS\times\mA}$, let $Q(t):=tQ+(1-t)Q'$ for $t \in [0,1]$, and define the average-path policy: $\bar\pi_{\Omega,[Q,Q']} = \int_0^1\pi_{\Omega,Q(t)}\,\mathrm{d}t$. Since $\pi_{\Omega,Q}(\cdot| s)=\nabla_{Q_s}\Omega^*(Q)(s)$, the fundamental theorem of calculus gives
\begin{equation*}
    \Omega^*(Q)(s)-\Omega^*(Q')(s) = \int_0^1 \left \langle \pi_{\Omega,Q(t)}(\cdot|s), Q_s - Q'_s \right\rangle\mathrm dt 
    = \left\langle \bar\pi_{\Omega,[Q,Q']}(\cdot|s), Q_s - Q'_s \right\rangle.
\end{equation*}

Plugging this result into the difference of Bellman updates gives the useful factorization: 
\begin{equation*}
    \mT_\Omega(Q) - \mT_\Omega(Q') = K_{\bar\pi_{\Omega,[Q,Q']}} (Q - Q').
\end{equation*}

This representation allows us to study the contraction of $\mT_\Omega$ through the mixing properties of the average-path policy. Iterating this identity expresses the difference after $m$ Bellman updates through a product of $m$ average-path kernels. How much of the common terminal-state mass $\beta_m\nu$ guaranteed under $\mu$ does this product retain?

To lower-bound this mass, we compare the regularized greedy policies with $\mu$ over action-value functions with bounded \emph{action span}, where $\|Q\|_\asp := \max_{s\in\mS}\|Q(s,\cdot)\|_{\mA,\spn}$.

\begin{definition}[Domination profile]
    \label{def:domination_profile}
    For $R\geq 0$, the \emph{domination profile} $\gamma_\Omega(R)$
    is the largest $\gamma\in[0,1]$ such that
    \begin{equation*}
        \pi_{\Omega,Q}(a|s)\geq\gamma\,\mu(a|s)
    \end{equation*}
    for every $(s,a)$ and every $Q$ with $\|Q\|_\asp\leq R$.
\end{definition}

If $\|Q\|_\asp,\|Q'\|_\asp\le R$, the path stays within radius $R$, so averaging gives $K_{\bar\pi_{\Omega,[Q,Q']}}\ge\gamma_\Omega(R)K_\mu$.

To bound the bias induced by regularization, we first quantify its effect on a single Bellman update.

\begin{definition}[Regularization modulus]
    \label{def:regularization_modulus}
    Let $g_Q(s):=\Omega^*(Q)(s)-\max_a Q(s,a)$ be the regularization gap at state $s$. The difference between the regularized and unregularized Bellman updates is the expected gap at the next state. We define the \emph{regularization modulus} by
    \begin{equation*}
        \Lambda_\Omega(R)
        :=\sup_{\|Q\|_\asp\leq R}\|g_Q\|_{\mS,\spn}.        
    \end{equation*}
    The profile $\gamma_\Omega$ is non-increasing, and the modulus $\Lambda_\Omega$ is finite and non-decreasing.
\end{definition}

Our finite-time analysis uses the following assumption.
\begin{assumption}[Positive domination profile]
    \label{assumption:positive-domination}
    For every finite $R\geq0$, $\gamma_\Omega(R)>0$.
\end{assumption}
In our finite setting, this is equivalent to full support of every regularized greedy policy and holds for $\alpha$-divergences with $\alpha\leq1$. Otherwise, an $\epsilon$-floor mixing the greedy policy with $\mu$ ensures $\gamma_{\Omega_\epsilon}(R)\geq\epsilon>0$. Our results then apply to $\Omega_\epsilon$ with its corresponding target and constants (\Cref{subsec:epsilon-floor-wrapped-regularization}).

Under \Cref{assumption:m-step-mixing,assumption:convex-analytic-regularity,assumption:reference_policy}, the regularized optimality equation $\mT_\Omega(Q^*_\Omega)=Q^*_\Omega+\rho^*_\Omega\ones$ admits a solution\footnote{Existence holds more generally (see \Cref{subsec:existence_fixed_point}).} with unique gain $\rho^*_\Omega$. If \Cref{assumption:positive-domination} also holds, $Q^*_\Omega$ is unique up to an additive constant (\Cref{thm:regularized_fixed_point}).

\begin{restatable}[Mixing properties of $\Omega$-regularization]{lemma}{mixingmain}
    \label{lem:mixing_main}
    Under \Cref{assumption:m-step-mixing,assumption:convex-analytic-regularity}:
    \begin{enumerate}[label=(\roman*)]
        \item \emph{Contraction}: write $R[Q,Q']:=\max\{\|Q\|_\asp,\|Q'\|_\asp\}$ and let $\mT^k_\Omega$ denote the $k$-th iterate of $\mT_\Omega$. For any $Q,Q'$
        \begin{equation*}
            \|\mT^m_\Omega(Q)-\mT^m_\Omega(Q')\|_\spn\leq\left(1-\beta_m\prod_{k=1}^{m-1}\gamma_\Omega(R[\mT_\Omega^k(Q),\mT_\Omega^k(Q')])\right)\|Q-Q'\|_\spn.
        \end{equation*}
        \item \emph{Coverage}: let $(s_n,a_n)_{n\geq0}$ be a trajectory with $a_n\sim\pi_{\Omega,Q_n}(\cdot|s_n)$ and $s_{n+1}\sim p(\cdot|s_n,a_n)$, where $\mF_n:=\sigma(Q_0,(s_k,a_k)_{k\leq n})$ and $Q_n$ is $\mF_{n-1}$-measurable for every $n\geq1$. If $\gamma_1,\dots,\gamma_m\geq0$ are $\mF_t$-measurable and $\pi_{\Omega,Q_{t+k}}\geq\gamma_k\mu$ almost surely, then almost surely 
        \begin{equation*}
            \bP\big((s_{t+m},a_{t+m})=(s,a)\mid\mF_t\big)\geq\beta_m\left(\prod_{k=1}^m\gamma_k\right)\eta(s,a), \quad \text{ for every } (s,a).
        \end{equation*}
        \item \emph{Bias}: for any $Q$, the one-step operator gap satisfies $\|\mT_\Omega(Q)-\mT(Q)\|_\spn\leq\Lambda_\Omega(\|Q\|_\asp)$.
        If \Cref{assumption:reference_policy,assumption:positive-domination} hold, then for any unregularized solution $(Q^*,\rho^*)$, the action-value bias satisfies
        \begin{equation*}
            \|Q^*_\Omega-Q^*\|_\spn \leq\frac{m\,\Lambda_\Omega(R^*)}{\beta_m\gamma_\Omega(R^*)^{m-1}}, \qquad R^*:=\max\Big\{\|Q^*_\Omega\|_\asp, \max_{0\leq k<m}\|\mT^k_\Omega(Q^*)\|_\asp\Big\}.
        \end{equation*}
        If also $|\mS|\geq2$ and all $\Omega_s$ have the same minimum value $\underline{\Omega}$ (zero for divergences to $\mu$), the gain bias satisfies
        \begin{equation*}
            -\Lambda_\Omega(\|Q^*\|_\asp)\leq \rho^*_\Omega-\rho^*+\tau\underline{\Omega}\leq0.
        \end{equation*}
    \end{enumerate}
\end{restatable}

\begin{figure}[!htbp]
    \centering
    \caption{Illustration of \Cref{lem:mixing_main}. (a) The greedy path and its average lie in the region of policies that dominate $\gamma_\Omega(R_0)\mu$. (b) The orange area shows the guaranteed share of the reference kernel, which increases with $\gamma$. (c) At a fixed radius $R_0$, increasing the temperature $\tau$ increases the domination profile, tightens the contraction bound, and increases the bias.}
    \begin{subfigure}[t]{0.32\linewidth}
        \centering
        \includegraphics[width=\linewidth]{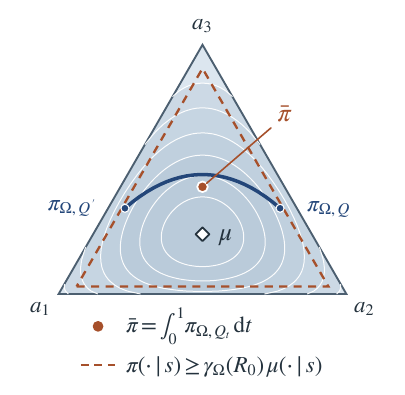}
        \caption{Average-path policy.}
        \label{fig:mixing_path}
    \end{subfigure}\hfill%
    \begin{subfigure}[t]{0.32\linewidth}
        \centering
        \includegraphics[width=\linewidth]{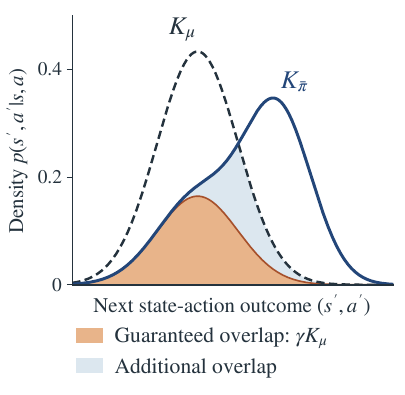}
        \caption{Kernel domination.}
        \label{fig:mixing_kernel}
    \end{subfigure}\hfill%
    \begin{subfigure}[t]{0.32\linewidth}
        \centering
        \includegraphics[width=\linewidth]{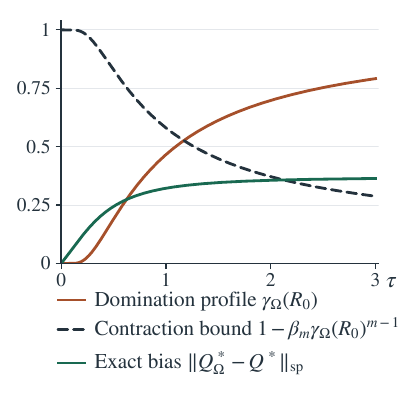}
        \caption{Temperature trade-off.}
        \label{fig:mixing_temperature}
    \end{subfigure}
    \label{fig:mixing_lemma}
\end{figure}

We make some remarks on \Cref{lem:mixing_main}. At a fixed action-span radius $R$, increasing $\tau$ reduces $R/\tau$ and can strengthen the contraction and coverage guarantees. The KL example in \cref{fig:mixing_temperature} illustrates this trade-off: the two-step contraction bound tightens as the exact action-value bias in span increases. For regularizers bounded on the simplex, $\Lambda_\Omega(R)\leq\tau\Delta_\Omega$ (\cref{rem:bounded_modulus}), so (iii) gives an $O(\tau)$ bound on the shifted gain bias. 

Related representations of unregularized Bellman differences through policy-induced Markov kernels appear in \citet[Lemma~B.6]{chen2026contraction} and \citet[Appendix~A.2, Lemma~6]{lee2026switching}. Our factorization identifies the average of the regularized greedy policies along the segment, which inherits their domination of $\mu$ (\cref{fig:mixing_path,fig:mixing_kernel}). Our analysis introduces the domination profile $\gamma_\Omega$ to quantify how regularization transfers the $m$-step mixing properties of a single reference policy. The same domination profile yields $m$-step span contraction in (i) and trajectory coverage in (ii).

\section{Algorithms and Finite-Time Analysis}
\label{sec:results}

We use the framework developed in \Cref{sec:mixing_regularized} to analyze synchronous and epoch $Q$-learning and localized $Q$-planning. Pseudocode is given in \Cref{app:pseudocodes} and proofs in \Cref{app:proofs_learning,app:planning}.

At a fixed temperature, both $Q$-learning algorithms target $Q^*_\Omega$, and their sample-complexity bounds quantify the cost of approximating it in span seminorm. Prior work likewise gives finite-sample guarantees relative to regularized targets in discounted planning and policy optimization \citep{grill2019planning,cayci2024convergence}. We bound the regularization bias relative to $Q^*$ separately in \Cref{lem:mixing_main}.(iii).

With a known transition model, $Q$-iteration applies the exact regularized Bellman update and centers the result, $Q_{t+1}=\Pi[\mT_\Omega(Q_t)]$. \Cref{lem:mixing_main}.(i) gives the following bound over a block of $m$ updates (see \Cref{cor:q-iteration}):
\begin{equation*}
    \|Q_{t+m}-Q^*_\Omega\|_{\spn}\leq\Bigl(1-\beta_m\gamma_\Omega(\|Q^*_\Omega\|_\asp+\|Q_t-Q^*_\Omega\|_\spn)^{m-1}\Bigr)\|Q_t-Q^*_\Omega\|_\spn.
\end{equation*}
The block contraction bound can tighten as the error decreases. The sample complexity bounds below use the uniform lower bound $\tilde\beta$ on the contraction strength over the localization region. For $m=1$, the factor is $1-\beta_1$ per sweep, recovering the classical relative value iteration rate under one-step minorization \citep[Sec.~8.5]{puterman1994markov}.

\subsection{Synchronous \texorpdfstring{$\Omega$}{Ω}-Regularized \texorpdfstring{$Q$}{Q}-Learning}
\label{subsec:synchronous_regularized_q_learning}

At each sweep $t$, synchronous $Q$-learning independently draws one next state $s_t'(s,a)\sim p(\cdot\mid s,a)$ for each state--action pair and forms the empirical Bellman backup
\begin{align*}
    (\widehat{\mT}_{\Omega,t}Q_t)(s,a)
    :=r(s,a)+\Omega^*(Q_t)(s_t'(s,a)).
\end{align*}
It then updates the table as $Q_{t+1}=\Pi\bigl[(1-\alpha_t)Q_t+\alpha_t\widehat{\mT}_{\Omega,t}Q_t\bigr]$, where $\Pi$ subtracts the mean of the table entries.

For both learning results, let $e_0:=\|Q_0-Q^*_\Omega\|_\spn>0$. On the span ball of radius $2e_0$ around $Q^*_\Omega$, the $m$-step contraction factor is at most $1-\tilde\beta$, and the sampling noise is bounded by $\bar w$, where
\begin{equation*}
    \tilde\beta:=\beta_m\gamma_\Omega(\|Q^*_\Omega\|_\asp+2e_0)^{m-1},
    \qquad
    \bar w:=\|Q^*_\Omega\|_\spn+2e_0+\|\Omega^*(0)\|_{\mS,\spn}.
\end{equation*}

\begin{theorem}[Synchronous $\Omega$-regularized $Q$-learning]
    \label{thm:sync-stoch-main}
    Let \Cref{assumption:m-step-mixing,assumption:convex-analytic-regularity,assumption:reference_policy,assumption:positive-domination} hold. For $\delta\in(0,1)$, run \Cref{alg:synchronous_regularized_q_learning} for $T$ sweeps with a constant step size $\alpha$ satisfying the condition in \Cref{thm:sync-stoch-main-detailed}. With probability at least $1-\delta$, the iterates remain within span distance $2e_0$ of $Q^*_\Omega$ and
    \begin{equation}
        \label{eq:sync-stoch-main}
        \|Q_T-Q_\Omega^*\|_\spn
        \leq 2e_0\left(1-\frac{\tilde\beta}{2}\right)^{\frac{2\alpha T}{9m}}
        + C\frac{m^{3/2}\bar w}{\tilde\beta}
        \sqrt{\alpha\log\!\left(\frac{6|\mS||\mA|(T+1)}{\delta}\right)},
    \end{equation}
    where $C$ is a universal constant. For any $\varepsilon\in(0,e_0)$, the step size and budget specified in \Cref{thm:sync-stoch-main-detailed} give span error at most $\varepsilon$ with probability at least $1-\delta$, using $\tilde O\bigl(|\mS||\mA|m^4\bar w^2/(\tilde\beta^3\varepsilon^2)\bigr)$ sampled transitions.
\end{theorem}

Each sweep samples every state--action pair, so the domination profile enters the bound through contraction. In the common one-step setting with rewards in $[0,1]$, $Q_0=0$, and a state-independent regularizer minimum, $\tilde\beta=\beta_1$ and \Cref{cor:sync-m-1} gives the worst-case rate $\tilde O(|\mS||\mA|\beta_1^{-5}\varepsilon^{-2})$, matching \citet{zhang2021finite}.

\subsection{Epoch \texorpdfstring{$\Omega$}{Ω}-Regularized \texorpdfstring{$Q$}{Q}-Learning}
\label{subsec:epoch_regularized_q_learning}

Epoch $Q$-learning uses a single trajectory (\Cref{alg:epoch_regularized_q_learning}). During epoch $j$, it follows the fixed policy $\pi_{\Omega,Q_j}$ and retains the first $B_j$ transitions from each state--action pair. Once every pair reaches this quota, the algorithm uses the per-pair averages to update $Q_j$.

On the localization region, the greedy policies dominate $\bar\gamma\mu$, where $\bar\gamma:=\gamma_\Omega(\|Q^*_\Omega\|_\asp+2e_0)$. By \Cref{lem:mixing_main}.(ii), this gives an $m$-step coverage lower bound of $\bar\gamma\tilde\beta\eta_{\min}$ for each state--action pair.

\begin{theorem}[Epoch $\Omega$-regularized $Q$-learning]
    \label{thm:epoch-main-body}
    Let \Cref{assumption:m-step-mixing,assumption:convex-analytic-regularity,assumption:reference_policy,assumption:positive-domination} hold. For $\delta\in(0,1)$, run \Cref{alg:epoch_regularized_q_learning} with $\alpha_j\equiv1$ and a constant per-pair batch size $B$ satisfying the condition in \Cref{thm:epoch-main-body-detailed}. Let $J(T)$ be the number of epochs completed within $T$ transitions. With probability at least $1-\delta$, the iterates through $J(T)$ remain within span distance $2e_0$ of $Q^*_\Omega$ and
    \begin{equation}
        \label{eq:epoch-main-bound}
        \|Q_{J(T)}-Q^*_\Omega\|_\spn
        \leq (1-\tilde\beta)^{\lfloor J(T)/m\rfloor}e_0
        + C\frac{m\bar w}{\tilde\beta}
        \sqrt{\frac{1}{B}\log\!\left(\frac{6|\mS||\mA|(T+1)}{\delta}\right)},
    \end{equation}
    where $C$ is a universal constant. Moreover, on the same event, $J(T)\geq\left\lfloor cT\bar\gamma\tilde\beta\eta_{\min}/(mB)\right\rfloor$ for a universal constant $c>0$.
    
    For any $\varepsilon\in(0,e_0)$, the batch size and budget specified in \Cref{thm:epoch-main-body-detailed} give span error at most $\varepsilon$ with probability at least $1-\delta$, using $\tilde O\!\left(\frac{m^4\bar w^2}{\tilde\beta^3\varepsilon^2}\cdot \frac{1}{\bar\gamma\tilde\beta\eta_{\min}}\right)$
    sampled transitions.
\end{theorem}

The factor $(\bar\gamma\tilde\beta\eta_{\min})^{-1}$ accounts for the time needed to collect samples from every state--action pair. Rolling out the reference policy $\mu$ instead gives the coverage lower bound $\beta_m\eta_{\min}$ (see \Cref{rem:ep_behavior_mu}).

\paragraph{Performance under the original objective.}
For bounded regularizers, \Cref{cor:greedy_performance} gives $\rho^*-\rho^{\pi_{\Omega,Q}}(s)\leq\|Q-Q^*_\Omega\|_\spn+\tau\Delta_\Omega$ for every initial state $s$. When $\Delta_\Omega>0$, applying either learning bound at temperature $\tau=\varepsilon/(2\Delta_\Omega)$ and span accuracy $\varepsilon/2$ therefore gives an $\varepsilon$-optimal policy. We leave the analysis of temperature schedules to future work.

\subsection{\texorpdfstring{$\Omega$}{Ω}-Regularized \texorpdfstring{$Q$}{Q}-Planning}
\label{subsec:planning_results}

Our final algorithm, $Q$-planning, performs localized learning from a reset distribution $\xi\in\Delta(\mS)$. As in Monte Carlo tree search~\citep{grill2020mcts}, it backs up a terminal value $Q_B$ over a horizon $H$. Let $\mS_h$ be the states reachable in exactly $h$ steps from $\operatorname{supp}(\xi)$, and $\mX_h:=\mS_h\times\mA$. The finite-horizon target is defined backward across depths by $Q^*_{\Omega,H}:=Q_B|_{\mX_H}$ and $Q^*_{\Omega,h}:=\mT_{\Omega,h}(Q^*_{\Omega,h+1})=(\mT_\Omega^{H-h}Q_B)|_{\mX_h}$ for $h<H$, where $\mT_{\Omega,h}$ is the Bellman operator from depth $h+1$ to depth $h$. Planning only requires mixing on the stabilized set of reachable states.

\begin{assumption}[Localized $m$-step Doeblin mixing]
    \label{assumption:pl_local_mixing}
    There exist integers $m\geq1$ and $h_1\geq0$, a full-support policy $\mu$, $\beta_m\in(0,1]$ and $\nu\in\Delta(\mS)$ with support $\mS_\infty$ such that (i) $\mS_h=\mS_\infty$ for every $h\geq h_1$, and (ii) $\sum_{a'}K^m_\mu(s',a'|s,a)\geq\beta_m\nu(s')$ for all $(s,a)\in\mS_\infty\times\mA$ and $s'\in\mS_\infty$.
\end{assumption}

Unlike \Cref{assumption:m-step-mixing}, this allows transient initial layers ($h<h_1$) and unreachable states. Here $Q^*_\Omega$ denotes the regularized solution on $\mS_\infty\times\mA$, extended to all reachable pairs (\Cref{lem:pl_extension}). At epoch $j$, the root error decomposes as
\[
    \|Q_{j,0}-Q^*_\Omega\|_{\mX_0,\spn}
    \leq
    \underbrace{\|Q^*_{\Omega,0}-Q^*_\Omega\|_{\mX_0,\spn}}_{\text{horizon bias}}
    +
    \underbrace{\|Q_{j,0}-Q^*_{\Omega,0}\|_{\mX_0,\spn}}_{\text{estimation error}}.
\]
Let $\Delta_B:=\|Q_B-Q^*_\Omega\|_{\mX_H,\spn}$ and $e_0:=\max_{h<H}\|Q_{0,h}-Q^*_{\Omega,h}\|_{\mX_h,\spn}+\Delta_B>0$.\footnote{For initial tables uniformly bounded in span and fixed $Q_B$, $e_0$ admits an $H$-independent upper bound (\Cref{lem:pl_uniform_bounds}).}
We use the constants $\tilde\beta$ and $\bar w$ defined above with this $e_0$, evaluated on the reachable set with the localized Doeblin parameters, and set $H_{\mathrm{eff}}:=m/\tilde\beta$.

The model-free planner (\Cref{alg:epoch_regularized_q_planning}) takes the reachable layers as inputs and estimates the depth-dependent Bellman updates from length-$H$ rollouts under the greedy policies $\pi_{\Omega,Q_{j,h}}$ (or $\mu$). In each epoch $j$, it collects $B$ samples per pair at each depth and updates $Q_{j+1,h}$ from the frozen table $Q_{j,h+1}$. Coverage comes from resets at shallow depths and from mixing at greater depths (\Cref{lem:pl_coverage}).

\begin{theorem}[Epoch $\Omega$-regularized $Q$-planning]
    \label{thm:pl-epoch-main-body}
    Let \Cref{assumption:pl_local_mixing,assumption:convex-analytic-regularity,assumption:reference_policy,assumption:positive-domination} hold.
    For $\delta\in(0,1)$ and $H\geq h_1+m$, run \Cref{alg:epoch_regularized_q_planning} for $J\geq H$ epochs with a constant per-pair, per-depth batch size $B$ satisfying the condition in \Cref{thm:pl-epoch-main}.
    With probability at least $1-\delta$,
    \begin{equation}
        \label{eq:pl-epoch-main-body}
        \|Q_{J,0}-Q^*_\Omega\|_{\mX_0,\spn}
        \leq
        (1-\tilde\beta)^{\lfloor(H-h_1)/m\rfloor}\Delta_B
        +C(h_1+H_{\mathrm{eff}})\bar w
        \sqrt{\frac{1}{B}\log\!\left(
            \frac{4JH\bar S_H|\mA|}{\delta}
        \right)},
    \end{equation}
    where $C$ is a universal constant and $\bar S_H:=\max_{h\leq H}|\mS_h|$.
    For any $\varepsilon\in(0,e_0)$, the horizon and batch size specified in \Cref{thm:pl-epoch-main}, with $J=H$, give root span error at most $\varepsilon$ with probability at least $1-\delta$, using $\tilde O\bigl((h_1+H_{\mathrm{eff}})^4\bar w^2/(q\varepsilon^2)\bigr)$ sampled transitions, where $q>0$ lower-bounds the per-rollout visitation probability across depths (\Cref{app:planning_results}).
\end{theorem}

Both terms in~\eqref{eq:pl-epoch-main-body} are controlled by the same local contraction: each block of $m$ levels beyond $h_1$ contracts deeper errors by a factor of at most $1-\tilde\beta$. The cumulative weight of sampling errors across depths is bounded by $h_1+H_{\mathrm{eff}}$. Without sampling noise, the depth-shifted recursion reaches $Q^*_{\Omega,h}$ in at most $H-h$ epochs, which explains the condition $J\geq H$.

When $h_1=0$ and the coverage constants are equal, pooling the reset samples into a single table saves a factor $1/\tilde\beta$ in the leading sample complexity term (\Cref{subsec:pl_reset_learning}). Further details and proofs are given in \Cref{app:planning_results,app:planning}.

\section{Instances and Experiments}

\subsection{Instances of Base Regularizers}
\label{subsec:instances_base_regularizers}

We instantiate our regularizers with the $\alpha$-divergences from Amari~\citep{cichocki2010families}: 
\begin{equation*}
    D_\alpha(p\|q):=\left(\sum_ap_a^\alpha q_a^{1-\alpha}-1\right)/\big(\alpha(\alpha-1)\big).
\end{equation*}
These include $D_1=\KL(p\|q)$ and $D_0=\KL(q\|p)$ by continuity, which have been used for KL-regularized reinforcement learning \citep{vieillard2020leverage}, the $\chi^2$-divergence underlying advantage-weighted updates \citep{belousov2019entropic}, and, for a uniform $\mu$, the Tsallis entropies up to a rescaling of $\tau$ \citep{lee2018sparse}. 

We define $\Omega_{\alpha,s}:=D_\alpha(\cdot\|\mu_s)$ and we compute the relevant quantities in \Cref{tab:alpha_family}\footnote{We write $\log_+(x):=\log(\max\{1,x\})$ for $x\geq0$.
Further notation and proofs are given in \Cref{app:proofs_alpha}.}. In particular, on the window $0<\alpha\leq1$, the divergences exhibit two favorable properties: the modulus is bounded uniformly in $R$ and the greedy policy is full support. 

\begin{table}[!htbp]
    \centering
    \caption{The $\alpha$-divergence regularizers (with $|\mS|,|\mA|\geq2$)}
    \label{tab:alpha_family}
    \footnotesize
    \setlength{\tabcolsep}{2.4pt}
    \renewcommand{\arraystretch}{1.35}
    \begin{tabular}{@{}l|ccccc@{}}
        \toprule
         & general $\alpha$ & reverse KL & forward KL & Hellinger & $\chi^2$ \\
         & & $\alpha=1$ & $\alpha=0$ & $\alpha=\frac12$ & $\alpha=2$ \\
        \midrule
        $\Omega_{\alpha,s}(\pi)$ & $D_\alpha(\pi\|\mu_s)$ & $\KL(\pi\|\mu_s)$ & $\KL(\mu_s\|\pi)$ & $2\|\sqrt{\pi}-\sqrt{\mu_s}\|_2^2$ & $\frac12\chi^2(\pi\|\mu_s)$\\
        strong convexity & $\mu_{\min}^{(1-\alpha)_+}$ ($\alpha\leq2$) & $1$ & $\mu_{\min}$ & $\sqrt{\mu_{\min}}$ & $1$\\
        $\pi_{\Omega,Q}(a|s)/\mu_a$ & $\big[1+\frac{(\alpha-1)(Q_a-\zeta_s)}{\tau}\big]_+^{\frac{1}{\alpha-1}}$ & $\frac{e^{Q_a/\tau}}{\sum_b\mu_be^{Q_b/\tau}}$ & $\big(1+\frac{\zeta_s-Q_a}{\tau}\big)^{-1}$ & $\big(1+\frac{\zeta_s-Q_a}{2\tau}\big)^{-2}$ & $\big[1+\frac{Q_a-\zeta_s}{\tau}\big]_+$\\
        $\gamma_{\Omega}(R)\geq$ & $\exp_{2-\alpha}(-R/\tau)$ & $e^{-R/\tau}$ & $\frac{\tau}{\tau+R}$ & $\big(1+\frac{R}{2\tau}\big)^{-2}$ & $\big[1-\frac{R}{\tau}\big]_+$\\
        $\Lambda_{\Omega}(R)\leq$ & $\tau\frac{\mu_{\min}^{1-\alpha}-1}{\alpha(\alpha-1)}$ if $\alpha>0$ & $\tau\log\frac{1}{\mu_{\min}}$ & $\tau\big(1+\log_+\frac{R}{\tau}\big)$ & $4\tau(1-\sqrt{\mu_{\min}})$ & $\frac{\tau}{2}\big(\frac{1}{\mu_{\min}}-1\big)$\\
        \Cref{assumption:positive-domination} holds & iff $\alpha\leq1$ & yes & yes & yes & no\\
        $\Lambda_\Omega$ bounded in $R$ & iff $\alpha>0$ & yes & no & yes & yes\\
        \bottomrule
    \end{tabular}
\end{table}

\subsection{Experiments}
\label{subsec:experiments}

\textit{Does knowledge of $\mu$ improve convergence and exploration?}
We consider the \textit{Two Rooms} MDP (\Cref{fig:two_rooms_mdp}): two $5\times5$ rooms joined by a two-cell corridor. Entering the corridor succeeds with probability $p_{\text{cross}}$; otherwise, the agent stays in place. The reference policy $\mu$ is uniform except at the doorway entrances, where the entering action has probability $\mu(\text{DOOR})$ ($0.25$ is uniform). We use the $\epsilon$-floored reverse KL ($\tau=2$, $\epsilon=0.05$); details and the bias are in \Cref{app:two_rooms}.

\begin{figure}[!htbp]
    \centering
    \caption{Empirical results on the Two Rooms MDP.
    (a) Two Rooms layout.
    (b) Iterations per asymptotic $e$-fold under synchronous regularized $Q$-iteration.
    (c) Epoch lengths relative to uniform $\mu$; dashed curves show the coverage proxy $\min_m m/(\beta_m\eta_{\min})$, normalized in the same way.}
    \begin{subfigure}[b]{0.32\textwidth}
        \centering
        \includegraphics[width=\textwidth]{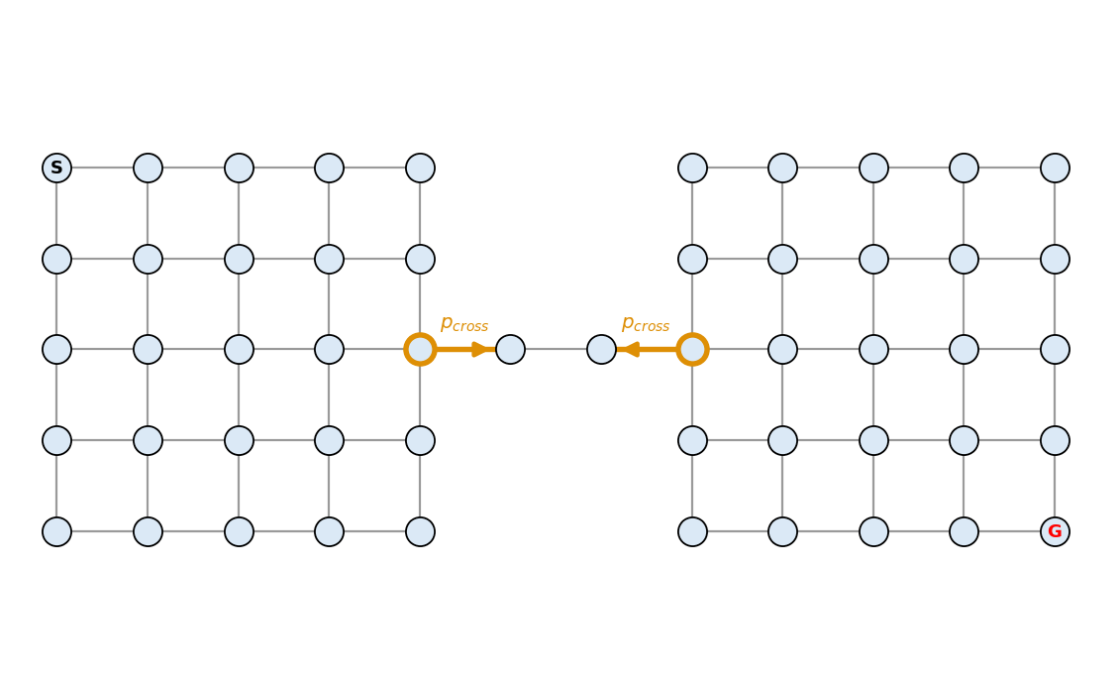}
        \caption{The Two Rooms MDP.}
        \label{fig:two_rooms_mdp}
    \end{subfigure}
    \hfill
    \begin{subfigure}[b]{0.32\textwidth}
        \centering
        \includegraphics[width=\textwidth]{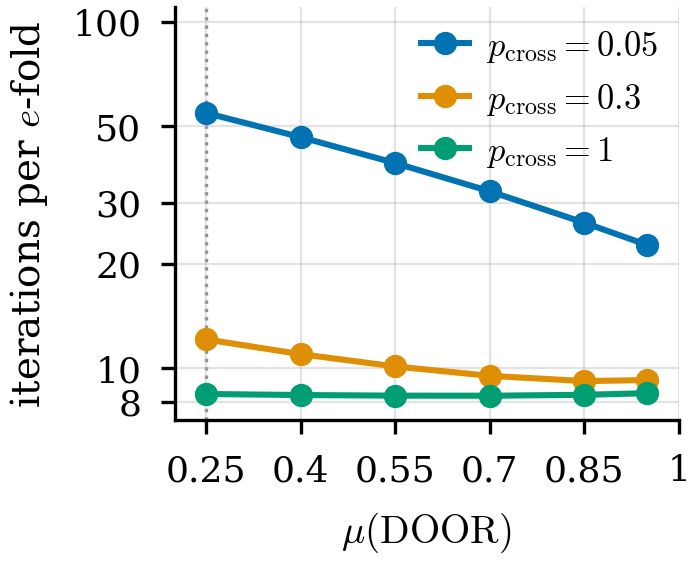}
        \caption{$\mu$'s effect on convergence.}
        \label{fig:two_rooms_convergence}
    \end{subfigure}
    \hfill
    \begin{subfigure}[b]{0.32\textwidth}
        \centering
        \includegraphics[width=\textwidth]{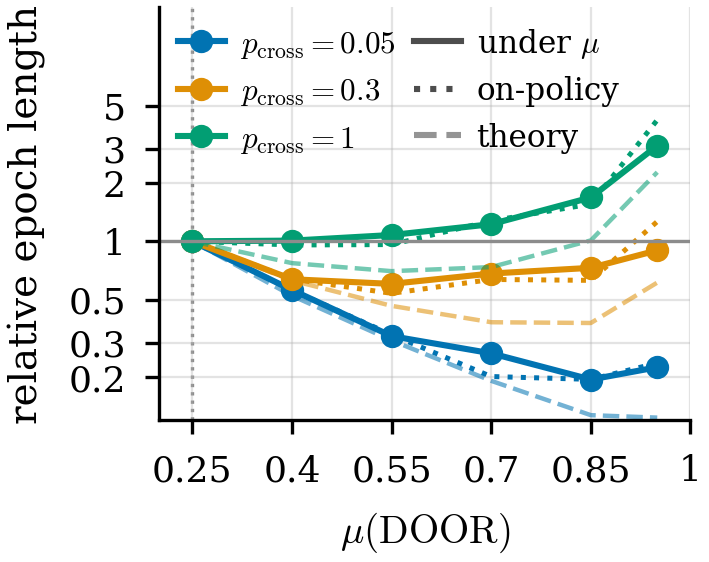}
        \caption{$\mu$'s effect on exploration.}
        \label{fig:two_rooms_exploration}
    \end{subfigure}
    \vspace{-0.5cm}
\end{figure}

Synchronous regularized $Q$-iteration isolates contraction from sampling noise (\Cref{fig:two_rooms_convergence}). Biasing $\mu$ towards the doorway reduces the iterations per $e$-fold by a factor of $2.4$ at $p_{\text{cross}}=0.05$, but has little effect at $p_{\text{cross}}=1$. For epoch regularized $Q$-learning (\Cref{fig:two_rooms_exploration}), it shortens epochs by up to a factor of $5$ at $p_{\text{cross}}=0.05$ but triples them at $p_{\text{cross}}=1$: here, knowledge of $\mu$ helps when it targets an actual bottleneck. The theoretical coverage proxy captures this dependence qualitatively.

\textit{Does theory predict regularizer ranking?}
We consider the RiverSwim MDP \citep{strehl2008analysis}, with the parameterization of \citet{osband2013more}, $n=6$ states and a uniform reference policy $\mu$. We compare the $\alpha$-divergence regularizers for $\alpha\in\{-1,0,\frac12,1,1.5,2\}$.

At a common temperature, the bias decreases as $\alpha$ increases (\Cref{fig:river_mdp_bias}). Comparing convergence at a common $\tau$ would therefore confound speed with bias. We thus choose, for each regularizer, the temperature giving the same bias $\|Q^*_\Omega-Q^*\|_\spn=2$, and run epoch regularized $Q$-learning ($B=64$, $20$ seeds).

\begin{figure}[!htbp]
    \centering
    \caption{Comparison of $\alpha$-divergence regularizers on RiverSwim.
    (a) RiverSwim: solid red arrows denote LEFT and dashed blue arrows RIGHT. The illustration uses $n=3$; experiments use $n=6$.
    (b) Span error to each method's target ($Q^*_\Omega$ for regularized methods, $Q^*$ for baselines), with span bias matched at $2$ across regularized methods (using $\epsilon=0.05$ flooring for $\alpha>1$).
    (c) Bias against temperature $\tau$: solid, base regularizers; dashed, $\epsilon$-floored variants for $\alpha>1$, with $\epsilon=0.05$. Colors match panel (b).}
    \begin{subfigure}[b]{0.32\textwidth}
        \centering
        \includegraphics[width=\textwidth, trim={6em 1em 5em 1em}, clip]{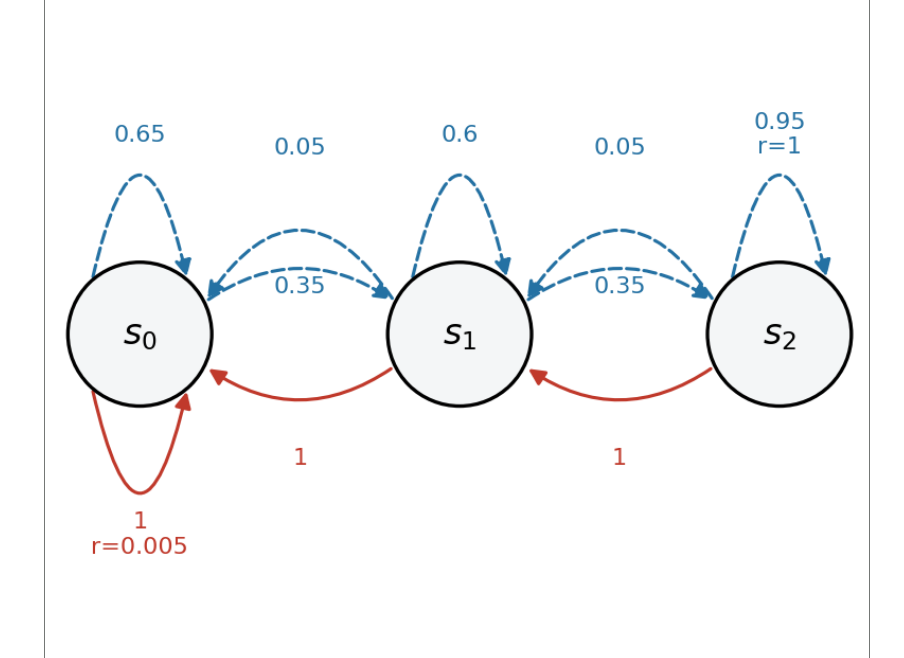}
        \caption{RiverSwim MDP ($n=3$).}
        \label{fig:river_mdp}
    \end{subfigure}
    \hfill
    \begin{subfigure}[b]{0.32\textwidth}
        \centering
        \includegraphics[width=\textwidth]{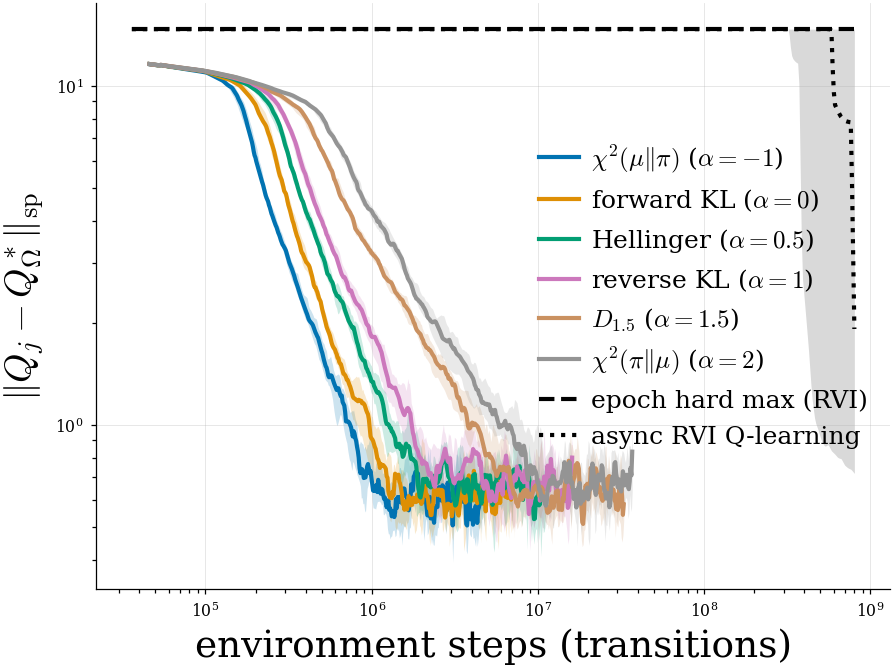}
        \caption{$\Omega$'s effect on convergence.}
        \label{fig:river_mdp_convergence}
    \end{subfigure}
    \hfill
    \begin{subfigure}[b]{0.32\textwidth}
        \centering
        \includegraphics[width=\textwidth]{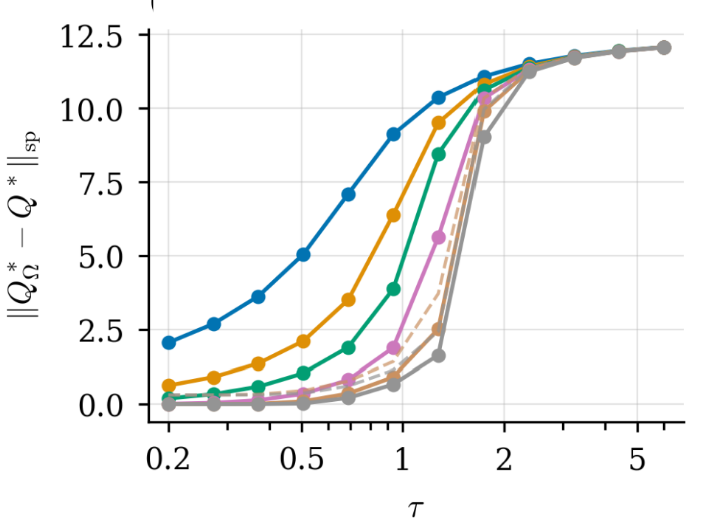}
        \caption{$\Omega$'s effect on bias.}
        \label{fig:river_mdp_bias}
    \end{subfigure}
\end{figure}

At matched bias, empirical convergence slows as $\alpha$ increases (\Cref{fig:river_mdp_convergence}), favoring smaller $\alpha$ on this benchmark. This is consistent with the stronger base domination bounds at smaller $\alpha$ for fixed $R/\tau$ (\Cref{tab:alpha_family}). The unregularized epoch baseline without added exploration never completes an epoch; asynchronous RVI $Q$-learning starts improving only after several $10^8$ transitions.

\section{Related Work}
\label{sec:related_work}

\begin{table}[!ht]
    \centering
    \small
    \caption{Leading sample-complexity scalings in sampled transitions, omitting logarithmic factors and lower-order terms. The cited bounds concern policy suboptimality or span error, with different probabilistic guarantees; our bounds control span error relative to $Q^*_\Omega$. Here $t_{\mathrm{mix}}$ uniformly bounds stationary-policy mixing times, $K$ is the worst expected hitting time of a common state over initial states and stationary deterministic policies, and $\theta<1$ is a span-contraction factor. For asynchronous baselines, $d_{\min}$ is the minimum positive stationary state--action mass and $C_{\mathrm b}$ denotes row-dependent behavior-mixing factors.}
    \label{tab:sample-complexity}
    \renewcommand{\arraystretch}{1.2}
    \begin{tabular}{@{}>{\raggedright\arraybackslash}p{0.20\linewidth}>{\raggedright\arraybackslash}p{0.42\linewidth}>{\raggedright\arraybackslash}p{0.31\linewidth}@{}}
        \toprule
        Paper & Assumptions & Sample complexity \\
        \midrule
        \multicolumn{3}{@{}l}{\emph{Generative model, model-based}} \\
        \citet{wang2024optimal} & all stationary policies uniformly mixing & $|\mS||\mA|\,t_{\mathrm{mix}}\,\varepsilon^{-2}$ \\
        \citet{zurek2024span} & weakly communicating & $|\mS||\mA|\,\|\max_a Q^*\|_\spn\,\varepsilon^{-2}$ \\
        \midrule
        \multicolumn{3}{@{}l}{\emph{Generative model, model-free}} \\
        \citet{lee2025anchoring} & weakly communicating & $|\mS||\mA|\,\|\max_a Q^*\|_\spn^2\,\varepsilon^{-2}$ \\
        \citet{zhang2021finite} & unichain; $J$-step span contraction with factor $\theta$ & $|\mS||\mA|\,J^3(1-\theta)^{-5}\,\varepsilon^{-2}$ \\
        \citet{chen2026contraction} & common-state reachability under all stationary deterministic policies; hitting time $K$ & $|\mS||\mA|\,2^{6K}\,\varepsilon^{-2}$ \\
        \textbf{Ours, \Cref{thm:sync-stoch-main}} & \Cref{assumption:m-step-mixing,assumption:convex-analytic-regularity,assumption:reference_policy,assumption:positive-domination} & $|\mS||\mA|\,m^4\,\bar w^2\,\tilde\beta^{-3}\,\varepsilon^{-2}$ \\
        \midrule
        \multicolumn{3}{@{}l}{\emph{Single trajectory, model-free}} \\
        \citet{chen2025nonasymptotic} (asynchronous) & one-step span contraction with factor $\theta$; fixed full-support behavior policy; irreducible aperiodic behavior chain & $C_{\mathrm b}|\mS|^3|\mA|^3\,d_{\min}^{-2}(1-\theta)^{-5}\,\varepsilon^{-2}$ \\
        \citet{chen2026contraction} (asynchronous) & reachability, hitting time $K$; fixed full-support behavior policy; aperiodic chain & $C_{\mathrm b}|\mS|^3|\mA|^3\,2^{6K}\,d_{\min}^{-2}\,\varepsilon^{-2}$ \\
        \textbf{Ours, \Cref{thm:epoch-main-body} (epoch)} & \Cref{assumption:m-step-mixing,assumption:convex-analytic-regularity,assumption:reference_policy,assumption:positive-domination} & $m^4\bar w^2\tilde\beta^{-3}(\bar\gamma\tilde\beta\eta_{\min})^{-1}\varepsilon^{-2}$ \\
        \midrule
        \multicolumn{3}{@{}l}{\emph{Planning (resets to $\xi$), model-free}} \\
        \textbf{Ours, \Cref{thm:pl-epoch-main-body}} & \Cref{assumption:pl_local_mixing,assumption:convex-analytic-regularity,assumption:reference_policy,assumption:positive-domination} & $(h_1+H_{\mathrm{eff}})^4\,\bar w^2\,q^{-1}\,\varepsilon^{-2}$ \\
        \bottomrule
    \end{tabular}
\end{table}

\paragraph{Where the contraction comes from.} The average-reward Bellman operator need not be a span contraction. \citet{zhang2021finite} and \citet{chen2025nonasymptotic} assume $J$-step and one-step span contraction, respectively. Under common-state reachability, \citet{chen2026contraction} derive contraction of a lazy Bellman operator in an instance-dependent seminorm, with explicit lazy sampling or implicit updates using original-dynamics samples. Other approaches use discounted approximations \citep{jin2024feasible,jiao2026sample} or Halpern iterations \citep{bravo2026halpern,lee2025anchoring}; \citet{jiao2026sample} obtain an $\varepsilon^{-3}$ dependence under weak communication. Our supplied mixing reference policy yields $m$-step span contraction and trajectory coverage, giving $\varepsilon^{-2}$ bounds relative to $Q^*_\Omega$ at fixed temperature, with regularization bias controlled separately (\Cref{lem:mixing_main}.(iii)).

\paragraph{Regularization toward a reference policy.} KL regularization is used for error averaging \citep{vieillard2020leverage}, behavior priors and transfer \citep{teh2017distral,tirumala2022behavior}, and offline RL and RLHF \citep{wu2019behavior,xiong2024iterative}. \citet{labbi2026beyond} couple $f$-softargmax parameterization and $f$-divergence regularization in discounted policy gradient. Average-reward entropy regularization is studied by \citet{neu2017unified,adamczyk2025avgentropy}. In discounted zero-sum Markov games, \citet{Ouhamma_2026} obtain exploration and mixing guarantees through Tsallis smoothing, assuming a mixing joint policy exists but need not be known. We regularize toward a supplied mixing reference policy and use a common domination profile to quantify average-reward span contraction and coverage.

\paragraph{Planning depth.} Discounted planning depth is controlled by the discount factor \citep{kearns2002sparse}. Entropy-regularized planners offer fast root-action identification \citep{xiao2019ments} and polynomial sample complexity \citep{grill2019planning}, while targeting a regularized objective whose greedy actions can be suboptimal \citep{painter2023monte}. For average-reward MDPs, \citet{schweitzer1977asymptotic} characterize convergence of finite-horizon values after subtracting linear growth. We give an explicit depth for root span accuracy under localized mixing, with $H_{\mathrm{eff}}=m/\tilde\beta$ playing the role of $1/(1-\gamma)$, and control regularization bias separately.

\section{Discussion and Limitations}

Our bounds use a uniform Doeblin pair $(m,\beta_m)$ and domination profile $\gamma_\Omega$, so a single bottleneck can determine the guaranteed rates. State-dependent constants, weighted along the paths of $\mu$, could yield tighter guarantees. Furthermore, our uniform accuracy criterion favors exploration with $\mu$ in the current coverage bounds and does not credit the concentration of on-policy data on states relevant to the task. Instance-dependent analyses could help capture these benefits.

Future directions include relaxing \Cref{assumption:m-step-mixing} to allow transient states in single-trajectory learning, with accuracy measured on the recurrent class; sharpening the worst-case effective-horizon dependence of synchronous learning toward $H_{\mathrm{eff}}^4\varepsilon^{-2}$, analogous to discounted $Q$-learning \citep{li2021minimax}; and extending our framework to fully asynchronous updates \citep{chen2025nonasymptotic} and function approximation. Finally, for our divergence instances, principled schedules $\tau_t\to0$, together with $\epsilon_t\to0$ when an $\epsilon$-floor is used, could combine early mixing benefits with vanishing operator bias. Establishing convergence would require controlling tracking error as contraction and coverage weaken.

\paragraph{AI use statement}

We used generative AI tools to identify relevant literature, polish writing, improve readability, and help implement the experiments. We have not used generative AI tools to develop the theoretical framework, to propose hypotheses, or for proving mathematical claims. We have reviewed all AI-assisted work. We take responsibility for the final content of this work, including text, claims or artifacts produced with the aid of generative AI.

\bibliographystyle{plainnat}
\bibliography{references}

@article{lee2026switching,
  title={Lyapunov-certified direct switching theory for Q-learning},
  author={Lee, Donghwan},
  journal={arXiv preprint arXiv:2604.19569},
  year={2026}
}

@article{schulman2017proximal,
  title={Proximal policy optimization algorithms},
  author={Schulman, John and Wolski, Filip and Dhariwal, Prafulla and Radford, Alec and Klimov, Oleg},
  journal={arXiv preprint arXiv:1707.06347},
  year={2017}
}

@inproceedings{yang2024robot,
  title={Robot fine-tuning made easy: Pre-training rewards and policies for autonomous real-world reinforcement learning},
  author={Yang, Jingyun and Mark, Max Sobol and Vu, Brandon and Sharma, Archit and Bohg, Jeannette and Finn, Chelsea},
  booktitle={2024 IEEE International Conference on Robotics and Automation (ICRA)},
  year={2024},
}

@article{tirumala2022behavior,
  title={Behavior priors for efficient reinforcement learning},
  author={Tirumala, Dhruva and Galashov, Alexandre and Noh, Hyeonwoo and Hasenclever, Leonard and Pascanu, Razvan and Schwarz, Jonathan and Desjardins, Guillaume and Czarnecki, Wojciech Marian and Ahuja, Arun and Teh, Yee Whye and others},
  journal={Journal of Machine Learning Research},
  year={2022}
}

@inproceedings{jin2021towards,
  title={Towards tight bounds on the sample complexity of average-reward MDPs},
  author={Jin, Yujia and Sidford, Aaron},
  booktitle={International Conference on Machine Learning},
  year={2021},
}

@article{strehl2008analysis,
  title   = {An analysis of model-based Interval Estimation for {M}arkov decision processes},
  author  = {Strehl, Alexander L. and Littman, Michael L.},
  journal = {Journal of Computer and System Sciences},
  year    = {2008}
}

@inproceedings{osband2013more,
  title     = {(More) Efficient Reinforcement Learning via Posterior Sampling},
  author    = {Osband, Ian and Russo, Daniel and Van Roy, Benjamin},
  booktitle = {Advances in Neural Information Processing Systems},
  year      = {2013}
}

@inproceedings{xie2021policy,
  title={Policy finetuning: Bridging sample-efficient offline and online reinforcement learning},
  author={Xie, Tengyang and Jiang, Nan and Wang, Huan and Xiong, Caiming and Bai, Yu},
  booktitle = {Advances in Neural Information Processing Systems},
  year={2021}
}

@article{li2021minimax,
  author  = {Li, Gen and Cai, Changxiao and Chen, Yuxin and Wei, Yuting and Chi, Yuejie},
  title   = {Is {$Q$}-Learning Minimax Optimal? A Tight Sample Complexity Analysis},
  journal = {Operations Research},
  year    = {2024}
}

@article{chen2026pointwise,
  title   = {From Set Convergence to Pointwise Convergence: Finite-Time Guarantees for Average-Reward {$Q$}-Learning with Adaptive Stepsizes},
  author  = {Chen, Zaiwei and Nanda, Phalguni},
  journal = {arXiv preprint arXiv:2504.18743v2},
  year    = {2026}
}

@book{mitzenmacher2005probability,
  author    = {Mitzenmacher, Michael and Upfal, Eli},
  title     = {Probability and Computing: Randomized Algorithms and Probabilistic Analysis},
  publisher = {Cambridge University Press},
  year      = {2005}
}

@book{rockafellar1970convex,
  author    = {Rockafellar, R. Tyrrell},
  title     = {Convex Analysis},
  publisher = {Princeton University Press},
  year      = {1970}
}

@inproceedings{labbi2026beyond,
  author    = {Labbi, Safwan and Tiapkin, Daniil and Mangold, Paul and Moulines, {\'E}ric},
  title     = {Beyond Softmax and Entropy: Convergence Rates of Policy Gradients with $f$-{SoftArgmax} Parameterization \& Coupled Regularization},
  booktitle = {International Conference on Learning Representations},
  year      = {2026}
}

@inproceedings{roulet2025loss,
  author    = {Roulet, Vincent and Liu, Tianlin and Vieillard, Nino and Sander, Michael Eli and Blondel, Mathieu},
  title     = {Loss Functions and Operators Generated by $f$-Divergences},
  booktitle = {International Conference on Machine Learning},
  year      = {2025}
}

@article{wan2024convergence,
  author  = {Wan, Yi and Yu, Huizhen and Sutton, Richard S.},
  title   = {On Convergence of Average-Reward {$Q$}-Learning in Weakly Communicating {M}arkov Decision Processes},
  journal = {arXiv preprint arXiv:2408.16262},
  year    = {2024}
}

@article{jiao2026sample,
  author  = {Jiao, Yuchen and Woo, Jiin and Li, Gen and Joshi, Gauri and Chi, Yuejie},
  title   = {Sample Complexity of Average-Reward {$Q$}-Learning: From Single-Agent to Federated Reinforcement Learning},
  journal = {arXiv preprint arXiv:2601.13642},
  year    = {2026}
}

@inproceedings{painter2023monte,
  author    = {Painter, Michael and Baioumy, Mohamed and Hawes, Nick and Lacerda, Bruno},
  title     = {{M}onte {C}arlo Tree Search with {B}oltzmann Exploration},
  booktitle = {Advances in Neural Information Processing Systems},
  year      = {2023}
}

@book{puterman1994markov,
  author    = {Puterman, Martin L.},
  title     = {{M}arkov Decision Processes: Discrete Stochastic Dynamic Programming},
  publisher = {John Wiley \& Sons},
  year      = {1994}
}

@book{bertsekas2012dynamic,
  author    = {Bertsekas, Dimitri P.},
  title     = {Dynamic Programming and Optimal Control, Vol.~{II}},
  edition   = {4th},
  publisher = {Athena Scientific},
  year      = {2012}
}

@book{hernandez1996discrete,
  author    = {Hern{\'a}ndez-Lerma, On{\'e}simo and Lasserre, Jean B.},
  title     = {Discrete-Time {M}arkov Control Processes: Basic Optimality Criteria},
  series    = {Applications of Mathematics},
  volume    = {30},
  publisher = {Springer},
  year      = {1996}
}

@article{arapostathis1993survey,
  author  = {Arapostathis, Aristotle and Borkar, Vivek S. and Fern{\'a}ndez-Gaucherand, Emmanuel and Ghosh, Mrinal K. and Marcus, Steven I.},
  title   = {Discrete-Time Controlled {M}arkov Processes with Average Cost Criterion: A Survey},
  journal = {SIAM Journal on Control and Optimization},
  volume  = {31},
  number  = {2},
  pages   = {282--344},
  year    = {1993}
}

@article{schweitzer1977asymptotic,
  author  = {Schweitzer, Paul J. and Federgruen, Awi},
  title   = {The Asymptotic Behavior of Undiscounted Value Iteration in {M}arkov Decision Problems},
  journal = {Mathematics of Operations Research},
  volume  = {2},
  number  = {4},
  pages   = {360--381},
  year    = {1977}
}

@article{schweitzer1978functional,
  author  = {Schweitzer, Paul J. and Federgruen, Awi},
  title   = {The Functional Equations of Undiscounted {M}arkov Renewal Programming},
  journal = {Mathematics of Operations Research},
  volume  = {3},
  number  = {4},
  pages   = {308--321},
  year    = {1978}
}

@article{schweitzer1984existence,
  author  = {Schweitzer, Paul J.},
  title   = {On the Existence of Relative Values for Undiscounted {M}arkovian Decision Processes with a Scalar Gain Rate},
  journal = {Journal of Mathematical Analysis and Applications},
  volume  = {104},
  number  = {1},
  pages   = {67--78},
  year    = {1984}
}

@article{schweitzer1985undiscounted,
  author  = {Schweitzer, Paul J.},
  title   = {On Undiscounted {M}arkovian Decision Processes with Compact Action Spaces},
  journal = {RAIRO -- Operations Research},
  volume  = {19},
  number  = {1},
  pages   = {71--86},
  year    = {1985}
}

@article{tsitsiklis2007nphard,
  author  = {Tsitsiklis, John N.},
  title   = {{NP}-Hardness of Checking the Unichain Condition in Average Cost {MDP}s},
  journal = {Operations Research Letters},
  volume  = {35},
  number  = {3},
  pages   = {319--323},
  year    = {2007}
}

@article{dobrushin1956central,
  author  = {Dobrushin, Roland L.},
  title   = {Central Limit Theorem for Nonstationary {M}arkov Chains. {I}},
  journal = {Theory of Probability and Its Applications},
  volume  = {1},
  number  = {1},
  pages   = {65--80},
  year    = {1956}
}

@book{seneta2006nonnegative,
  author    = {Seneta, Eugene},
  title     = {Non-negative Matrices and {M}arkov Chains},
  series    = {Springer Series in Statistics},
  publisher = {Springer},
  year      = {2006}
}

@article{gaubert2015dobrushin,
  author  = {Gaubert, St{\'e}phane and Qu, Zheng},
  title   = {{D}obrushin's Ergodicity Coefficient for {M}arkov Operators on Cones},
  journal = {Integral Equations and Operator Theory},
  volume  = {81},
  number  = {1},
  pages   = {127--150},
  year    = {2015}
}

@article{depasquale2024dual,
  author  = {De Pasquale, Giulia and Smith, Kevin D. and Bullo, Francesco and Valcher, Maria Elena},
  title   = {Dual Seminorms, Ergodic Coefficients and Semicontraction Theory},
  journal = {IEEE Transactions on Automatic Control},
  volume  = {69},
  number  = {5},
  pages   = {3040--3053},
  year    = {2024}
}

@article{wielandt1950,
  author  = {Wielandt, Helmut},
  title   = {Unzerlegbare, nicht negative {M}atrizen},
  journal = {Mathematische Zeitschrift},
  volume  = {52},
  pages   = {642--648},
  year    = {1950}
}

@book{lemmens2012nonlinear,
  author    = {Lemmens, Bas and Nussbaum, Roger},
  title     = {Nonlinear {P}erron--{F}robenius Theory},
  series    = {Cambridge Tracts in Mathematics},
  volume    = {189},
  publisher = {Cambridge University Press},
  year      = {2012}
}

@article{crandall1980some,
  author  = {Crandall, Michael G. and Tartar, Luc},
  title   = {Some Relations between Nonexpansive and Order Preserving Mappings},
  journal = {Proceedings of the American Mathematical Society},
  volume  = {78},
  number  = {3},
  pages   = {385--390},
  year    = {1980}
}

@book{hiriart2001fundamentals,
  author    = {Hiriart-Urruty, Jean-Baptiste and Lemar{\'e}chal, Claude},
  title     = {Fundamentals of Convex Analysis},
  series    = {Grundlehren Text Editions},
  publisher = {Springer},
  year      = {2001}
}

@article{kakade2012regularization,
  author  = {Kakade, Sham M. and Shalev-Shwartz, Shai and Tewari, Ambuj},
  title   = {Regularization Techniques for Learning with Matrices},
  journal = {Journal of Machine Learning Research},
  volume  = {13},
  number  = {59},
  pages   = {1865--1890},
  year    = {2012}
}

@book{dieudonne1960foundations,
  author    = {Dieudonn{\'e}, Jean},
  title     = {Foundations of Modern Analysis},
  series    = {Pure and Applied Mathematics},
  volume    = {10},
  publisher = {Academic Press},
  year      = {1960}
}

@article{freedman1975tail,
  author  = {Freedman, David A.},
  title   = {On Tail Probabilities for Martingales},
  journal = {The Annals of Probability},
  volume  = {3},
  number  = {1},
  pages   = {100--118},
  year    = {1975}
}

@article{kuszmaul2021multiplicative,
  author  = {Kuszmaul, William and Qi, Qi},
  title   = {The Multiplicative Version of {A}zuma's Inequality, with an Application to Contention Analysis},
  journal = {arXiv preprint arXiv:2102.05077},
  year    = {2021}
}

@book{roch2024modern,
  author    = {Roch, S{\'e}bastien},
  title     = {Modern Discrete Probability: An Essential Toolkit},
  series    = {Cambridge Series in Statistical and Probabilistic Mathematics},
  publisher = {Cambridge University Press},
  year      = {2024}
}

@article{bravo2026halpern,
  author  = {Bravo, Mario and Contreras, Juan Pablo},
  title   = {Stochastic {H}alpern Iteration in Normed Spaces and Applications to Reinforcement Learning},
  journal = {arXiv preprint arXiv:2403.12338},
  year    = {2024}
}

@article{nanda2026minimal,
  author  = {Nanda, Phalguni and Chen, Zaiwei},
  title   = {A Minimal-Assumption Analysis of {$Q$}-Learning with Time-Varying Policies},
  journal = {Proceedings of the ACM on Measurement and Analysis of Computing Systems},
  volume  = {10},
  number  = {2},
  year    = {2026}
}

@inproceedings{zhang2021finite,
  author    = {Zhang, Sheng and Zhang, Zhe and Maguluri, Siva Theja},
  title     = {Finite Sample Analysis of Average-Reward {TD} Learning and {$Q$}-Learning},
  booktitle = {Advances in Neural Information Processing Systems},
  year      = {2021}
}

@inproceedings{jin2024feasible,
  author    = {Jin, Ying and Gummadi, Ramki and Zhou, Zhengyuan and Blanchet, Jose},
  title     = {Feasible {$Q$}-Learning for Average Reward Reinforcement Learning},
  booktitle = {International Conference on Artificial Intelligence and Statistics},
  year      = {2024}
}

@inproceedings{chen2025nonasymptotic,
  author    = {Chen, Zaiwei},
  title     = {Non-Asymptotic Guarantees for Average-Reward {$Q$}-Learning with Adaptive Stepsizes},
  booktitle = {Advances in Neural Information Processing Systems},
  year      = {2025}
}

@article{chen2026contraction,
  author  = {Chen, Zijun and Chen, Zaiwei and Si, Nian and Wang, Shengbo},
  title   = {Achieving {$\varepsilon^{-2}$} Dependence for Average-Reward {$Q$}-Learning with a New Contraction Principle},
  journal = {arXiv preprint arXiv:2601.21301},
  year    = {2026}
}

@inproceedings{lee2025anchoring,
  author    = {Lee, Jongmin and Bravo, Mario and Cominetti, Roberto},
  title     = {Near-Optimal Sample Complexity for {MDP}s via Anchoring},
  booktitle = {International Conference on Machine Learning},
  year      = {2025}
}

@inproceedings{wang2024optimal,
  author    = {Wang, Shengbo and Blanchet, Jose and Glynn, Peter},
  title     = {Optimal Sample Complexity for Average Reward {M}arkov Decision Processes},
  booktitle = {International Conference on Learning Representations},
  year      = {2024}
}

@inproceedings{zurek2024span,
  author    = {Zurek, Matthew and Chen, Yudong},
  title     = {Span-Based Optimal Sample Complexity for Weakly Communicating and General Average Reward {MDP}s},
  booktitle = {Advances in Neural Information Processing Systems},
  year      = {2024}
}

@inproceedings{Ouhamma_2026,
  author    = {Ouhamma, Reda and Kamgarpour, Maryam},
  title     = {Learning in Zero-Sum {M}arkov Games: Relaxing Strong Reachability and Mixing Time Assumptions},
  booktitle = {AAAI Conference on Artificial Intelligence},
  year      = {2026}
}

@article{jaksch2010near,
  author  = {Jaksch, Thomas and Ortner, Ronald and Auer, Peter},
  title   = {Near-Optimal Regret Bounds for Reinforcement Learning},
  journal = {Journal of Machine Learning Research},
  volume  = {11},
  pages   = {1563--1600},
  year    = {2010}
}

@inproceedings{bartlett2009regal,
  author    = {Bartlett, Peter L. and Tewari, Ambuj},
  title     = {{REGAL}: A Regularization Based Algorithm for Reinforcement Learning in Weakly Communicating {MDP}s},
  booktitle = {Conference on Uncertainty in Artificial Intelligence},
  year      = {2009}
}

@inproceedings{geist2019theory,
  author    = {Geist, Matthieu and Scherrer, Bruno and Pietquin, Olivier},
  title     = {A Theory of Regularized {M}arkov Decision Processes},
  booktitle = {International Conference on Machine Learning},
  year      = {2019}
}

@article{neu2017unified,
  author  = {Neu, Gergely and Jonsson, Anders and G{\'o}mez, Vicen{\c{c}}},
  title   = {A Unified View of Entropy-Regularized {M}arkov Decision Processes},
  journal = {arXiv preprint arXiv:1705.07798},
  year    = {2017}
}

@inproceedings{todorov2007linearly,
  author    = {Todorov, Emanuel},
  title     = {Linearly-Solvable {M}arkov Decision Problems},
  booktitle = {Advances in Neural Information Processing Systems},
  year      = {2007}
}

@inproceedings{haarnoja2018soft,
  author    = {Haarnoja, Tuomas and Zhou, Aurick and Abbeel, Pieter and Levine, Sergey},
  title     = {Soft Actor-Critic: Off-Policy Maximum Entropy Deep Reinforcement Learning with a Stochastic Actor},
  booktitle = {International Conference on Machine Learning},
  year      = {2018}
}

@inproceedings{vieillard2020leverage,
  author    = {Vieillard, Nino and Kozuno, Tadashi and Scherrer, Bruno and Pietquin, Olivier and Munos, R{\'e}mi and Geist, Matthieu},
  title     = {Leverage the Average: An Analysis of {KL} Regularization in Reinforcement Learning},
  booktitle = {Advances in Neural Information Processing Systems},
  year      = {2020}
}

@article{adamczyk2025avgentropy,
  author  = {Adamczyk, Jacob and Makarenko, Volodymyr and Tiomkin, Stas and Kulkarni, Rahul V.},
  title   = {Average-Reward Soft Actor-Critic},
  journal = {Reinforcement Learning Journal},
  volume  = {6},
  pages   = {412--430},
  year    = {2025}
}

@inproceedings{teh2017distral,
  author    = {Teh, Yee Whye and Bapst, Victor and Czarnecki, Wojciech M. and Quan, John and Kirkpatrick, James and Hadsell, Raia and Heess, Nicolas and Pascanu, Razvan},
  title     = {Distral: Robust Multitask Reinforcement Learning},
  booktitle = {Advances in Neural Information Processing Systems},
  year      = {2017}
}

@article{wu2019behavior,
  author  = {Wu, Yifan and Tucker, George and Nachum, Ofir},
  title   = {Behavior Regularized Offline Reinforcement Learning},
  journal = {arXiv preprint arXiv:1911.11361},
  year    = {2019}
}

@inproceedings{xiong2024iterative,
  author    = {Xiong, Wei and Dong, Hanze and Ye, Chenlu and Wang, Ziqi and Zhong, Han and Ji, Heng and Jiang, Nan and Zhang, Tong},
  title     = {Iterative Preference Learning from Human Feedback: Bridging Theory and Practice for {RLHF} under {KL}-Constraint},
  booktitle = {International Conference on Machine Learning},
  year      = {2024}
}

@article{cichocki2010families,
  author  = {Cichocki, Andrzej and Amari, Shun-ichi},
  title   = {Families of Alpha- Beta- and Gamma- Divergences: Flexible and Robust Measures of Similarities},
  journal = {Entropy},
  volume  = {12},
  number  = {6},
  pages   = {1532--1568},
  year    = {2010}
}

@article{belousov2019entropic,
  author  = {Belousov, Boris and Peters, Jan},
  title   = {Entropic Regularization of {M}arkov Decision Processes},
  journal = {Entropy},
  volume  = {21},
  number  = {7},
  pages   = {674},
  year    = {2019}
}

@article{lee2018sparse,
  author  = {Lee, Kyungjae and Choi, Sungjoon and Oh, Songhwai},
  title   = {Sparse {M}arkov Decision Processes with Causal Sparse {T}sallis Entropy Regularization for Reinforcement Learning},
  journal = {IEEE Robotics and Automation Letters},
  volume  = {3},
  number  = {3},
  pages   = {1466--1473},
  year    = {2018}
}

@inproceedings{zhu2023general,
  author    = {Zhu, Lingwei and Chen, Zheng and Schlegel, Matthew and White, Martha},
  title     = {General {M}unchausen Reinforcement Learning with {T}sallis {K}ullback--{L}eibler Divergence},
  booktitle = {Advances in Neural Information Processing Systems},
  year      = {2023}
}

@inproceedings{peters2019sparse,
  author    = {Peters, Ben and Niculae, Vlad and Martins, Andr{\'e} F. T.},
  title     = {Sparse Sequence-to-Sequence Models},
  booktitle = {Annual Meeting of the Association for Computational Linguistics},
  year      = {2019}
}

@book{sutton2018reinforcement,
  author    = {Sutton, Richard S. and Barto, Andrew G.},
  title     = {Reinforcement Learning: An Introduction},
  edition   = {2nd},
  publisher = {MIT Press},
  year      = {2018}
}

@article{kearns2002sparse,
  author  = {Kearns, Michael and Mansour, Yishay and Ng, Andrew Y.},
  title   = {A Sparse Sampling Algorithm for Near-Optimal Planning in Large {M}arkov Decision Processes},
  journal = {Machine Learning},
  volume  = {49},
  number  = {2--3},
  pages   = {193--208},
  year    = {2002}
}

@inproceedings{xiao2019ments,
  author    = {Xiao, Chenjun and Huang, Ruitong and Mei, Jincheng and Schuurmans, Dale and M{\"u}ller, Martin},
  title     = {Maximum Entropy {M}onte-{C}arlo Planning},
  booktitle = {Advances in Neural Information Processing Systems},
  year      = {2019}
}

@inproceedings{grill2019planning,
  author    = {Grill, Jean-Bastien and Darwiche Domingues, Omar and M{\'e}nard, Pierre and Munos, R{\'e}mi and Valko, Michal},
  title     = {Planning in Entropy-Regularized {M}arkov Decision Processes and Games},
  booktitle = {Advances in Neural Information Processing Systems},
  year      = {2019}
}

@inproceedings{grill2020mcts,
  author    = {Grill, Jean-Bastien and Altch{\'e}, Florent and Tang, Yunhao and Hubert, Thomas and Valko, Michal and Antonoglou, Ioannis and Munos, R{\'e}mi},
  title     = {{M}onte-{C}arlo Tree Search as Regularized Policy Optimization},
  booktitle = {International Conference on Machine Learning},
  year      = {2020}
}

@article{cayci2024convergence,
  title   = {Convergence of Entropy-Regularized Natural Policy Gradient
             with Linear Function Approximation},
  author  = {Cayci, Semih and He, Niao and Srikant, R.},
  journal = {SIAM Journal on Optimization},
  volume  = {34},
  number  = {3},
  pages   = {2729--2755},
  year    = {2024},
  doi     = {10.1137/22M1540156}
}

\newpage 
\appendix

\etocdepthtag.toc{mtappendix}
\etocsettagdepth{mtmain}{none}          %
\etocsettagdepth{mtappendix}{subsection} %
\etocstandardlines                      %
\etocsettocstyle{\section*{Appendix Contents}}{}
\tableofcontents
\clearpage

\makeatletter
\@for\thm@ctr:=assumption,lemma,theorem,remark,note,definition,corollary,proposition,example,convention\do{%
  \expandafter\counterwithin\expandafter{\thm@ctr}{section}%
  \expandafter\xdef\csname theH\thm@ctr\endcsname{\noexpand\thesection.\noexpand\arabic{\thm@ctr}}}
\makeatother
\counterwithin{equation}{section}
\renewcommand{\theHequation}{\thesection.\arabic{equation}}

\crefalias{section}{appendix}
\crefalias{subsection}{appendix}
\crefalias{subsubsection}{appendix}

\section{Notation Summary}

\begin{table}[H]
    \centering
    \small
    \caption{Notation used across sections.}
    \label{tab:notation}
    \begin{tabular}{@{}lp{0.7\linewidth}@{}}
        \toprule
        \multicolumn{2}{@{}l}{\emph{MDP and kernels}}\\
        $\mS$, $\mA$, $p$, $r$ & finite state and action sets, transition kernel, deterministic reward\\
        $\bP_\pi(\cdot\mid s_0=s)$, $\bE_\pi[\cdot\mid s_0=s]$ & law and expectation of the chain started at $s$, with $a_t\sim\pi(\cdot|s_t)$ and $s_{t+1}\sim p(\cdot|s_t,a_t)$\\
        $P$ & rectangular kernel $(s,a)\mapsto p(\cdot|s,a)$, from $\mS\times\mA$ to $\mS$\\
        $P_\pi$ & state kernel $P_\pi(s'|s)=\sum_a\pi(a|s)p(s'|s,a)$, acting on $\bR^\mS$\\
        $K_\pi$ & state--action kernel $K_\pi(s',a'|s,a)=p(s'|s,a)\pi(a'|s')$, acting on $\bR^{\mS\times\mA}$\\
        $K_{\pi_1:\pi_m}$, $K^m_\mu$ & $m$-step state--action kernels, \eqref{eq:mstep_sa_kernel}; $K^m_\mu(s'|s,a)$ is the state marginal\\
        $\Pi$ & centering projection $\Pi[Q]=Q-\big(\ones^\top Q/|\mS||\mA|\big)\ones$\\
        \midrule
        \multicolumn{2}{@{}l}{\emph{Norms and seminorms}}\\
        $\|\cdot\|_\spn$, $\|Q\|_\asp$ & span seminorm, action-span seminorm\\
        $\|\cdot\|_{\spn,\op}$ & operator span seminorm, \Cref{def:operator_span_seminorm}\\
        $\|\cdot\|_{\mathrm{TV}}$, $\alpha(P)$ & total variation distance; Dobrushin coefficient, \Cref{def:dobrushin}\\
        \midrule
        \multicolumn{2}{@{}l}{\emph{Regularization}}\\
        $\tau$, $\Omega_s$, $\dom(\Omega_s)$ & temperature; regularizer at $s$; its effective domain\\
        $\Omega^*$, $\pi_{\Omega,Q}$ & conjugate of $\tau\Omega$ and greedy policy, \Cref{subsec:regularized}\\
        $\mT$, $\mT_\Omega$ & hard and regularized Bellman operators\\
        $Q^*,\rho^*$; $Q^*_\Omega,\rho^*_\Omega$ & optimal and regularized relative value and gain\\
        $\tilde r_\pi$ & regularized reward $\langle\pi_s,r_s\rangle-\tau\Omega_s(\pi_s)$\\
        $\underline{\Omega}_s$, $\underline{\Omega}$, $\overline{\Omega}$, $\Delta_\Omega$ & $\min\Omega_s$; $\min_s\underline{\Omega}_s$; $\max_{s,a}\Omega_s(\delta_a)$; $\overline{\Omega}-\underline{\Omega}$\\
        $\bar\pi_{\Omega,[Q,Q']}$, $R[Q,Q']$ & average-path policy and radius\\
        $\gamma_\Omega(R)$ & domination profile, \Cref{def:domination_profile}\\
        $g_Q$, $\Lambda_\Omega(R)$, $\Lambda_\Omega^{\mathrm{pt}}(R)$ & regularization gap and moduli, \Cref{def:regularization_modulus}, \Cref{def:pointwise-modulus}\\
        \midrule
        \multicolumn{2}{@{}l}{\emph{Mixing and rates}}\\
        $\mu$, $m$, $\beta_m$, $\nu$ & mixing policy and its $m$-step Doeblin constants, \Cref{assumption:m-step-mixing}\\
        $\eta$, $\eta_{\min}$ & $\eta(s,a)=\nu(s)\mu(a|s)$ and its minimum\\
        $\mu_{\min}$, $\nu_{\min}$ & $\min_{s,a}\mu(a|s)$ and $\min_s\nu(s)$\\
        $\tilde\beta_m(A)$ & regularized $m$-step contraction rate $\beta_m\gamma_\Omega(A)^{m-1}$\\
        \midrule
        \multicolumn{2}{@{}l}{\emph{Learning analysis}}\\
        $\delta$, $T$, $e_0$, $\iota$ & failure probability, budget, initial error $\|Q_0-Q^*_\Omega\|_\spn$, $\log(6|\mS||\mA|(T+1)/\delta)$, \Cref{app:proofs_learning}\\
        $\bar\gamma$, $\tilde\beta$, $\bar w$ & domination constant $\gamma_\Omega(\|Q^*_\Omega\|_\asp+2e_0)$, contraction rate $\beta_m\bar\gamma^{\,m-1}$ and noise scale on the ball of radius $2e_0$, \Cref{sec:results}\\
        $A_R$, $C_R$, $\bar\beta(R)$, $\bar w(R)$ & localized radii and constants, \Cref{def:localized-radii}\\
        $p_m$, $\kappa_0$ & probability of $m$ effective updates per block and sub-block budget, \Cref{lem:pickup_synchronous}\\
        $\sigma_R$, $e_t$ & exit time from the ball of radius $R$ and error at time $t$, \Cref{def:localized-radii}\\
        $\mB_i$, $b_i$, $\mB_{\mathrm{res}}$, $I(T)$ & blocks, their first indices, last incomplete block, number of completed blocks, \Cref{def:sub_blocks,subsubsec:sync_stochastic}\\
        $J(T)$, $\ell_j$, $L_j(\delta')$, $q_R$ & completed epochs, epoch length and its bound, coverage constant, \Cref{lem:ep_cover_time}\\
        \midrule
        \multicolumn{2}{@{}l}{\emph{Planning analysis}}\\
        $\xi$, $H$, $Q_B$ & reset distribution, planning horizon, boundary value, \Cref{subsec:planning_results}\\
        $\mS_h$, $\mX_h$, $\mS'$, $\bar S_H$ & reachable layers, $\mS_h\times\mA$, reachable set $\bigcup_h\mS_h$, $\max_{h\leq H}|\mS_h|$, \Cref{app:planning_results}\\
        $h_1$, $\mS_\infty$ & depth after which the layers no longer change, and their common value, \Cref{assumption:pl_local_mixing}\\
        $\mT_{\Omega,h}$, $Q^*_{\Omega,h}$ & depth-dependent operator from $\bR^{\mX_{h+1}}$ to $\bR^{\mX_h}$, planning solution, \Cref{subsec:planning_results}\\
        $\Delta_B$, $D_0$ & boundary mismatch $\|Q_B-Q^*_\Omega\|_{\mX_H,\spn}$, initialization error $\max_{h<H}\|Q_{0,h}-Q^*_{\Omega,h}\|_\spn$; in planning, $e_0:=D_0+\Delta_B$\\
        $H_{\mathrm{eff}}(A)$, $M(h,A)$ & effective horizon $m/\tilde\beta_m(A)$ and $(h_1-h)_++H_{\mathrm{eff}}(A)$, \eqref{eq:pl_Heff}; $H_{\mathrm{eff}}$ and $M$ denote their values at $A=A_{2e_0}$ and $h=0$, \Cref{app:planning_results}\\    
        $q_\mu$, $q_\pi(A)$, $q$ & coverage constants under $\mu$ and under a behavior dominating $\gamma_\Omega(A)\mu$, \eqref{eq:pl_q}; $q$ is the constant of \Cref{thm:pl-epoch-main}\\
        $\iota_H$ & $\log(4JH\bar S_H|\mA|/\delta)$, \Cref{thm:pl-epoch-main}\\
        \bottomrule
    \end{tabular}
\end{table}

\section{Algorithm Pseudocodes}
\label{app:pseudocodes}

\subsection{Learning}

All the learning algorithms apply the centering projection $\Pi[Q]:=Q-\frac{1}{|\mS||\mA|}\ones^\top Q\ones$, which only shifts the iterates by a constant and does not affect the analysis in span seminorm.

\begin{algorithm}[!ht]
\caption{Synchronous $\Omega$-Regularized $Q$-Iteration}
\label{alg:synchronous_regularized_q_iteration}
\begin{algorithmic}[1]
\STATE \textbf{Input:} initial $Q_0$, regularizer $\Omega$, transition kernel $p$.
\FOR{$t=0,1,\dots$}
    \STATE $\mT_\Omega(Q_t)(s,a)\leftarrow r(s,a)+\sum_{s'}p(s'|s,a)\,\Omega^*(Q_t)(s')$ for every $(s,a)\in\mS\times\mA$.
    \STATE $Q_{t+1}\leftarrow\Pi[\mT_\Omega(Q_t)]$.
\ENDFOR
\end{algorithmic}
\end{algorithm}

\begin{algorithm}[!ht]
\caption{Synchronous $\Omega$-Regularized $Q$-Learning}
\label{alg:synchronous_regularized_q_learning}
\begin{algorithmic}[1]
\STATE \textbf{Input:} initial $Q_0$, regularizer $\Omega$, step sizes $\{\alpha_t\}_{t\geq0}$, generative model of $p$.
\FOR{$t=0,1,\dots$}
    \STATE Sample $s'_t(s,a)\sim p(\cdot|s,a)$ independently for every $(s,a)\in\mS\times\mA$.
    \STATE $\widehat{\mT_\Omega Q_t}(s,a)\leftarrow r(s,a)+\Omega^*(Q_t)\big(s'_t(s,a)\big)$ for every $(s,a)$.
    \STATE $Q_{t+1}\leftarrow\Pi\big[(1-\alpha_t)Q_t+\alpha_t\widehat{\mT_\Omega Q_t}\big]$.
\ENDFOR
\end{algorithmic}
\end{algorithm}

\begin{algorithm}[htbp]
\caption{Epoch $\Omega$-Regularized $Q$-Learning}
\label{alg:epoch_regularized_q_learning}
\begin{algorithmic}[1]
\STATE \textbf{Input:} initial $Q_0$, regularizer $\Omega$, step sizes $\{\alpha_j\}$, batch sizes $\{B_j\}$, initial state $s_0$.
\STATE Set $n\leftarrow0$ and $j\leftarrow0$.
\WHILE{\textbf{true}}
    \STATE $N_j\leftarrow0$ and $\widehat{\mT_\Omega Q_j}\leftarrow0$ on $\mS\times\mA$.
    \WHILE{$\min_{(s,a)}N_j(s,a)<B_j$}
        \STATE Sample $a_n\sim\pi_{\Omega,Q_j}(\cdot|s_n)$ and $s_{n+1}\sim p(\cdot|s_n,a_n)$.
        \IF{$N_j(s_n,a_n)<B_j$}
            \STATE $\widehat{\mT_\Omega Q_j}(s_n,a_n)\pluseq r(s_n,a_n)+\Omega^*(Q_j)(s_{n+1})$ and $N_j(s_n,a_n)\pluseq1$.
        \ENDIF
        \STATE $n\leftarrow n+1$.
    \ENDWHILE
    \STATE $Q_{j+1}\leftarrow\Pi\big[(1-\alpha_j)Q_j+\alpha_j\widehat{\mT_\Omega Q_j}/B_j\big]$ and $j\leftarrow j+1$.
\ENDWHILE
\end{algorithmic}
\end{algorithm}

\subsection{Planning}

The planning algorithms take the reachable layers $\mS_0,\dots,\mS_H$ of \Cref{app:planning_results} as input; they depend only on $\xi$ and on the support of $p$, and the stopping rules of the model-free algorithms require them.

\begin{algorithm}[H]
\caption{Backward Synchronous $\Omega$-Regularized $Q$-Planning}
\label{alg:pl_backward_sync}
\begin{algorithmic}[1]
\STATE \textbf{Input:} boundary value $Q_B$, regularizer $\Omega$, transition kernel $p$, layers $\mS_0,\dots,\mS_H$.
\STATE Set $Q_H\leftarrow Q_B|_{\mX_H}$.
\FOR{$h=H-1,\dots,0$}
    \STATE $Q_h(s,a)\leftarrow r(s,a)+\sum_{s'\in\mS_{h+1}}p(s'|s,a)\,\Omega^*(Q_{h+1})(s')$ for every $(s,a)\in\mX_h$.
\ENDFOR
\STATE \textbf{Return} $(Q_h)_{h=0}^{H}$.
\end{algorithmic}
\end{algorithm}

\begin{algorithm}[H]
\caption{Epoch $\Omega$-Regularized $Q$-Planning}
\label{alg:epoch_regularized_q_planning}
\begin{algorithmic}[1]
\STATE \textbf{Input:} initial $Q_0$, boundary value $Q_B$, regularizer $\Omega$, batch sizes $\{B_j\}$, reset distribution $\xi$, layers $\mS_0,\dots,\mS_H$.
\STATE Set $Q_{0,h}\leftarrow Q_0|_{\mX_h}$ for $h<H$, $Q_{j,H}:=Q_B|_{\mX_H}$ for every $j$, and $j\leftarrow0$.
\WHILE{\textbf{true}}
    \STATE $N_{j,h}\leftarrow0$ and $\widehat{\mT_{\Omega,h}Q_{j,h+1}}\leftarrow0$ on $\mX_h$, for every $h<H$.
    \WHILE{$\min_{h<H}\min_{(s,a)\in\mX_h}N_{j,h}(s,a)<B_j$}
        \STATE Sample $s_0\sim\xi$.
        \FOR{$h=0,\dots,H-1$}
            \STATE Sample $a_h\sim\pi_{\Omega,Q_{j,h}}(\cdot|s_h)$ and $s_{h+1}\sim p(\cdot|s_h,a_h)$.
            \IF{$N_{j,h}(s_h,a_h)<B_j$}
                \STATE $\widehat{\mT_{\Omega,h}Q_{j,h+1}}(s_h,a_h)\pluseq r(s_h,a_h)+\Omega^*(Q_{j,h+1})(s_{h+1})$ and $N_{j,h}(s_h,a_h)\pluseq1$.
            \ENDIF
        \ENDFOR
    \ENDWHILE
    \STATE $Q_{j+1,h}\leftarrow\widehat{\mT_{\Omega,h}Q_{j,h+1}}/B_j$ for every $h<H$, and $j\leftarrow j+1$.
\ENDWHILE
\end{algorithmic}
\end{algorithm}

\begin{algorithm}[H]
\caption{Epoch $\Omega$-Regularized $Q$-Learning with Resets}
\label{alg:pl_reset_learning}
\begin{algorithmic}[1]
\STATE \textbf{Input:} initial $Q_0\in\bR^{\mS_\infty\times\mA}$, regularizer $\Omega$, step sizes $\{\alpha_j\}$, batch sizes $\{B_j\}$, reset distribution $\xi$, burn-in depth $h_1$ and core $\mS_\infty$, rollout length $H>h_1$.
\STATE Set $j\leftarrow0$.
\WHILE{\textbf{true}}
    \STATE $N_j\leftarrow0$ and $\widehat{\mT_\Omega Q_j}\leftarrow0$ on $\mS_\infty\times\mA$.
    \WHILE{$\min_{(s,a)\in\mS_\infty\times\mA}N_j(s,a)<B_j$}
        \STATE Sample $s_0\sim\xi$.
        \FOR{$h=0,\dots,H-1$}
            \STATE Sample $a_h\sim\mu(\cdot|s_h)$ if $h<h_1$ and $a_h\sim\pi_{\Omega,Q_j}(\cdot|s_h)$ otherwise, and $s_{h+1}\sim p(\cdot|s_h,a_h)$.
            \IF{$h\geq h_1$ \textbf{and} $N_j(s_h,a_h)<B_j$}
                \STATE $\widehat{\mT_\Omega Q_j}(s_h,a_h)\pluseq r(s_h,a_h)+\Omega^*(Q_j)(s_{h+1})$ and $N_j(s_h,a_h)\pluseq1$.
            \ENDIF
        \ENDFOR
    \ENDWHILE
    \STATE $Q_{j+1}\leftarrow\Pi\big[(1-\alpha_j)Q_j+\alpha_j\widehat{\mT_\Omega Q_j}/B_j\big]$ and $j\leftarrow j+1$.
\ENDWHILE
\end{algorithmic}
\end{algorithm}

\section{Additional Experiments}
\label{app:experiments}

\subsection{Two Rooms: experimental details and bias}
\label{app:two_rooms}

\paragraph{Setting.} The Two Rooms MDP consists of two $5\times5$ gridworld rooms connected by a two-cell corridor, with the four cardinal moves as actions. Moves are deterministic except when entering the corridor from either room: entry succeeds with probability $p_{\text{cross}}$ and otherwise leaves the agent in place. The agent starts in the top-left corner of the left room and receives reward $1$ in the bottom-right corner of the right room, and $0$ elsewhere. At the two doorway entrances, $\mu$ puts mass $\mu(\text{DOOR})$ on the action entering the doorway and $(1-\mu(\text{DOOR}))/3$ on each of the three others; everywhere else it is uniform. The regularizer is the $\epsilon$-floored reverse KL with $\tau=2$ and $\epsilon=0.05$ in all panels. We sweep $\mu(\text{DOOR})\in\{0.25,0.4,0.55,0.7,0.85,0.95\}$ and $p_{\text{cross}}\in\{0.05,0.3,1\}$.

\paragraph{Convergence.} Synchronous regularized $Q$-iteration is exact, so no seeds are involved. The rate is fitted on the part of the error curve with $10^{-9}<\|Q_k-Q^*_\Omega\|_\spn/\|Q_0-Q^*_\Omega\|_\spn<10^{-3}$. The convergence accelerates before settling: early iterates have a larger action-span, hence a sharper policy, a smaller domination profile $\gamma_\Omega$ and a slower contraction. On the slowest configuration ($p_\text{cross}=0.05$, uniform $\mu$) the first $e$-fold takes about $170$ iterations against $54.5$ once settled, so a rate fitted on a fixed initial window depends on the window; we report the settled rate, which does not.

\paragraph{Exploration.} Epoch lengths are for $B_j=1$ accepted sample per state--action pair, with a cap of $4\cdot10^5$ transitions per epoch. Under $\mu$: median over $12$ seeds. On-policy: $4$ epochs, median over $5$ seeds of the mean of the last $2$ epochs, the first epochs being dominated by the initialization $Q_0=0$. All curves are normalized by their value at the uniform $\mu$. In absolute terms, one epoch under the uniform $\mu$ costs $12\,304$, $3\,096$ and $1\,901$ transitions for $p_\text{cross}=0.05,0.3,1$, respectively. The coverage-time proxy $\min_m m/(\beta_m\eta_{\min})$ is $69$, $52$ and $39$ times the corresponding empirical epoch lengths, respectively, which motivates the normalized comparison.

\begin{figure}[H]
    \centering
    \includegraphics[width=0.32\textwidth]{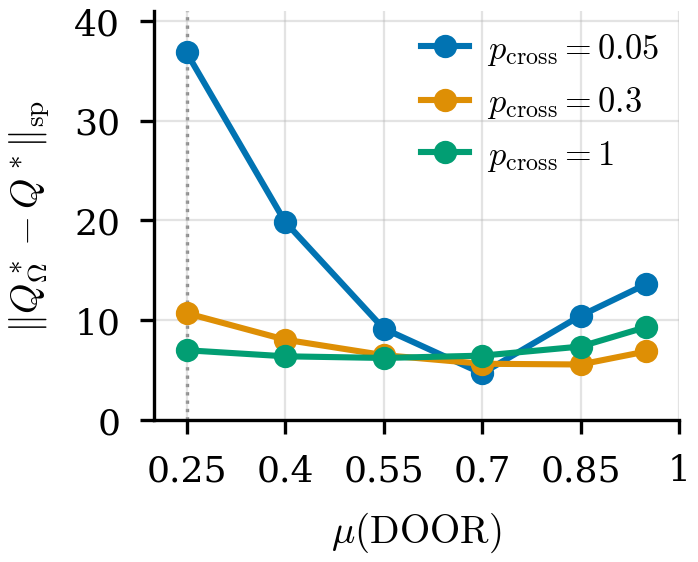}
    \caption{Bias $\|Q^*_\Omega-Q^*\|_\spn$ on Two Rooms against $\mu(\text{DOOR})$, exact. Dotted line: uniform $\mu$.}
    \label{fig:two_rooms_bias}
\end{figure}

\paragraph{Bias.} \Cref{fig:two_rooms_bias} reports the bias $\|Q^*_\Omega-Q^*\|_\spn$, computed exactly. When the doorway is a bottleneck, biasing $\mu$ towards it reduces the bias from $36.8$ to $4.7$, but past $\mu(\text{DOOR})\approx0.7$ the bias rises again, likely because the same knob also pushes the agent from the goal room back through the doorway. The convergence and exploration panels reward a large $\mu(\text{DOOR})$; the bias is what prevents pushing it indefinitely.

\subsection{Choice of the regularizer}
\label{app:omega_choice}

\begin{figure}[H]
    \centering
    \includegraphics[width=\linewidth]{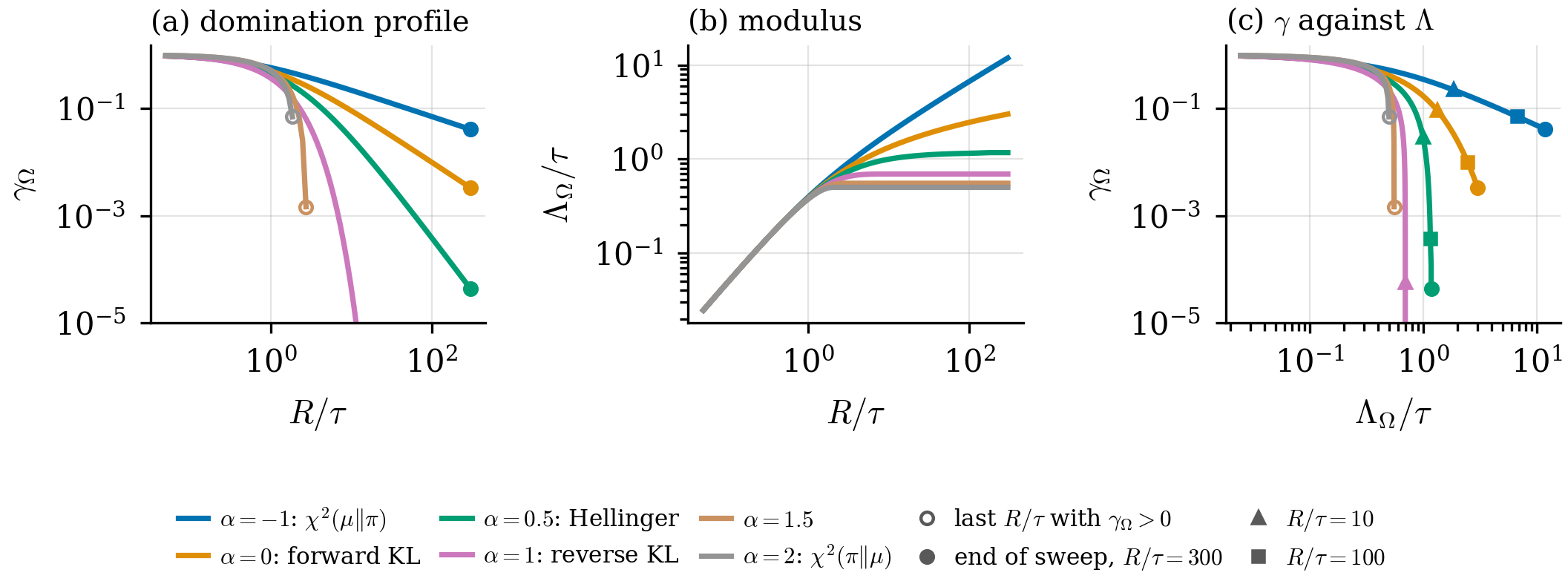}
    \caption{Mixing quantities for the $\alpha$-divergences to a uniform $\mu$ over two actions, and the resulting bias. (a) Domination profile $\gamma_\Omega$ and (b) regularization modulus $\Lambda_\Omega/\tau$, both functions of $R/\tau$ only. (c) The pair $(\Lambda_\Omega/\tau,\gamma_\Omega)$ as $R/\tau$ grows.}
    \label{fig:omega_choice}
\end{figure}

\Cref{fig:omega_choice} computes the two objects that govern our rates and bias for the $\alpha$-divergences to $\mu$. Since the greedy policy only depends on $Q/\tau$, $\gamma_\Omega(R)$ and $\Lambda_\Omega(R)/\tau$ only depend on $R/\tau$, so panels (a) and (b) hold for every temperature.

Panel (a) illustrates \Cref{thm:alpha}.(iv): for $\alpha>1$, the profile vanishes at a finite radius (open circles), at $R/\tau=\mu_{\max}^{1-\alpha}/(\alpha-1)$, that is $2$ for $\alpha=2$ and $2\sqrt2$ for $\alpha=1.5$, so that \Cref{assumption:positive-domination} fails; for $\alpha\leq1$, it stays positive, with an exponential decay for the reverse KL and a polynomial one for $\alpha<1$, slower for smaller $\alpha$. Panel (b) illustrates \Cref{thm:alpha}.(v): the modulus saturates for $\alpha>0$, at $\log2$ for the reverse KL and $4(1-2^{-1/2})$ for the Hellinger regularizer, and grows without bound for $\alpha\leq0$, as $\log(R/\tau)$ for the forward KL and $(R/\tau)^{1/2}$ for $\alpha=-1$.

Panel (c) traces both objects together as the radius grows: a regularizer is preferable when its curve lies higher, for a faster contraction and coverage, and further left, for a smaller bias. No value of $\alpha$ dominates. Divergences with $\alpha\leq0$ keep a large profile at the price of an unbounded modulus, and those with $\alpha>1$ have a small modulus but lose \Cref{assumption:positive-domination}; only the window $0<\alpha\leq1$ keeps the modulus bounded and the profile positive. Within it, the Hellinger regularizer keeps a larger profile at large $R/\tau$, whereas the reverse KL has the smaller modulus for this balanced $\mu$, as discussed in \Cref{rem:alpha_exact}.

The value bias on RiverSwim (\Cref{fig:river_mdp_bias}) follows the ordering of the moduli in panel (b): at a given temperature, it decreases with $\alpha$, and as $\tau\to0$ it vanishes fastest for the bounded regularizers $\alpha>0$, consistently with \Cref{lem:mixing_main}.(iii) and \Cref{thm:alpha}.(v). For large $\tau$, all greedy policies approach $\mu$ and the biases coincide. The $\epsilon$-floor restores \Cref{assumption:positive-domination} for $\alpha>1$ at the price of a slightly larger bias at small temperature.

\paragraph{Conservativeness of the span bound on RiverSwim.}
For the six-state RiverSwim instance of \Cref{fig:river_mdp_convergence}, the uniform reference policy and the canonical Doeblin pair of \Cref{app:mixing_assumption} at $m=14$ give
\begin{equation*} 
    \beta_{14}\simeq0.813924,\qquad \nu_{\min}\simeq0.00246331,\qquad \beta_{14}\nu_{\min}\simeq0.00200495.
\end{equation*}

Since $\|r\|_\spn=1$, the reference-based bound of \Cref{subsubsec:sufficient_structural_conditions} is $U\simeq6\,983.73$. Here $\|Q^*\|_\spn\simeq14.6668$. At the matched bias $\|Q^*_\Omega-Q^*\|_\spn=2$, the triangle inequality gives
\begin{equation*} 
    \|Q^*_\Omega\|_\spn\leq\|Q^*\|_\spn+2\simeq16.6668, 
\end{equation*}
so $U$ exceeds each regularized solution span by a factor of more than $400$. Together with the coverage-time comparisons in \Cref{app:two_rooms}, this quantifies the conservativeness of individual ingredients of the theoretical prescriptions. The full prescribed learning-rate and batch-size schedules have not been benchmarked.

\section{Discussion on \pdfref{Assumption}{assumption:m-step-mixing}}
\label{app:mixing_assumption}

We characterize the MDPs that satisfy \Cref{assumption:m-step-mixing} and its constants. Throughout, $\mu$ has full support and $K^n_\mu(s'|s,a)=\sum_{s_1}p(s_1|s,a)P^{n-1}_\mu(s'|s_1)$ (\Cref{app:math}).

\paragraph{Canonical constants.}
Let $\underline{V}^n_\mu(s'):=\min_{(s,a)}K^n_\mu(s'|s,a)$. Since the minorization of \Cref{assumption:m-step-mixing} is uniform in the starting pair, a pair $(\beta,\tilde\nu)$ satisfies it at horizon $m$ if and only if $\beta\tilde\nu\leq\underline{V}^m_\mu$ pointwise. The assumption thus holds at horizon $m$ if and only if $\underline{V}^m_\mu>0$, and the canonical pair $\beta_m:=\|\underline{V}^m_\mu\|_1$, $\nu:=\underline{V}^m_\mu/\beta_m$, which we always use, dominates every other one, both in the contraction constant $\beta$ and in the coverage profile $\beta\tilde\nu\otimes\mu$. Averaging $K^n_\mu(s'|s_1,a_1)\geq\underline{V}^n_\mu(s')$ over $a_1\sim\mu(\cdot|s_1)$ and decomposing $K^{n+1}_\mu$ on its first step gives $\underline{V}^{n+1}_\mu\geq\underline{V}^n_\mu$: the assumption then holds at every horizon $n\geq m$, with canonical constants at least as good.

In fact, the contraction and the exploration rely on two different properties of $K^m_\mu$. The contraction only requires a bound on its Dobrushin coefficient, which the minorization provides as $\alpha(K^m_\mu)\leq1-\beta_m$, possibly strictly (\Cref{lem:dobrushin_doeblin}), whereas the exploration requires the minorization itself, which bounds below the probability of reaching each pair. We state both through the Doeblin pair $(\beta_m,\nu)$, so that the domination profile transfers them at once: the greedy policies contract at rate $\beta_m\gamma_\Omega^{m-1}$ and visit every pair with probability at least $\beta_m\gamma_\Omega^m\eta_{\min}$ every $m$ steps (\Cref{lem:mixing_main}).

\begin{definition}[Union graph]
    \label{def:union_graph}
    The \emph{union graph} $G$ is the directed graph on $\mS$ with an edge $s\to s'$ whenever $p(s'|s,a)>0$ for some $a\in\mA$. For $X\subseteq\mS$ and $k\geq0$, $XG^k$ is the set of states reached from $X$ by a walk of length exactly $k$. $G$ is \emph{primitive} if it is strongly connected and aperiodic, and its \emph{exponent} $\mathrm{ex}(G)$ is then the smallest $k\geq1$ such that $\{s\}G^k=\mS$ for every $s$.
\end{definition}

Since $\mu$ has full support, $P^n_\mu(s'|s)>0$ if and only if $s'\in\{s\}G^n$: $G$ is the support graph of $P_\mu$ for every full-support $\mu$.

\begin{proposition}[Which MDPs satisfy \Cref{assumption:m-step-mixing}]
    \label{prop:mixing_characterization}
    For every full-support $\mu$, \Cref{assumption:m-step-mixing} holds for some $m$ if and only if $P_\mu$ is irreducible and aperiodic, if and only if $G$ is primitive, if and only if the MDP is communicating and $G$ is aperiodic. In particular, \Cref{assumption:m-step-mixing} holds with some full-support policy if and only if it holds with every full-support policy: whether it holds only depends on the support of $p$, and $\mu$ only affects its constants $\beta_m$ and $\nu$.
\end{proposition}

\begin{proof}
    The second equivalence holds because $G$ is the support graph of $P_\mu$, and the third because $G$ is strongly connected if and only if the MDP is communicating. If $\underline{V}^m_\mu>0$, then $P^m_\mu>0$ and $P^{m+1}_\mu>0$ entrywise, so $P_\mu$ is irreducible, with closed walks of lengths $m$ and $m+1$ through every state, hence aperiodic. Conversely, if $G$ is primitive, then $\{s_1\}G^{\mathrm{ex}(G)}=\mS$ for every $s_1$, so that, for $m:=\mathrm{ex}(G)+1$ and any $s_1$ with $p(s_1|s,a)>0$, $K^m_\mu(\cdot|s,a)\geq p(s_1|s,a)P^{m-1}_\mu(\cdot|s_1)>0$, i.e.\ $\underline{V}^m_\mu>0$.
\end{proof}

\begin{proposition}[The minimal horizon]
    \label{prop:minimal_horizon}
    Under \Cref{assumption:m-step-mixing}, let $e(X):=\min\{k\geq0\mid XG^k=\mS\}$ for nonempty $X\subseteq\mS$. The smallest horizon at which \Cref{assumption:m-step-mixing} holds is
    \begin{equation*}
        m^*=1+\max_{(s,a)\in\mS\times\mA}e\big(\mathrm{supp}\,p(\cdot|s,a)\big),
    \end{equation*}
    and $\mathrm{ex}(G)\leq m^*\leq\mathrm{ex}(G)+1\leq(|\mS|-1)^2+2$.
\end{proposition}

\begin{proof}
    By the first-step decomposition and the full support of $\mu$, $K^m_\mu(\cdot|s,a)>0$ if and only if $X_{s,a}G^{m-1}=\mS$, with $X_{s,a}:=\mathrm{supp}\,p(\cdot|s,a)$. As $G$ is strongly connected, every state has a predecessor, so $XG^k=\mS$ implies $XG^{k+1}=\mS$, and $\{k\mid XG^k=\mS\}=\{k\geq e(X)\}$. Hence $\underline{V}^m_\mu>0$ if and only if $m-1\geq e(X_{s,a})$ for every $(s,a)$, which gives the formula. Since $e(X)\leq\mathrm{ex}(G)$ for every nonempty $X$, $m^*\leq\mathrm{ex}(G)+1$; since $P^{m^*}_\mu(s'|s)=\sum_a\mu(a|s)K^{m^*}_\mu(s'|s,a)>0$ for every $(s,s')$, $m^*\geq\mathrm{ex}(G)$; and $\mathrm{ex}(G)\leq(|\mS|-1)^2+1$ by \citet{wielandt1950}.
\end{proof}

The threshold $m^*$ is thus determined by the transition supports $\mathrm{supp}\,p(\cdot|s,a)$ and is the same for every full-support reference policy. If every state has a self-loop under some action, then $m^*$ is at most the diameter of $G$ plus one. The horizon $m$ of our rates, in contrast, is not a property of the MDP: any $m\geq m^*$ is valid and, since the algorithms do not use it, our guarantees hold at every such $m$. Larger horizons can only improve $\beta_m$ and $\beta_m\eta_{\min}$, but they weaken the transfer factors $\gamma_\Omega^{m-1}$ for the contraction and $\gamma_\Omega^m$ for the exploration, so the best horizon to read our guarantees balances the mixing of $\mu$ against the domination of the greedy policies.

\begin{remark}[Comparison with other assumptions]
    \label{rem:mixing_comparison}
    Under \Cref{assumption:m-step-mixing}, the MDP is communicating (\Cref{cor:mixing_implies_span}) and every full-support policy induces an irreducible aperiodic chain (\Cref{prop:mixing_characterization}). The assumption is thus equivalent to the standard one on the behavior policy of asynchronous $Q$-learning, a full-support policy inducing an irreducible aperiodic chain \citep[Assumption~3.1]{chen2025nonasymptotic}. \citet{chen2025nonasymptotic} additionally assumes that the Bellman operator is a span contraction (his Assumption~2.1), whereas we obtain the contraction from regularization toward $\mu$, and the constants of $\mu$ enter the rates. It is incomparable with the unichain condition: two states swapped by a single action form a unichain MDP whose $G$ is periodic, and adding a self-loop action at both states makes $G$ primitive but the MDP not unichain. Unlike the unichain condition, whose verification is NP-hard \citep{tsitsiklis2007nphard}, it is checked in polynomial time on $G$, and it is preserved by adding actions or transitions. Periodicity can be removed by laziness, i.e.\ replacing $p(\cdot|s,a)$ by $\frac12(\delta_s+p(\cdot|s,a))$, which preserves the gain of every stationary policy. Finally, if $P_\mu$ is unichain and aperiodic but has transient states, then for $m$ large enough $\beta_m$ is still positive but $\nu$ vanishes exactly on these states, which a trajectory following stationary full-support policies (or policies uniformly dominating $\epsilon\mu$) visits finitely often: full support of $\nu$ is what makes single-trajectory learning possible on every pair.
\end{remark}

\section{Mathematical Background}
\label{app:math}

The following convention will be useful in the proofs.

\begin{convention}
    \label{convention:product_order}
    Throughout, $\mX$, $\mY$ and $\mZ$ denote finite sets, typically $\mS$ or $\mS\times\mA$, and $P$ is a row-stochastic matrix from $\mX$ to $\mY$. It maps functions $f\in\bR^{\mY}$ to $Pf\in\bR^{\mX}$ and distributions $\phi\in\Delta(\mX)$ to $\phi P\in\Delta(\mY)$:
    \[
        (Pf)(x)=\sum_{y\in\mY}P(y|x)f(y), \qquad (\phi P)(y)=\sum_{x\in\mX}\phi(x)P(y|x).
    \]
    The product of $P_1$ from $\mX$ to $\mY$ and $P_2$ from $\mY$ to $\mZ$ is $(P_1P_2)(z|x)=\sum_{y\in\mY}P_1(y|x)P_2(z|y)$, so that $(P_1P_2)f=P_1(P_2f)$ and $\phi(P_1P_2)=(\phi P_1)P_2$. In a product $P_1P_2\cdots P_n$, the Markov chain therefore takes its first step with $P_1$ and its last with $P_n$, whereas a function is transformed first by $P_n$ and last by $P_1$.
\end{convention}

For a policy $\pi$, we also write $P_\pi(s'|s):=\sum_a\pi(a|s)\,p(s'|s,a)$ for its state kernel, and $K^m_\pi(s'|s,a):=\sum_{a'}K^m_\pi(s',a'|s,a)$ for the law of the state reached after $m$ steps from $(s,a)$; in this notation, \Cref{assumption:m-step-mixing} reads $K^m_\mu(\cdot|s,a)\geq\beta_m\nu$.

\subsection{Concentration inequalities}
\label{app:concentration}

The analysis of the model-free algorithms controls two different random quantities.

The number of visits to a state--action pair is a sum of indicators whose conditional probabilities are only known to be bounded below, on an event, by a small constant such as $q_R=\beta_m\gamma_\Omega(A_R)^m\eta_{\min}$; \Cref{lem:self-bounding} bounds its lower tail multiplicatively, which keeps the cover times of \Cref{lem:ep_cover_time} linear in $1/q_R$.

The noise in the Bellman targets is a martingale with bounded increments, whose running maximum is controlled by \Cref{lem:azuma_range}.

\begin{lemma}[Self-bounding lower tail]\label{lem:self-bounding}
    Let $(Z_i)_{i=1}^{B}$ be adapted to a filtration $(\mathcal{G}_i)$ with $Z_i\in[0,b]$ almost surely, where $b>0$. Let $S:=\sum_{i=1}^B Z_i$ and $\bar{S}:=\sum_{i=1}^B\bE[Z_i\mid\mathcal{G}_{i-1}]$. For any deterministic $a>0$ and $c\in(0,a)$,
    \begin{equation}
        \bP\big(S<c,\ \bar{S}\geq a\big) \leq\exp\Big(-\frac{1}{b}\Big(a-c-c\log\frac{a}{c}\Big)\Big)\label{eq:exact-chernoff}
    \end{equation}
\end{lemma}

\begin{proof}
    This is the lower tail of the multiplicative Azuma inequality \citep[App.~A]{kuszmaul2021multiplicative}, stated on the event $\{\bar{S}\geq a\}$ in the manner of \citet{freedman1975tail} rather than under an almost-sure lower bound on $\bar{S}$; we recall the argument.
    Let $\theta>0$ and $p_i:=\bE[Z_i/b\mid\mathcal{G}_{i-1}]$. Convexity of $y\mapsto e^{-\theta y}$ on $[0,1]$ and $1-x\leq e^{-x}$ give $\bE[e^{-\theta Z_i/b}\mid\mathcal{G}_{i-1}]\leq1-p_i(1-e^{-\theta})\leq e^{-p_i(1-e^{-\theta})}$, so that $M_k:=\exp\big(-\frac{\theta}{b}\sum_{i\leq k}Z_i+(1-e^{-\theta})\sum_{i\leq k}p_i\big)$ is a non-negative supermartingale with $M_0=1$.
    On the event $\{S<c,\ \bar{S}\geq a\}$ we have $M_B\geq\exp\big(-\frac{\theta c}{b}+(1-e^{-\theta})\frac{a}{b}\big)$, so Markov's inequality bounds its probability by $\exp\big(\frac{1}{b}(\theta c-(1-e^{-\theta})a)\big)$, which is minimized at $\theta=\log(a/c)$; this is~\eqref{eq:exact-chernoff}.
\end{proof}

\begin{lemma}[Maximal Azuma--Hoeffding inequality with predictable ranges]
    \label{lem:azuma_range}
    Let $(M_k)_{0\leq k\leq n}$ be a martingale with $M_0=0$, adapted to $(\mathcal{F}_k)$, such that for each $k$ the increment $M_k-M_{k-1}$ lies a.s.\ in an interval $[\mathcal{L}_k,\mathcal{L}_k+c_k]$ with $\mathcal{L}_k$ $\mathcal{F}_{k-1}$-measurable and $c_k>0$ deterministic. Then for every $u>0$,
    \begin{equation*}
        \bP\left(\max_{0\leq k\leq n}|M_k|\geq u\right)\leq2\exp\left(-\frac{2u^2}{\sum_{k=1}^nc_k^2}\right).
    \end{equation*}
\end{lemma}

\begin{proof}
    Apply \citet[Theorem~3.2.1]{roch2024modern} to $(M_k)$ with $A_k=\mathcal{L}_k$ and $B_k=\mathcal{L}_k+c_k$, and to $(-M_k)$ with $A_k=-\mathcal{L}_k-c_k$ and $B_k=-\mathcal{L}_k$, then take a union bound.
\end{proof}

\subsection{On span seminorms and the Dobrushin coefficient}

This section provides additional definitions and lemmas that are useful when working with the span seminorm.

\begin{definition}[Total-variation distance]
    \label{def:tv_norm}
    The total-variation distance between two probability distributions $\phi, \phi' \in \Delta(\mX)$ is defined as:
    \begin{equation}
        \|\phi - \phi'\|_{\mathrm{TV}} = \frac{1}{2}\sum_{x\in \mX} |\phi(x) - \phi'(x)|.
    \end{equation}
\end{definition}

\begin{definition}[Dobrushin coefficient]
    \label{def:dobrushin}
    The Dobrushin coefficient of a row-stochastic matrix $P\in\bR^{\mX\times\mY}$ is defined as:
    \begin{equation}
        \alpha(P) = \sup_{\phi\neq\phi' \in \Delta(\mX)}\frac{\|\phi P -\phi' P\|_{\mathrm{TV}}}{\|\phi -\phi'\|_{\mathrm{TV}}}=\sup_{x,x'\in\mX}\|P(\cdot|x)-P(\cdot|x')\|_{\mathrm{TV}},
    \end{equation} 
    the equality of the two expressions being due to \citet{dobrushin1956central}.
\end{definition}

The following lemma turns a Doeblin minorization, such as \Cref{assumption:m-step-mixing}, into a bound on the Dobrushin coefficient, hence into a contraction in span seminorm by \Cref{lem:row-stochastic-span}.

\begin{lemma}[Doeblin and Dobrushin coefficients]
    \label{lem:dobrushin_doeblin}
    Let $\mX$ and $\mY$ be finite and let $P$ be a row-stochastic matrix from $\mX$ to $\mY$. If there exist $\beta\in(0,1]$ and $\nu\in\Delta(\mY)$ with $P(y|x)\geq\beta\nu(y)$ for all $(x,y)\in\mX\times\mY$, then $\alpha(P)\leq 1-\beta$.
\end{lemma}

\begin{proof}
    This is classical \citep[\S4.3]{seneta2006nonnegative}; we recall the argument.
    Since $|a-b|=a+b-2\min\{a,b\}$ and both rows sum to one, $\|P(\cdot|x)-P(\cdot|x')\|_{\mathrm{TV}}=1-\sum_{y}\min\{P(y|x),P(y|x')\}$ for all $x,x'\in\mX$.
    Each minimum is at least $\beta\nu(y)$ by the minorization, so the sum is at least $\beta$ and every row distance is at most $1-\beta$.
\end{proof}

\begin{lemma}[The terminal action does not affect the contraction]
    \label{lem:terminal_action}
    Let $n\geq1$, let $\pi_1,\dots,\pi_n$ be policies and $K:=K_{\pi_1}\cdots K_{\pi_n}$, the chain taking its first step with $K_{\pi_1}$ (\Cref{convention:product_order}). If $\pi_k\geq\gamma_k\mu$ with $\gamma_k\geq0$ for every $k\in\llbracket1,n-1\rrbracket$ and $K^n_\mu(\cdot|s,a)\geq\beta\nu$ for every $(s,a)$, then $\alpha(K)\leq1-\beta\prod_{k=1}^{n-1}\gamma_k$.
\end{lemma}

\begin{proof}
    The rows of $K$ factor as $K(s',a'|x)=A(s'|x)\,\pi_n(a'|s')$, where $A(\cdot|x)$ is the law of the $n$-th state. The terminal factor being common to all rows, $1-\alpha(K)=\min_{x,y}\sum_{s'}\min\{A(s'|x),A(s'|y)\}=1-\alpha(A)$. Only the $n-1$ intermediate actions enter $A$, so the transport inequality $P_{\pi_k}\geq\gamma_kP_\mu$ gives $A(\cdot|x)\geq\big(\prod_{k=1}^{n-1}\gamma_k\big)K^n_\mu(\cdot|x)\geq\beta\big(\prod_{k=1}^{n-1}\gamma_k\big)\nu$, and \Cref{lem:dobrushin_doeblin} concludes.
\end{proof}

\begin{definition}[Operator span seminorm]
    \label{def:operator_span_seminorm}
    The operator span seminorm of a linear operator $A:\bR^p \to \bR^q$ is defined as:
    \begin{equation}
        \|A\|_{\spn,\op} = \sup_{\|x\|_{\spn}\leq 1}\|Ax\|_{\spn}\in\bR^+\cup\{+\infty\}.
    \end{equation}
\end{definition}

\begin{lemma}[Operator span seminorm]
    \label{lem:spop}
    Let $A:\bR^{p}\to\bR^{q}$ be linear. The following two conditions are equivalent:
    \begin{enumerate}[label=(\roman*)]
        \item $\|A\|_{\spn,\op}<+\infty$;
        \item $A\cdot\mathrm{Ker}\|\cdot\|_{\spn}\subseteq\mathrm{Ker}\|\cdot\|_{\spn}$, i.e. $A\ones\in\bR\ones$, i.e. $A$ has constant row sums.
    \end{enumerate}
    When they hold, for every $f\in\bR^{p}$ and every linear map $B:\bR^{p'}\to\bR^{p}$ satisfying (ii),
    \begin{equation}
        \label{eq:sub-multiplicativity}
        \|Af\|_\spn\leq\|A\|_{\spn,\op}\|f\|_\spn, \qquad \|AB\|_{\spn,\op}\leq\|A\|_{\spn,\op}\|B\|_{\spn,\op}.
    \end{equation}
\end{lemma}

\begin{proof}
    Assume $A\ones\notin\bR\ones$, i.e.\ $\|A\ones\|_\spn>0$. For every $c>0$ the vector $c\ones$ satisfies $\|c\ones\|_\spn=0\leq1$, while $\|A(c\ones)\|_\spn=c\|A\ones\|_\spn\to+\infty$ as $c\to+\infty$. Hence $\|A\|_{\spn,\op}=+\infty$.

    Conversely, the span seminorm vanishes exactly on the constant vectors, so it induces a norm, still denoted $\|\cdot\|_\spn$, on the quotient spaces $E_{p}:=\bR^{p}/\bR\ones$ and $E_{q}:=\bR^{q}/\bR\ones$.
    Under (ii), the factorization theorem gives that $A$ induces a well-defined linear map $\tilde{A}:E_{p}\to E_{q}$, $\tilde{A}[f]=[Af]$.
    By construction $\|Af\|_\spn=\|\tilde{A}[f]\|_\spn$ and $\|f\|_\spn=\|[f]\|_\spn$, so $\|A\|_{\spn,\op}$ is the operator norm of $\tilde{A}$ between the normed spaces $(E_{p},\|\cdot\|_\spn)$ and $(E_{q},\|\cdot\|_\spn)$, which is finite since $E_{p}$ is finite-dimensional.

    Finally, the first inequality in~\eqref{eq:sub-multiplicativity} is the operator-norm inequality $\|\tilde{A}[f]\|_\spn\leq\|\tilde{A}\|\,\|[f]\|_\spn$. For the second, $AB$ also satisfies (ii) and induces the map $\tilde{A}\tilde{B}$, and operator norms are submultiplicative.
\end{proof}

\begin{lemma}[Span seminorm of a row-stochastic matrix]
    \label{lem:row-stochastic-span}
    Let $P\in\bR^{q\times p}$ be row-stochastic. Then its operator span seminorm equals its Dobrushin coefficient, $\|P\|_{\spn,\op}=\alpha(P)$, and consequently $\|Pf\|_\spn\leq\alpha(P)\|f\|_\spn$ for every $f\in\bR^p$.
\end{lemma}

\begin{proof}
    Since $P$ is row-stochastic, $P\ones=\ones$, so $P$ satisfies condition~(ii) of \Cref{lem:spop}: $\|P\|_{\spn,\op}$ is finite, and the last claim follows from~\eqref{eq:sub-multiplicativity} once the equality is proved.
    For $f\in\bR^p$ we have $\|Pf\|_\spn=\max_{i,j}\left\{(Pf)_i-(Pf)_j\right\}=\max_{i,j}\sum_k\left(P_{ik}-P_{jk}\right)f_k$.

    Fix $i,j$, let $c_k:=P_{ik}-P_{jk}$, and let $I^+$ (resp.\ $I^-$) be the set of indices $k$ with $c_k\geq0$ (resp.\ $c_k<0$). Since $P$ is row-stochastic, $\sum_kc_k=0$, hence $C:=\sum_{k\in I^+}c_k=-\sum_{k\in I^-}c_k=\frac{1}{2}\sum_k|c_k|$, and with $\overline{f}:=\max_kf_k$ and $\underline{f}:=\min_kf_k$, $\sum_kc_kf_k=\sum_{k\in I^+}c_kf_k+\sum_{k\in I^-}c_kf_k\ \leq\ C\overline{f}-C\underline{f}=C\|f\|_\spn$.
    Maximizing over $(i,j)$ and using \Cref{def:tv_norm,def:dobrushin} gives $\|Pf\|_\spn\leq\frac12\max_{i,j}\sum_k|P_{ik}-P_{jk}|\,\|f\|_\spn=\alpha(P)\|f\|_\spn$, i.e.\ $\|P\|_{\spn,\op}\leq\alpha(P)$.

    If $\alpha(P)=0$ there is nothing to prove, so assume $\alpha(P)>0$ and let $(i^*,j^*)$ attain the maximum defining $\alpha(P)$. The coefficients $c_k:=P_{i^*k}-P_{j^*k}$ are not all zero and sum to zero, so they take both signs, and the vector $f^*$ defined by $f^*_k:=1$ if $c_k\geq0$ and $f^*_k:=0$ otherwise is non-constant, with $\|f^*\|_\spn=1$. Then $\|Pf^*\|_\spn\ \geq\ (Pf^*)_{i^*}-(Pf^*)_{j^*}=\sum_{k\in I^+}c_k=\frac{1}{2}\sum_k|P_{i^*k}-P_{j^*k}|=\alpha(P)$, hence $\|P\|_{\spn,\op}\geq\alpha(P)$.
\end{proof}

\subsection{A mean value inequality for the span seminorm}
\label{app:mean-value}

Under \Cref{assumption:convex-analytic-regularity}, the regularized value function $\Omega^*$ is differentiable with a Lipschitz gradient (\Cref{lem:omega_star_props}.(i)), and the regularized Bellman operator $\mT_\Omega$ is continuously differentiable with Jacobian $K_{\pi_{\Omega,Q}}$ (\Cref{cor:T_omega_props}). This Jacobian is row-stochastic, hence has constant row sums, which allows a mean value argument in span seminorm.

\Cref{prop:mvi-span} is the classical mean value inequality for curves in a normed space \citep[\S8.5]{dieudonne1960foundations}, applied in the quotient $\bR^q/\bR\ones$, on which $\|\cdot\|_\spn$ is a norm; condition (H2) is what makes the Jacobians act on that quotient, and we include a proof for completeness. For linear Markov operators, the contraction rate in this seminorm, known as Hopf's oscillation seminorm, is the Dobrushin coefficient (\Cref{lem:row-stochastic-span}); see \citet{gaubert2015dobrushin} for Markov operators on cones and \citet{depasquale2024dual} for the corresponding semicontraction theory.

Throughout this section, for $Q,Q'\in\bR^p$ we write $Q(t):=Q'+t(Q-Q')$, $t\in[0,1]$, for the segment $[Q',Q]$ from $Q'$ to $Q$, as in \Cref{def:mixing_objects}.

\begin{proposition}[Mean value inequality in span seminorm]
    \label{prop:mvi-span}
    Let $U\subseteq\bR^p$ be open, let $\Phi:U\to\bR^q$, and let $Q,Q'\in U$ be such that $[Q',Q]\subseteq U$. Assume:
    \begin{enumerate}
        \item[(H1)] $\Phi$ is differentiable at every point of $[Q',Q]$;
        \item[(H2)] for every $t\in[0,1]$, the Jacobian $D\Phi(Q(t))$ has constant row sums, i.e.\ $D\Phi(Q(t))\ones\in\bR\ones$.
    \end{enumerate}
    Then
    \begin{equation}
        \label{eq:mvi-span}
        \left\|\Phi(Q)-\Phi(Q')\right\|_\spn\ \leq\ \sup_{t\in(0,1)}\left\|D\Phi(Q(t))\right\|_{\spn,\op}\,\|Q-Q'\|_\spn,
    \end{equation}
    and in particular with the supremum taken over the closed segment.
\end{proposition}

\begin{proof}
    Let $\pi:\bR^q\to E_q:=\bR^q/\bR\ones$ be the quotient map, $E_q$ being normed by $\|\cdot\|_\spn$ as in the proof of \Cref{lem:spop}, and let $h:=Q-Q'$.
    By (H1) the curve $g(t):=\pi(\Phi(Q(t)))$ is differentiable on $[0,1]$ with $g'(t)=\pi(D\Phi(Q(t))h)$, and by (H2) and~\eqref{eq:sub-multiplicativity}, $\|g'(t)\|_\spn=\|D\Phi(Q(t))h\|_\spn\leq\|D\Phi(Q(t))\|_{\spn,\op}\|h\|_\spn$.
    The mean value inequality for differentiable curves in a normed space \citep[\S8.5]{dieudonne1960foundations} gives $\|g(1)-g(0)\|_\spn\leq\sup_{t\in(0,1)}\|g'(t)\|_\spn$, which is~\eqref{eq:mvi-span} since $\|g(1)-g(0)\|_\spn=\|\Phi(Q)-\Phi(Q')\|_\spn$.
\end{proof}

\begin{corollary}[$\mathcal{C}^1$ case: an exact averaged kernel]
    \label{cor:avg}
    If moreover $\Phi$ is $\mathcal{C}^1$ on a neighborhood of $[Q',Q]$, then
    \begin{equation*}
        \Phi(Q)-\Phi(Q')=\bar{A}\,(Q-Q'),\qquad\bar{A}:=\int_0^1D\Phi(Q(t))\,\mathrm{d}t.
    \end{equation*}
    $\bar{A}$ is the average of the Jacobians of $\Phi$ along the segment; it has constant row sums and $\|\bar{A}\|_{\spn,\op}\leq\sup_{t\in(0,1)}\|D\Phi(Q(t))\|_{\spn,\op}$.
\end{corollary}

\begin{proof}
    The identity is the fundamental theorem of calculus applied coordinatewise to $t\mapsto\Phi(Q(t))$, whose derivative $D\Phi(Q(t))(Q-Q')$ is continuous in $t$.
    Integrating (H2) gives $\bar{A}\ones=\int_0^1D\Phi(Q(t))\ones\,\mathrm{d}t\in\bR\ones$.
    Finally, for every $f\in\bR^p$, the triangle inequality gives $\|\bar{A}f\|_\spn\leq\int_0^1\|D\Phi(Q(t))f\|_\spn\,\mathrm{d}t\leq\sup_{t\in(0,1)}\|D\Phi(Q(t))\|_{\spn,\op}\,\|f\|_\spn$.
\end{proof}

\begin{remark}
    For $\Phi=\mT_\Omega$, which is continuously differentiable with Jacobian $K_{\pi_{\Omega,Q}}$ at $Q$ (\Cref{cor:T_omega_props}), \Cref{cor:avg} applies on every segment. The state--action kernel is linear in the policy, so $\bar{A}=\int_0^1K_{\pi_{\Omega,Q(t)}}\,\mathrm{d}t$ is the state--action kernel of the average-path policy $\bar{\pi}_{\Omega,[Q,Q']}$, and the identity is \Cref{lem:mixing-properties-of-regularization}.(i).
\end{remark}

\section{Proofs for \pdfref{Section}{sec:background}}
\label{app:proofs_regularized}

\begin{lemma}[Properties of $\Omega^*$]
    \label{lem:omega_star_props}
    For each $s\in\mS$, let $\Omega_s:\Delta(\mA)\to\bR\cup\{+\infty\}$ be proper, closed and strongly convex, and let $\sigma>0$ be a common modulus of strong convexity of the $\Omega_s$ with respect to $\|\cdot\|_1$, which exists by equivalence of norms on $\bR^\mA$ and finiteness of $\mS$; let $\tau\geq0$ and $Q,Q'\in\bR^{\mS\times\mA}$.
    \begin{enumerate}[label=(\roman*)]
        \item \emph{Greedy policy}: if $\tau>0$, $\Omega^*(\cdot)(s)$ is differentiable, $\nabla\Omega^*(Q)(s)=\pi_{\Omega,Q}(\cdot|s)$ is the unique maximizer in its definition, and $\|\pi_{\Omega,Q}(\cdot|s)-\pi_{\Omega,Q'}(\cdot|s)\|_1\leq\frac{1}{2\tau\sigma}\|Q_s-Q'_s\|_{\mA,\spn}$.
        \item \emph{Boundedness}: if $L_{\Omega_s}\leq\Omega_s\leq U_{\Omega_s}$ on $\dom(\Omega_s)$, then $\Omega^*(Q)(s)\leq\max_aQ_s(a)-\tau L_{\Omega_s}$; if moreover $\delta_a\in\dom(\Omega_s)$ for every $a\in\mA$, then $\Omega^*(Q)(s)\geq\max_aQ_s(a)-\tau U_{\Omega_s}$.
        \item \emph{Translation equivariance}: $\Omega^*(Q+c\ones)=\Omega^*(Q)+c\ones$ for every $c\in\bR$.
        \item \emph{Monotonicity}: if $Q_s\leq Q'_s$, then $\Omega^*(Q)(s)\leq\Omega^*(Q')(s)$.
        \item \emph{Non-expansiveness in $\|\cdot\|_\infty$}: $\|\Omega^*(Q)-\Omega^*(Q')\|_{\mS,\infty}\leq\|Q-Q'\|_{\mS\times\mA,\infty}$.
        \item \emph{Non-expansiveness in $\|\cdot\|_\spn$}: $\|\Omega^*(Q)-\Omega^*(Q')\|_{\mS,\spn}\leq\|Q-Q'\|_{\spn}$.
        \item \emph{Sub-affinity in $\|\cdot\|_\spn$}: $\|\Omega^*(Q)\|_{\mS,\spn}\leq\|Q\|_{\spn}+\|\Omega^*(0)\|_{\mS,\spn}$.
    \end{enumerate}
\end{lemma}

\begin{proof}[Proof of \Cref{lem:omega_star_props}]
    \emph{(i)}. $\Omega^*(\cdot)(s)$ is the convex conjugate of the closed, $\tau\sigma$-strongly convex function $\tau\Omega_s$, so it is differentiable, its gradient is the unique maximizer, and this gradient is $\frac{1}{\tau\sigma}$-Lipschitz from $\|\cdot\|_\infty$ to $\|\cdot\|_1$ \citep{kakade2012regularization}; see \citet[\S E.4]{hiriart2001fundamentals} for the Euclidean case and \citet[Prop.~1]{geist2019theory} for $\Omega_s$ finite on $\Delta(\mA)$. Adding a constant to $Q'_s$ shifts the objective by that constant and leaves the maximizer unchanged, so the bound holds with $\|Q_s-Q'_s-c\ones\|_\infty$ for every $c\in\bR$, and $\min_{c\in\bR}\|Q_s-Q'_s-c\ones\|_\infty=\frac12\|Q_s-Q'_s\|_{\mA,\spn}$.

    \emph{(ii)}. This refines \citet[Prop.~1]{geist2019theory}, where $\Omega_s$ is finite on the whole simplex. For every $\pi_s\in\dom(\Omega_s)$, $\langle\pi_s,Q_s\rangle\leq\max_aQ_s(a)$ and $\tau\Omega_s(\pi_s)\geq\tau L_{\Omega_s}$; taking the supremum gives the upper bound. For the lower bound, let $a^\star\in\argmax_aQ_s(a)$: since $\delta_{a^\star}\in\dom(\Omega_s)$, $\Omega^*(Q)(s)\geq\max_aQ_s(a)-\tau\Omega_s(\delta_{a^\star})\geq\max_aQ_s(a)-\tau U_{\Omega_s}$.

    \emph{(iii)--(iv)}. For every $\pi_s\in\Delta(\mA)$, $\langle\pi_s,Q_s+c\ones\rangle=\langle\pi_s,Q_s\rangle+c$, and $\langle\pi_s,Q_s\rangle\leq\langle\pi_s,Q'_s\rangle$ when $Q_s\leq Q'_s$ since $\pi_s\geq0$; take the supremum over $\pi_s$.

    \emph{(v)}. This is the argument of \citet{crandall1980some}. With $C:=\|Q-Q'\|_{\mS\times\mA,\infty}$, we have $Q'-C\ones\leq Q\leq Q'+C\ones$, and (iii)--(iv) give $\Omega^*(Q')-C\ones\leq\Omega^*(Q)\leq\Omega^*(Q')+C\ones$.

    \emph{(vi)}. The same argument with $c:=\min_{s,a}(Q-Q')(s,a)$ and $C:=\max_{s,a}(Q-Q')(s,a)$ gives $\Omega^*(Q')+c\ones\leq\Omega^*(Q)\leq\Omega^*(Q')+C\ones$, whence $\|\Omega^*(Q)-\Omega^*(Q')\|_{\mS,\spn}\leq C-c=\|Q-Q'\|_{\spn}$. Maps that are monotone and translation equivariant are called topical, and are non-expansive for both norms \citep{lemmens2012nonlinear}.

    \emph{(vii)}. By the triangle inequality and (vi) with $Q'=0$, $\|\Omega^*(Q)\|_{\mS,\spn}\leq\|\Omega^*(Q)-\Omega^*(0)\|_{\mS,\spn}+\|\Omega^*(0)\|_{\mS,\spn}\leq\|Q\|_{\spn}+\|\Omega^*(0)\|_{\mS,\spn}$.
\end{proof}

\begin{remark}[Strict convexity suffices]
    \label{rem:strict_convexity}
    We keep strong convexity in \Cref{assumption:convex-analytic-regularity}, as is standard since \citet{geist2019theory}, but our results only use part of it.
    Uniqueness and continuity of the greedy map, and continuous differentiability of $\Omega^*$, i.e.\ \Cref{lem:omega_star_props}.(i) without its Lipschitz constant, hold under strict convexity alone: $\tau\Omega_s$ is closed with bounded domain, so $\Omega^*(\cdot)(s)$ is finite and convex, its subdifferential, the set of maximizers, is a singleton \citep[\S E.4]{hiriart2001fundamentals}, and a differentiable convex function is continuously differentiable \citep[Cor.~25.5.1]{rockafellar1970convex}.
    This is all that \Cref{lem:mixing-properties-of-regularization}, the mean value inequality (\Cref{prop:mvi-span,cor:avg}), the existence of the fixed point (\Cref{thm:regularized_fixed_point}) and the finite-time bounds of \Cref{sec:results} require, so they hold verbatim for strictly convex regularizers; the modulus only enters the Lipschitz constant of \Cref{lem:omega_star_props}.(i).
    For instance, when $|\mA|\geq3$, the $\alpha$-divergences with $\alpha>2$ are strictly but not strongly convex (\Cref{thm:alpha}.(ii)); they fail to meet \Cref{assumption:positive-domination}, but their $\epsilon$-floored versions satisfy it (\Cref{lem:floor_wrapped_props}.(iii) only needs a strictly convex base), and our results apply to them.
\end{remark}

\begin{remark}[Span convergence and centering]
    \label{rem:span_to_sup}
    Adding a constant to $Q_s$ shifts the objective defining $\Omega^*(Q)(s)$ by that constant, so $\pi_{\Omega,Q+c\ones}=\pi_{\Omega,Q}$: as for the unregularized relative $Q$-function \citep{puterman1994markov}, the greedy policy depends on $Q$ only modulo constants, which is why the analysis is carried out in $\|\cdot\|_\spn$. The centering projection $\Pi$ turns span convergence into convergence in $\|\cdot\|_\infty$ to the representative $\Pi[Q^*_\Omega]$: since $\Pi[Q]-\Pi[Q']$ has zero mean, $\|\Pi[Q]-\Pi[Q']\|_\infty\leq\|\Pi[Q]-\Pi[Q']\|_\spn=\|Q-Q'\|_\spn$. A centering step plays the same role for asynchronous average-reward $Q$-learning in \citet{chen2026pointwise}, where it upgrades span-seminorm convergence to pointwise convergence at the same rate.
\end{remark}

\begin{corollary}[Properties of $\mT_\Omega$]
    \label{cor:T_omega_props}
    Under the assumptions of \Cref{lem:omega_star_props}, $\mT_\Omega$ is monotone, translation equivariant, and non-expansive in both $\|\cdot\|_\infty$ and $\|\cdot\|_\spn$, and $\|\mT_\Omega(Q)\|_\spn\leq\|Q\|_\spn+\|\Omega^*(0)\|_{\mS,\spn}+\|r\|_\spn$ for every $Q$, where $\Omega^*(0)(s)=-\tau\min\Omega_s$; in particular $\|\Omega^*(0)\|_{\mS,\spn}=0$ when $\min\Omega_s$ does not depend on $s$, as for divergences to $\mu$. If $\tau>0$, $\mT_\Omega$ is continuously differentiable, with Jacobian $D\mT_\Omega(Q)=K_{\pi_{\Omega,Q}}$.
\end{corollary}

\begin{proof}
    $\mT_\Omega(Q)=r+P\,\Omega^*(Q)$, with $P$ the transition kernel, row-stochastic from $\mS\times\mA$ to $\mS$. Since $P\geq0$ and $P\ones=\ones$, monotonicity and translation equivariance follow from \Cref{lem:omega_star_props}.(iii)--(iv), and non-expansiveness from (v)--(vi), $P$ being non-expansive in $\|\cdot\|_\infty$ and, as $\alpha(P)\leq 1$ (\Cref{lem:row-stochastic-span}), in $\|\cdot\|_\spn$. The bound follows from (vii) and the triangle inequality, and $\Omega^*(0)(s)=\sup_{\pi_s}\{-\tau\Omega_s(\pi_s)\}=-\tau\min\Omega_s$, the minimum being attained since $\Omega_s$ is closed and $\Delta(\mA)$ is compact. For a divergence to $\mu$, $\min\Omega_s=\Omega_s(\mu_s)=0$ for every $s$. Finally, by \Cref{lem:omega_star_props}.(i), $\partial \mT_\Omega(Q)(s,a)/\partial Q(s',a')=p(s'|s,a)\,\pi_{\Omega,Q}(a'|s')=K_{\pi_{\Omega,Q}}(s',a'|s,a)$, which is continuous in $Q$.
\end{proof}

\section{Proofs for \pdfref{Section}{sec:mixing_regularized}}

\subsection{Proofs for \pdfref{Subsection}{subsec:mixing_properties}}
\label{app:proofs_mixing}

We recall \Cref{def:domination_profile,def:regularization_modulus} together with the average-path policy and radius used in the proofs.

\begin{definition}[Mixing-regularization objects]
    \label{def:mixing_objects}
    Let $\Omega$ satisfy \Cref{assumption:convex-analytic-regularity}, let $\mu$ be the mixing policy of \Cref{assumption:m-step-mixing}, let $Q,Q'\in\bR^{\mS\times\mA}$, $Q(t):=tQ+(1-t)Q'$ for $t\in[0,1]$, and $R\geq0$.
    \begin{enumerate}[label=(\roman*)]
        \item \emph{Average-path policy}: $\bar{\pi}_{\Omega,[Q,Q']}:=\int_0^1\pi_{\Omega,Q(t)}\,\mathrm{d}t$ is the mean greedy policy along the segment $[Q',Q]$.
        \item \emph{Average-path radius}: $R[Q,Q']:=\max_{t\in[0,1]}\|Q(t)\|_\asp$ is the largest action span along the segment; by convexity of $\|\cdot\|_\asp$, it equals $\max\{\|Q\|_\asp,\|Q'\|_\asp\}$.
        \item \emph{Domination profile}: $\gamma_\Omega(R):=\inf\big\{\pi_{\Omega,Q}(a|s)/\mu(a|s) \mid (s,a)\in\mS\times\mA,\ \|Q\|_\asp\leq R\big\}$ is the largest factor such that $\pi_{\Omega,Q}\geq\gamma_\Omega(R)\,\mu$ for every $Q$ of action span at most $R$. It lies in $[0,1]$, is non-increasing in $R$, and governs contraction and exploration (\Cref{lem:mixing-properties-of-regularization}).
        \item \emph{Regularization modulus}: $\Lambda_\Omega(R):=\sup\big\{\|g_Q\|_{\mS,\spn}:\ \|Q\|_\asp\leq R\big\}$ is the largest variation across states of the \emph{regularization gap} $g_Q(s):=\Omega^*(Q)(s)-\max_aQ(s,a)$ between the regularized and the hard maximum. It is finite, non-decreasing in $R$ (\Cref{lem:modulus_properties}), and governs the bias (\Cref{prop:gap_optimal_q}).
    \end{enumerate}
\end{definition}

The following lemma gives some basic properties of the regularization modulus defined in~\Cref{def:mixing_objects}.(iv).

\begin{lemma}[Properties of $\Lambda_\Omega$]
    \label{lem:modulus_properties}
    Under \Cref{assumption:convex-analytic-regularity}, the supremum defining $\Lambda_\Omega(R)$ is attained and finite for every $R\geq0$, $\Lambda_\Omega$ is non-decreasing, and $\Lambda_\Omega(0)=\tau\big(\max_s\underline{\Omega}_s-\underline{\Omega}\big)$, where $\underline{\Omega}_s:=\min_{\pi_s\in\Delta(\mA)}\Omega_s(\pi_s)$ and $\underline{\Omega}:=\min_s\underline{\Omega}_s$.
\end{lemma}

\begin{proof}
    Since $\Omega^*(Q)(s)$ depends on $Q$ only through $Q_s$, \Cref{lem:omega_star_props}.(iii) applied state by state shows that $\Omega^*(Q)(s)$ and $\max_aQ_s(a)$ both increase by $c$ when $Q_s$ is replaced by $Q_s+c\ones$: $g_Q$ is invariant under the addition of a constant to each $Q_s$ separately. We may therefore restrict the supremum to the $Q$ with $\max_aQ_s(a)=0$ for every $s$, for which $\|Q\|_\asp\leq R$ reads $Q_s\in[-R,0]^\mA$: a compact set. On it, $Q\mapsto\|g_Q\|_{\mS,\spn}$ is continuous, since $\Omega^*$ is non-expansive in $\|\cdot\|_\infty$ (\Cref{lem:omega_star_props}.(v)) and $Q\mapsto\max_aQ(\cdot,a)$ is $1$-Lipschitz; hence the supremum is attained and finite. It is non-decreasing in $R$ because the feasible set grows with $R$. Finally, for $R=0$ each $Q_s$ is a constant $c_s\ones$, so $\Omega^*(Q)(s)=c_s+\Omega^*(0)(s)=c_s-\tau\underline{\Omega}_s$, the minimum defining $\underline{\Omega}_s$ being attained since $\Omega_s$ is closed on the compact simplex. Thus $g_Q=-\tau(\underline{\Omega}_s)_{s\in\mS}$, whose span is $\tau(\max_s\underline{\Omega}_s-\underline{\Omega})$.
\end{proof}

\begin{remark}[Boundedness of the regularization modulus]
    \label{rem:bounded_modulus}
    Let $\overline{\Omega}:=\max_{s\in\mS}\max_{a\in\mA}\Omega_s(\delta_a)$, $\underline{\Omega}=\min_s\underline{\Omega}_s$ and $\Delta_\Omega:=\overline{\Omega}-\underline{\Omega}$. If $\overline{\Omega}<+\infty$, i.e.\ if every Dirac lies in the effective domain, then \Cref{lem:omega_star_props}.(ii) applied state-wise with $L_{\Omega_s}=\underline{\Omega}$ and $U_{\Omega_s}=\max_a\Omega_s(\delta_a)$ gives $\Lambda_\Omega(R)\leq\tau\Delta_\Omega$ for every $R$: the modulus is bounded by a radius-independent constant and all the bounds below reduce to their constant counterparts.

    The $\alpha$-divergences with $\alpha\leq0$, including the forward KL, and the floor-wrapped regularizers of \Cref{def:floor_wrapped} exclude every Dirac from their effective domain, so $\overline{\Omega}=+\infty$, while $\Lambda_\Omega(R)$ stays finite (\Cref{thm:alpha}.(v) and \Cref{lem:floor_wrapped_props}.(v)); this is what makes the radius-dependent form necessary.
\end{remark}

\begin{lemma}[Mixing properties of $\Omega$-regularization]
    \label{lem:mixing-properties-of-regularization}
    Let \Cref{assumption:m-step-mixing,assumption:convex-analytic-regularity} hold and let $Q,Q'\in\bR^{\mS\times\mA}$. Then:
    \begin{enumerate}[label=(\roman*)]
        \item \emph{Average-path factorization of the conjugate and of the operator}: for every $s\in\mS$, $\Omega^*(Q)(s)-\Omega^*(Q')(s)=\langle\bar{\pi}_{\Omega,[Q,Q']}(\cdot|s),Q_s-Q'_s\rangle$ and $\mT_\Omega(Q)-\mT_\Omega(Q')=K_{\bar{\pi}_{\Omega,[Q,Q']}}(Q-Q')$. 
        \item \emph{Domination of the average-path policy}: $\bar{\pi}_{\Omega,[Q,Q']}(a|s)\geq\gamma_\Omega(R[Q,Q'])\,\mu(a|s)$ for every $(s,a)\in\mS\times\mA$.
        \item \emph{$m$-step minorization}: if $\pi_1,\dots,\pi_m$ are policies with $\pi_k\geq\gamma_k\mu$ with $\gamma_k\geq0$ for every $k\in\llbracket1,m\rrbracket$, the $m$-step kernel $K_{\pi_1:\pi_m}:=K_{\pi_1}\cdots K_{\pi_m}$ satisfies
        \begin{equation*}
            K_{\pi_1:\pi_m}(s',a'|s,a)\geq\beta_m\left(\prod_{k=1}^{m}\gamma_k\right)\eta(s',a'),\qquad\forall(s,a),(s',a')\in\mS\times\mA.
        \end{equation*}
        \item \emph{Radius-dependent $m$-step contraction}: 
        \begin{equation*}
            \left\|\mT^m_\Omega(Q)-\mT^m_\Omega(Q')\right\|_\spn\leq\left(1-\beta_m\prod_{k=1}^{m-1}\gamma_\Omega\left(R[\mT^{k}_{\Omega}(Q), \mT^{k}_{\Omega}(Q')]\right)\right)\|Q-Q'\|_\spn.
        \end{equation*}
        \item \emph{State--action coverage}: let $(Q_n)_{n\geq0}$ be $Q$-functions and $(s_n,a_n)_{n\geq0}$ a trajectory with $a_n\sim\pi_{\Omega,Q_n}(\cdot|s_n)$ and $s_{n+1}\sim p(\cdot|s_n,a_n)$; let $\mF_n:=\sigma(Q_0,(s_0,a_0),\dots,(s_n,a_n))$ and assume that $Q_n$ is $\mF_{n-1}$-measurable for $n\geq1$. If $t\geq0$ and $\gamma_1,\dots,\gamma_m$ are non-negative $\mF_t$-measurable random variables with $\pi_{\Omega,Q_{t+k}}\geq\gamma_k\mu$ almost surely for $k\in\llbracket1,m\rrbracket$, then almost surely
        \begin{equation*}
            \bP\big((s_{t+m},a_{t+m})=(s,a)\mid\mF_t\big)\geq\beta_m\left(\prod_{k=1}^{m}\gamma_k\right)\eta(s,a),\qquad\forall(s,a)\in\mS\times\mA.
        \end{equation*}
        \item \emph{One-step operator gap}: for every $Q\in\bR^{\mS\times\mA}$,
        \begin{equation*}
            \|\mT_\Omega(Q)-\mT(Q)\|_\spn\leq\alpha(P)\,\Lambda_\Omega\big(\|Q\|_\asp\big)\leq\Lambda_\Omega\big(\|Q\|_\asp\big),
        \end{equation*}
        where $\alpha(P)$ is the Dobrushin coefficient (\Cref{def:dobrushin}) of the rectangular matrix $(p(s'|s,a))_{(s,a),s'}$. In particular, under the assumptions of \Cref{rem:bounded_modulus}, $\|\mT_\Omega(Q)-\mT(Q)\|_\spn\leq\alpha(P)\tau\Delta_\Omega\leq\tau\Delta_\Omega$ uniformly in $Q$.
    \end{enumerate}
\end{lemma}

\begin{proof}[Proof of \Cref{lem:mixing-properties-of-regularization}]
    Throughout the proof we write $\bar{\pi}=\bar{\pi}_{\Omega,[Q,Q']}$, $\gamma_k$ for the domination coefficients of each item and $\eta(s',a')=\nu(s')\mu(a'|s')$ for the minorizing measure of \Cref{assumption:m-step-mixing}. Points (i), (ii), (iv) and (v) use $\tau>0$ through \Cref{lem:omega_star_props}.(i), whereas (iii) requires no regularity at all and (vi) only uses \Cref{assumption:convex-analytic-regularity} through the finiteness of $\Lambda_\Omega$ (\Cref{lem:modulus_properties}).

    \emph{(i)}. Let $g(t):=\Omega^*(Q(t))(s)$ for $t\in[0,1]$. Since $\Omega^*(Q)(s)$ depends on $Q$ only through $Q_s$, \Cref{lem:omega_star_props}.(i) and the chain rule give that $g$ is continuously differentiable, with $g'(t)=\langle\pi_{\Omega,Q(t)}(\cdot|s),Q_s-Q'_s\rangle$. The fundamental theorem of calculus and the linearity of the inner product then give
    \begin{equation}
        \label{eq:value-factorization}
        \Omega^*(Q)(s)-\Omega^*(Q')(s)=g(1)-g(0)=\int_0^1\langle\pi_{\Omega,Q(t)}(\cdot|s),Q_s-Q'_s\rangle\,\mathrm{d}t=\langle\bar{\pi}(\cdot|s),Q_s-Q'_s\rangle,
    \end{equation}
    which is the first equality. Since $\mT_\Omega=r+P\,\Omega^*$, the reward terms cancel and applying~\eqref{eq:value-factorization} at each next state $s'\in\mS$ gives, for every $(s,a)\in\mS\times\mA$,
    \begin{equation*}
    \begin{aligned}
        \mT_\Omega(Q)(s,a)-\mT_\Omega(Q')(s,a)
        &=\sum_{s'}p(s'|s,a)\big(\Omega^*(Q)(s')-\Omega^*(Q')(s')\big)\\
        &=\sum_{s',a'}p(s'|s,a)\,\bar{\pi}(a'|s')\,(Q-Q')(s',a')
        =\big(K_{\bar{\pi}}(Q-Q')\big)(s,a),
    \end{aligned}
    \end{equation*}
    which is the second equality.

    \emph{(ii)}. By \Cref{def:mixing_objects}.(ii), $\|Q(t)\|_\asp\leq R[Q,Q']$ for all $t\in[0,1]$, and \Cref{def:mixing_objects}.(iii) applied with $R=R[Q,Q']$ then gives $\pi_{\Omega,Q(t)}(a|s) \geq \gamma_\Omega(R[Q,Q'])\mu(a|s)$ for all $(s,a)$ and all $t\in[0,1]$, uniformly in $t$. Integrating this inequality against the uniform measure on $[0,1]$ and using \Cref{def:mixing_objects}.(i) yields $\bar{\pi}(a|s) = \int_0^1 \pi_{\Omega,Q(t)}(a|s)\,\mathrm{d}t \geq \gamma_\Omega(R[Q,Q'])\mu(a|s)$, which is the claim.

    \emph{(iii)}. Set $\Gamma_{1:m}=\prod_{k=1}^m\gamma_k$. First, for any policy $\pi$ with $\pi(a|s)\geq\gamma\mu(a|s)$ we have the transport inequality $P_\pi(s'|s) = \sum_a \pi(a|s)p(s'|s,a) \geq \gamma\sum_a\mu(a|s)p(s'|s,a) = \gamma P_\mu(s'|s)$. 
    By definition of $K_{\pi_1:\pi_m}=K_{\pi_1}\cdots K_{\pi_m}$ and \Cref{convention:product_order}, the first action is the one supplied by the conditioning, the intermediate actions are drawn from $\pi_1,\dots,\pi_{m-1}$ and the terminal action from $\pi_m$:
    \begin{equation}
        \label{eq:mstep_sa_kernel}
        K_{\pi_1:\pi_m}(s',a'|s,a) = \pi_m(a'|s') \sum_{s_1,\dots,s_{m-1}} p(s_1|s,a)\prod_{k=1}^{m-1} P_{\pi_k}(s_{k+1}|s_k),
    \end{equation}
    with $s_m=s'$, and since all factors are non-negative, the bounds multiply: applying the transport inequality to the $m-1$ intermediate factors and the domination hypothesis to the terminal factor $\pi_m(a'|s')\geq\gamma_m\mu(a'|s')$ gives
    \begin{align*}
        K_{\pi_1:\pi_m}(s',a'|s,a) &\geq \Gamma_{1:m}\,\mu(a'|s') \sum_{s_1,\dots,s_{m-1}} p(s_1|s,a)\prod_{k=1}^{m-1} P_\mu(s_{k+1}|s_k)\\
        &= \Gamma_{1:m}\,\mu(a'|s') \sum_{s_1} p(s_1|s,a)\, P_\mu^{m-1}(s'|s_1)\\
        &= \Gamma_{1:m}\,\mu(a'|s')\, K_\mu^m(s'|s,a)\\
        &\geq \Gamma_{1:m}\,\beta_m\,\nu(s')\mu(a'|s') = \Gamma_{1:m}\,\beta_m\,\eta(s',a'),
    \end{align*}
    the second-to-last line by the definition $K^m_\mu(s'|s,a)=\sum_{s_1}p(s_1|s,a)P_\mu^{m-1}(s'|s_1)$ and the last by \Cref{assumption:m-step-mixing}.

    \emph{(iv)}. For $k\in\llbracket 0,m-1\rrbracket$ write $Q_k:=\mT_\Omega^k(Q)$, $Q_k':=\mT_\Omega^k(Q')$, $R_k:=R[Q_k,Q_k']$, $\bar{\pi}^{(k)} = \bar{\pi}_{\Omega,[Q_k,Q_k']}$ and $\gamma_k = \gamma_\Omega(R_k)$. (i) applied at level $k$ gives the one-step identity $Q_{k+1}-Q_{k+1}' = \mT_\Omega(Q_k)-\mT_\Omega(Q_k') = K_{\bar{\pi}^{(k)}}(Q_k-Q_k')$, where $K_{\bar{\pi}^{(k)}}$ is a row-stochastic state--action kernel. Iterating this identity from $k=m-1$ down to $k=0$, we obtain
    \begin{equation*}
        \mT^m_\Omega(Q)-\mT^m_\Omega(Q') = Q_m - Q_m' = K^{(m)}_{Q,Q'}(Q-Q'), \qquad K^{(m)}_{Q,Q'} := K_{\bar{\pi}^{(m-1)}}K_{\bar{\pi}^{(m-2)}}\cdots K_{\bar{\pi}^{(0)}}.
    \end{equation*}
    Following \Cref{convention:product_order}, $K^{(m)}_{Q,Q'}$ is exactly the $m$-step state--action kernel $K_{\bar{\pi}^{(m-1)}:\bar{\pi}^{(0)}}$ associated with the sequence of policies $\bar{\pi}^{(m-1)},\dots,\bar{\pi}^{(0)}$. By (ii), each of them satisfies $\bar{\pi}^{(k)}\geq\gamma_k\mu$, so (iii) applies and gives the Doeblin condition $K^{(m)}_{Q,Q'}(s',a'|s,a) \geq \beta_m\left(\prod_{k=0}^{m-1}\gamma_k\right)\eta(s',a')$. \Cref{lem:terminal_action}, applied with $\pi_1=\bar\pi^{(m-1)},\dots,\pi_m=\bar\pi^{(0)}$, gives $\alpha\big(K^{(m)}_{Q,Q'}\big)\leq1-\beta_m\prod_{k=1}^{m-1}\gamma_\Omega(R_k)$, and by \Cref{lem:row-stochastic-span}, $\|K^{(m)}_{Q,Q'}\|_{\spn,\op} = \alpha\left(K^{(m)}_{Q,Q'}\right)$. Therefore
    \begin{equation*}
        \left\|\mT^m_\Omega(Q)-\mT^m_\Omega(Q')\right\|_\spn = \left\|K^{(m)}_{Q,Q'}(Q-Q')\right\|_\spn \leq \left(1-\beta_m\prod_{k=1}^{m-1}\gamma_\Omega(R_k)\right)\|Q-Q'\|_\spn.
    \end{equation*}

    \emph{(v)}. Fix $t\geq 0$ and $(s,a)\in\mS\times\mA$. For $j\in\llbracket0,m\rrbracket$, let $\Gamma_j:=\prod_{k=j+1}^{m}\gamma_k$, which is non-negative and $\mF_t$-measurable, hence $\mF_{t+j}$-measurable. We prove by backward induction on $j$ that, a.s.,
    \begin{equation*}
        \bE\left[\indicator_{\{(s_{t+m},a_{t+m})=(s,a)\}}\ \middle|\ \mF_{t+j}\right]\ \geq\ \Gamma_j\,K^{m-j}_\mu\big(s,a\mid s_{t+j},a_{t+j}\big).
    \end{equation*}
    The case $j=m$ is an equality. Assume the claim at $j+1$. Conditionally on $\mF_{t+j}$, $s_{t+j+1}\sim p(\cdot|s_{t+j},a_{t+j})$ and $a_{t+j+1}\sim\pi_{\Omega,Q_{t+j+1}}(\cdot|s_{t+j+1})$, where $Q_{t+j+1}$ is $\mF_{t+j}$-measurable. By the tower property, the induction hypothesis and the $\mF_{t+j}$-measurability of $\Gamma_{j+1}$, the left-hand side at level $j$ is therefore at least
    \begin{align*}
        &\Gamma_{j+1}\sum_{s'}p(s'|s_{t+j},a_{t+j})\sum_{a'}\pi_{\Omega,Q_{t+j+1}}(a'|s')\,K^{m-j-1}_\mu\big(s,a\mid s',a'\big)\\
        &\qquad\geq \Gamma_{j+1}\gamma_{j+1}\sum_{s',a'}K_\mu(s',a'|s_{t+j},a_{t+j})\,K^{m-j-1}_\mu\big(s,a\mid s',a'\big)
        = \Gamma_j\,K^{m-j}_\mu\big(s,a\mid s_{t+j},a_{t+j}\big),
    \end{align*}
    the inequality by the domination hypothesis and the non-negativity of $\Gamma_{j+1}$ and of $K^{m-j-1}_\mu$. Taking $j=0$ and using that the $m$-step state--action kernel of $\mu$ satisfies $K^m_\mu(s',a'|s,a)=K^m_\mu(s'|s,a)\mu(a'|s')\geq\beta_m\nu(s')\mu(a'|s')=\beta_m\eta(s',a')$ by \Cref{assumption:m-step-mixing} gives the claim.

    \emph{(vi)}. \Cref{def:mixing_objects}.(iv) applied with $R=\|Q\|_\asp$ immediately gives $\|g_Q\|_{\mS,\spn}\leq\Lambda_\Omega(\|Q\|_\asp)$. The two operators $\mT_\Omega$ and $\mT$ share the reward term and differ only through the state-value function, so for all $(s,a)$,
    \begin{equation*}
        \mT_\Omega(Q)(s,a)-\mT(Q)(s,a) = \sum_{s'}p(s'|s,a)\left(\Omega^*(Q)(s')-\max_{a'}Q(s',a')\right) = (Pg_Q)(s,a).
    \end{equation*}
    \Cref{lem:row-stochastic-span} gives $\|Pg_Q\|_\spn\leq\alpha(P)\|g_Q\|_{\mS,\spn}$, so $\|\mT_\Omega(Q)-\mT(Q)\|_\spn\leq\alpha(P)\Lambda_\Omega(\|Q\|_\asp)$. The second inequality follows from $\alpha(P)\leq1$, and the last claim from \Cref{rem:bounded_modulus}.
\end{proof}

\mixingmain*

\begin{proof}[Proof of \Cref{lem:mixing_main}]
    Items (i) and (ii) are \Cref{lem:mixing-properties-of-regularization}.(iv) and~(v). The first claim of (iii) is \Cref{lem:mixing-properties-of-regularization}.(vi) with $\alpha(P)\leq1$, and the second is \Cref{prop:gap_optimal_q}.(i), whose right-hand side is finite under \Cref{assumption:positive-domination} and the gain bounds are \Cref{prop:gap_optimal_q}.(ii), the last one using $|\mS|\geq2$ and a common minimum.
\end{proof}

\begin{remark}[Only $m$ regularized steps are needed]
    \label{rem:tail_only}
    Let $\pi_1,\dots,\pi_M$ with $M\geq m$ be given, and assume that only the last $m$ of them are known to be regularized, say $\pi_k\geq\gamma_k\mu$ for $k\in\llbracket M-m+1,M\rrbracket$, with $\Gamma_{M-m+1:M}:=\prod_{k=M-m+1}^{M}\gamma_k$. Then \Cref{lem:mixing-properties-of-regularization}.(iii) still holds for the full $M$-step kernel, with the same parameter $\Gamma_{M-m+1:M}\beta_m$ and the same minorizing measure $\eta$.

    Writing $K_{\mathrm{head}}:=K_{\pi_1:\pi_{M-m}}$ for the kernel of the first $M-m$ steps and $K_{\mathrm{tail}}:=K_{\pi_{M-m+1}:\pi_M}$ for that of the last $m$ ones, we have, consistently with \Cref{convention:product_order},
    \begin{equation*}
        K_{\pi_1:\pi_M} = K_{\mathrm{head}}K_{\mathrm{tail}}, \qquad\text{i.e.}\qquad K_{\pi_1:\pi_M}(s',a'|s,a) = \sum_{(\tilde s,\tilde a)} K_{\mathrm{head}}(\tilde s,\tilde a|s,a)K_{\mathrm{tail}}(s',a'|\tilde s,\tilde a).
    \end{equation*}
    \Cref{lem:mixing-properties-of-regularization}.(iii) applies to $K_{\mathrm{tail}}$ and gives $K_{\mathrm{tail}}(s',a'|\tilde s,\tilde a)\geq\Gamma_{M-m+1:M}\beta_m\,\eta(s',a')$ uniformly in the starting pair $(\tilde s,\tilde a)$; the head, which requires no regularization, only contributes a factor $1$ by row-stochasticity:
    \begin{equation*}
        K_{\pi_1:\pi_M}(s',a'|s,a) \geq \Gamma_{M-m+1:M}\beta_m\eta(s',a') \sum_{(\tilde s,\tilde a)} K_{\mathrm{head}}(\tilde s,\tilde a|s,a) = \Gamma_{M-m+1:M}\beta_m\eta(s',a').
    \end{equation*}

    The same holds for the coverage bound of \Cref{lem:mixing-properties-of-regularization}.(v): if its domination hypothesis holds for $k\in\llbracket M-m+1,M\rrbracket$ with deterministic coefficients $\gamma_k\geq0$, then, almost surely,
    \begin{equation*}
        \bP\big((s_{t+M},a_{t+M})=(s,a)\mid\mF_t\big)\geq\Gamma_{M-m+1:M}\beta_m\,\eta(s,a).
    \end{equation*}
    Indeed, \Cref{lem:mixing-properties-of-regularization}.(v) applied at time $t+M-m$ gives this lower bound for the probability conditional on $\mF_{t+M-m}$, and the tower property transfers it to $\mF_t$, the bound being deterministic: the conditional expectation plays the role of the row-stochastic head. Since the proof of (v) only uses the steps after time $t+M-m$, the actions of the head may moreover follow any adapted rule.
\end{remark}

\subsection{Strict domination}
\label{app:pd}

\begin{lemma}[\Cref{assumption:positive-domination} is full support]
    \label{lem:pd_full_support}
    Under \Cref{assumption:convex-analytic-regularity}, $\Omega$ satisfies \Cref{assumption:positive-domination} if and only if $\pi_{\Omega,Q}(a|s)>0$ for every $Q$ and every $(s,a)$.
\end{lemma}

\begin{proof}
    If \Cref{assumption:positive-domination} holds, $\pi_{\Omega,Q}\geq\gamma_\Omega(\|Q\|_\asp)\,\mu>0$. Conversely, $\pi_{\Omega,Q}(\cdot|s)=\nabla\Omega^*_s(Q_s)$ is continuous and invariant under $Q_s\mapsto Q_s+c\ones$, so $\gamma_\Omega(R)$ is the minimum of a positive continuous function over the compact set of $Q_s$ with $\max_aQ_s(a)=0$ and $\min_aQ_s(a)\geq-R$, $s\in\mS$.
\end{proof}

\Cref{assumption:positive-domination} is pointwise in $R$: for the reverse KL, $\gamma_\Omega(R)=\big(\mu_{\min}+(1-\mu_{\min})e^{R/\tau}\big)^{-1}$ vanishes as $R\to\infty$ (\Cref{rem:alpha_exact}). It makes the contraction rate of \Cref{lem:mixing-properties-of-regularization}.(iv) positive on bounded sets, guarantees that epochs terminate (\Cref{prop:ep_wellposed}) and that $q_R>0$, and is the full-support requirement of \Cref{prop:uniqueness_fixed_point}. It holds for the $\alpha$-divergences with $\alpha\leq1$ and fails for $\alpha>1$ (\Cref{thm:alpha}.(iv)); the floor-wrapped regularizers satisfy it with $\gamma_{\Omega_\epsilon}\geq\epsilon$ (\Cref{lem:floor_wrapped_props}.(iii)).

\subsection{Existence and uniqueness of the fixed point of \texorpdfstring{$\mT_\Omega$}{TΩ}}
\label{subsec:existence_fixed_point}

\begin{theorem}[Existence of a fixed point for $\mT_\Omega$]
    \label{thm:regularized_fixed_point}
    Under \Cref{assumption:m-step-mixing,assumption:convex-analytic-regularity,assumption:reference_policy}, there exist $Q^*_\Omega\in\bR^{\mS\times\mA}$ and $\rho_\Omega^*\in\bR$ such that $\mT_\Omega(Q^*_\Omega)=Q^*_\Omega+\rho_\Omega^*\ones$, with
    \begin{equation*}
        \left\|Q^*_\Omega\right\|_\spn\leq\|r\|_\spn+\frac{4mR_{\max}}{\beta_m\nu_{\min}}, \qquad \left\lvert\rho^*_\Omega\right\rvert\leq2R_{\max},
    \end{equation*}
    where $\nu_{\min}=\min_s\nu(s)>0$ and $R_{\max}=\|r\|_\infty+\tau\max_s\max\{\lvert\underline{\Omega}_s\rvert,\lvert\Omega_s(\mu_s)\rvert\}$, both finite under \Cref{assumption:reference_policy}.
    Moreover $\rho^*_\Omega$ is unique and is the optimal regularized gain (\Cref{prop:uniqueness_gain,prop:gain_interpretation}), and under \Cref{assumption:positive-domination} $Q^*_\Omega$ is unique up to an additive constant (\Cref{prop:uniqueness_fixed_point}).
\end{theorem}

This subsection proves \Cref{thm:regularized_fixed_point} as a corollary of a more general statement, \Cref{thm:existence_fixed_point_general}, which rests on a condition that is both necessary and sufficient and does not presuppose \Cref{assumption:m-step-mixing}. The argument is the regularized counterpart of the \emph{vanishing discount} approach. In the unregularized case, the equivalence it establishes is classical \citep{schweitzer1984existence}, as is the bound on the span of the relative discounted values under weak communication, which underlies regret analysis \citep{jaksch2010near,bartlett2009regal} (see also \citealp[\S8.4.4]{puterman1994markov}; \citealp[\S4.1--4.3]{bertsekas2012dynamic}; \citealp[\S5.5]{hernandez1996discrete}; and the survey of \citealp{arapostathis1993survey}). We nevertheless give a self-contained argument, since the classical statements for compact action sets require the reward to be continuous in the action \citep{schweitzer1985undiscounted}, which fails for $\tilde{r}_\pi$ as soon as $\dom(\Omega_s)\subsetneq\Delta(\mA)$, as for the $\alpha$-divergences with $\alpha\leq0$. Entropy-regularized average-reward MDPs can also be studied through convex duality \citep{neu2017unified}. In the linearly solvable KL-control models of \citet{todorov2007linearly}, the penalty acts directly on controlled transition probabilities relative to passive dynamics, and an exponential change of variables turns the average-cost equation into a Perron--Frobenius eigenproblem.

Throughout $\Omega$ satisfies \Cref{assumption:convex-analytic-regularity} and we also assume the following.

\begin{assumption}[Admissible policy]
    \label{assumption:reference_policy_general}
    There exists a full-support policy $\tilde{\pi}$ such that $\tilde{\pi}_s\in\dom(\Omega_s)$ for every $s\in\mS$ and $\max_{s\in\mS}\lvert\Omega_s(\tilde{\pi}_s)\rvert<+\infty$.
\end{assumption}

\paragraph{Notation and constants.}
Since $\Omega_s$ is proper and closed and $\Delta(\mA)$ is compact, $\underline{\Omega}_s:=\min_{\pi_s\in\Delta(\mA)}\Omega_s(\pi_s)$ is finite.
Under \Cref{assumption:reference_policy_general}, the following constants are well-defined and finite:  
\begin{equation}
    \label{eq:existence_constants}
    \delta_\Omega:=\max_{s\in\mS}\max\left\{\lvert\underline{\Omega}_s\rvert,\ \lvert\Omega_s(\tilde{\pi}_s)\rvert\right\}<+\infty,
    \qquad R_{\max}:=\|r\|_\infty+\tau\delta_\Omega .
\end{equation}
For an \emph{admissible} policy $\pi$ (i.e.\ $\pi_s\in\dom(\Omega_s)$ for every $s$) and discount factor $\lambda\in (0,1)$ define the \emph{regularized reward} and the discounted evaluation and optimality operators on $\bR^\mS$,
\begin{equation}
    \label{eq:reg_reward}
    \tilde{r}_\pi(s):=\langle\pi_s,r_s\rangle-\tau\,\Omega_s(\pi_s),
    \qquad
    \mathcal{E}^{\lambda}_{\pi}V:=\tilde{r}_\pi+\lambda P_\pi V,
    \qquad
    \mathcal{T}^\lambda_\Omega V:=\max_{\pi\text{ admissible}}\mathcal{E}^\lambda_\pi V,
\end{equation}
the maximum in $\mathcal{T}^\lambda_\Omega$ being taken state-wise.
By~\eqref{eq:existence_constants}, $\tilde{r}_\pi(s)\leq\|r\|_\infty-\tau\underline{\Omega}_s\leq R_{\max}$ for every admissible $\pi$, and $\lvert\tilde{r}_{\tilde{\pi}}(s)\rvert\leq R_{\max}$.

\begin{remark}[Sufficient conditions for \Cref{assumption:reference_policy_general}]
    \label{rem:bounded_reg_sufficient}
    Under \Cref{assumption:m-step-mixing}, \Cref{assumption:reference_policy} gives \Cref{assumption:reference_policy_general} with $\tilde{\pi}=\mu$. More generally, if every Dirac lies in the domain, as in \Cref{rem:bounded_modulus}, convexity gives $\Omega_s(\pi_s)\leq\max_a\Omega_s(\delta_a)<+\infty$ on $\Delta(\mA)$, so any full-support $\tilde{\pi}$ works; this covers the negative entropy, the Tsallis negentropies, the quadratic regularizer and the $\alpha$-divergences to $\mu$ with $\alpha>0$ (\Cref{thm:alpha}.(i)). For $\alpha\leq0$, including the forward KL, $\Omega_s$ is unbounded on the relative interior of $\Delta(\mA)$, but $\tilde{\pi}=\mu$ still gives $\Omega_s(\mu_s)=0$; for the floor-wrapped regularizers, $\Omega_{\epsilon,s}(\mu_s)=(1-\epsilon)\Omega_{0,s}(\mu_s)$ (\Cref{lem:floor_wrapped_props}). The assumption says nothing about the domination profile: regularizers with sparse greedy policies satisfy it while failing \Cref{assumption:positive-domination} (\Cref{rem:full_support_uniqueness}).
\end{remark}

\subsubsection{The discounted regularized problem}

\begin{lemma}[Discounted fixed point]
    \label{lem:discounted_fixed_point}
    Let $\lambda\in(0,1)$ and $\mT^\lambda_\Omega(Q):=r+\lambda P\,\Omega^*(Q)$.
    Then $\mT^\lambda_\Omega$ is a $\lambda$-contraction for $\|\cdot\|_\infty$ with unique fixed point $Q^\lambda$, and $V^\lambda:=\Omega^*(Q^\lambda)$ satisfies
    \begin{equation}
        \label{eq:discounted_Q}
        Q^\lambda(s,a)=r(s,a)+\lambda\sum_{s'}p(s'|s,a)\,V^\lambda(s'),
        \qquad
        V^\lambda=\mathcal{T}^\lambda_\Omega V^\lambda .
    \end{equation}
    The operator $\mathcal{T}^\lambda_\Omega$ is monotone and $\lambda$-translation equivariant, and the maximum defining $\mathcal{T}^\lambda_\Omega V^\lambda$ is attained at the unique $\pi^\lambda:=\pi_{\Omega,Q^\lambda}=\nabla\Omega^*(Q^\lambda)$; in particular $V^\lambda=\mathcal{E}^\lambda_{\pi^\lambda}V^\lambda\geq\mathcal{E}^\lambda_{\pi}V^\lambda$ for every admissible $\pi$.
\end{lemma}

\begin{proof}
    We have $\mT^\lambda_\Omega(Q)-\mT^\lambda_\Omega(Q')=\lambda P(\Omega^*(Q)-\Omega^*(Q'))$ with $P$ row-stochastic and $\Omega^*$ non-expansive for $\|\cdot\|_\infty$ (\Cref{lem:omega_star_props}.(v)), so $\mT^\lambda_\Omega$ is a $\lambda$-contraction on the complete space $(\bR^{\mS\times\mA},\|\cdot\|_\infty)$ and Banach's theorem applies; this is the regularized discounted operator of~\citep[Thm.~1]{geist2019theory}.
    Substituting the first identity of~\eqref{eq:discounted_Q} into the definition of $\Omega^*$ gives, for every $s$,
    \begin{equation}
        \label{eq:discounted_ACOE}
        V^\lambda(s)=\max_{\pi_s\in\dom(\Omega_s)}\left\{\tilde{r}_\pi(s)+\lambda\left(P_\pi V^\lambda\right)(s)\right\}
        =\left(\mathcal{T}^\lambda_\Omega V^\lambda\right)(s),
    \end{equation}
    the maximizer being unique and equal to $\nabla\Omega^*(Q^\lambda)$ by \Cref{lem:omega_star_props}.(i), since $\tau>0$.
    Monotonicity and translation equivariance of $\mathcal{T}^\lambda_\Omega$ are immediate from~\eqref{eq:reg_reward}, each $\mathcal{E}^\lambda_\pi$ having them and the state-wise maximum preserving them.
\end{proof}

\begin{lemma}[Variational representation of the discounted values]
    \label{lem:variational_representation}
    For an admissible stationary policy $\pi$ define the discounted regularized return $J^\lambda_\pi(s):=\bE_\pi\left[\sum_{t\geq0}\lambda^t\tilde{r}_\pi(s_t) | s_0=s\right]$, which is finite since $\tilde{r}_\pi$ is a finite vector.
    Then, for every $\lambda\in(0,1)$ and every $s\in\mS$,
    \begin{equation}
        \label{eq:variational_representation}
        V^\lambda(s)=\max_{\pi\text{ admissible stationary}}J^\lambda_\pi(s),
    \end{equation}
    the maximum being attained at $\pi^\lambda$.
\end{lemma}

\begin{proof}
    Let $\pi$ be an admissible stationary policy.
    Iterating $V^\lambda\geq\mathcal{E}^\lambda_\pi V^\lambda$ (\Cref{lem:discounted_fixed_point}) $K$ times, the operator $\mathcal{E}^\lambda_\pi$ being monotone, and using the Markov property we obtain $V^\lambda(s)\geq\left(\mathcal{E}^\lambda_\pi\right)^K V^\lambda(s)=\bE_\pi\left[\sum_{t<K}\lambda^t\tilde{r}_\pi(s_t)|s_0=s\right]+\lambda^K P^K_\pi V^\lambda(s)$.
    Since $V^\lambda$ is a finite vector and $P^K_\pi$ is row-stochastic, $\left\|\lambda^K P^K_\pi V^\lambda\right\|_{\mS,\infty}\leq\lambda^K\|V^\lambda\|_{\mS,\infty}\to0$ as $K\to+\infty$ at fixed $\lambda$, so $V^\lambda\geq J^\lambda_\pi$.
    The same computation with $\pi=\pi^\lambda$, for which $V^\lambda=\mathcal{E}^\lambda_{\pi^\lambda}V^\lambda$ holds with equality, gives $V^\lambda=J^\lambda_{\pi^\lambda}$.
\end{proof}

\begin{lemma}[Normalized values are uniformly bounded]
    \label{lem:uniform_sup_bound}
    For every $\lambda\in(0,1)$,
    \begin{equation}
        \label{eq:uniform_sup_bound}
        (1-\lambda)\left\|V^\lambda\right\|_{\mS,\infty}\leq R_{\max},
        \qquad
        (1-\lambda)\left\|Q^\lambda\right\|_{\infty}\leq 2R_{\max}.
    \end{equation}
\end{lemma}

\begin{proof}
    By~\eqref{eq:existence_constants}, every admissible $\pi$ satisfies the one-sided bound $\tilde{r}_\pi\leq\|r\|_\infty-\tau\underline{\Omega}_s\leq R_{\max}$, while the policy $\tilde{\pi}$ of \Cref{assumption:reference_policy_general} satisfies the two-sided bound $\lvert\tilde{r}_{\tilde{\pi}}(s)\rvert\leq\|r\|_\infty+\tau\lvert\Omega_s(\tilde{\pi}_s)\rvert\leq R_{\max}$.
    Bounding the summand of $J^\lambda_\pi$ termwise and summing the geometric series therefore gives $J^\lambda_\pi\leq\frac{R_{\max}}{1-\lambda}$ for every admissible $\pi$, and $J^\lambda_{\tilde{\pi}}\geq-\frac{R_{\max}}{1-\lambda}$.
    Since $V^\lambda$ is the maximum of $J^\lambda_\pi$ over admissible $\pi$ (\Cref{lem:variational_representation}), the upper bound propagates to $V^\lambda$ and the lower bound is inherited from the single policy $\tilde{\pi}$, whence $\|V^\lambda\|_{\mS,\infty}\leq\frac{R_{\max}}{1-\lambda}$.
    The claim on $Q^\lambda$ follows from~\eqref{eq:discounted_Q}: $(1-\lambda)\|Q^\lambda\|_\infty\leq(1-\lambda)\|r\|_\infty+\lambda(1-\lambda)\|V^\lambda\|_{\mS,\infty}\leq2R_{\max}$.
\end{proof}

\subsubsection{A necessary and sufficient condition for the existence of the regularized fixed point}

\begin{assumption}[Uniformly bounded relative values]
    \label{assumption:span_bounded}
    There exist $\lambda_0\in(0,1)$ and $C_V<+\infty$ such that $\left\|V^\lambda\right\|_{\mS,\spn}\leq C_V$ for every $\lambda\in[\lambda_0,1)$.
\end{assumption}

\begin{theorem}[Existence of a fixed point for $\mT_\Omega$]
    \label{thm:existence_fixed_point_general}
    Let \Cref{assumption:convex-analytic-regularity,assumption:reference_policy_general,assumption:span_bounded} hold.
    Then there exist $Q^*_\Omega\in\bR^{\mS\times\mA}$ and $\rho^*_\Omega\in\bR$ such that
    \begin{equation}
        \label{eq:fixed_point_equation}
        \mT_\Omega\left(Q^*_\Omega\right)=Q^*_\Omega+\rho^*_\Omega\ones,
    \end{equation}
    with $\left\|Q^*_\Omega\right\|_\spn\leq\|r\|_\spn+C_V$ and $\lvert\rho^*_\Omega\rvert\leq2R_{\max}$.
\end{theorem}

\begin{proof}
    Fix $(s_0,a_0)\in\mS\times\mA$ and set, for $\lambda\in[\lambda_0,1)$, $\bar{Q}^\lambda:=Q^\lambda-Q^\lambda(s_0,a_0)\ones$ and $\rho^\lambda:=(1-\lambda)Q^\lambda(s_0,a_0)$.

    By translation equivariance of $\Omega^*$ (\Cref{lem:omega_star_props}.(iii)) and $P\ones=\ones$, substituting $Q^\lambda=\bar{Q}^\lambda+Q^\lambda(s_0,a_0)\ones$ into~\eqref{eq:discounted_Q} and subtracting $\lambda Q^\lambda(s_0,a_0)$ gives
    \begin{equation}
        \label{eq:normalized_discounted}
        \bar{Q}^\lambda(s,a)+\rho^\lambda=r(s,a)+\lambda\sum_{s'}p(s'|s,a)\,\Omega^*\!\left(\bar{Q}^\lambda\right)(s').
    \end{equation}

    By \Cref{lem:uniform_sup_bound}, $\lvert\rho^\lambda\rvert\leq(1-\lambda)\|Q^\lambda\|_\infty\leq2R_{\max}$.
    By~\eqref{eq:discounted_Q} and \Cref{lem:row-stochastic-span}, $\|Q^\lambda\|_\spn\leq\|r\|_\spn+\lambda\|V^\lambda\|_{\mS,\spn}\leq\|r\|_\spn+C_V$ under \Cref{assumption:span_bounded}, hence $\|\bar{Q}^\lambda\|_\infty=\max_{s,a}\lvert Q^\lambda(s,a)-Q^\lambda(s_0,a_0)\rvert\leq\|Q^\lambda\|_\spn\leq\|r\|_\spn+C_V$.
    Both bounds are uniform in $\lambda\in[\lambda_0,1)$.

    Let $\lambda_n\uparrow1$ with $\lambda_n\geq\lambda_0$.
    The pairs $(\bar{Q}^{\lambda_n},\rho^{\lambda_n})$ lie in a compact subset of $\bR^{\mS\times\mA}\times\bR$, so by Bolzano--Weierstrass some subsequence $\lambda_{\varphi(n)}\uparrow1$ satisfies $\bar{Q}^{\lambda_{\varphi(n)}}\to Q^*_\Omega$ and $\rho^{\lambda_{\varphi(n)}}\to\rho^*_\Omega$.
    Since $\Omega^*$ is non-expansive hence continuous, letting $n\to+\infty$ in~\eqref{eq:normalized_discounted} gives $Q^*_\Omega+\rho^*_\Omega\ones=r+P\,\Omega^*(Q^*_\Omega)=\mT_\Omega(Q^*_\Omega)$.
    The bounds are preserved in the limit.
\end{proof}

\begin{proposition}[\Cref{assumption:span_bounded} is also necessary]
    \label{prop:span_bounded_necessary}
    Conversely, if~\eqref{eq:fixed_point_equation} admits a solution $(Q^*_\Omega,\rho^*_\Omega)$, then, with $V^*:=\Omega^*(Q^*_\Omega)$,
    \begin{equation*}
        \left\|V^\lambda\right\|_{\mS,\spn}\ \leq\ 2\left\|V^*\right\|_{\mS,\spn}
        \qquad\text{for every }\lambda\in(0,1),
    \end{equation*}
    so \Cref{assumption:span_bounded} holds.
    It therefore characterizes exactly the solvability of the regularized average-reward optimality equation, as in the unregularized case \citetext{\citealp{schweitzer1978functional}; \citealp[\S8.10]{puterman1994markov}}.
\end{proposition}

\begin{proof}
    Replacing $Q^*_\Omega$ by $Q^*_\Omega+c\ones$ changes $V^*$ by $c$ (\Cref{lem:omega_star_props}.(iii)) and leaves~\eqref{eq:fixed_point_equation} unchanged, so we may center $V^*$ so that $\|V^*\|_{\mS,\infty}=\frac{1}{2}\|V^*\|_{\mS,\spn}=:\kappa$.
    Taking the conjugate in~\eqref{eq:fixed_point_equation} as in~\eqref{eq:discounted_ACOE} gives the regularized average-reward optimality equation
    \begin{equation}
        \label{eq:ACOE}
        V^*+\rho^*_\Omega\ones=\max_{\pi\text{ admissible}}\left\{\tilde{r}_\pi+P_\pi V^*\right\}.
    \end{equation}
    Set $\tilde{V}:=V^*+\frac{\rho^*_\Omega}{1-\lambda}\ones$.
    Using~\eqref{eq:ACOE} and $P_\pi V^*\geq-\kappa\ones$ termwise,
    \begin{equation*}
        \begin{aligned}
            \mathcal{T}^\lambda_\Omega\tilde{V}
            &=\max_\pi\left\{\tilde{r}_\pi+P_\pi V^*-(1-\lambda)P_\pi V^*\right\}+\frac{\lambda\rho^*_\Omega}{1-\lambda}\ones\\
            &\leq V^*+\rho^*_\Omega\ones+(1-\lambda)\kappa\ones+\frac{\lambda\rho^*_\Omega}{1-\lambda}\ones
            =\tilde{V}+(1-\lambda)\kappa\ones.
        \end{aligned}
    \end{equation*}
    Hence $\mathcal{T}^\lambda_\Omega(\tilde{V}+\kappa\ones)\leq\tilde{V}+(1-\lambda)\kappa\ones+\lambda\kappa\ones=\tilde{V}+\kappa\ones$ by $\lambda$-translation equivariance, and monotonicity of $\mathcal{T}^\lambda_\Omega$ (\Cref{lem:discounted_fixed_point}) makes the iterates from $\tilde{V}+\kappa\ones$ non-increasing, so their limit $V^\lambda$ satisfies $V^\lambda\leq\tilde{V}+\kappa\ones$.
    The symmetric argument gives $V^\lambda\geq\tilde{V}-\kappa\ones$, whence $\|V^\lambda-\tilde{V}\|_{\mS,\infty}\leq\kappa$ and $\|V^\lambda\|_{\mS,\spn}\leq\|V^*\|_{\mS,\spn}+2\kappa=2\|V^*\|_{\mS,\spn}$.
\end{proof}

\subsubsection{Sufficient structural conditions}
\label{subsubsec:sufficient_structural_conditions}

We now record the classical setting in which \Cref{assumption:span_bounded} holds, and show that \Cref{assumption:m-step-mixing} is a special case of it with explicit constants.

\begin{assumption}[Weakly communicating MDP]
    \label{assumption:weakly_communicating}
    The state space splits as $\mS=\mathcal{C}\sqcup\Str$, where
    \begin{enumerate}[label=(\roman*)]
        \item $\mathcal{C}$ is \emph{closed}: $s\in\mathcal{C}$, $a\in\mA$ and $p(s'|s,a)>0$ imply $s'\in\mathcal{C}$;
        \item $\mathcal{C}$ is \emph{communicating}: for all $s,s'\in\mathcal{C}$ there exist a stationary policy $\pi$ and $k\geq0$ with $P^k_\pi(s'|s)>0$;
        \item every state of $\Str$ is transient under every stationary policy.
    \end{enumerate}
    This is the standard definition \citep[\S8.3.1]{puterman1994markov}; $\Str$ may be empty, in which case the MDP is communicating.
\end{assumption}

\begin{lemma}[Stopped evaluation]
    \label{lem:stopped_comparison}
    Let $\lambda\in(0,1)$, $\mathcal{B}\subseteq\mS$, $\sigma:=\inf\{t\geq0 | s_t\in\mathcal{B}\}\in\bN\cup\{+\infty\}$, and let $\pi$ be an admissible stationary policy.
    For every $K\geq0$ and $s\in\mS$,
    \begin{equation}
        \label{eq:stopped_comparison}
        V^\lambda(s)\ \geq\ \bE_\pi\left[\sum_{t=0}^{\sigma\wedge K-1}\lambda^t\tilde{r}_\pi(s_t)+\lambda^{\sigma\wedge K}V^\lambda\left(s_{\sigma\wedge K}\right) \mid s_0=s\right],
    \end{equation}
    with equality when $\pi=\pi^\lambda$.
    If moreover $\sigma<+\infty$ $\,\bP_\pi(\cdot|s_0=s)$-almost surely, the same holds with $\sigma\wedge K$ replaced by $\sigma$.
\end{lemma}

\begin{proof}
    Let $G_K(s)$ denote the right-hand side of~\eqref{eq:stopped_comparison}, empty sums being $0$.
    We show $V^\lambda\geq G_K$ pointwise by induction on $K$, simultaneously over all states.

    For $K=0$ we have $\sigma\wedge0=0$, so $G_0=V^\lambda$ for every $\pi$.
    Assume $V^\lambda\geq G_K$.
    If $s\in\mathcal{B}$ then $\sigma=0$ and $G_{K+1}(s)=V^\lambda(s)$.
    If $s\notin\mathcal{B}$ then $\sigma\geq1$, so splitting off the first step and applying the Markov property at $s_1$ gives $G_{K+1}(s)=\tilde{r}_\pi(s)+\lambda\left(P_\pi G_K\right)(s)$; the induction hypothesis, $P_\pi$ being non-negative, and \Cref{lem:discounted_fixed_point} then give
    \begin{equation*}
        G_{K+1}(s)\ \leq\ \tilde{r}_\pi(s)+\lambda\left(P_\pi V^\lambda\right)(s)\ =\ \left(\mathcal{E}^\lambda_\pi V^\lambda\right)(s)\ \leq\ V^\lambda(s).
    \end{equation*}
    Both inequalities are equalities at $\pi^\lambda$, which attains the state-wise maximum of \Cref{lem:discounted_fixed_point}, whence the equality case.

    For the last claim, $\mS$ is finite and $\pi$ is admissible, so $\max_s\lvert\tilde{r}_\pi(s)\rvert$ and $\|V^\lambda\|_{\mS,\infty}$ are finite and the integrand of~\eqref{eq:stopped_comparison} is bounded, uniformly in $K$, by $\frac{\max_s\lvert\tilde{r}_\pi(s)\rvert}{1-\lambda}+\|V^\lambda\|_{\mS,\infty}$.
    On $\{\sigma<+\infty\}$ we have $\sigma\wedge K=\sigma$ for $K$ large enough, so the integrand converges almost surely to its value at $\sigma$; conclude by dominated convergence.
\end{proof}

\begin{lemma}[Uniformly bounded absorption time]
    \label{lem:uniform_absorption}
    Under \Cref{assumption:weakly_communicating} there is $L<+\infty$ with $\bE_\pi[\sigma_{\mathcal{C}}|s_0 = s]\leq L$ for every stationary policy $\pi$ and every $s\in\mS$, where $\sigma_{\mathcal{C}}:=\inf\{t\geq0 | s_t\in\mathcal{C}\}$.
\end{lemma}

\begin{proof}
    This is the uniform-transience estimate underlying stochastic shortest path problems \citep[\S3.1]{bertsekas2012dynamic}; we give the short argument.
    If $\Str=\emptyset$, take $L=0$ and the claim is immediate. Otherwise, let $n:=\lvert\Str\rvert\geq1$, and fix a stationary $\pi$ and $s\in\Str$.
    If $\bP_\pi(\forall t, s_t\in\Str | s_0 = s)>0$, then with positive probability some state of the finite set $\Str$ is visited infinitely often, hence is recurrent for $P_\pi$, contradicting \Cref{assumption:weakly_communicating}.(iii).
    So $\sigma_{\mathcal{C}}<+\infty$ a.s., and deleting cycles from a path reaching $\mathcal{C}$ gives $\bP_\pi(\sigma_{\mathcal{C}}>n | s_0 = s)<1$.
    The map $\pi\mapsto\bP_\pi(\sigma_{\mathcal{C}}>n | s_0 = s)$ is polynomial, hence continuous on the compact $\prod_s\Delta(\mA)$, so $\theta:=\max_{s\in\Str}\max_\pi\bP_\pi(\sigma_{\mathcal{C}}>n | s_0 = s)<1$; the Markov property gives $\bP_\pi(\sigma_{\mathcal{C}}>kn | s_0 = s)\leq\theta^k$ and $\bE_\pi[\sigma_{\mathcal{C}}|s_0 = s] = \sum_{j\geq0}\bP_\pi(\sigma_{\mathcal{C}}>j | s_0 = s)\leq n\sum_{k\geq 0}\bP_\pi(\sigma_{\mathcal{C}}>kn | s_0 = s)\leq\frac{n}{1-\theta}=:L$.
    If $s\in\mathcal{C}$ then $\sigma_{\mathcal{C}}=0$, so the bound holds for every $s\in\mS$.
\end{proof}

\begin{proposition}[Weakly communicating implies \Cref{assumption:span_bounded}]
    \label{prop:weakly_communicating_implies_span}
    Let \Cref{assumption:weakly_communicating,assumption:reference_policy_general} hold.
    Let $K^\star\geq1$ and $p^\star>0$ be such that for every ordered pair $(s,s')$ of states of $\mathcal{C}$ there is $k(s,s')\leq K^\star$ with $P^{k(s,s')}_{\tilde{\pi}}(s'|s)\geq p^\star$.
    Such constants exist, with $K^\star\leq\max\{1,\lvert\mathcal{C}\rvert-1\}\leq\lvert\mS\rvert$, and they depend only on $p$ and $\tilde{\pi}$, not on $\lambda$. For any such pair, \Cref{assumption:span_bounded} holds with
    \begin{equation}
        \label{eq:span_global}
        \lambda_0:=1-\frac{1}{2K^\star},
        \qquad
        C_V:=\frac{4K^\star R_{\max}}{p^\star}+4LR_{\max}.
    \end{equation}
\end{proposition}

\begin{proof}
    Let $s,s'\in\mathcal{C}$. By \Cref{assumption:weakly_communicating}.(ii) there are $\pi$ and $k$ with $P^k_\pi(s'|s)>0$, hence a path $s=s_0,\dots,s_k=s'$ with $p(s_{i+1}|s_i,a_i)>0$ for some actions $a_i$; it stays in $\mathcal{C}$, which is closed, and deleting its cycles leaves a simple path, so $k(s,s')\leq\max\{1,\lvert\mathcal{C}\rvert-1\}$.     Since $\tilde\pi$ has full support, $P_{\tilde{\pi}}(s_{i+1}|s_i)\geq\tilde{\pi}(a_i|s_i)p(s_{i+1}|s_i,a_i)>0$ along that path, so $P^{k(s,s')}_{\tilde{\pi}}(s'|s)>0$; taking for $K^\star$ the largest of these lengths and for $p^\star$ the smallest of these finitely many positive numbers proves that admissible constants exist, and they involve only $p$ and $\tilde{\pi}$. Note also that $\tilde{\pi}$ is admissible by \Cref{assumption:reference_policy_general} and $\pi^\lambda=\nabla\Omega^*(Q^\lambda)$ is admissible as a maximizer, so \Cref{lem:stopped_comparison} applies to both, with $\lvert\tilde{r}_{\tilde{\pi}}\rvert\leq R_{\max}$ and $\tilde{r}_\pi\leq R_{\max}$ for every admissible $\pi$ by~\eqref{eq:existence_constants}.

    We first bound the span on $\mathcal{C}$. Fix $\lambda\in[\lambda_0,1)$ and let $s_{\max}$ and $s_{\min}$ maximize and minimize $V^\lambda$ over $\mathcal{C}$; if $s_{\max}=s_{\min}$ that span vanishes and there is nothing to prove, so assume $s_{\max}\neq s_{\min}$ and set $K:=k(s_{\min},s_{\max})\geq1$ and $p:=P^K_{\tilde{\pi}}(s_{\max}|s_{\min})\geq p^\star$. Applying \Cref{lem:stopped_comparison} at $s_{\min}$ with $\mathcal{B}=\emptyset$, so that $\sigma\wedge K=K$, with $\pi=\tilde{\pi}$ and $\sum_{t<K}\lambda^t\leq K$, gives $V^\lambda(s_{\min})\geq -KR_{\max}+\lambda^K\bE_{\tilde{\pi}}[V^\lambda(s_K)\mid s_0=s_{\min}]$. As $\mathcal{C}$ is closed, $s_K\in\mathcal{C}$ almost surely, so $V^\lambda(s_K)\geq V^\lambda(s_{\min})$ pointwise and, isolating the event $\{s_K=s_{\max}\}$ of probability $p$, $\bE_{\tilde{\pi}}[V^\lambda(s_K)\mid s_0=s_{\min}]\geq pV^\lambda(s_{\max})+(1-p)V^\lambda(s_{\min})$. Rearranging, the coefficient of $V^\lambda(s_{\min})$ being $1-\lambda^K$, and using $1-\lambda^K\leq K(1-\lambda)$ with \Cref{lem:uniform_sup_bound},
    \begin{equation*}
        p\,\lambda^K\left(V^\lambda(s_{\max})-V^\lambda(s_{\min})\right)\ \leq\ KR_{\max}+\left(1-\lambda^K\right)V^\lambda(s_{\min})\ \leq\ 2KR_{\max}.
    \end{equation*}
    Since $K\leq K^\star$ and $\lambda^K\geq\lambda_0^{K^\star}=\left(1-\frac{1}{2K^\star}\right)^{K^\star}\geq\frac12$, we get $\left\|V^\lambda|_{\mathcal{C}}\right\|_{\mathcal{C},\spn}\leq 4K^\star R_{\max}/p^\star=:C_{\mathcal{C}}$.

    It remains to pass from $\mathcal{C}$ to $\mS$. Write $\sigma:=\sigma_{\mathcal{C}}$ and fix $s\in\mS$. By \Cref{lem:uniform_absorption}, $\bE_\pi[\sigma\mid s_0=s]\leq L$ for every stationary $\pi$, so $\sigma<+\infty$ almost surely and \Cref{lem:stopped_comparison} applies with $\mathcal{B}=\mathcal{C}$ and the stopping time $\sigma$ itself. The reward term is at most $LR_{\max}$ for $\pi=\pi^\lambda$ and at least $-LR_{\max}$ for $\pi=\tilde{\pi}$, since $\sum_{t<\sigma}\lambda^t\leq\sigma$. Writing the boundary term as $\bE_\pi[V^\lambda(s_\sigma)\mid s_0=s]+\bE_\pi[(\lambda^\sigma-1)V^\lambda(s_\sigma)\mid s_0=s]$, the second expectation has absolute value at most $LR_{\max}$, because $0\leq1-\lambda^\sigma\leq\sigma(1-\lambda)$ and $(1-\lambda)\|V^\lambda\|_{\mS,\infty}\leq R_{\max}$ by \Cref{lem:uniform_sup_bound}. Since $s_\sigma\in\mathcal{C}$, the previous paragraph gives $0\leq V^\lambda(s_\sigma)-V^\lambda(s_{\min})\leq C_{\mathcal{C}}$. Taking $\pi=\pi^\lambda$, for which \Cref{lem:stopped_comparison} holds with equality, and $\tilde{r}_{\pi^\lambda}\leq R_{\max}$, yields $V^\lambda(s)\leq V^\lambda(s_{\min})+C_{\mathcal{C}}+2LR_{\max}$; taking $\pi=\tilde{\pi}$ and $\tilde{r}_{\tilde{\pi}}\geq-R_{\max}$ yields $V^\lambda(s)\geq V^\lambda(s_{\min})-2LR_{\max}$. Both hold for every $s$, so $\left\|V^\lambda\right\|_{\mS,\spn}\leq C_{\mathcal{C}}+4LR_{\max}=C_V$.
\end{proof}

\begin{corollary}[\Cref{assumption:m-step-mixing} implies \Cref{assumption:span_bounded}, hence \Cref{thm:regularized_fixed_point}]
    \label{cor:mixing_implies_span}
    Let \Cref{assumption:m-step-mixing,assumption:convex-analytic-regularity,assumption:reference_policy} hold.
    Then the MDP is communicating ($\mathcal{C}=\mS$, $\Str=\emptyset$, $L=0$), \Cref{prop:weakly_communicating_implies_span} applies with $\tilde{\pi}=\mu$, $K^\star=m$ and $p^\star=\beta_m\nu_{\min}$, where $\nu_{\min}:=\min_s\nu(s)>0$, so \Cref{assumption:span_bounded} holds with
    \begin{equation*}
        \lambda_0=1-\frac{1}{2m},
        \qquad
        C_V=\frac{4mR_{\max}}{\beta_m\nu_{\min}},
    \end{equation*}
    and \Cref{thm:existence_fixed_point_general} gives \Cref{thm:regularized_fixed_point}.
\end{corollary}

\begin{proof}
    \Cref{assumption:reference_policy} makes \Cref{assumption:reference_policy_general} hold with $\tilde{\pi}=\mu$, which has full support, so that $\delta_\Omega=\max_{s}\max\{\lvert\underline{\Omega}_s\rvert,\lvert\Omega_s(\mu_s)\rvert\}$ and $R_{\max}$ are the constants of \Cref{thm:regularized_fixed_point}.
    For all $s,s'$, $P^m_\mu(s'|s)=\sum_a\mu(a|s)K^m_\mu(s'|s,a)\geq\beta_m\nu(s')\geq\beta_m\nu_{\min}>0$, so every ordered pair of states is joined by a path of length $m$ under $\mu$: \Cref{assumption:weakly_communicating} holds with $\mathcal{C}=\mS$, which is closed and communicating, and $\Str=\emptyset$, so $\sigma_{\mathcal{C}}\equiv0$ and $L=0$ in \Cref{lem:uniform_absorption}.
    \Cref{prop:weakly_communicating_implies_span} then applies with $K^\star=m$ and $p^\star=\beta_m\nu_{\min}$, its last paragraph being vacuous, and gives \Cref{assumption:span_bounded} with the stated constants.
    Finally \Cref{thm:existence_fixed_point_general} yields~\eqref{eq:fixed_point_equation} with $\|Q^*_\Omega\|_\spn\leq\|r\|_\spn+C_V$ and $\lvert\rho^*_\Omega\rvert\leq2R_{\max}$. The gain $\rho^*_\Omega$ is unique and optimal by \Cref{prop:uniqueness_gain,prop:gain_interpretation}, and under \Cref{assumption:positive-domination} $Q^*_\Omega$ is unique up to an additive constant by \Cref{prop:uniqueness_fixed_point}, whose \Cref{assumption:unique_terminal_class} holds since the MDP is communicating (\Cref{rem:unique_terminal_weaker}).
\end{proof}

\paragraph{A temperature-independent span bound for divergences.}
Under \Cref{assumption:m-step-mixing,assumption:convex-analytic-regularity,assumption:reference_policy}, assume additionally that $\Omega_s\geq0$ and $\Omega_s(\mu_s)=0$ for every $s$, as for the divergence regularizers considered here. Then every regularized solution satisfies
\begin{equation*}
    \|Q^*_\Omega\|_\spn \leq U, \qquad U:=\|r\|_\spn \left(1+\frac{m}{\beta_m\nu_{\min}}\right),
\end{equation*}
uniformly in the temperature $\tau>0$.

\begin{proof}[Proof of the bound]
    Let $h:=\Omega^*(Q^*_\Omega)$, $r_{\min}:=\min_{s,a}r(s,a)$ and $r_{\max}:=\max_{s,a}r(s,a)$. Evaluating the state optimality equation~\eqref{eq:ACOE} at $\mu$ gives $h+\rho^*_\Omega\ones\geq r_\mu+P_\mu h$, where $r_\mu(s):=\sum_a\mu(a|s)r(s,a)$. At a maximizer of $h$, the same optimality equation and the non-negativity of the regularizer give $\rho^*_\Omega\leq r_{\max}$. Iterating the inequality for $\mu$ yields
    \begin{equation*}
        h+m\rho^*_\Omega\ones \geq \sum_{t=0}^{m-1}P_\mu^t r_\mu+P_\mu^m h \geq mr_{\min}\ones+P_\mu^m h.
    \end{equation*}
    Set $g:=h-\min_s h(s)\ones\geq0$, and choose $s_-\in\argmin_s h(s)$ and $s_+\in\argmax_s h(s)$. Evaluating the preceding inequality at $s_-$ gives
    \begin{equation*}
        (P_\mu^m g)(s_-) \leq m(\rho^*_\Omega-r_{\min}) \leq m\|r\|_\spn.
    \end{equation*}
    Averaging the initial action in \Cref{assumption:m-step-mixing} gives $P_\mu^m(s'|s)\geq\beta_m\nu(s')$. Consequently,
    \begin{equation*}
        (P_\mu^m g)(s_-) \geq\beta_m\nu(s_+)g(s_+) \geq\beta_m\nu_{\min}\|h\|_{\mS,\spn}.
    \end{equation*}
    Finally, $Q^*_\Omega=r+Ph-\rho^*_\Omega\ones$, so span non-expansiveness of $P$ gives $\|Q^*_\Omega\|_\spn \leq\|r\|_\spn+\|h\|_{\mS,\spn}\leq U$.
\end{proof}

For $Q_0=0$, we have $e_0=\|Q^*_\Omega\|_\spn\leq U$. Moreover, $\Omega^*(0)=0$, so the constants of \Cref{sec:results} satisfy
\begin{equation*}
    \bar w\leq3U, \qquad \|Q^*_\Omega\|_\asp+2e_0\leq3U,\qquad \tilde\beta\geq \beta_m\gamma_\Omega(3U)^{m-1}.
\end{equation*}
Thus, for $U>0$, the localized bounds underlying \Cref{thm:sync-stoch-main-detailed,thm:epoch-main-body-detailed} can be calibrated with span radius $2U$, noise bound $3U$, and contraction lower bound $\underline{\beta}:=\beta_m\underline{\gamma}^{m-1}$, where $0<\underline{\gamma}\leq\gamma_\Omega(3U)$ is an analytic lower bound, such as those in \Cref{tab:alpha_family}. Any known upper bound on the reward span can replace $\|r\|_\spn$ in the definition of $U$. The single-trajectory transition budget additionally uses $\eta_{\min}$.
These prescriptions require a certified Doeblin minorization for $\mu$; knowledge of the policy alone does not supply its mixing constants.

\begin{remark}[On the structural hypothesis]
    \label{rem:other_sufficient}
    Among the classical ergodic classes of finite MDPs \citep[\S8.3.1]{puterman1994markov}, weak communication is the weakest: both \emph{unichain} models --- in which every stationary policy induces a chain with a single recurrent class plus a possibly empty set of transient states --- and \emph{communicating} models are weakly communicating, while unichain and communicating are incomparable, neither class containing the other. \Cref{prop:weakly_communicating_implies_span} therefore covers all these models at once.

    \Cref{assumption:span_bounded} is nonetheless strictly weaker, since it depends on the rewards and the regularizer as well as the transition kernel. For instance, consider two absorbing states with $r(s,a)\equiv c$ and the same regularizer minimum $\min\Omega_s=\underline{\Omega}$. This MDP is not weakly communicating, yet $V^\lambda(s)=(c-\tau\underline{\Omega})/(1-\lambda)$ in both states, so its span vanishes and the regularized optimality equation admits a solution. Moreover, \Cref{assumption:span_bounded} is necessary as well as sufficient (\Cref{prop:span_bounded_necessary}), directly for the regularized model.
\end{remark}

\subsubsection{Uniqueness and interpretation of the gain}

\begin{proposition}[Uniqueness of the gain]
    \label{prop:uniqueness_gain}
    Let \Cref{assumption:convex-analytic-regularity} hold.
    If $\mT_\Omega(Q_i)=Q_i+\rho_i\ones$ for $i=1,2$, then $\rho_1=\rho_2$. No structural assumption on the MDP is needed here, unlike for the uniqueness of $Q^*_\Omega$ (\Cref{prop:uniqueness_fixed_point}).
\end{proposition}

\begin{proof}
    Let $c:=\max_{s,a}(Q_1-Q_2)(s,a)$, so $Q_1\leq Q_2+c\ones$.
    By monotonicity and translation equivariance of $\mT_\Omega$ (\Cref{cor:T_omega_props}), $\mT_\Omega(Q_1)\leq \mT_\Omega(Q_2)+c\ones$, i.e.\ $Q_1-Q_2\leq(\rho_2-\rho_1+c)\ones$.
    Evaluating at a maximizer of $Q_1-Q_2$ gives $c\leq\rho_2-\rho_1+c$, i.e.\ $\rho_1\leq\rho_2$; exchanging the roles of $Q_1$ and $Q_2$ gives the reverse inequality.
\end{proof}

\begin{assumption}[Unique terminal class]
    \label{assumption:unique_terminal_class}
    The union graph $G$ of \Cref{def:union_graph} has a unique terminal strongly connected component, i.e.\ a unique strongly connected component $\mathcal{C}$ with no outgoing edge.
\end{assumption}

\begin{remark}
    \label{rem:unique_terminal_weaker}
    \Cref{assumption:weakly_communicating} implies \Cref{assumption:unique_terminal_class}: its set $\mathcal{C}$ is terminal by closedness and connectedness, and a second terminal component $\mathcal{C}'\subseteq\Str$ would be closed, so every stationary policy then has a recurrent state inside, contradicting~(iii).
    The converse fails: with $\mS=\{1,2,3\}$, a self-loop and an edge $1\to2$ at state $1$, and $2\to3$, $3\to3$, the terminal component $\{3\}$ is unique, yet state $1$ is recurrent under the self-loop policy, so the MDP is not weakly communicating.
\end{remark}

\begin{proposition}[Uniqueness of $Q^*_\Omega$ up to an additive constant]
    \label{prop:uniqueness_fixed_point}
    Let \Cref{assumption:unique_terminal_class} hold and assume  that all greedy policies have full support.
    If $\mT_\Omega(Q_i)=Q_i+\rho_i\ones$ for $i=1,2$, then $Q_1-Q_2\in\bR\ones$.
    Conversely $Q^*_\Omega+c\ones$ solves~\eqref{eq:fixed_point_equation} with the same $\rho^*_\Omega$, for every $c\in\bR$.
\end{proposition}

\begin{proof}
    The converse is translation equivariance (\Cref{cor:T_omega_props}).
    Directly, $\rho_1=\rho_2$ by \Cref{prop:uniqueness_gain}, so $h:=Q_1-Q_2$ satisfies $h=\mT_\Omega(Q_1)-\mT_\Omega(Q_2)=K_{\bar{\pi}}h$ by the calculation in the proof of \Cref{lem:mixing-properties-of-regularization}.(i), which uses only \Cref{assumption:convex-analytic-regularity}, where $\bar{\pi}:=\bar{\pi}_{\Omega,[Q_1,Q_2]}$: the function $h$ is harmonic for the state--action kernel $K_{\bar{\pi}}$.

    By \Cref{def:mixing_objects}.(i), $\bar{\pi}$ is a mixture of greedy policies, hence has full support, so $P_{\bar{\pi}}(s'|s)>0$ exactly when $s\to s'$ in $G$.
    The recurrent classes of the state chain are therefore the terminal strongly connected components of $G$, and there is only one, $\mathcal{C}$, by \Cref{assumption:unique_terminal_class}; since $\bar{\pi}(a|s)>0$ for every $a$, the state--action chain has $\mathcal{C}\times\mA$ as its unique recurrent class.
    A bounded harmonic function of a finite Markov chain is constant on each recurrent class \citep[App.~A.4]{puterman1994markov}, and on a transient $s$ we have $h(s,a)=\bE[h(s_n,a_n)|s_0 = s, a_0 = a]$ for every $n$, which converges by absorption to the value of $h$ on that unique class.
    Hence $h$ is constant.
\end{proof}

\begin{remark}[Where the full-support requirement bites]
    \label{rem:full_support_uniqueness}
    The full-support requirement of \Cref{prop:uniqueness_fixed_point} is exactly \Cref{assumption:positive-domination} (\Cref{lem:pd_full_support}). It holds for the $\alpha$-divergences with $\alpha\leq1$ (\Cref{thm:alpha}.(iv)), hence for both KL divergences and for the Hellinger regularizer, for the negative entropy, and for every $\epsilon$-floor-wrapped regularizer, where $\pi_{\Omega_\epsilon,Q}\geq\epsilon\mu>0$ by \Cref{lem:floor_wrapped_props}.(ii).
    It fails for regularizers whose greedy policy may be sparse, that is for $\alpha>1$ and in particular for the $\chi^2$: there $\rho^*_\Omega$ is still unique by \Cref{prop:uniqueness_gain}, but $Q^*_\Omega$ need not be, exactly as at $\tau=0$ \citep[\S9.1]{puterman1994markov}.
\end{remark}

\begin{proposition}[$\rho^*_\Omega$ is the optimal regularized gain]
    \label{prop:gain_interpretation}
    Let $(Q^*_\Omega,\rho^*_\Omega)$ solve~\eqref{eq:fixed_point_equation} and set $\pi^*:=\pi_{\Omega,Q^*_\Omega}$.
    For an admissible stationary policy $\pi$ define the regularized gain $\tilde{\rho}_\pi(s):=\liminf_{n\to+\infty}\frac{1}{n}\bE_\pi\left[\sum_{t=0}^{n-1}\tilde{r}_\pi(s_t) | s_0 = s\right]$.
    Then $\tilde{\rho}_\pi(s)\leq\rho^*_\Omega=\tilde{\rho}_{\pi^*}(s)$ for every $s$ and every admissible stationary $\pi$: the gain is constant in $s$, $\rho^*_\Omega$ is optimal, and $\pi^*$ is gain-optimal.
\end{proposition}

\begin{proof}
    Equation~\eqref{eq:ACOE} holds with $V^*=\Omega^*(Q^*_\Omega)$, the maximum being attained state-wise at $\pi^*$.
    For any admissible $\pi$ it gives $\tilde{r}_\pi\leq V^*+\rho^*_\Omega\ones-P_\pi V^*$; applying $P^t_\pi$ and summing over $t<n$ telescopes to $\bE_\pi\left[\sum_{t<n}\tilde{r}_\pi(s_t) | s_0 = s\right]\leq n\rho^*_\Omega+V^*(s)-(P^n_\pi V^*)(s)$.
    Divide by $n$ and let $n\to+\infty$, $V^*$ being bounded.
    Equality holds throughout for $\pi=\pi^*$.
    This is the standard verification argument \citetext{\citealp[\S8.4]{puterman1994markov}; \citealp[\S4.2]{bertsekas2012dynamic}}.
\end{proof}

\subsection{Control of the bias}\label{app:proofs_bias}

We compare the regularized solution with the solutions $(Q^*,\rho^*)$ of the average-reward optimality equation $\mT(Q^*)=Q^*+\rho^*\ones$. Since \Cref{assumption:m-step-mixing} makes the MDP communicating, such solutions exist \citep[Theorem~3.1(b)]{schweitzer1978functional} and $\rho^*$ is the optimal gain, but $Q^*$ need not be unique up to an additive constant \citep{wan2024convergence}. The comparison uses two tools: a pointwise version of the regularization modulus, and the classical bounds on the gain given by an approximate fixed point.

\begin{definition}[Pointwise regularization modulus]
    \label{def:pointwise-modulus}
    For a radius $R\geq0$, with $g_Q$ the regularization gap of \Cref{def:mixing_objects}.(iv),
    \begin{equation*}
        \Lambda_\Omega^{\mathrm{pt}}(R):=\max_{Q\,:\,\|Q\|_\asp\leq R}\left\|g_Q\right\|_{\mS,\infty}.
    \end{equation*}
    The maximum is attained and finite by the argument of the proof of \Cref{lem:modulus_properties}, and $\Lambda_\Omega(R)\leq2\Lambda_\Omega^{\mathrm{pt}}(R)$.
\end{definition}

\begin{lemma}[Gain bounds from an approximate fixed point]
    \label{lem:gain_sandwich}
    Assume that $\mT_\Omega(Q^*_\Omega)=Q^*_\Omega+\rho^*_\Omega\ones$. If $Q\in\bR^{\mS\times\mA}$ and $c_-,c_+\in\bR$ satisfy $Q+c_-\ones\leq \mT_\Omega(Q)\leq Q+c_+\ones$, then $c_-\leq\rho^*_\Omega\leq c_+$.
\end{lemma}

\begin{proof}
    This is the regularized form of the classical value-iteration bounds on the gain \citep[\S8.5.5]{puterman1994markov}, \citep[\S4.3]{bertsekas2012dynamic}.
    By monotonicity and translation equivariance of $\mT_\Omega$ (\Cref{cor:T_omega_props}), $\mT^{n+1}_\Omega(Q)\leq \mT^n_\Omega(Q+c_+\ones)=\mT^n_\Omega(Q)+c_+\ones$, so $\mT^n_\Omega(Q)\leq Q+nc_+\ones$ by induction.
    Since $\mT_\Omega$ is non-expansive for $\|\cdot\|_\infty$ (\Cref{cor:T_omega_props}), with $d:=\|Q-Q^*_\Omega\|_\infty$ we get $Q^*_\Omega+n\rho^*_\Omega\ones=\mT^n_\Omega(Q^*_\Omega)\leq \mT^n_\Omega(Q)+d\ones\leq Q+nc_+\ones+d\ones$.
    Reading any coordinate, dividing by $n$ and letting $n\to+\infty$ gives $\rho^*_\Omega\leq c_+$; the other inequality is symmetric.
\end{proof}

\begin{proposition}[Cost of $\Omega$-regularization]
    \label{prop:gap_optimal_q}
    Let \Cref{assumption:m-step-mixing,assumption:convex-analytic-regularity,assumption:reference_policy} hold, and let $(Q^*,\rho^*)$ solve $\mT(Q^*)=Q^*+\rho^*\ones$.
    \begin{enumerate}[label=(\roman*)]
        \item \emph{Value gap}: with $R^*:=\max\big\{\|Q^*_\Omega\|_\asp,\max_{0\leq k<m}\|\mT^k_\Omega(Q^*)\|_\asp\big\}$ and $\alpha(P)\leq1$ the Dobrushin coefficient of the transition kernel $P$ (\Cref{def:dobrushin}),
        \begin{equation*}
            \|Q^*_\Omega-Q^*\|_\spn\leq\frac{m\,\alpha(P)\,\Lambda_\Omega(R^*)}{\beta_m\gamma_\Omega(R^*)^{m-1}}\leq\frac{m\,\Lambda_\Omega(R^*)}{\beta_m\gamma_\Omega(R^*)^{m-1}},
        \end{equation*}
        the right-hand sides being $+\infty$ if $m>1$ and $\gamma_\Omega(R^*)=0$.
        \item \emph{Gain gap}: with $\underline{\Omega}:=\min_{s\in\mS}\min_{\pi_s\in\Delta(\mA)}\Omega_s(\pi_s)$,
        \begin{equation}
            \label{eq:gain_gap}
            \min_{s\in\mS}g_{Q^*}(s)\ \leq\ \rho^*_\Omega-\rho^*\ \leq\ \max_{s\in\mS}g_{Q^*}(s)\ \leq\ -\tau\underline{\Omega},
        \end{equation}
        and consequently
        \begin{equation}
            \label{eq:gain_gap_modulus}
            \left\lvert\rho^*_\Omega-\rho^*\right\rvert\ \leq\ \|g_{Q^*}\|_{\mS,\infty}\ \leq\ \Lambda_\Omega^{\mathrm{pt}}(\|Q^*\|_\asp).
        \end{equation}
        If moreover $|\mS|\geq2$ and $\min\Omega_s=\underline{\Omega}$ for every $s$, then $-\Lambda_\Omega(\|Q^*\|_\asp)\leq\rho^*_\Omega-\rho^*+\tau\underline{\Omega}\leq0$.
    \end{enumerate}
\end{proposition}

\begin{proof}
    The regularized fixed point $(Q^*_\Omega,\rho^*_\Omega)$ exists by \Cref{thm:regularized_fixed_point}, and both $\mT$ and $\mT_\Omega$ are translation equivariant (\Cref{cor:T_omega_props}), so that $\mT^k(Q^*)=Q^*+k\rho^*\ones$ and $\mT^k_\Omega(Q^*_\Omega)=Q^*_\Omega+k\rho^*_\Omega\ones$ for every $k\geq0$.

    \emph{(i)} The span seminorm being invariant under constants,
    \begin{equation}
        \|Q^*_\Omega-Q^*\|_\spn = \|\mT^m_\Omega(Q^*_\Omega) - \mT^m(Q^*)\|_\spn \leq \underbrace{\|\mT^m_\Omega(Q^*_\Omega) - \mT^m_\Omega(Q^*)\|_\spn}_{\mathrm{(a)}} + \underbrace{\|\mT^m_\Omega(Q^*) - \mT^m(Q^*)\|_\spn}_{\mathrm{(b)}}.
        \label{eq:perturbation_split_omega}
    \end{equation}
    \emph{Term (a).} Since $\|\cdot\|_\asp$ is invariant under constants, $\|\mT^k_\Omega(Q^*_\Omega)\|_\asp=\|Q^*_\Omega\|_\asp$, so that for every $k\in\llbracket 1,m-1\rrbracket$,
    \begin{equation*}
        R\big[\mT_\Omega^k(Q^*_\Omega),\mT_\Omega^k(Q^*)\big]=\max\big\{\|Q^*_\Omega\|_\asp,\|\mT^k_\Omega(Q^*)\|_\asp\big\}\leq R^*.
    \end{equation*}
    As $\gamma_\Omega$ is non-increasing (\Cref{def:mixing_objects}.(iii)), every factor of the product in \Cref{lem:mixing-properties-of-regularization}.(iv), applied with $Q=Q^*_\Omega$ and $Q'=Q^*$, is at least $\gamma_\Omega(R^*)$, hence (a)~$\leq\big(1-\beta_m\gamma_\Omega(R^*)^{m-1}\big)\|Q^*_\Omega-Q^*\|_\spn$.

    \emph{Term (b).} Let $A_k:=\mT^k_\Omega\big(\mT^{m-k}(Q^*)\big)$ for $k\in\llbracket0,m\rrbracket$, so that $A_0=\mT^m(Q^*)$ and $A_m=\mT^m_\Omega(Q^*)$. Since $\mT^{m-k-1}(Q^*)=Q^*+(m-k-1)\rho^*\ones$, translation equivariance makes the constants cancel, and $k$ applications of the non-expansiveness of $\mT_\Omega$ in $\|\cdot\|_\spn$ give
    \begin{equation*}
        \begin{aligned}
            \|A_{k+1}-A_k\|_\spn&=\big\|\mT^k_\Omega\big(\mT_\Omega(Q^*)\big)-\mT^k_\Omega\big(\mT(Q^*)\big)\big\|_\spn\leq\|\mT_\Omega(Q^*)-\mT(Q^*)\|_\spn\\
            &\leq\alpha(P)\Lambda_\Omega(\|Q^*\|_\asp)\leq\alpha(P)\Lambda_\Omega(R^*),
        \end{aligned}
    \end{equation*}
    by \Cref{lem:mixing-properties-of-regularization}.(vi), $\|Q^*\|_\asp\leq R^*$ and the monotonicity of $\Lambda_\Omega$ (\Cref{lem:modulus_properties}). Summing over $k$, (b)~$\leq m\,\alpha(P)\Lambda_\Omega(R^*)$.

    Injecting both bounds into~\eqref{eq:perturbation_split_omega} yields $\beta_m\gamma_\Omega(R^*)^{m-1}\|Q^*_\Omega-Q^*\|_\spn\leq m\,\alpha(P)\Lambda_\Omega(R^*)$, which is the claim when $m=1$ or $\gamma_\Omega(R^*)>0$; otherwise the bound is trivial.

    \emph{(ii)} By \Cref{lem:mixing-properties-of-regularization}.(vi) and its proof, $\mT_\Omega(Q)-\mT(Q)=Pg_Q$ for every $Q$, so that $\mT_\Omega(Q^*)=Q^*+\rho^*\ones+Pg_{Q^*}$. Each entry of $Pg_{Q^*}$ is a convex combination of the entries of $g_{Q^*}$, and \Cref{lem:gain_sandwich} with $c_-:=\rho^*+\min_sg_{Q^*}(s)$ and $c_+:=\rho^*+\max_sg_{Q^*}(s)$ gives the first two inequalities of~\eqref{eq:gain_gap}. The last one is the upper part of \Cref{lem:omega_star_props}.(ii), which needs no assumption on the Diracs: $g_{Q^*}(s)\leq-\tau\underline{\Omega}_s\leq-\tau\underline{\Omega}$. Finally,~\eqref{eq:gain_gap_modulus} follows from~\eqref{eq:gain_gap} and \Cref{def:pointwise-modulus}.
    For the last claim, fix $s$ and $s_0\neq s$, and let $\tilde{Q}$ equal $Q^*$ except at $s_0$, where it is constant in the action. Then $\|\tilde{Q}\|_\asp\leq\|Q^*\|_\asp$, and since $g_Q(s)$ only depends on $Q(s,\cdot)$ and equals $-\tau\min\Omega_{s_0}$ at a constant, $g_{\tilde{Q}}(s)=g_{Q^*}(s)$ and $g_{\tilde{Q}}(s_0)=-\tau\underline{\Omega}$; hence $-\tau\underline{\Omega}-g_{Q^*}(s)\leq\|g_{\tilde{Q}}\|_{\mS,\spn}\leq\Lambda_\Omega(\|Q^*\|_\asp)$. Together with $g_{Q^*}\leq-\tau\underline{\Omega}$ and~\eqref{eq:gain_gap}, this gives the claim.
\end{proof}

Item~(i) bounds the gap by the regularization modulus times the effective horizon $m/(\beta_m\gamma_\Omega(R^*)^{m-1})$, the average-reward counterpart of the discounted bias bound $\tau(\overline{\Omega}-\underline{\Omega})/(1-\gamma)$ of \citet{geist2019theory}, where $\overline{\Omega}:=\max_{s,a}\Omega_s(\delta_a)$. It exhibits the trade-off that runs through the paper: a larger temperature raises $\gamma_\Omega$, hence the contraction rates of \Cref{sec:results}, but also inflates $\Lambda_\Omega$. Unlike $1/(1-\gamma)$, this effective horizon depends on $\tau$, so (i) is informative only when $\gamma_\Omega(R^*)$ is not too small; for the reverse KL the bound grows as $e^{(m-1)R^*/\tau}$ when $\tau\to0$. Item~(ii) has no such limitation, as the next corollary shows. Regularization also removes a degeneracy of the unregularized problem: under \Cref{assumption:positive-domination}, $Q^*_\Omega$ is unique up to an additive constant (\Cref{prop:uniqueness_fixed_point}), and (i) bounds its distance to every solution $Q^*$.

\begin{corollary}[The gain gap vanishes with $\tau$]
    \label{cor:gain_gap_vanishes}
    Under the assumptions of \Cref{rem:bounded_modulus}, $-\tau\overline{\Omega}\leq\rho^*_\Omega-\rho^*\leq-\tau\underline{\Omega}$, so $\rho^*_\Omega\to\rho^*$ linearly as $\tau\to0$, with constants that are free of $\beta_m$, of $\gamma_\Omega$ and of the radius.
\end{corollary}

\begin{proof}
    When every Dirac lies in the effective domain, \Cref{lem:omega_star_props}.(ii) gives $g_{Q^*}(s)\geq-\tau\overline{\Omega}$ for every $s$, and~\eqref{eq:gain_gap} concludes.
\end{proof}

\begin{example}
    For the reverse KL (\Cref{cor:landmarks}, $\alpha=1$), $\underline{\Omega}_s=0$ and $\overline{\Omega}=\log(1/\mu_{\min})$ with $\mu_{\min}:=\min_{s,a}\mu(a|s)$, so \Cref{cor:gain_gap_vanishes} reads $-\tau\log(1/\mu_{\min})\leq\rho^*_\Omega-\rho^*\leq0$. For the forward KL (\Cref{cor:landmarks}, $\alpha=0$), the Diracs lie outside the effective domain, but $\underline{\Omega}_s=0$ still gives $\rho^*_\Omega\leq\rho^*$, and~\eqref{eq:gain_gap_modulus} and \Cref{thm:alpha}.(v) give $\lvert\rho^*_\Omega-\rho^*\rvert\leq\tau\big(1+\log_+(\|Q^*\|_\asp/\tau)\big)=O\big(\tau\log(1/\tau)\big)$ as $\tau\to0$.
\end{example}

The gain gap also controls the performance of the policies returned by our algorithms: the greedy policy of any $Q$ close to $Q^*_\Omega$ in span is near-optimal.

\begin{corollary}[Performance of the greedy policy]
    \label{cor:greedy_performance}
    Let \Cref{assumption:m-step-mixing,assumption:convex-analytic-regularity,assumption:reference_policy} hold, let $Q\in\bR^{\mS\times\mA}$ and $\pi:=\pi_{\Omega,Q}$. Then, for every $s\in\mS$, the regularized gain of \Cref{prop:gain_interpretation} satisfies $\rho^*_\Omega-\|Q-Q^*_\Omega\|_\spn\leq\tilde{\rho}_\pi(s)\leq\rho^*_\Omega$, and the gain of \Cref{subsec:average_reward_mdps} satisfies $\rho^\pi(s)\geq\tilde{\rho}_\pi(s)+\tau\underline{\Omega}$. In particular, under the assumptions of \Cref{rem:bounded_modulus}, $\rho^*-\rho^\pi(s)\leq\|Q-Q^*_\Omega\|_\spn+\tau\Delta_\Omega$.
\end{corollary}

\begin{proof}
    The upper bound is \Cref{prop:gain_interpretation}. For the lower bound, let $\mT^\pi(U)(s,a):=r(s,a)+\sum_{s'}p(s'|s,a)\big(\langle\pi_{s'},U_{s'}\rangle-\tau\Omega_{s'}(\pi_{s'})\big)$ be the regularized evaluation operator of $\pi$; it is monotone and translation equivariant, and $\mT^\pi(Q)=\mT_\Omega(Q)$ since $\pi$ is greedy for $Q$. With $u:=Q-Q^*_\Omega$, \Cref{lem:mixing-properties-of-regularization}.(i) gives $\mT_\Omega(Q)-Q-\rho^*_\Omega\ones=K_{\bar\pi}u-u\geq-\|u\|_\spn\ones$, the entries of $K_{\bar\pi}u$ lying in $[\min u,\max u]$. Hence $\mT^\pi(Q)\geq Q+(\rho^*_\Omega-\|u\|_\spn)\ones$ and, as in the proof of \Cref{lem:gain_sandwich}, $(\mT^\pi)^n(Q)\geq Q+n(\rho^*_\Omega-\|u\|_\spn)\ones$. Since $(\mT^\pi)^n(Q)(s,a)$ is $r(s,a)$ plus the expected sum of $\tilde{r}_\pi(s_t)$, $1\leq t<n$, along $\pi$ from $(s,a)$, up to a bounded term, dividing by $n$ gives $\tilde{\rho}_\pi(s)\geq\rho^*_\Omega-\|u\|_\spn$. Next, $\langle\pi_s,r_s\rangle=\tilde{r}_\pi(s)+\tau\Omega_s(\pi_s)\geq\tilde{r}_\pi(s)+\tau\underline{\Omega}$, whence $\rho^\pi\geq\tilde{\rho}_\pi+\tau\underline{\Omega}$. Finally, \Cref{cor:gain_gap_vanishes} gives $\rho^*-\rho^*_\Omega\leq\tau\overline{\Omega}$, and $\rho^*-\rho^\pi\leq(\rho^*-\rho^*_\Omega)+\|u\|_\spn-\tau\underline{\Omega}\leq\|u\|_\spn+\tau\Delta_\Omega$.
\end{proof}

\subsection{\texorpdfstring{$\epsilon$}{ε}-Floor-Wrapped Regularization}
\label{subsec:epsilon-floor-wrapped-regularization}

Regularizers with sparse greedy policies, such as the $\alpha$-divergences with $\alpha>1$, violate \Cref{assumption:positive-domination}, and even when \Cref{assumption:positive-domination} holds the domination profile $\gamma_\Omega(R)$ may go to zero as $R$ grows, making contraction and exploration results local. We wrap the base regularization in a hard policy floor. This restores \Cref{assumption:positive-domination} with $\gamma_{\Omega_\epsilon}\geq\epsilon$ uniformly in $R$, which makes the results global (\Cref{lem:floor_global}), at the price of an additional $\epsilon R$ term in our bound on the regularization modulus at radius $R$ (\Cref{rem:floor_bias}). The resulting greedy policy $\epsilon\mu+(1-\epsilon)\pi_{\Omega_0,Q}$ (\Cref{lem:floor_wrapped_props}.(ii)) is, for a uniform $\mu$ and the negative entropy as $\Omega_0$, the $\epsilon$-softmax learning policy of \citet{nanda2026minimal}, where the mixing weight $\epsilon$ plays the same role: a uniform lower bound on the action probabilities.

\begin{definition}[$\epsilon$-floored simplex]
    \label{def:floored_simplex}
    Let $\epsilon\in[0,1)$ and let $\mu$ be the full-support mixing policy of \Cref{assumption:m-step-mixing}. For $s\in\mS$ define $\Delta_\epsilon(s) := \{\pi_s\in\Delta(\mA) : \pi_s \geq \epsilon\mu_s\}$ and the affine bijection $\Phi_{\epsilon,s} : \Delta(\mA)\to\Delta_\epsilon(s), \Phi_{\epsilon,s}(p) = \epsilon\mu_s + (1-\epsilon)p$, whose inverse is $\Phi_{\epsilon,s}^{-1}(\pi_s) = (\pi_s-\epsilon\mu_s)/(1-\epsilon)$.
\end{definition}

\begin{definition}[$\epsilon$-floor-wrapped regularization]
    \label{def:floor_wrapped}
    Let $\Omega_0$ be a base regularization satisfying \Cref{assumption:convex-analytic-regularity} and let $\epsilon\in[0,1)$. The $\epsilon$-floor-wrapped regularization $\Omega_\epsilon$ is defined state-wise by
    \begin{equation*}
        \Omega_{\epsilon,s}(\pi_s) = \begin{cases}
            (1-\epsilon)\,\Omega_{0,s}\left(\frac{\pi_s-\epsilon\mu_s}{1-\epsilon}\right) & \text{if } \pi_s\in\Delta_\epsilon(s),\\[2ex]
            +\infty & \text{otherwise.}
        \end{cases}
    \end{equation*}
    Equivalently, $\Omega_{\epsilon,s} = (1-\epsilon)\left(\Omega_{0,s}\circ\Phi_{\epsilon,s}^{-1}\right) + \iota_{\Delta_\epsilon(s)}$, where $\iota_C$ denotes the convex indicator of the set $C$, equal to $0$ on $C$ and $+\infty$ elsewhere. For $\epsilon=0$ we recover $\Omega_0$.
\end{definition}

\begin{lemma}[Properties of $\Omega_\epsilon$]
    \label{lem:floor_wrapped_props}
    Let $\Omega_0$ satisfy \Cref{assumption:convex-analytic-regularity} with strong convexity modulus $\sigma$, and let $\epsilon\in[0,1)$. Then $\Omega_{\epsilon,s}$ is proper, closed, and $\frac{\sigma}{1-\epsilon}$-strongly convex, with effective domain $\Phi_{\epsilon,s}(\dom(\Omega_{0,s}))$; in particular $\Omega_\epsilon$ satisfies \Cref{assumption:convex-analytic-regularity}, and it satisfies \Cref{assumption:reference_policy} whenever $\Omega_0$ does. Moreover:
    \begin{enumerate}[label=(\roman*)]
        \item \emph{Conjugate}: $\Omega^*_\epsilon(Q)(s) = \epsilon\langle\mu_s,Q_s\rangle + (1-\epsilon)\Omega^*_0(Q)(s)$; in particular $\Omega^*_\epsilon$ is differentiable with Lipschitz gradient.
        \item \emph{Greedy policy}: $\pi_{\Omega_\epsilon,Q}(a|s) = \epsilon\mu(a|s) + (1-\epsilon)\pi_{\Omega_0,Q}(a|s)$.
        \item \emph{Domination profile}: $\gamma_{\Omega_\epsilon}(R) = \epsilon + (1-\epsilon)\gamma_{\Omega_0}(R) \geq \epsilon$ for every $R\geq0$.
        \item \emph{Span of the conjugate at the origin}: $\|\Omega^*_\epsilon(0)\|_{\mS,\spn} = (1-\epsilon)\|\Omega^*_0(0)\|_{\mS,\spn}$.
        \item \emph{Regularization modulus}: $\Lambda_{\Omega_\epsilon}(R) \leq \epsilon R + (1-\epsilon)\Lambda_{\Omega_0}(R)$ for every $R\geq0$. In particular, if $\Omega_0$ satisfies the assumptions of \Cref{rem:bounded_modulus}, then $\Lambda_{\Omega_\epsilon}(R)\leq\epsilon R+(1-\epsilon)\tau\Delta_{\Omega_0}$.
    \end{enumerate}
\end{lemma}

\begin{proof}[Proof of \Cref{lem:floor_wrapped_props}]
    Fix $s\in\mS$ and write $\mu_s=\mu(\cdot|s)$, $\Phi=\Phi_{\epsilon,s}$ for the affine bijection of \Cref{def:floored_simplex}.

    \emph{Effective domain and properness}. Since $\Omega_{\epsilon,s}(\Phi(p))=(1-\epsilon)\Omega_{0,s}(p)$ for $p\in\Delta(\mA)$ and $\Omega_{\epsilon,s}=+\infty$ outside $\Delta_\epsilon(s)=\Phi(\Delta(\mA))$, the effective domain of $\Omega_{\epsilon,s}$ is $\Phi(\dom(\Omega_{0,s}))$. It is non-empty since $\Omega_{0,s}$ is proper, and $\Omega_{\epsilon,s}$ never takes the value $-\infty$ since $\Omega_{0,s}$ does not; hence $\Omega_{\epsilon,s}$ is proper. Moreover $\Phi(\mu_s)=\mu_s$, so $\Omega_{\epsilon,s}(\mu_s)=(1-\epsilon)\Omega_{0,s}(\mu_s)$, which gives \Cref{assumption:reference_policy} for $\Omega_\epsilon$ whenever it holds for $\Omega_0$.

    \emph{Closedness}. $\Delta_\epsilon(s)$ is the intersection of the simplex with the closed half-spaces $\{\pi_s : \pi_s(a)\geq\epsilon\mu(a|s)\}$, hence closed and convex. The map $\Omega_{0,s}\circ\Phi^{-1}$ is proper, convex and lower semicontinuous as the composition of a proper, closed convex function with a continuous affine map, hence closed, a proper convex function being closed if and only if it is lower semicontinuous \citep[Theorem~7.1]{rockafellar1970convex}; and $\iota_{\Delta_\epsilon(s)}$ is closed since $\Delta_\epsilon(s)$ is closed. A sum of closed functions is closed.

    \emph{Strong convexity}. Let $\pi_s,\pi'_s\in\dom(\Omega_{\epsilon,s})$, $t\in[0,1]$, and set $p=\Phi^{-1}(\pi_s)$, $p'=\Phi^{-1}(\pi'_s)$, so that $\pi_s-\pi'_s=(1-\epsilon)(p-p')$. Since $\Phi^{-1}$ is affine, $\Phi^{-1}(t\pi_s+(1-t)\pi'_s)=tp+(1-t)p'$, and the $\sigma$-strong convexity of $\Omega_{0,s}$ gives
    \begin{align*}
        \Omega_{\epsilon,s}(t\pi_s+(1-t)\pi'_s) &= (1-\epsilon)\Omega_{0,s}\!\left(tp+(1-t)p'\right)\\
        &\leq (1-\epsilon)\left[t\,\Omega_{0,s}(p)+(1-t)\,\Omega_{0,s}(p')-\frac{\sigma}{2}t(1-t)\|p-p'\|^2\right]\\
        &= t\,\Omega_{\epsilon,s}(\pi_s)+(1-t)\,\Omega_{\epsilon,s}(\pi'_s)-\frac{\sigma}{2(1-\epsilon)}t(1-t)\|\pi_s-\pi'_s\|^2,
    \end{align*}
    using $\|p-p'\|^2=\|\pi_s-\pi'_s\|^2/(1-\epsilon)^2$ in the last line. Hence $\Omega_{\epsilon,s}$ is $\frac{\sigma}{1-\epsilon}$-strongly convex, and together with $\tau>0$ this shows that $\Omega_\epsilon$ satisfies \Cref{assumption:convex-analytic-regularity}.

    \emph{(i)}. The constraint $\pi_s\in\Delta_\epsilon(s)$ is exactly the range of $\Phi$, so the change of variable $\pi_s=\Phi(p)=\epsilon\mu_s+(1-\epsilon)p$ is a bijection from $\Delta(\mA)$ onto $\Delta_\epsilon(s)$. Since $\Omega_{\epsilon,s}(\Phi(p))=(1-\epsilon)\Omega_{0,s}(p)$ and $\langle\Phi(p),Q_s\rangle=\epsilon\langle\mu_s,Q_s\rangle+(1-\epsilon)\langle p,Q_s\rangle$, we get
    \begin{align*}
        \Omega_\epsilon^*(Q)(s) &= \max_{\pi_s\in\Delta(\mA)}\left\{\langle\pi_s,Q_s\rangle-\tau\Omega_{\epsilon,s}(\pi_s)\right\}\\
        &= \max_{\pi_s\in\Delta_\epsilon(s)}\left\{\langle\pi_s,Q_s\rangle-\tau\Omega_{\epsilon,s}(\pi_s)\right\}\\
        &= \max_{p\in\Delta(\mA)}\left\{\epsilon\langle\mu_s,Q_s\rangle+(1-\epsilon)\langle p,Q_s\rangle-\tau(1-\epsilon)\,\Omega_{0,s}(p)\right\}\\
        &= \epsilon\langle\mu_s,Q_s\rangle+(1-\epsilon)\max_{p\in\Delta(\mA)}\left\{\langle p,Q_s\rangle-\tau\Omega_{0,s}(p)\right\}\\
        &= \epsilon\langle\mu_s,Q_s\rangle+(1-\epsilon)\Omega_0^*(Q)(s),
    \end{align*}
    where the second equality holds because $\Omega_{\epsilon,s}=+\infty$ outside $\Delta_\epsilon(s)$, the third is the change of variable, the fourth uses $1-\epsilon>0$, and the last is the definition of $\Omega_0^*$. The map $Q_s\mapsto\epsilon\langle\mu_s,Q_s\rangle$ is linear with constant gradient $\epsilon\mu_s$, so $\Omega_\epsilon^*$ is differentiable as soon as $\Omega_0^*$ is (which is the case by \Cref{lem:omega_star_props}.(i) since $\tau >0$), with $\nabla\Omega_\epsilon^*=\epsilon\mu_s+(1-\epsilon)\nabla\Omega_0^*$; if $\nabla\Omega_0^*$ is $L_0$-Lipschitz then $\nabla\Omega_\epsilon^*$ is $(1-\epsilon)L_0$-Lipschitz.

    \emph{(ii)}. The maximum over $p$ in the third line above is attained at $p^\star=\argmax_{p\in\Delta(\mA)}\{\langle p,Q_s\rangle-\tau\Omega_{0,s}(p)\}=\pi_{\Omega_0,Q}(\cdot|s)$, which is unique by \Cref{lem:omega_star_props}.(i) applied to $\Omega_0$. Since $\Phi$ is a bijection, the maximizer of the original problem is $\pi_{\Omega_\epsilon,Q}(\cdot|s)=\Phi(p^\star)=\epsilon\mu_s+(1-\epsilon)\pi_{\Omega_0,Q}(\cdot|s)$. This is consistent with (i) and \Cref{lem:omega_star_props}.(i), which identify the greedy policy with $\nabla\Omega_\epsilon^*(Q)(s)$.

    \emph{(iii)}. Let $(s,a)\in\mS\times\mA$ and $Q$ with $\|Q\|_\asp\leq R$. Dividing (ii) by $\mu(a|s)>0$ gives
    \begin{equation*}
        \frac{\pi_{\Omega_\epsilon,Q}(a|s)}{\mu(a|s)} = \epsilon+(1-\epsilon)\frac{\pi_{\Omega_0,Q}(a|s)}{\mu(a|s)}.
    \end{equation*}
    Taking the infimum over $(s,a)$ and over $\{Q : \|Q\|_\asp\leq R\}$, we obtain $\gamma_{\Omega_\epsilon}(R)=\epsilon+(1-\epsilon)\gamma_{\Omega_0}(R)$ from \Cref{def:mixing_objects}.(iii). Since $\gamma_{\Omega_0}(R)\geq0$ and $\epsilon<1$, we get $\gamma_{\Omega_\epsilon}(R)\geq\epsilon$ for every $R\geq0$.

    \emph{(iv)}. Trivial from (i).

    \emph{(v)}. Let $Q$ satisfy $\|Q\|_\asp\leq R$ and write $g^\epsilon_Q$ and $g^0_Q$ for the regularization gaps of $\Omega_\epsilon$ and $\Omega_0$. Subtracting $\max_a Q_s(a)$ from (i) gives, for every $s\in\mS$,
    \begin{equation*}
        g^\epsilon_Q(s) = \epsilon\,h_Q(s) + (1-\epsilon)\,g^0_Q(s), \qquad h_Q(s) := \langle\mu_s,Q_s\rangle-\max_a Q_s(a).
    \end{equation*}
    Since $\mu_s\in\Delta(\mA)$, we have $h_Q(s)\in\left[-\|Q_s\|_{\mA,\spn},0\right]\subseteq[-R,0]$ for every $s$, hence $\|h_Q\|_{\mS,\spn}\leq R$. The triangle inequality and \Cref{def:mixing_objects}.(iv) applied to $\Omega_0$ give
    \begin{equation*}
        \|g^\epsilon_Q\|_{\mS,\spn} \leq \epsilon\|h_Q\|_{\mS,\spn} + (1-\epsilon)\|g^0_Q\|_{\mS,\spn} \leq \epsilon R + (1-\epsilon)\Lambda_{\Omega_0}(R).
    \end{equation*}
    Taking the maximum over $\{Q : \|Q\|_\asp\leq R\}$ yields $\Lambda_{\Omega_\epsilon}(R)\leq\epsilon R+(1-\epsilon)\Lambda_{\Omega_0}(R)$. The last claim follows from \Cref{rem:bounded_modulus} applied to $\Omega_0$.
\end{proof}

\begin{lemma}[Global properties of $\Omega_\epsilon$]
    \label{lem:floor_global}
    Let \Cref{assumption:m-step-mixing} hold, let $\Omega_0$ satisfy \Cref{assumption:convex-analytic-regularity}, let $\epsilon\in(0,1)$, and let $Q,Q'$ be two $Q$-value functions. Then:
    \begin{enumerate}[label=(\roman*)]
        \item \emph{Global $m$-step contraction}:
        \begin{equation*}
            \begin{aligned}
                \|\mT^m_{\Omega_\epsilon}(Q)-\mT^m_{\Omega_\epsilon}(Q')\|_\spn
                &\leq\Big(1-\beta_m\prod_{k=1}^{m-1}\gamma_{\Omega_\epsilon}\big(R[\mT_{\Omega_\epsilon}^k(Q),\mT_{\Omega_\epsilon}^k(Q')]\big)\Big)\|Q-Q'\|_\spn\\
                &\leq\left(1-\beta_m\epsilon^{m-1}\right)\|Q-Q'\|_\spn.
            \end{aligned}
        \end{equation*}
        \item \emph{Global coverage}: let $(Q_n)_{n\geq0}$ be $Q$-functions and let $(s_n,a_n)_{n\geq0}$ be a trajectory with $a_n\sim\pi_{\Omega_\epsilon,Q_n}(\cdot|s_n)$ and $s_{n+1}\sim p(\cdot|s_n,a_n)$; let $\mF_n:=\sigma(Q_0,(s_0,a_0),\dots,(s_n,a_n))$ and assume that $Q_n$ is $\mF_{n-1}$-measurable for $n\geq1$. Then for all $t\geq0$ and all $(s,a)$:
        \begin{equation}
            \bP\left((s_{t+m},a_{t+m})=(s,a)\mid\mF_t\right) \geq \beta_m\epsilon^m\eta(s,a) \geq \beta_m\epsilon^m\eta_{\min}.
        \end{equation}
        \item \emph{Fixed point up to a constant}: there exists $Q^*_{\Omega_\epsilon}\in\bR^{\mS\times\mA}$, unique up to an additive constant, and a unique $\rho^*_{\Omega_\epsilon}\in\bR$, such that $\mT_{\Omega_\epsilon}(Q^*_{\Omega_\epsilon}) = Q^*_{\Omega_\epsilon} + \rho^*_{\Omega_\epsilon}\ones$.
    \end{enumerate}
\end{lemma}

\begin{proof}[Proof of \Cref{lem:floor_global}]
    By \Cref{lem:floor_wrapped_props}, $\Omega_\epsilon$ satisfies \Cref{assumption:convex-analytic-regularity}, so \Cref{lem:mixing-properties-of-regularization} applies to $\Omega_\epsilon$.

    \emph{(i)}. The first inequality is \Cref{lem:mixing-properties-of-regularization}.(iv) applied to $\Omega_\epsilon$. By \Cref{lem:floor_wrapped_props}.(iii), $\gamma_{\Omega_\epsilon}(R)\geq\epsilon$ for every $R\geq0$, so $\prod_{k=1}^{m-1}\gamma_{\Omega_\epsilon}(R[\mT_{\Omega_\epsilon}^k(Q),\mT_{\Omega_\epsilon}^k(Q')])\geq\epsilon^{m-1}$ regardless of the radii, and the second inequality follows.

    \emph{(ii)}. By \Cref{lem:floor_wrapped_props}.(ii), $\pi_{\Omega_\epsilon,Q_n}\geq\epsilon\mu$ for every $n$, so \Cref{lem:mixing-properties-of-regularization}.(v) applies to $\Omega_\epsilon$ with the deterministic constants $\gamma_1=\dots=\gamma_m=\epsilon$.

    \emph{(iii)}. By translation equivariance (\Cref{cor:T_omega_props}), $\mT_{\Omega_\epsilon}$ induces a well-defined map $\bar{\mT}$ on the quotient space $\bR^{\mS\times\mA}/\bR\ones$, which is complete for the norm induced by $\|\cdot\|_\spn$, and (i) makes $\bar{\mT}^m$ a contraction with constant $1-\beta_m\epsilon^{m-1}<1$. Banach's theorem gives a unique fixed point $\bar{Q}^*$ of $\bar{\mT}^m$; since $\bar{\mT}$ commutes with $\bar{\mT}^m$, $\bar{\mT}(\bar{Q}^*)$ is also a fixed point of $\bar{\mT}^m$, hence equals $\bar{Q}^*$. Any representative $Q^*_{\Omega_\epsilon}$ of $\bar{Q}^*$ therefore satisfies $\mT_{\Omega_\epsilon}(Q^*_{\Omega_\epsilon})=Q^*_{\Omega_\epsilon}+\rho^*_{\Omega_\epsilon}\ones$ for some $\rho^*_{\Omega_\epsilon}\in\bR$, unique by \Cref{prop:uniqueness_gain}, and uniqueness of $Q^*_{\Omega_\epsilon}$ up to an additive constant is that of $\bar{Q}^*$.
\end{proof}

\begin{remark}[Bias of the floor]
    \label{rem:floor_bias}
    Applying \Cref{prop:gap_optimal_q} to $\Omega_\epsilon$, which satisfies its assumptions by \Cref{lem:floor_wrapped_props}, and using \Cref{lem:floor_wrapped_props}.(iii) and (v) gives
    \begin{equation*}
        \|Q^*_{\Omega_\epsilon}-Q^*\|_\spn \leq \frac{m\left(\epsilon R^* + (1-\epsilon)\Lambda_{\Omega_0}(R^*)\right)}{\beta_m(\epsilon+(1-\epsilon)\gamma_{\Omega_0}(R^*))^{m-1}},
    \end{equation*}
    where $R^*$ is the radius of \Cref{prop:gap_optimal_q} associated with $\Omega_\epsilon$. The floor guarantees a \emph{global} $m$-step contraction factor at most $1-\beta_m\epsilon^{m-1}$. At a fixed action-span radius $R$, it contributes the term $\epsilon R$ to our upper bound on the regularization modulus, independently of $\tau$. A smaller floor reduces this contribution but weakens the uniform contraction guarantee; monotonicity of the fixed-point bias in $\epsilon$ does not follow from the displayed bound, whose denominator and radius $R^*$ also depend on $\epsilon$.
\end{remark}

\section{Proofs for the learning algorithms of \pdfref{Section}{sec:results}}
\label{app:proofs_learning}

\paragraph{Proof techniques.}
Three ingredients drive the analysis. First, since $\nabla\Omega^*$ is the greedy policy, differences of regularized Bellman operators are kernels of average-path policies (\Cref{lem:mixing-properties-of-regularization}); unlike the single randomized policy that represents differences of hard maxima \citep[Lemma~B.6]{chen2026contraction}, this policy inherits the domination of the greedy policies, which turns the Doeblin condition on $\mu$ into a Doeblin condition on every product of $m$ such kernels, much as lower bounds on learning policies transfer irreducibility in discounted on-policy $Q$-learning \citep{nanda2026minimal}. Second, we analyze the unmodified one-step update although $\mT_\Omega$ only contracts over $m$ steps: expanding the Jacobian of a block of relaxed updates as an expectation over Bernoulli selections shows that the step size acts as a lazy transformation of the state--action kernels, the counterpart of the aperiodicity transformation of \citet[\S8.5.4]{puterman1994markov}, so that blocks with at least $m$ selected updates inherit the $m$-step contraction. This avoids both the multi-step estimators of \citet{zhang2021finite} and the lazy Bellman operator used by \citet{chen2026contraction}. Third, since the contraction rate degrades with the radius, a stopping-time argument keeps the iterates, with high probability, in a ball on which the rate holds, with a self-consistent choice of the radius.

Throughout this section, \Cref{assumption:m-step-mixing,assumption:convex-analytic-regularity,assumption:reference_policy,assumption:positive-domination} hold, and $Q^*_\Omega$ denotes a solution of~\eqref{eq:fixed_point_equation}, unique up to an additive constant (\Cref{thm:regularized_fixed_point}); every quantity below is invariant under the choice of that constant. The $Q$-iteration bound is the regularized counterpart of the span-contraction analysis of relative value iteration \citep[\S8.5]{puterman1994markov}. For $Q$-learning, we first analyze the noiseless relaxed flow, which contracts over blocks of updates (\Cref{subsubsec:synchronous_full_knowledge_updates}), then track the stochastic iterates along this flow up to their exit time from a span ball (\Cref{subsubsec:sync_stochastic}); the epoch algorithm follows the same scheme, with batch averaging and a cover-time bound (\Cref{subsec:proofs_epoch_q_learning}).

The results for general step size schedules use the following assumption.

\begin{assumption}[Step sizes]
    \label{assumption:lr}
    The sequence of step sizes $(\alpha_t)_{t\geq 0}$ is deterministic and non-increasing, satisfies $0<\alpha_t\leq1$ for all $t\geq 0$, and $\sum_{t\geq 0} \alpha_t=+\infty$.
\end{assumption}

For the explicit bounds below, let $\delta\in(0,1)$ be the failure probability, $e_0:=\|Q_0-Q^*_\Omega\|_\spn>0$ the initial span error, and $T$ the algorithmic budget. We use
\begin{align*}
    \bar\gamma&:=\gamma_\Omega(\|Q^*_\Omega\|_\asp+2e_0), & \tilde\beta&:=\beta_m\bar\gamma^{m-1},\\
    \bar w&:=\|Q^*_\Omega\|_\spn+2e_0+\|\Omega^*(0)\|_{\mS,\spn}, & \iota&:=\log\!\left(\frac{6|\mS||\mA|(T+1)}{\delta}\right).
\end{align*}
Here $T$ counts sweeps for $Q$-iteration and synchronous $Q$-learning, and transitions for epoch $Q$-learning. The constants $\tilde\beta$ and $\bar w$ are those of \Cref{sec:results}.

\subsection{Proof of \pdfref{Corollary}{cor:q-iteration} (Synchronous \texorpdfstring{$\Omega$}{Ω}-regularized \texorpdfstring{$Q$}{Q}-iteration)}
\label{subsec:convergence_synchronous_regularized_q_iteration}

We consider the iterates of \Cref{alg:synchronous_regularized_q_iteration}, $Q_{t+1}=\Pi[\mT_\Omega(Q_t)]$, and write $e_t:=\|Q_t-Q^*_\Omega\|_\spn$.

\begin{corollary}[Synchronous $\Omega$-regularized $Q$-iteration]
    \label{cor:q-iteration}
    Let \Cref{assumption:m-step-mixing,assumption:convex-analytic-regularity,assumption:reference_policy,assumption:positive-domination} hold. Then $(\|Q_t-Q^*_\Omega\|_\spn)_{t\geq0}$ is non-increasing and, for every $t\geq0$,
    \begin{equation}
        \label{eq:q-iteration-rate}
        \|Q_{t+m} - Q^*_\Omega \|_\spn \leq\Big(1-\beta_m\,\gamma_\Omega\big(\|Q^*_\Omega\|_\asp+ \|Q_{t} - Q^*_\Omega \|_\spn \big)^{m-1}\Big) \|Q_{t} - Q^*_\Omega \|_\spn.
    \end{equation}
    Consequently, for $\varepsilon\in(0,e_0)$, $\|Q_T-Q^*_\Omega\|_\spn\leq\varepsilon$ for every $T\geq m\lceil\tilde{\beta}^{-1}\log(e_0/\varepsilon)\rceil$, i.e.\ after $O\big(|\mS|^2|\mA|\,m\,\tilde{\beta}^{-1}\log(e_0/\varepsilon)\big)$ model queries, one query being an evaluation of $p(s'|s,a)$.
\end{corollary}

\begin{proof}
    Let $Q^*_\Omega$ be given by \Cref{thm:regularized_fixed_point}. By translation equivariance (\Cref{cor:T_omega_props}), $\mT^k_\Omega(Q^*_\Omega)=Q^*_\Omega+k\rho^*_\Omega\ones$ for every $k\geq0$ and, since the projection $\Pi$ subtracts a constant, $Q_{t+k}=\mT^k_\Omega(Q_t)-c\ones$ for some $c\in\bR$. Since $\|\cdot\|_\spn$ annihilates constants, $e_{t+k}=\|\mT^k_\Omega(Q_t)-\mT^k_\Omega(Q^*_\Omega)\|_\spn$ for every $t,k\geq0$. With $k=1$, the non-expansiveness of $\mT_\Omega$ in $\|\cdot\|_\spn$ (\Cref{cor:T_omega_props}) shows that $(e_t)_t$ is non-increasing.

    With $k=m$, \Cref{lem:mixing-properties-of-regularization}.(iv) applied to $(Q_t,Q^*_\Omega)$ gives~\eqref{eq:q-iteration-rate}, since $\gamma_\Omega$ is non-increasing and $R[\mT_\Omega^k(Q_t),\mT_\Omega^k(Q^*_\Omega)]\leq\|Q^*_\Omega\|_\asp+e_t$ for $k\in\llbracket0,m-1\rrbracket$. Indeed, $\|\cdot\|_\asp$ is invariant under constants, so $\|\mT^k_\Omega(Q^*_\Omega)\|_\asp=\|Q^*_\Omega\|_\asp$, and the triangle inequality together with $\|\cdot\|_\asp\leq\|\cdot\|_\spn$ gives $\|\mT^k_\Omega(Q_t)\|_\asp\leq\|Q^*_\Omega\|_\asp+e_{t+k}\leq\|Q^*_\Omega\|_\asp+e_t$.

    Finally, $e_t\leq e_0$ for every $t$ and $\gamma_\Omega$ is non-increasing, so each factor in~\eqref{eq:q-iteration-rate} is at most $1-\beta_m\gamma_\Omega(\|Q^*_\Omega\|_\asp+e_0)^{m-1}\leq1-\tilde\beta$. Hence $e_{km}\leq(1-\tilde\beta)^ke_0\leq e^{-\tilde\beta k}e_0$, which is at most $\varepsilon$ for $k\geq\tilde\beta^{-1}\log(e_0/\varepsilon)$, and monotonicity extends the bound to every $t\geq km$. Each sweep evaluates $p(s'|s,a)$ for all $(s,a,s')$, i.e.\ uses $|\mS|^2|\mA|$ queries, which gives the announced count.
\end{proof}

\subsection{Proofs for \pdfref{Subsection}{subsec:synchronous_regularized_q_learning} (Synchronous \texorpdfstring{$\Omega$}{Ω}-regularized \texorpdfstring{$Q$}{Q}-learning)}
\label{subsec:proofs_synchronous_q_learning}

\subsubsection{Synchronous full-knowledge updates}
\label{subsubsec:synchronous_full_knowledge_updates}

We let $(Q_t)_{t\geq 0}$ denote the sequence produced from an arbitrary initialization $Q_0$ by the \emph{full-knowledge synchronous update}
\begin{equation}
    \label{eq:sync_fk_update}
    Q_{t+1} = F_t(Q_t), \qquad \text{where } F_t(Q) := (1-\alpha_t)Q + \alpha_t \mT_\Omega(Q),
\end{equation}
and we measure the error in span seminorm, $e_t := \|Q_t - Q^*_\Omega\|_\spn$.

Without noise, \eqref{eq:sync_fk_update} is a damped value iteration, whose linearization $(1-\alpha_t)I+\alpha_tK_{\pi}$ is a lazy version of the greedy state--action kernel: the step size plays the role of the aperiodicity transformation of \citet[\S8.5.4]{puterman1994markov}, which acts on the state kernel and which \citet{chen2026contraction} impose on the dynamics through lazy sampling. The analysis below exploits this laziness at the level of the Jacobians (\Cref{lem:jacobian_bound_synchronous}): it turns the $m$-step minorization of \Cref{assumption:m-step-mixing} into a contraction over blocks, and it is this deterministic flow that the stochastic analysis of \Cref{subsubsec:sync_stochastic} tracks.

The fixed point is stationary only up to a constant, so it is not a fixed point of $F_t$; the next lemma records that it nonetheless generates a trajectory of~\eqref{eq:sync_fk_update} along which the error is exactly $e_t$. This is what allows us to dispense with a second sequence.

\begin{lemma}[Reference trajectory]
    \label{lem:reference_trajectory_synchronous}
    Let $Q^*_t := Q^*_\Omega + c_t\ones$ with $c_t := \rho^*_\Omega\sum_{k<t}\alpha_k$. Then $Q^*_{t+1} = F_t(Q^*_t)$ for every $t\geq0$, and $\|Q_t - Q^*_t\|_\spn = e_t$.
\end{lemma}

\begin{proof}
    By translation equivariance of $\mT_\Omega$ (\Cref{cor:T_omega_props}) and the fixed-point equation, $\mT_\Omega(Q^*_\Omega+c\ones) = Q^*_\Omega+(c+\rho^*_\Omega)\ones$, hence $F_t(Q^*_\Omega+c\ones) = Q^*_\Omega+(c+\alpha_t\rho^*_\Omega)\ones$, which is the announced recursion. The second claim follows since the span seminorm annihilates constants.
\end{proof}

\begin{lemma}[Monotonicity of the error]
    \label{lem:q_boundedness_synchronous_full_knowledge}
    For any $Q_0$, the sequence $(e_t)_{t\geq 0}$ is non-increasing.
\end{lemma}

\begin{proof}
    Using $\mT_\Omega(Q^*_\Omega) = Q^*_\Omega+\rho^*_\Omega\ones$, we have $Q_{t+1}-Q^*_\Omega = (1-\alpha_t)(Q_t-Q^*_\Omega) + \alpha_t\left(\mT_\Omega(Q_t)-\mT_\Omega(Q^*_\Omega)\right) + \alpha_t\rho^*_\Omega\ones$, and the constant term vanishes in the span seminorm. The triangle inequality and the non-expansiveness of $\mT_\Omega$ in $\|\cdot\|_\spn$ (\Cref{cor:T_omega_props}) then give $e_{t+1} \leq (1-\alpha_t)e_t + \alpha_t e_t = e_t$.
\end{proof}

\begin{definition}[Sub-blocks and global blocks]
    \label{def:sub_blocks}
    Let $\kappa_0 > 0$ be a fixed sub-block budget. We define recursively the sub-blocks $J_n = \{t_n, \ldots, t_{n+1}-1\}$ by $t_0 = 0$ and
    \begin{equation}
        \label{eq:subblocks}
        t_{n+1} = \min\left\{t > t_n \Big| \sum_{k=t_n}^{t-1} \alpha_k \geq \kappa_0\right\},
    \end{equation}
    and the $i$-th global block as the chronological concatenation of $m$ consecutive sub-blocks,
    \begin{equation}
        \label{eq:block}
        \mB_i = \bigcup_{\ell=0}^{m-1} J_{im+\ell}, \qquad b_i := \min \mB_i = t_{im}.
    \end{equation}
    Note that these definitions are valid under \Cref{assumption:lr}.
    
    We write $\Phi^{(i)} := F_{b_{i+1}-1}\circ\ldots\circ F_{b_i}$, so that $Q_{b_{i+1}} = \Phi^{(i)}(Q_{b_i})$ and, by \Cref{lem:reference_trajectory_synchronous}, $Q^*_{b_{i+1}} = \Phi^{(i)}(Q^*_{b_i})$.
\end{definition}

\begin{lemma}
    \label{lem:sum_block}
    Under \Cref{assumption:lr}, for any sub-block $J_n$ we have $\sum_{k \in J_n} \alpha_k < \kappa_0 + \alpha_0$, and consequently, for any global block $\mB_i$, $\sum_{k \in \mB_i} \alpha_k < m(\kappa_0 + \alpha_0)$.
\end{lemma}

\begin{proof}
    This follows from the minimality in the definition of a sub-block~\eqref{eq:subblocks} and \Cref{assumption:lr} on the step sizes. A global block is the chronological concatenation of $m$ consecutive sub-blocks, which gives the second result.
\end{proof}

\begin{lemma}[Boundedness along the frozen flow]
    \label{lem:frozen_flow_bounded_synchronous}
    Let $\bar{R}_i := \|Q^*_\Omega\|_\asp + e_{b_i}$. The sequence $(\bar{R}_i)_{i\geq0}$ is non-increasing, and for every $\tilde{Q} \in [Q_{b_i}, Q^*_{b_i}]$ and every $k \in \mB_i$, the frozen-flow iterate $\tilde{Q}_k := F_{k-1}\circ\ldots\circ F_{b_i}(\tilde{Q})$ (with $\tilde{Q}_{b_i} := \tilde{Q}$) satisfies $\|\tilde{Q}_k\|_\asp \leq \bar{R}_i$.
\end{lemma}

\begin{proof}
    The monotonicity of $(\bar{R}_i)_i$ is that of $(e_t)_t$ (\Cref{lem:q_boundedness_synchronous_full_knowledge}).

    Let now $\tilde{Q} = (1-u)Q_{b_i} + uQ^*_{b_i}$ for some $u\in[0,1]$. Since $Q^*_{b_i} = Q^*_\Omega+c_{b_i}\ones$, we get $\tilde{Q}-Q^*_\Omega = (1-u)(Q_{b_i}-Q^*_\Omega) + uc_{b_i}\ones$, and the constant is annihilated by the span seminorm, so $\|\tilde{Q}-Q^*_\Omega\|_\spn = (1-u)e_{b_i} \leq e_{b_i}$. The sequence $(\tilde{Q}_k)_{k\geq b_i}$ follows the update~\eqref{eq:sync_fk_update} started at $\tilde{Q}$, so \Cref{lem:q_boundedness_synchronous_full_knowledge} applied to it, whose proof holds for any starting time, gives $\|\tilde{Q}_k-Q^*_\Omega\|_\spn \leq \|\tilde{Q}-Q^*_\Omega\|_\spn \leq e_{b_i}$, whence $\|\tilde{Q}_k\|_\asp \leq \|Q^*_\Omega\|_\asp + \|\tilde{Q}_k-Q^*_\Omega\|_\spn \leq \bar{R}_i$.
\end{proof}

\begin{lemma}[Contraction of products of at least $m$ regularized greedy kernels]
    \label{lem:product_doeblin_synchronous}
    Let $R\geq0$, let $(U_k)_{k\in\mathcal{K}}$ be a finite family of $Q$-functions with $\|U_k\|_\asp\leq R$ for every $k$, and let $K_k:=K_{\pi_{\Omega,U_k}}$, i.e.\ $K_k(s',a'|s,a)=p(s'|s,a)\,\pi_{\Omega,U_k}(a'|s')$. For a subset $\mathcal{C}=\{k_1<\ldots<k_n\}\subseteq\mathcal{K}$, let $K_{\mathcal{C}}:=K_{k_n}\cdots K_{k_1}$ be the associated ordered product. Then $K_{\mathcal{C}}$ is row-stochastic and:
    \begin{enumerate}[label=(\roman*)]
        \item if $n\geq m$, then $\|K_{\mathcal{C}}\|_{\spn,\op}\leq1-\tilde{\beta}_m(R)$, where $\tilde{\beta}_m(R):=\beta_m\,\gamma_\Omega(R)^{m-1}$;
        \item if $n<m$, then $\|K_{\mathcal{C}}\|_{\spn,\op}\leq1$.
    \end{enumerate}
\end{lemma}

\begin{proof}
    Each $K_k$ is row-stochastic, hence so is $K_{\mathcal{C}}$, and (ii) follows from \Cref{lem:row-stochastic-span}.

    Assume $n\geq m$. With \Cref{convention:product_order}, the Markov chain associated with $K_{\mathcal{C}}=K_{k_n}\cdots K_{k_1}$ applies the factors in the order $K_{k_n},K_{k_{n-1}},\dots,K_{k_1}$. We therefore split off its \emph{last} $m$ steps in chain order, i.e.\ the $m$ \emph{oldest} algorithmic indices: $K_{\mathcal{C}}=HB$ with $H:=K_{k_n}\cdots K_{k_{m+1}}$ and $B:=K_{k_m}\cdots K_{k_1}$, and $H=I$ when $n=m$. Since $\|U_{k_\ell}\|_\asp\leq R$ for each $\ell$, \Cref{def:mixing_objects}.(iii) gives $\pi_{\Omega,U_{k_\ell}}\geq\gamma_\Omega(R)\mu$, so \Cref{lem:terminal_action}, applied to $B$, whose $m-1$ intermediate policies are $\pi_{\Omega,U_{k_m}},\dots,\pi_{\Omega,U_{k_2}}$, gives $\alpha(B)\leq1-\beta_m\gamma_\Omega(R)^{m-1}=1-\tilde{\beta}_m(R)$. Finally, $H$ is row-stochastic, so \Cref{lem:row-stochastic-span} and the submultiplicativity~\eqref{eq:sub-multiplicativity} give $\|K_{\mathcal{C}}\|_{\spn,\op}\leq\|H\|_{\spn,\op}\,\|B\|_{\spn,\op}\leq\alpha(B)$.
\end{proof}

It remains to count the effective updates within a global block. This is where the sub-block structure is used, and a plain independence argument suffices: no concentration inequality is required (a Chernoff bound gives a sharper estimate for large $m$, see \Cref{rem:chernoff_alternative}).

\begin{lemma}[Effective updates within a global block]
    \label{lem:pickup_synchronous}
    Let $\mB_i$ be the $i$-th global block of \Cref{def:sub_blocks}, let $(X_k)_{k \in \mB_i}$ be auxiliary independent Bernoulli random variables with $X_k \sim \mathcal{B}(\alpha_k)$ and let $N_i := \sum_{k\in \mB_i} X_k$. Under \Cref{assumption:lr},
    \begin{equation}
        \label{eq:pickup_synchronous}
        \bP(N_i \geq m) \geq p_m := \left(1-e^{-\kappa_0}\right)^m.
    \end{equation}
\end{lemma}

\begin{proof}
    For each of the $m$ sub-blocks $J_n \subseteq \mB_i$, let $\mathcal{U}_n := \bigcup_{k\in J_n}\{X_k = 1\}$ be the event that at least one effective update occurs in $J_n$. By independence of the $X_k$ and $1-u \leq e^{-u}$,
    \begin{equation*}
        \bP(\mathcal{U}_n) = 1 - \prod_{k\in J_n}(1-\alpha_k) \geq 1 - \exp\left(-\sum_{k\in J_n}\alpha_k\right) \geq 1 - e^{-\kappa_0},
    \end{equation*}
    the last inequality by the definition~\eqref{eq:subblocks} of the sub-block. The sub-blocks of $\mB_i$ have disjoint index sets, so the events $\mathcal{U}_{im}, \ldots, \mathcal{U}_{im+m-1}$ are independent, and $\bigcap_{\ell=0}^{m-1} \mathcal{U}_{im+\ell} \subseteq \{N_i \geq m\}$. Hence $\bP(N_i \geq m) \geq \prod_{\ell=0}^{m-1}\bP(\mathcal{U}_{im+\ell}) \geq (1-e^{-\kappa_0})^m$.
\end{proof}

\begin{lemma}[Jacobian bound over a global block]
    \label{lem:jacobian_bound_synchronous}
    Under \Cref{assumption:lr}, for every global block $\mB_i$ and every $\tilde{Q}\in[Q_{b_i},Q^*_{b_i}]$,
    \begin{equation*}
        \|D\Phi^{(i)}(\tilde{Q})\|_{\spn,\op} \leq 1 - p_m\,\tilde{\beta}_m(\bar{R}_i).
    \end{equation*}
\end{lemma}

\begin{proof}
    By the chain rule and \Cref{cor:T_omega_props}, the Jacobian of $\Phi^{(i)}$ at $\tilde{Q}$ is
    \begin{equation*}
        D\Phi^{(i)}(\tilde{Q}) = \prod_{k\in \mB_i} \left((1-\alpha_{k})I + \alpha_{k} K_{k}\right), \qquad K_k:=K_{\pi_{\Omega,\tilde{Q}_k}},\quad \tilde{Q}_k := F_{k-1}\circ\ldots\circ F_{b_i}(\tilde{Q}),
    \end{equation*}
    the product being written in decreasing order of $k$. By \Cref{lem:frozen_flow_bounded_synchronous}, $\|\tilde{Q}_k\|_\asp \leq \bar{R}_i$ for every $k\in \mB_i$, so the family $(K_k)_{k\in \mB_i}$ satisfies the hypotheses of \Cref{lem:product_doeblin_synchronous} with $U_k=\tilde{Q}_k$ and $R=\bar{R}_i$.

    We introduce a probabilistic argument to bound the Jacobian. Let $(X_k)_{k\in \mB_i}$ be the auxiliary Bernoulli variables of \Cref{lem:pickup_synchronous}, realized as the coordinate maps of the finite product space $\{0,1\}^{\mB_i}$ endowed with $\bigotimes_{k\in\mB_i}\mathcal{B}(\alpha_k)$, which is well defined since $\alpha_k\in(0,1]$ by \Cref{assumption:lr}. The flow~\eqref{eq:sync_fk_update} being deterministic, this finite space carries all the randomness of the argument. We define $Y_k = (1-X_k)I + X_k K_{k}$, so that $(1-\alpha_{k})I + \alpha_{k} K_{k} = \bE[Y_k]$. By independence of the $Y_k$, we have $\prod_{k\in \mB_i} \bE[Y_k] = \bE\left[\prod_{k\in \mB_i} Y_k\right]$, and the triangle inequality gives
    \begin{equation*}
        \|D\Phi^{(i)}(\tilde{Q})\|_{\spn,\op} = \left\|\bE\left[\prod_{k\in \mB_i} Y_k\right]\right\|_{\spn,\op} \leq \bE\left[\left\|\prod_{k\in \mB_i} Y_k\right\|_{\spn,\op}\right].
    \end{equation*}
    Let $\mathcal{C} := \{k \in \mB_i \mid X_k = 1\}$, so that $N_i = |\mathcal{C}|$. Since the factors equal to $I$ can be removed from the product, we have the pathwise identity $\prod_{k\in \mB_i} Y_k = K_{\mathcal{C}}$ with the notation of \Cref{lem:product_doeblin_synchronous}.

    On the event $\{N_i \geq m\}$, point (i) of \Cref{lem:product_doeblin_synchronous} gives $\|K_{\mathcal{C}}\|_{\spn,\op} \leq 1 - \tilde{\beta}_m(\bar{R}_i)$, and on the event $\{N_i < m\}$, point (ii) gives $\|K_{\mathcal{C}}\|_{\spn,\op} \leq 1$; hence $\|K_{\mathcal{C}}\|_{\spn,\op} \leq 1-\tilde{\beta}_m(\bar{R}_i)\indicator_{\{N_i\geq m\}}$ pointwise. Taking expectations,
    \begin{equation*}
        \|D\Phi^{(i)}(\tilde{Q})\|_{\spn,\op} \leq 1 - \bP(N_i \geq m)\,\tilde{\beta}_m(\bar{R}_i) \leq 1 - p_m\tilde{\beta}_m(\bar{R}_i),
    \end{equation*}
    the last inequality by \Cref{lem:pickup_synchronous}.
\end{proof}

\begin{lemma}[Per-block contraction]
    \label{lem:block_contraction_synchronous}
    Under \Cref{assumption:lr}, for every global block $\mB_i$,
    \begin{equation*}
        e_{b_{i+1}} \leq \left(1 - p_m\,\tilde{\beta}_m(\bar{R}_i)\right) e_{b_i}.
    \end{equation*}
\end{lemma}

\begin{proof}
    By \Cref{lem:reference_trajectory_synchronous}, $e_{b_{i+1}} = \|\Phi^{(i)}(Q_{b_i}) - \Phi^{(i)}(Q^*_{b_i})\|_\spn$. The map $\Phi^{(i)}$ is continuously differentiable (\Cref{cor:T_omega_props}) and all its Jacobians $D\Phi^{(i)}(\tilde{Q})$, $\tilde{Q}\in[Q_{b_i},Q^*_{b_i}]$, are products of row-stochastic matrices, hence row-stochastic, so the mean value inequality in the span seminorm (\Cref{prop:mvi-span}) applies and gives $e_{b_{i+1}} \leq \sup_{\tilde{Q}\in[Q_{b_i},Q^*_{b_i}]} \|D\Phi^{(i)}(\tilde{Q})\|_{\spn,\op}\,\|Q_{b_i} - Q^*_{b_i}\|_\spn$. \Cref{lem:jacobian_bound_synchronous} bounds the supremum by $1-p_m\tilde{\beta}_m(\bar{R}_i)$, while $\|Q_{b_i}-Q^*_{b_i}\|_\spn = e_{b_i}$.
\end{proof}

\begin{corollary}[Full-knowledge updates for $m=1$]
    \label{cor:full_knowledge_m_1}
    Assume $m=1$ and $\alpha_t\in(0,1]$ for every $t$. For any initialization $Q_0$, the iterates of~\eqref{eq:sync_fk_update} satisfy
    \begin{equation*}
        e_T \leq \prod_{k=0}^{T-1} (1-\alpha_k\beta_1)\,e_0.
    \end{equation*}
    In particular, $Q$-iteration ($\alpha_t\equiv1$) contracts at rate $1-\beta_1$ per sweep, and reaches $e_T\leq\varepsilon$ after $O\big(|\mS|^2|\mA|\,\beta_1^{-1}\log(e_0/\varepsilon)\big)$ model queries.
\end{corollary}

\begin{proof}
    Let $\Psi_T := F_{T-1}\circ\ldots\circ F_0$ and $\tilde{Q}\in[Q_0,Q^*_0]$. As in \Cref{lem:jacobian_bound_synchronous}, $D\Psi_T(\tilde{Q}) = \prod_{k=0}^{T-1}\left((1-\alpha_k)I + \alpha_k K_k\right)$. By \Cref{lem:terminal_action} with $n=1$, each factor satisfies
    \begin{equation*}
        \left((1-\alpha_k)I + \alpha_k K_k\right)(s',a'|s,a) \geq \alpha_k K_k(s',a'|s,a) \geq \alpha_k\beta_1\,\nu(s')\,\pi_{\Omega,\tilde{Q}_k}(a'|s'),
    \end{equation*}
    which is a Doeblin condition with parameter $\alpha_k\beta_1$. \Cref{lem:dobrushin_doeblin,lem:row-stochastic-span} give $\left\|(1-\alpha_k)I + \alpha_k K_k\right\|_{\spn,\op} \leq 1-\alpha_k\beta_1$, and the submultiplicativity of the operator span seminorm together with \Cref{prop:mvi-span} and \Cref{lem:reference_trajectory_synchronous} conclude. For $\alpha_t\equiv1$, the update~\eqref{eq:sync_fk_update} is $Q$-iteration without the centering projection, which is invisible to $\|\cdot\|_\spn$, so $e_T\leq(1-\beta_1)^Te_0\leq\varepsilon$ for $T\geq\beta_1^{-1}\log(e_0/\varepsilon)$, each sweep using $|\mS|^2|\mA|$ queries.
\end{proof}

\begin{lemma}[Number of completed blocks]
    \label{lem:block_count_synchronous}
    Under \Cref{assumption:lr}, let $I(T) := \max\{i \mid b_i \leq T\}$ be the number of global blocks completed within the first $T$ updates. Then $I(T)\geq \left\lfloor\frac{\sum_{t=0}^{T-1}\alpha_t}{m(\kappa_0+\alpha_0)}\right\rfloor$.
\end{lemma}

\begin{proof}
    By definition of $I(T)$ we have $b_{I(T)} \leq T < b_{I(T)+1}$, and by \Cref{lem:sum_block} each global block consumes at most $m(\kappa_0+\alpha_0)$. The step sizes being non-negative,
    \begin{equation*}
        \sum_{k=0}^{T-1}\alpha_k \leq \sum_{k=0}^{b_{I(T)+1}-1}\alpha_k < (I(T)+1)\,m(\kappa_0+\alpha_0), \qquad \text{i.e.} \qquad I(T) > \frac{\sum_{k=0}^{T-1}\alpha_k}{m(\kappa_0+\alpha_0)} - 1,
    \end{equation*}
    and $I(T)$ being an integer gives the claim.
\end{proof}

\begin{remark}[Radius-adaptive rate]
    \label{rem:per_block_radius_synchronous}
    Chaining \Cref{lem:block_contraction_synchronous} over the completed blocks gives the radius-adaptive bound $e_{b_{I(T)}} \leq \prod_{i<I(T)}\left(1-p_m\beta_m\gamma_\Omega(\bar{R}_i)^{m-1}\right)e_0$, in which the per-block rate improves as the error decreases, since $\bar{R}_i = \|Q^*_\Omega\|_\asp + e_{b_i}$ is non-increasing. The uniform constant $\tilde{\beta}_m(\bar{R}_0)$ is its worst case, attained at the first block. Asymptotically, $\bar{R}_i \downarrow \|Q^*_\Omega\|_\asp$, so the rate approaches $1-p_m\beta_m\gamma_\Omega(\|Q^*_\Omega\|_\asp)^{m-1}$: the local geometry of the fixed point, not the initialization, governs the asymptotic speed. The bounds of \Cref{thm:sync-stoch-R-app,thm:ep_main} use instead the worst-case constant $\bar{\beta}(R)$ of \Cref{def:localized-radii} on the whole ball of radius $R$; making them radius-adaptive would require shrinking the radius along the run.
\end{remark}

\begin{remark}[An alternative bound on $\bP(N_i \geq m)$ through a Chernoff inequality]
    \label{rem:chernoff_alternative}
    Each global block has a total budget $\bar\alpha_i := \sum_{k\in \mB_i}\alpha_k \geq m\kappa_0$, so for $\kappa_0 > 1$ the multiplicative Chernoff inequality for the lower tail \citep[Theorem~4.5]{mitzenmacher2005probability} gives
    \begin{equation*}
        \bP(N_i < m) \leq \exp\left(-\frac{(\bar\alpha_i-m)^2}{2\bar\alpha_i}\right) \leq \exp\left(-\frac{(\kappa_0-1)^2}{2\kappa_0}m\right),
    \end{equation*}
    the last inequality because $u \mapsto \frac{(u-m)^2}{2u}$ has derivative $\frac{1}{2}\left(1-\frac{m^2}{u^2}\right) > 0$ on $(m,\infty)$ and is therefore minimized at $u = m\kappa_0$. This variant requires $\kappa_0 > 1$ but is much better for large $m$, since it tends to $1$ whereas $(1-e^{-\kappa_0})^m$ decays exponentially in $m$. Both bounds being valid, \Cref{lem:pickup_synchronous} holds with
    \begin{equation*}
        p_m = \max\left\{\left(1-e^{-\kappa_0}\right)^m,\ \left(1-\exp\left(-\frac{(\kappa_0-1)^2}{2\kappa_0}m\right)\right)\indicator\{\kappa_0>1\}\right\}.
    \end{equation*}
    The proof of \Cref{thm:sync-stoch-main-detailed} uses $\kappa_0 = 4$, for which $p_m \geq 1-e^{-9/8} \geq \frac12$.
\end{remark}

\subsubsection{Setting and localized constants}
\label{subsubsec:sync_stochastic}

We now study \Cref{alg:synchronous_regularized_q_learning}, the synchronous $\Omega$-regularized $Q$-learning algorithm, and write $e_t:=\|Q_t-Q^*_\Omega\|_\spn$ for its error. We define the filtration $\mF_t := \sigma\left\{Q_0,\ \left(s'_k(s,a)\right)_{(s,a)\in\mS\times\mA} : k<t\right\}$ and write the stochastic update as
\begin{equation}
    \label{eq:sync-stoch-update}
    Q_{t+1} = Q_t + \alpha_t\left(\mT_\Omega(Q_t)-Q_t+W_{t+1}\right),
\end{equation}
where for every pair $(s,a)$ the next state $s'_t(s,a)\sim p(\cdot|s,a)$ is sampled independently and
\begin{equation}
    \label{eq:sync-noise}
    W_{t+1}(s,a) := \Omega^*(Q_t)\left(s'_t(s,a)\right) - \sum_{s'}p(s'|s,a)\,\Omega^*(Q_t)(s').
\end{equation}

By construction $W_{t+1}$ is $\mF_{t+1}$-measurable and $\bE[W_{t+1}\mid\mF_t]=0$. By translation equivariance of $\Omega^*$ (\Cref{lem:omega_star_props}.(iii)), the centering projection of \Cref{alg:synchronous_regularized_q_learning} only shifts the iterates by a constant, which $\|\cdot\|_\spn$ and $\|\cdot\|_\asp$ annihilate, so we omit it; its only role is to upgrade span convergence to $\|\cdot\|_\infty$ convergence (\Cref{rem:span_to_sup}). We reuse the sub-blocks $J_n$ and global blocks $\mB_i$ of \Cref{def:sub_blocks}; the step sizes being deterministic, the block boundaries $b_i$, the number $I(T)$ of completed blocks, the bound $I_{\max}:=\left\lfloor\frac{1}{m\kappa_0}\sum_{t<T}\alpha_t\right\rfloor$ and the last incomplete block $\mB_{\mathrm{res}}:=\llbracket b_{I(T)},T-1\rrbracket$ are all deterministic.

\begin{definition}[Localized radii, exit time and localized constants]
    \label{def:localized-radii}
    For a radius $R\geq e_0$, define the exit time
    \begin{equation}
        \label{eq:exit-time}
        \sigma_R := \inf\left\{t\in\bN\ \middle|\ e_t>R\right\} \in \bN\cup\{+\infty\},
    \end{equation}
    the \emph{domination radius} and the \emph{noise radius}
    \begin{equation*}
        A_R := \|Q^*_\Omega\|_\asp + R, \qquad C_R := \|Q^*_\Omega\|_\spn + R,
    \end{equation*}
    and the localized constants
    \begin{equation}
        \label{eq:localized-constants}
        \tilde{\beta}_m(A_R) := \beta_m\,\gamma_\Omega(A_R)^{m-1}, \quad \bar{\beta}(R) := p_m\,\tilde{\beta}_m(A_R) \in (0,1), \quad \bar{w}(R) := C_R + \left\|\Omega^*(0)\right\|_{\mS,\spn},
    \end{equation}
    where $p_m$ is the constant of \Cref{lem:pickup_synchronous} (see also \Cref{rem:chernoff_alternative}).
\end{definition}

\begin{remark}
    \label{rem:exit-time-facts}
    (i) $\sigma_R$ is an $(\mF_t)$-stopping time and $\sigma_R\geq1$ since $e_0\leq R$. (ii) By minimality, $t<\sigma_R$ implies $e_t\leq R$, hence $\|Q_t\|_\spn\leq C_R$ and, since $\|\cdot\|_\asp\leq\|\cdot\|_\spn$, also $\|Q_t\|_\asp\leq\|Q^*_\Omega\|_\asp+\|Q_t-Q^*_\Omega\|_\asp\leq A_R$. (iii) $\gamma_\Omega$ is non-increasing (\Cref{def:mixing_objects}.(iii)), so $R\mapsto\bar{\beta}(R)$ is non-increasing and any contraction estimate valid at a radius $A\leq A_R$ holds a fortiori with $\bar{\beta}(R)$; moreover $\bar{\beta}(R)>0$ by \Cref{assumption:positive-domination}. (iv) The localization serves two purposes: it bounds the noise scale by $\bar{w}(R)$, which grows linearly in $R$ through the sub-affinity of $\Omega^*$ (\Cref{lem:omega_star_props}.(vii)), and, since $\gamma_\Omega$ may vanish at infinity, it keeps the contraction rate at least $\bar{\beta}(R)$.
\end{remark}

\subsubsection{The noiseless flow before the exit time}

For each global block $\mB_i$ we define the \emph{virtual sequence} $(\tilde{Q}_k)_{k\geq b_i}$ by
\begin{equation}
    \label{eq:virtual-sequence}
    \tilde{Q}_{b_i} := Q_{b_i}, \qquad \tilde{Q}_{k+1} := \tilde{Q}_k+\alpha_k\left(\mT_\Omega(\tilde{Q}_k)-\tilde{Q}_k\right),
\end{equation}
i.e.\ the full-knowledge synchronous update~\eqref{eq:sync_fk_update} started at $Q_{b_i}$; for $i=I(T)$, it is defined on the last incomplete block $\mB_{\mathrm{res}}$. It is a deterministic function of $Q_{b_i}$, hence $\mF_{b_i}$-measurable.

\begin{lemma}[Contraction of the virtual flow at radius $R$]
    \label{lem:virtual-flow-R}
    Let $i\in\bN$ with $e_{b_i}\leq R$. Then, pathwise:
    \begin{enumerate}[label=(\roman*)]
        \item $\|\tilde{Q}_k-Q^*_\Omega\|_\spn \leq e_{b_i}$ for every $k\geq b_i$;
        \item $\|\tilde{Q}_{b_{i+1}}-Q^*_\Omega\|_\spn \leq \left(1-\bar{\beta}(R)\right)e_{b_i}$.
    \end{enumerate}
\end{lemma}

\begin{proof}
    The virtual sequence is exactly the full-knowledge update~\eqref{eq:sync_fk_update} started at $Q_{b_i}$, so (i) is \Cref{lem:q_boundedness_synchronous_full_knowledge} applied to that sequence, and (ii) is \Cref{lem:block_contraction_synchronous} applied within the single block $\mB_i$, which gives the factor $1-p_m\tilde{\beta}_m(\bar{R}_i)$ where $\bar{R}_i=\|Q^*_\Omega\|_\asp+e_{b_i}$ is the block radius of \Cref{lem:frozen_flow_bounded_synchronous}. It remains to check $\bar{R}_i\leq A_R$, which follows from $e_{b_i}\leq R$; $\gamma_\Omega$ being non-increasing, $\tilde{\beta}_m(\bar{R}_i)\geq\tilde{\beta}_m(A_R)$, i.e.\ $1-p_m\tilde{\beta}_m(\bar{R}_i)\leq1-\bar{\beta}(R)$.
\end{proof}

\subsubsection{Localized noise control}

\begin{lemma}[Stopped noise]
    \label{lem:stopped-noise}
    Define $\widehat{W}_{t+1} := \indicator_{\{t<\sigma_R\}}W_{t+1}$. Then $\widehat{W}_{t+1}$ is $\mF_{t+1}$-measurable, $\bE\left[\widehat{W}_{t+1}\mid\mF_t\right]=0$, and each entry $\widehat{W}_{t+1}(s,a)$ lies in an $\mF_t$-measurable interval containing $0$ and of length at most $\bar{w}(R)$; in particular $\left\|\widehat{W}_{t+1}\right\|_\infty\leq\bar{w}(R)$. Moreover $\widehat{W}_{t+1}=W_{t+1}$ on $\{t<\sigma_R\}$.
\end{lemma}

\begin{proof}
    Since $\sigma_R$ is a stopping time, $\{t<\sigma_R\}=\{\sigma_R\leq t\}^c\in\mF_t$, so $\bE[\widehat{W}_{t+1}\mid\mF_t]=\indicator_{\{t<\sigma_R\}}\bE[W_{t+1}\mid\mF_t]=0$. Conditionally on $\mF_t$, $\Omega^*(Q_t)$ is fixed, and $W_{t+1}(s,a)$ is a value of $\Omega^*(Q_t)$ minus the $\mF_t$-measurable convex combination $\sum_{s'}p(s'|s,a)\,\Omega^*(Q_t)(s')$; it therefore lies in an $\mF_t$-measurable interval containing $0$ and of length $\|\Omega^*(Q_t)\|_{\mS,\spn}$, and so does $\widehat{W}_{t+1}(s,a)$, with length $\indicator_{\{t<\sigma_R\}}\|\Omega^*(Q_t)\|_{\mS,\spn}$. On $\{t<\sigma_R\}$ we have $\|Q_t\|_\spn\leq C_R$ by \Cref{rem:exit-time-facts}.(ii), and \Cref{lem:omega_star_props}.(vii) gives $\|\Omega^*(Q_t)\|_{\mS,\spn}\leq\|Q_t\|_\spn+\|\Omega^*(0)\|_{\mS,\spn}\leq C_R+\|\Omega^*(0)\|_{\mS,\spn}=\bar{w}(R)$. Off the event, $\widehat{W}_{t+1}=0$.
\end{proof}

For an interval $\mathcal{K}=\llbracket t_a,t_b\rrbracket$ we define the stopped noise martingale $\widehat{U}^{(\mathcal{K})}_{t_a}:=0$, $\widehat{U}^{(\mathcal{K})}_{k+1}:=\widehat{U}^{(\mathcal{K})}_k+\alpha_k\widehat{W}_{k+1}$ for $k\in\mathcal{K}$, we write $v_\mathcal{K}:=\sum_{k\in\mathcal{K}}\alpha_k^2$, and we set
\begin{align}
        \bar{U}_\mathcal{K}(\delta') &:= \bar{w}(R)\sqrt{\frac{1}{2}v_\mathcal{K}\log\left(\frac{2|\mS||\mA|}{\delta'}\right)},\\
    \varepsilon_\mathcal{K}(\delta') &:= 2\left(1+2\sum_{k\in\mathcal{K}}\alpha_k\right)\bar{U}_\mathcal{K}(\delta') \ \leq\ 2\left(1+2m(\kappa_0+\alpha_0)\right)\bar{U}_\mathcal{K}(\delta'),
    \label{eq:eps-K}
\end{align}
the last inequality by \Cref{lem:sum_block} when $\mathcal{K}$ is contained in a global block.

\begin{lemma}[Good event]
    \label{lem:good-event}
    Let $\delta\in(0,1)$ and $\delta' := \frac{\delta}{I_{\max}+1}$. The event
    \begin{equation}
        \label{eq:good-event}
        G_R(\delta) := \bigcap_{i=0}^{I(T)-1}\left\{\max_{b_i\leq k\leq b_{i+1}}\left\|\widehat{U}^{(\mB_i)}_k\right\|_\infty\leq\bar{U}_{\mB_i}(\delta')\right\} \cap \left\{\max_{b_{I(T)}\leq k\leq T}\left\|\widehat{U}^{(\mB_{\mathrm{res}})}_k\right\|_\infty\leq\bar{U}_{\mB_{\mathrm{res}}}(\delta')\right\}
    \end{equation}
    satisfies $\bP\left(G_R(\delta)\right)\geq1-\delta$.
\end{lemma}

\begin{proof}
    Fix an interval $\mathcal{K}$ among $\mB_0,\dots,\mB_{I(T)-1},\mB_{\mathrm{res}}$ and a pair $(s,a)$. By \Cref{lem:stopped-noise}, the scalar process $\left(\widehat{U}^{(\mathcal{K})}_k(s,a)\right)_k$ is an $\mF_k$-martingale whose increment $\alpha_k\widehat{W}_{k+1}(s,a)$ lies in an $\mF_k$-measurable interval of length at most $\alpha_k\bar{w}(R)$, so $\sum_{k\in\mathcal{K}}\left(\alpha_k\bar{w}(R)\right)^2=\bar{w}(R)^2v_\mathcal{K}$ and the maximal Azuma--Hoeffding inequality (\Cref{lem:azuma_range}) bounds the failure probability of each brace, for each pair, by $\frac{\delta'}{|\mS||\mA|}$. A union bound over the $|\mS||\mA|$ pairs and the $I(T)+1\leq I_{\max}+1$ intervals concludes.
\end{proof}

\subsubsection{Tracking before the exit time}

\begin{lemma}[Tracking up to the exit time]
    \label{lem:tracking-R}
    Work on the event $G_R(\delta)$. Let $i\leq I(T)$, let $\mathcal{K}:=\mB_i$ if $i<I(T)$ and $\mathcal{K}:=\mB_{\mathrm{res}}$ if $i=I(T)$, write $t_a:=b_i$ and $t_b+1:=b_{i+1}$ (resp.\ $T$), and let $\tilde{Q}$ be the virtual sequence started at $Q_{t_a}$. Then, for every $k$ with $t_a\leq k\leq\min\{t_b+1,\sigma_R\}$, the tracking error $E_k:=Q_k-\tilde{Q}_k$ satisfies $\|E_k\|_\spn\leq\varepsilon_\mathcal{K}(\delta')$.
\end{lemma}

\begin{proof}
    If $\sigma_R<t_a$ the range is empty, and if $\sigma_R=t_a$ it reduces to $k=t_a$, where $E_{t_a}=0$. Otherwise, for $t_a\leq k<\min\{t_b+1,\sigma_R\}$ the true update~\eqref{eq:sync-stoch-update} occurs and $W_{k+1}=\widehat{W}_{k+1}$ by \Cref{lem:stopped-noise}. Subtracting~\eqref{eq:virtual-sequence} from~\eqref{eq:sync-stoch-update} and using the average-path identity of \Cref{lem:mixing-properties-of-regularization}.(i) at the pair $(Q_k,\tilde{Q}_k)$, i.e.\ $\mT_\Omega(Q_k)-\mT_\Omega(\tilde{Q}_k)=\bar{K}_kE_k$ with $\bar{K}_k := K_{\bar{\pi}_{\Omega,[Q_k,\tilde{Q}_k]}}$,
    \begin{equation}
        \label{eq:sync-err-rec-R}
        E_{k+1} = (1-\alpha_k)E_k+\alpha_k\bar{K}_kE_k+\alpha_k\widehat{W}_{k+1} = M_kE_k+\alpha_k\widehat{W}_{k+1},
    \end{equation}
    where $M_k := (1-\alpha_k)I+\alpha_k\bar{K}_k$ is row-stochastic as a convex combination of row-stochastic matrices.

    Decompose $E_k = \widehat{R}_k+\widehat{U}^{(\mathcal{K})}_k$ with $\widehat{R}_{t_a}=0$. The noise terms in~\eqref{eq:sync-err-rec-R} cancel with the definition of $\widehat{U}^{(\mathcal{K})}_{k+1}$, so that $\widehat{R}_{k+1} = M_k\widehat{R}_k+\alpha_k\left(\bar{K}_k-I\right)\widehat{U}^{(\mathcal{K})}_k$. Since $\bar{K}_k$ is row-stochastic, every entry of $\bar{K}_k\widehat{U}^{(\mathcal{K})}_k$ is a convex combination of entries of $\widehat{U}^{(\mathcal{K})}_k$, hence lies in $\left[\min\widehat{U}^{(\mathcal{K})}_k,\max\widehat{U}^{(\mathcal{K})}_k\right]$, and so does every entry of $\widehat{U}^{(\mathcal{K})}_k$; therefore every entry of $\left(\bar{K}_k-I\right)\widehat{U}^{(\mathcal{K})}_k$ lies in $\left[-\|\widehat{U}^{(\mathcal{K})}_k\|_\spn,\|\widehat{U}^{(\mathcal{K})}_k\|_\spn\right]$ and $\left\|\alpha_k\left(\bar{K}_k-I\right)\widehat{U}^{(\mathcal{K})}_k\right\|_\spn \leq 2\alpha_k\left\|\widehat{U}^{(\mathcal{K})}_k\right\|_\spn \leq 4\alpha_k\left\|\widehat{U}^{(\mathcal{K})}_k\right\|_\infty$. By \Cref{lem:row-stochastic-span} the row-stochastic matrix $M_k$ is span non-expansive, so $\|\widehat{R}_{k+1}\|_\spn\leq\|\widehat{R}_k\|_\spn+4\alpha_k\|\widehat{U}^{(\mathcal{K})}_k\|_\infty$. Summing from $t_a$ and using the definition of $G_R(\delta)$, for every $k\in\llbracket t_a,\min\{t_b+1,\sigma_R\}\rrbracket$, $\|\widehat{R}_k\|_\spn \leq 4\left(\sum_{k'\in\mathcal{K}}\alpha_{k'}\right)\bar{U}_\mathcal{K}(\delta')$ and
    \begin{equation*}
        \|E_k\|_\spn \leq \|\widehat{R}_k\|_\spn+2\left\|\widehat{U}^{(\mathcal{K})}_k\right\|_\infty \leq 2\left(1+2\sum_{k'\in\mathcal{K}}\alpha_{k'}\right)\bar{U}_\mathcal{K}(\delta') = \varepsilon_\mathcal{K}(\delta').
    \end{equation*}
\end{proof}

\subsubsection{No escape before the horizon}
\label{subsec:no-escape}

\begin{proposition}[No escape]
    \label{prop:no-escape}
    Let $\delta\in(0,1)$, $R\geq e_0$, $\delta'=\frac{\delta}{I_{\max}+1}$ and
    \begin{equation*}
        \bar{\varepsilon}_R(\delta) := \max\left\{\max_{0\leq i<I(T)}\varepsilon_{\mB_i}(\delta'),\ \varepsilon_{\mB_{\mathrm{res}}}(\delta')\right\}.
    \end{equation*}
    Assume the \emph{radius admissibility condition}
    \begin{equation}
        \label{eq:R-admissible}
        \bar{\varepsilon}_R(\delta) \leq \frac{\bar{\beta}(R)}{\bar{\beta}(R)+1}\left(R-e_0\right).
    \end{equation}
    Then, on the event $G_R(\delta)$, we have $\sigma_R>T$, in particular $\sup_{0\leq t\leq T}e_t\leq R$, and for every $0\leq i\leq I(T)$,
    \begin{equation}
        \label{eq:block-recursion-R}
        e_{b_i} \leq (1-\bar{\beta}(R))^ie_0+\sum_{\ell=0}^{i-1}(1-\bar{\beta}(R))^{i-1-\ell}\varepsilon_{\mB_\ell}(\delta').
    \end{equation}
\end{proposition}

\begin{proof}
    Work on $G_R(\delta)$. We first record that~\eqref{eq:block-recursion-R} implies, bounding the geometric sum and using~\eqref{eq:R-admissible},
    \begin{equation}
        \label{eq:margin}
        e_{b_i} \leq e_0+\frac{\bar{\varepsilon}_R(\delta)}{\bar{\beta}(R)}, \qquad\text{hence}\qquad e_{b_i}+\bar{\varepsilon}_R(\delta) \leq e_0+\frac{\bar{\beta}(R)+1}{\bar{\beta}(R)}\bar{\varepsilon}_R(\delta) \leq R.
    \end{equation}

    We prove by induction on $i\in\llbracket0,I(T)\rrbracket$ that $\sigma_R>b_i$ and~\eqref{eq:block-recursion-R} holds at rank $i$.

    \emph{Base case.} $b_0=0$ and $\sigma_R\geq1$ by \Cref{rem:exit-time-facts}.(i);~\eqref{eq:block-recursion-R} at $i=0$ reads $e_0\leq e_0$.

    \emph{Inductive step.} Assume the claim at rank $i<I(T)$ and let $\tilde{Q}$ be the virtual sequence started at $Q_{b_i}$; by~\eqref{eq:margin}, $e_{b_i}\leq R$, so \Cref{lem:virtual-flow-R} applies to $\mB_i$. Suppose $\sigma_R\leq b_{i+1}$; since $\sigma_R>b_i$, we would have $t^\star:=\sigma_R\in\llbracket b_i+1,b_{i+1}\rrbracket$, and as $t^\star\leq\min\{b_{i+1},\sigma_R\}$, \Cref{lem:tracking-R} gives $\|Q_{t^\star}-\tilde{Q}_{t^\star}\|_\spn\leq\varepsilon_{\mB_i}(\delta')\leq\bar{\varepsilon}_R(\delta)$ while \Cref{lem:virtual-flow-R}.(i) gives $\|\tilde{Q}_{t^\star}-Q^*_\Omega\|_\spn\leq e_{b_i}$. Hence, by~\eqref{eq:margin}, $e_{\sigma_R} \leq e_{b_i}+\bar{\varepsilon}_R(\delta) \leq R$, contradicting the definition~\eqref{eq:exit-time} of $\sigma_R$. Therefore $\sigma_R>b_{i+1}$; in particular $b_{i+1}\leq\min\{b_{i+1},\sigma_R\}$, so \Cref{lem:tracking-R} at $k=b_{i+1}$ combined with \Cref{lem:virtual-flow-R}.(ii) yields $e_{b_{i+1}} \leq \|\tilde{Q}_{b_{i+1}}-Q^*_\Omega\|_\spn+\|Q_{b_{i+1}}-\tilde{Q}_{b_{i+1}}\|_\spn \leq (1-\bar{\beta}(R))e_{b_i}+\varepsilon_{\mB_i}(\delta')$, which, injected into~\eqref{eq:block-recursion-R} at rank $i$, gives it at rank $i+1$.

    \emph{Residual interval.} The last step gives $\sigma_R>b_{I(T)}$. If $\sigma_R\leq T$, the same argument on $\mB_{\mathrm{res}}$ --- \Cref{lem:tracking-R} with $\mathcal{K}=\mB_{\mathrm{res}}$ and \Cref{lem:virtual-flow-R}.(i) --- gives $e_{\sigma_R}\leq e_{b_{I(T)}}+\bar{\varepsilon}_R(\delta)\leq R$, again a contradiction. Hence $\sigma_R>T$ and the containment follows from the definition of $\sigma_R$.
\end{proof}

\subsubsection{Main result}

\begin{theorem}[Finite-time bound for $\Omega$-regularized synchronous $Q$-learning]
    \label{thm:sync-stoch-R-app}
    Let $\delta\in(0,1)$ and $T\geq1$. Under \Cref{assumption:lr}, let the sub-block budget $\kappa_0>0$ be arbitrary and let the radius $R\geq e_0$ satisfy the admissibility condition~\eqref{eq:R-admissible}. Then, with probability at least $1-\delta$, $\sup_{0\leq t\leq T}e_t \leq R$ and
    \begin{align*}
        e_T &\leq (1-\bar{\beta}(R))^{I(T)}e_0 + 2\left(1+2m(\kappa_0+\alpha_0)\right)\bar{w}(R)\sqrt{\frac{1}{2}\log\left(\frac{2|\mS||\mA|(I_{\max}+1)}{\delta}\right)}\\
        &\qquad \times\left(\sum_{i=0}^{I(T)-1}(1-\bar{\beta}(R))^{I(T)-1-i}\sqrt{\sum_{k\in \mB_i}\alpha_k^2}+\sqrt{\sum_{k\in\mB_{\mathrm{res}}}\alpha_k^2}\right),
    \end{align*}
    where $\bar{\beta}(R)$ and $\bar{w}(R)$ are the localized constants of \Cref{def:localized-radii}, $I(T)\geq\left\lfloor\frac{\sum_{t<T}\alpha_t}{m(\kappa_0+\alpha_0)}\right\rfloor$ is the number of completed blocks and $I_{\max}\geq I(T)$.
\end{theorem}

\begin{proof}
    Work on $G_R(\delta)$, which holds with probability at least $1-\delta$ by \Cref{lem:good-event}. \Cref{prop:no-escape} gives $\sigma_R>T$, so the ranges of \Cref{lem:tracking-R} cover every block entirely, as well as the residual block. Applying \Cref{lem:tracking-R} with $\mathcal{K}=\mB_{\mathrm{res}}$ at $k=T$, and \Cref{lem:virtual-flow-R}.(i) to the virtual sequence started at $Q_{b_{I(T)}}$, the triangle inequality gives $e_T \leq e_{b_{I(T)}}+\varepsilon_{\mB_{\mathrm{res}}}(\delta')$. Injecting~\eqref{eq:block-recursion-R} at $i=I(T)$ and the explicit form~\eqref{eq:eps-K} of the $\varepsilon_\mathcal{K}(\delta')$ yields the announced bound; the lower bound on $I(T)$ is \Cref{lem:block_count_synchronous}.
\end{proof}

\begin{corollary}[Constant step size]
    \label{cor:sync-const-step-R}
    Under the assumptions of \Cref{thm:sync-stoch-R-app}, let the step size be constant, $\alpha_t=\alpha\in(0,1)$, and set $C_\kappa := \sqrt{2}\left(1+2m(\kappa_0+\alpha)\right)$. If the radius $R\geq e_0$ satisfies
    \begin{equation}
        \label{eq:R-admissible-const}
        C_\kappa\,\bar{w}(R)\sqrt{m(\kappa_0+\alpha)\,\alpha\,\log\left(\frac{2|\mS||\mA|\left(\frac{\alpha T}{m\kappa_0}+1\right)}{\delta}\right)} \leq \frac{\bar{\beta}(R)}{\bar{\beta}(R)+1}\left(R-e_0\right),
    \end{equation}
    then~\eqref{eq:R-admissible} holds and, with probability at least $1-\delta$,
    \begin{align*}
        e_T &\leq \frac{1}{1-\bar{\beta}(R)}\left(1-\bar{\beta}(R)\right)^{\frac{\alpha T}{m(\kappa_0+\alpha)}}e_0\\
        &\qquad + \frac{(\bar{\beta}(R)+1)C_\kappa\bar{w}(R)}{\bar{\beta}(R)}\sqrt{m(\kappa_0+\alpha)\,\alpha\,\log\left(\frac{2|\mS||\mA|\left(\frac{\alpha T}{m\kappa_0}+1\right)}{\delta}\right)}.
    \end{align*}
\end{corollary}

\begin{proof}
    With a constant step size, every sub-block has exactly $\lceil\kappa_0/\alpha\rceil$ indices, so $|\mB_i|=m\lceil\kappa_0/\alpha\rceil$ and, using $\lceil\kappa_0/\alpha\rceil<\frac{\kappa_0+\alpha}{\alpha}$ and $\lfloor x\rfloor>x-1$, $I(T)>\frac{\alpha T}{m(\kappa_0+\alpha)}-1$; moreover $v_{\mB_i}=|\mB_i|\alpha^2\leq m(\kappa_0+\alpha)\alpha$ for every $i$, and the same bound holds for $v_{\mB_{\mathrm{res}}}$ since $\mB_{\mathrm{res}}\subseteq\mB_{I(T)}$. Hence, for every interval $\mathcal{K}$ among the $\mB_i$ and $\mB_{\mathrm{res}}$,
    \begin{equation*}
        \varepsilon_\mathcal{K}(\delta') \leq 2\left(1+2m(\kappa_0+\alpha)\right)\bar{w}(R)\sqrt{\frac{1}{2}m(\kappa_0+\alpha)\alpha\log\left(\frac{2|\mS||\mA|\left(\frac{\alpha T}{m\kappa_0}+1\right)}{\delta}\right)},
    \end{equation*}
    which is the left-hand side of~\eqref{eq:R-admissible-const}, so~\eqref{eq:R-admissible-const} implies~\eqref{eq:R-admissible}. \Cref{thm:sync-stoch-R-app} then applies, and the noise term is summed using $\sum_{i=0}^{I(T)-1}(1-\bar{\beta}(R))^{I(T)-1-i}\leq\frac{1}{\bar{\beta}(R)}$, so that the bracket is at most $\left(\frac{1}{\bar{\beta}(R)}+1\right)\sqrt{m(\kappa_0+\alpha)\alpha}$. Finally $1-\bar{\beta}(R)\in(0,1)$ and $I(T)>\frac{\alpha T}{m(\kappa_0+\alpha)}-1$ give the prefactor $\frac{1}{1-\bar{\beta}(R)}$ on the contraction term.
\end{proof}

\subsubsection{Proof of \pdfref{Theorem}{thm:sync-stoch-main}}
\label{subsubsec:proof_sync_main}

\begin{theorem}[Synchronous $\Omega$-regularized $Q$-learning]
    \label{thm:sync-stoch-main-detailed}
    Let \Cref{assumption:m-step-mixing,assumption:convex-analytic-regularity,assumption:reference_policy,assumption:positive-domination} hold, and set $\alpha(u):=\tilde\beta^2u^2/[9m(2+\tilde\beta)^2(1+9m)^2\bar w^2\iota]$ for $u>0$. Run \Cref{alg:synchronous_regularized_q_learning} for $T$ sweeps with a constant step size $\alpha\in(0,\alpha(e_0)]$. Then, with probability at least $1-\delta$, $\sup_{t\leq T}\|Q_t-Q^*_\Omega\|_\spn\leq2e_0$ and
    \begin{equation}
        \label{eq:sync-stoch-main-detailed}
        \|Q_T-Q^*_\Omega\|_\spn\leq\frac{(1-\tilde\beta/2)^{2\alpha T/(9m)}}{1-\tilde\beta/2}\,\|Q_0-Q_\Omega^*\|_\spn+\frac{3(2+\tilde\beta)(1+9m)}{\tilde\beta}\,\bar w\sqrt{m\alpha\iota}.
    \end{equation}
    In particular, for $\varepsilon\in(0,e_0)$, the choices $\alpha=\alpha(\varepsilon/2)$ and $T\geq\frac{9m}{\alpha\tilde\beta}\log\!\left(\frac{2e_0}{(1-\tilde\beta/2)\varepsilon}\right)$ give $\|Q_T-Q^*_\Omega\|_\spn\leq\varepsilon$, i.e.\ $\tilde O\big(|\mS||\mA|m^4\bar w^2/(\tilde\beta^3\varepsilon^2)\big)$ sampled transitions.
\end{theorem}

\begin{proof}[Proof of \Cref{thm:sync-stoch-main-detailed}]
    Take $\kappa_0=4$ and the Chernoff bound of \Cref{rem:chernoff_alternative}, so that $p_m\geq1-e^{-9/8}\geq\frac12$; for $\alpha\leq\frac12$, we also have $m(\kappa_0+\alpha)\leq\frac{9}{2}m$ and $C_\kappa\leq\sqrt{2}(1+9m)$, and $\alpha(e_0)\leq\frac{1}{2}$ since $\bar w\geq e_0$. We apply \Cref{cor:sync-const-step-R} with $R=2e_0$: then $A_R=\|Q^*_\Omega\|_\asp+2e_0$, so that $\bar\beta(R)=p_m\tilde\beta_m(A_R)\geq\tilde\beta/2$ and $\bar w(R)=\bar w$, and the logarithm in~\eqref{eq:R-admissible-const} is at most $\iota$ since $\frac{\alpha T}{m\kappa_0}+1\leq T+1$.

    \emph{Admissibility.} The left-hand side of~\eqref{eq:R-admissible-const} is at most $\sqrt2(1+9m)\bar w\sqrt{\frac92m\alpha\iota}=3(1+9m)\bar w\sqrt{m\alpha\iota}$, which is at most $\frac{\tilde\beta}{2+\tilde\beta}e_0$ when $\alpha\leq\alpha(e_0)$. Since $x\mapsto\frac{x}{1+x}$ is increasing and $\bar\beta(R)\geq\tilde\beta/2$, this is at most $\frac{\bar\beta(R)}{1+\bar\beta(R)}(R-e_0)$, which is~\eqref{eq:R-admissible-const}.

    \emph{Bound.} With probability at least $1-\delta$, \Cref{thm:sync-stoch-R-app} thus gives $\sup_{t\leq T}e_t\leq2e_0$. Since $\bar\beta(R)\geq\tilde\beta/2$, $I(T)\geq0$, and $I(T)>\frac{\alpha T}{m(\kappa_0+\alpha)}-1\geq\frac{2\alpha T}{9m}-1$, its contraction term satisfies $(1-\bar\beta(R))^{I(T)}e_0\leq(1-\tilde\beta/2)^{I(T)}e_0\leq(1-\tilde\beta/2)^{\frac{2\alpha T}{9m}-1}e_0$. Since $x\mapsto(1+x)/x$ is non-increasing, the noise term of \Cref{cor:sync-const-step-R} is at most $\frac{2+\tilde\beta}{\tilde\beta}\,3(1+9m)\bar w\sqrt{m\alpha\iota}$. This is~\eqref{eq:sync-stoch-main-detailed}.

    \emph{Sample complexity.} For $\varepsilon\in(0,e_0)$, the choice $\alpha=\alpha(\varepsilon/2)\leq\alpha(e_0)$ makes the noise term equal to $\varepsilon/2$, and $(1-\tilde\beta/2)^x\leq e^{-\tilde\beta x/2}$ makes the contraction term at most $\varepsilon/2$ as soon as $T\geq\frac{9m}{\alpha\tilde\beta}\log\left(\frac{2e_0}{(1-\tilde\beta/2)\varepsilon}\right)$, which is possible since $\iota$ depends on $T$ only logarithmically. With $\alpha(\varepsilon/2)=\tilde\Theta\big(\tilde\beta^2\varepsilon^2/(m^3\bar w^2)\big)$, this gives $\tilde O\big(m^4\bar w^2/(\tilde\beta^3\varepsilon^2)\big)$ iterations, each drawing $|\mS||\mA|$ transitions. Since the radius is $2e_0$, the constants $\tilde\beta$ and $\bar w$ do not depend on $\varepsilon$.
\end{proof}

\begin{proof}[Proof of \Cref{thm:sync-stoch-main}]
    Apply \Cref{thm:sync-stoch-main-detailed} with its step size condition. Since $0<\tilde\beta\leq1$, $(1-\tilde\beta/2)^{-1}\leq2$. Moreover, $3(2+\tilde\beta)(1+9m)\sqrt m\leq90m^{3/2}$ for $m\geq1$. Thus~\eqref{eq:sync-stoch-main-detailed} implies~\eqref{eq:sync-stoch-main} with $C=90$. The containment and accuracy guarantees, including the stated sample complexity, follow from the same theorem with its explicit step size and budget.
\end{proof}

\begin{corollary}[Synchronous $\Omega$-regularized $Q$-learning for $m=1$]
    \label{cor:sync-m-1}
    Let $m=1$. Then $\tilde\beta=\beta_1$, and \Cref{thm:sync-stoch-main-detailed} holds even without \Cref{assumption:positive-domination}. Moreover $\mT_\Omega$ is a $(1-\beta_1)$-contraction for $\|\cdot\|_\spn$ and $\|Q^*_\Omega\|_\spn\leq(\|r\|_\spn+\|\Omega^*(0)\|_{\mS,\spn})/\beta_1$; in particular, for $Q_0=0$ and a regularizer whose minimal value does not depend on the state, $\bar w\leq3\|r\|_\spn/\beta_1$, and the sample complexity of \Cref{thm:sync-stoch-main-detailed} is $\tilde O\big(|\mS||\mA|\,\|r\|_\spn^2/(\beta_1^5\varepsilon^2)\big)$.
\end{corollary}

\begin{proof}
    For $m=1$, $\tilde\beta_1(R)=\beta_1$ for every $R$, so the localized contraction constants of \Cref{def:localized-radii} do not depend on the radius and are positive without \Cref{assumption:positive-domination}, which the proof of \Cref{thm:sync-stoch-main-detailed} only uses to ensure $\tilde\beta>0$; that proof then applies verbatim. For the second claim, \Cref{lem:mixing-properties-of-regularization}.(iv) with $m=1$, whose product is empty, gives $\|\mT_\Omega(Q)-\mT_\Omega(Q')\|_\spn\leq(1-\beta_1)\|Q-Q'\|_\spn$. Since $\|\cdot\|_\spn$ annihilates constants, $\|Q^*_\Omega\|_\spn=\|\mT_\Omega(Q^*_\Omega)\|_\spn\leq(1-\beta_1)\|Q^*_\Omega\|_\spn+\|\mT_\Omega(0)\|_\spn$, and $\|\mT_\Omega(0)\|_\spn=\|r+P\,\Omega^*(0)\|_\spn\leq\|r\|_\spn+\|\Omega^*(0)\|_{\mS,\spn}$. Finally, for $Q_0=0$, $e_0=\|Q^*_\Omega\|_\spn$, and for such a regularizer $\|\Omega^*(0)\|_{\mS,\spn}=0$ (\Cref{cor:T_omega_props}), so $\bar w=3\|Q^*_\Omega\|_\spn\leq3\|r\|_\spn/\beta_1$.
\end{proof}

\subsection{Proofs for \pdfref{Subsection}{subsec:epoch_regularized_q_learning} (Epoch \texorpdfstring{$\Omega$}{Ω}-regularized \texorpdfstring{$Q$}{Q}-learning)}
\label{subsec:proofs_epoch_q_learning}

Every constant below is a function of the domination profile $\gamma_\Omega$ (\Cref{def:mixing_objects}.(iii)) evaluated at the radius actually reached by the iterates, which drives both the contraction rate and the exploration rate.

\Cref{alg:epoch_regularized_q_learning} has two mechanisms for controlling noise: averaging across epochs through the step sizes $\alpha_j$, and averaging within each epoch through the batch size $B_j$. We therefore control the noise in two ways, both valid at all parameters:
\begin{itemize}
    \item a \emph{batch bound} (\Cref{lem:ep_tracking}.(i)), which bounds the per-epoch noise by $O(\bar{w}(R)/\sqrt{B_j})$ and sums it over a block by the triangle inequality;
    \item a \emph{block-martingale bound} (\Cref{lem:ep_tracking}.(ii)), which concatenates all the transition-level increments of a block into a single martingale, as in \Cref{subsec:proofs_synchronous_q_learning}.
\end{itemize}
Neither dominates: for constant parameters their ratio is $(1+2N\alpha)/\sqrt{N}$, with $N$ the number of epochs of a block, so the batch bound is tighter at $\alpha_j\equiv1$ (\Cref{cor:ep_alpha_one}), the setting of \Cref{thm:epoch-main-body-detailed}, where the block-martingale bound loses a factor $\Theta(\sqrt{m})$ on the noise term, while the block-martingale bound wins for small step sizes and handles small batches, down to $B_j\equiv1$, where it recovers \Cref{thm:sync-stoch-R-app} (\Cref{rem:ep_recovers_sync}).

\subsubsection{Setting, terminology and the epoch update}
\label{subsubsec:ep_setting}

We write $e_j:=\|Q_j-Q^*_\Omega\|_\spn$ for the error after $j$ epochs and, for a radius $R\geq e_0$, we use the localized constants $A_R$, $C_R$, $\tilde{\beta}_m(A_R)$, $\bar{\beta}(R)$ and $\bar{w}(R)$ of \Cref{def:localized-radii}; under \Cref{assumption:positive-domination}, $\tilde{\beta}_m(A_R)$ and $\bar{\beta}(R)$ are positive. We reserve the index $j$ and the notation $E_j$ for algorithmic epochs, and the index $i$ for the blocks $\mB_i$ introduced by the analysis; $B_j$ is always the batch size. Three time scales coexist: transitions ($n$), epochs ($j$) and blocks ($i$).

The $j$-th epoch is $E_j:=\{n_j,\dots,n_{j+1}-1\}$ with $n_0:=0$ and
\begin{equation*}
    n_{j+1}:=\min\left\{n>n_j\ \middle|\ \min_{(s,a)}\sum_{k=n_j}^{n-1}\indicator_{\{(s_k,a_k)=(s,a)\}}\geq B_j\right\}.
\end{equation*}
We reuse the sub-blocks and blocks $\mB_i$ of \Cref{def:sub_blocks} on the epoch scale, i.e.\ with the iteration index replaced by the epoch index, so that $m\kappa_0\leq\sum_{j\in\mB_i}\alpha_j<m(\kappa_0+\alpha_0)$ (\Cref{def:sub_blocks,lem:sum_block}), and we write $b_i:=\min\mB_i$. The step sizes and batch sizes being deterministic, so are the blocks and the $b_i$.

We write $\ell_j:=n_{j+1}-n_j\in\bN\cup\{+\infty\}$ for the (random) length of epoch $j$, with the convention
\begin{equation}
    \label{eq:ep_convention_infinite}
    Q_{j+1}:=Q_j \qquad\text{on }\{\ell_j=+\infty\}.
\end{equation}

Let $\mG_j:=\sigma\left(Q_0,(s_0,a_0),\dots,(s_{n_j-1},a_{n_j-1}),s_{n_j}\right)$; then $Q_j$ is $\mG_j$-measurable and, by construction of \Cref{alg:epoch_regularized_q_learning}, the behavior policy $\pi_j:=\pi_{\Omega,Q_j}$ is frozen throughout epoch $j$. Set $\mathcal{H}^{(j)}_n:=\sigma\left(\mG_j,(s_{n_j},a_{n_j}),\dots,(s_{n_j+n},a_{n_j+n})\right)$. As in \Cref{subsubsec:sync_stochastic}, the centering projection $\Pi$ only shifts $Q_{j+1}$ by a constant, which leaves the greedy policies, hence the law of the trajectory, as well as $\|\cdot\|_\spn$ and $\|\cdot\|_\asp$ unchanged (\Cref{cor:T_omega_props} and \Cref{lem:omega_star_props}.(i),(iii)), so we omit it; its only role is to upgrade span convergence to $\|\cdot\|_\infty$ convergence (\Cref{rem:span_to_sup}).

\begin{convention}[Extended trajectory]
    \label{conv:ep_extended}
    Within epoch $j$ we let the trajectory be defined for \emph{all} $n\geq0$ under $\pi_j$: $\tilde{s}^{(j)}_0:=s_{n_j}$, $\tilde{a}^{(j)}_n\sim\pi_j(\cdot|\tilde{s}^{(j)}_n)$, $\tilde{s}^{(j)}_{n+1}\sim p(\cdot|\tilde{s}^{(j)}_n,\tilde{a}^{(j)}_n)$, coupled with the algorithm for $n<\ell_j$. All visit times below refer to this extended trajectory, which conditionally on $\mG_j$ is a time-homogeneous Markov chain with kernel $K_{\pi_j}$; its filtration $\widetilde{\mathcal{H}}^{(j)}_n$ coincides with $\mathcal{H}^{(j)}_n$ on $\{n<\ell_j\}$. The extension is a definition and requires no probabilistic input; \Cref{prop:ep_wellposed} shows a posteriori that it is almost surely never used.
\end{convention}

\begin{lemma}[Exact form of the epoch update]
    \label{lem:ep_update}
    \begin{enumerate}[label=(\roman*)]
        \item On $\{\ell_j<+\infty\}$ we have $N_j(s,a)=B_j$ for every pair at exit, and
        \begin{equation}
            \label{eq:ep_update}
            Q_{j+1}=F_j(Q_j)+\alpha_jW_{j+1},\qquad F_j(Q):=(1-\alpha_j)Q+\alpha_j \mT_\Omega(Q),
        \end{equation}
        where, with $t_1(s,a)<\dots<t_{B_j}(s,a)$ the first $B_j$ visit times to $(s,a)$ in the extended trajectory of epoch $j$, counted from the beginning $n_j$ of the epoch,
        \begin{equation}
            \label{eq:ep_noise}
            \begin{aligned}
                W_{j+1}(s,a)&=\frac{1}{B_j}\sum_{l=1}^{B_j}Z^{(j)}_l(s,a),\\
                Z^{(j)}_l(s,a)&:=\Omega^*(Q_j)\big(\tilde{s}^{(j)}_{t_l(s,a)+1}\big)-\bE_{s'\sim p(\cdot|s,a)}\big[\Omega^*(Q_j)(s')\big].
            \end{aligned}
        \end{equation}
        \item Every $Q_j$ is a well-defined $\mG_j$-measurable random vector of $\bR^{\mS\times\mA}$, with $\|Q_j\|_\infty<+\infty$ everywhere. Consequently $\pi_j$ is a well-defined policy and the extended trajectory of \Cref{conv:ep_extended} is a well-defined Markov chain, for every $j$ and on every path.
    \end{enumerate}
\end{lemma}

\begin{proof}
    \emph{(i)}. The inner loop takes into account a transition at $(s,a)$ only while $N_j(s,a)<B_j$, so it accumulates exactly the first $B_j$ visits and exits when the last pair reaches $B_j$; hence $N_j\equiv B_j$ at exit and $t_l(s,a)<\ell_j$ for $l\leq B_j$, so the extended trajectory coincides with the algorithm at every visit time used and at the next state. After normalization the estimate is $\widehat{\mT_\Omega Q_j}(s,a)=r(s,a)+\frac{1}{B_j}\sum_l\Omega^*(Q_j)(\tilde{s}^{(j)}_{t_l(s,a)+1})$, the reward being deterministic; subtracting $\mT_\Omega(Q_j)(s,a)$ gives~\eqref{eq:ep_noise}, and injecting into the update gives~\eqref{eq:ep_update}.

    \emph{(ii)}. By induction on $j$; the argument is deterministic. The initialization $Q_0$ is a fixed finite vector. Assume $Q_j$ finite. On $\{\ell_j=+\infty\}$, \eqref{eq:ep_convention_infinite} gives $Q_{j+1}=Q_j$. On $\{\ell_j<+\infty\}$, (i) gives $Q_{j+1}=F_j(Q_j)+\alpha_jW_{j+1}$, and each coordinate of $W_{j+1}$ is a difference between a value of $\Omega^*(Q_j)$ and a convex combination of such values, so $\|W_{j+1}\|_\infty\leq\|\Omega^*(Q_j)\|_{\mS,\spn}\leq\|Q_j\|_\spn+\|\Omega^*(0)\|_{\mS,\spn}<+\infty$ by \Cref{lem:omega_star_props}.(vii); hence $Q_{j+1}$ is finite. Measurability is that of the update, and $\pi_j=\nabla\Omega^*(Q_j)$ is well defined for any finite argument by \Cref{lem:omega_star_props}.(i).
\end{proof}

The length $\ell_j$ is random, but because of the cap the number of terms in~\eqref{eq:ep_noise} is deterministic, equal to $B_j$. The martingale arguments of \Cref{lem:ep_noise_epoch,lem:ep_block_martingale} rely on this, the second one through the deterministic length $\sum_{j\in\mB_i}B_j$ of the concatenated martingale.

\subsubsection{Exit time, well-posedness and epoch length}
\label{subsubsec:ep_length}

We define the exit time of the iterates in the epoch time scale, at fixed radius $R\geq e_0$, by
\begin{equation}
    \label{eq:ep_exit_time}
    \sigma_R:=\inf\left\{j\geq0\ \middle|\ e_j>R\right\}\in\bN\cup\{+\infty\}.
\end{equation}

\begin{remark}[Exit time]
    \label{rem:ep_exit_time}
    The iterates are frozen within an epoch, so $\sigma_R$ is a stopping time for $(\mG_j)_{j\geq0}$ and $\{j<\sigma_R\}\in\mG_j$. Since $e_0\leq R$, $\sigma_R\geq1$. As in \Cref{rem:exit-time-facts}.(ii), $j<\sigma_R$ implies $\|Q_j\|_\spn\leq C_R$ and $\|Q_j\|_\asp\leq A_R$. As in \Cref{rem:exit-time-facts}.(iv), the localization bounds the noise and, since $\gamma_\Omega$ may vanish at infinity, keeps both the contraction and the coverage rates at least at their values at radius $A_R$.
\end{remark}

\begin{lemma}[Conditional cover time at a deterministic rate]
    \label{lem:ep_cover_time}
    Let $j\geq0$, let $u\geq0$ be a radius and let $\delta'\in(0,1)$. Set
    \begin{equation}
        \label{eq:ep_qr}
        \mathcal{D}_j(u) := \left\{\|Q_j\|_\asp\leq u\right\}\in\mG_j, \qquad q(u) := \beta_m\,\gamma_\Omega(u)^m\,\eta_{\min},
    \end{equation}
    \begin{equation}
        \label{eq:ep_KL}
        K(u,\delta') := \left\lceil\frac{\left(\sqrt{B_j}+\sqrt{\iota(\delta')}\right)^2}{q(u)}\right\rceil,
        \qquad L(u,\delta') := m\,K(u,\delta')+1,
        \qquad \iota(\delta') := \log\left(\frac{|\mS||\mA|}{\delta'}\right).
    \end{equation}
    Under \Cref{assumption:positive-domination}, $q(u)>0$ and, almost surely,
    \begin{equation}
        \label{eq:ep_conditional_length}
        \bP\left(\left\{\ell_j>L(u,\delta')\right\}\cap\mathcal{D}_j(u)\ \middle|\ \mG_j\right)\ \leq\ \delta'\,\indicator_{\mathcal{D}_j(u)}.
    \end{equation}
    Moreover $L(u,\delta')\geq m|\mS||\mA|B_j$. In particular, for $R\geq e_0$, the coverage constant and the localized length
    \begin{equation}
        \label{eq:ep_length_bound}
        q_R:=q(A_R)=\gamma_\Omega(A_R)\,\tilde\beta_m(A_R)\,\eta_{\min},\qquad
        L_j(\delta') := L(A_R,\delta') = m\left\lceil\frac{\left(\sqrt{B_j}+\sqrt{\iota(\delta')}\right)^2}{q_R}\right\rceil+1
    \end{equation}
    satisfy $\bP\left(\left\{\ell_j>L_j(\delta')\right\}\cap\{j<\sigma_R\}\mid\mG_j\right)\leq\delta'\indicator_{\{j<\sigma_R\}}$ almost surely.
\end{lemma}

\begin{proof}
    Since $\mathcal{D}_j(u)\in\mG_j$, it suffices to prove $\bP(\ell_j>L(u,\delta')\mid\mG_j)\leq\delta'$ almost surely on $\mathcal{D}_j(u)$; we work under a regular conditional probability given $\mG_j$ and fix an outcome in $\mathcal{D}_j(u)$. All the constants $q(u)>0$, $K:=K(u,\delta')$ and $L(u,\delta')$ are deterministic, $u$ being deterministic; positivity of $q(u)$ is \Cref{assumption:positive-domination} at the deterministic radius $u$.

    By \Cref{lem:ep_update}.(ii), the extended trajectory of epoch $j$ is well defined and, conditionally on $\mG_j$, is a time-homogeneous Markov chain with kernel $K_{\pi_j}$; the conditioning is legitimate because $\mathcal{D}_j(u)\in\mG_j\subseteq\widetilde{\mathcal{H}}^{(j)}_n$. On $\mathcal{D}_j(u)$, \Cref{def:mixing_objects}.(iii) applied at radius $u$ gives $\pi_j(a|s)\geq\gamma_\Omega(u)\mu(a|s)$ for all $(s,a)$, so \Cref{lem:mixing-properties-of-regularization}.(iii) applied with $\pi_1=\cdots=\pi_m=\pi_j$ gives, for every $n\geq0$ and every $(s,a)\in\mS\times\mA$,
    \begin{equation}
        \label{eq:ep_coverage}
        \bP\left(\left(\tilde{s}^{(j)}_{n+m},\tilde{a}^{(j)}_{n+m}\right)=(s,a)\ \middle|\ \widetilde{\mathcal{H}}^{(j)}_n\right)
        =K_{\pi_j}^m\big(s,a\,|\,\tilde{s}^{(j)}_n,\tilde{a}^{(j)}_n\big)
        \ \geq\ \beta_m\,\gamma_\Omega(u)^m\,\eta(s,a)\ \geq\ q(u).
    \end{equation}

    Fix $(s,a)$ and set $Z_l:=\indicator_{\{(\tilde{s}^{(j)}_{lm},\tilde{a}^{(j)}_{lm})=(s,a)\}}\in[0,1]$ for $l\in\llbracket1,K\rrbracket$, adapted to $\mathcal{G}^{\mathrm{sub}}_l:=\widetilde{\mathcal{H}}^{(j)}_{lm}$. Consecutive sub-sampling instants being exactly $m$ steps apart, \eqref{eq:ep_coverage} gives $\bE[Z_l\mid\mathcal{G}^{\mathrm{sub}}_{l-1}]\geq q(u)$, so the predictable mean satisfies $\bar{S}\geq Kq(u)$ deterministically. The exact multiplicative-Chernoff form~\eqref{eq:exact-chernoff} of \Cref{lem:self-bounding}, applied with increment bound $b=1$, $a=Kq(u)$ and $c=B_j$, all deterministic, gives
    \begin{equation}
        \label{eq:ep_chernoff}
        \bP\left(\sum_{l=1}^KZ_l<B_j\ \middle|\ \mG_j\right)\leq\exp\left(-B_j\,h(\theta)\right),
        \qquad \theta:=\frac{Kq(u)}{B_j},\quad h(\theta):=\theta-1-\log\theta,
    \end{equation}
    the hypothesis $c<a$ holding since $Kq(u)\geq(\sqrt{B_j}+\sqrt{\iota(\delta')})^2>B_j$.

    Now $h(\theta)\geq(\sqrt\theta-1)^2$ for $\theta\geq1$: both sides vanish at $\theta=1$ and $h'(\theta)=1-\theta^{-1}\geq1-\theta^{-1/2}$ there, the right-hand side being the derivative of $\theta\mapsto(\sqrt\theta-1)^2$. By the choice of $K$, $\theta\geq(1+\sqrt{\iota(\delta')/B_j})^2$, hence $B_jh(\theta)\geq B_j(\sqrt\theta-1)^2\geq\iota(\delta')$ and the right-hand side of~\eqref{eq:ep_chernoff} is at most $\delta'/(|\mS||\mA|)$.

    A union bound over the $|\mS||\mA|$ pairs shows that, with conditional probability at least $1-\delta'$, every pair is visited at least $B_j$ times among the instants $\{m,2m,\dots,Km\}\subseteq\llbracket0,mK\rrbracket$, so $\ell_j\leq mK+1=L(u,\delta')$. This is~\eqref{eq:ep_conditional_length}.

    Finally $q(u)\leq\eta_{\min}\leq(|\mS||\mA|)^{-1}$, so $K(u,\delta')\geq|\mS||\mA|B_j$ and $L(u,\delta')\geq m|\mS||\mA|B_j$. For the last claim, \Cref{rem:ep_exit_time} gives $\{j<\sigma_R\}\subseteq\mathcal{D}_j(A_R)$, both events belonging to $\mG_j$; apply~\eqref{eq:ep_conditional_length} with the deterministic radius $u=A_R$ and intersect with $\{j<\sigma_R\}$.
\end{proof}

\begin{proposition}[Well-posedness]
    \label{prop:ep_wellposed}
    For every $j\geq0$, $\ell_j<+\infty$ almost surely. Consequently convention~\eqref{eq:ep_convention_infinite} is vacuous, the extension of \Cref{conv:ep_extended} is almost surely never used, and \Cref{alg:epoch_regularized_q_learning} produces the update~\eqref{eq:ep_update} at every epoch almost surely.
\end{proposition}

\begin{proof}
    Fix $j$. For every deterministic $u\geq0$ and every $\delta'\in(0,1)$, integrating~\eqref{eq:ep_conditional_length} gives
    \begin{equation*}
        \bP\left(\ell_j=+\infty\right)
        \ \leq\ \bP\left(\left\{\ell_j>L(u,\delta')\right\}\cap\mathcal{D}_j(u)\right)+\bP\left(\mathcal{D}_j(u)^c\right)
        \ \leq\ \delta'+\bP\left(\|Q_j\|_\asp>u\right).
    \end{equation*}
    Letting $\delta'\downarrow0$ at fixed $u$ gives $\bP(\ell_j=+\infty)\leq\bP(\|Q_j\|_\asp>u)$ for every $u$, and letting $u\uparrow+\infty$ gives $\bP(\ell_j=+\infty)=0$, since $\|Q_j\|_\asp\leq2\|Q_j\|_\infty<+\infty$ everywhere by \Cref{lem:ep_update}.(ii).
\end{proof}

Given a transition budget $T\geq 0$ we write
\begin{equation}
    \label{eq:ep_counts}
    J(T):=\max\{j\mid n_j\leq T\},\qquad I(T):=\max\{i\mid b_i\leq J(T)\},
\end{equation}
and $\mB_{\mathrm{res}}:=\llbracket b_{I(T)},J(T)-1\rrbracket\subseteq\mB_{I(T)}$ for the last incomplete block.

\begin{lemma}[Epoch and block counts]
    \label{lem:ep_counts}
    Let $T\geq 0$, $\underline{B}:=\min_jB_j\geq1$. Then:
    \begin{enumerate}[label=(\roman*)]
        \item \emph{Deterministic upper bounds}: $\ell_j\geq|\mS||\mA|B_j$, hence $J(T)\leq J_{\max}(T):=\left\lfloor\frac{T}{|\mS||\mA|\underline{B}}\right\rfloor$ and $I(T)\leq I_{\max}(T):=\left\lfloor\frac{J_{\max}(T)}{m}\right\rfloor$.
        \item \emph{From epochs to blocks}: $I(T) > \frac{\sum_{j<J(T)}\alpha_j}{m(\kappa_0+\alpha_0)}-1$, i.e.\ $I(T)\geq\left\lfloor\frac{\sum_{j<J(T)}\alpha_j}{m(\kappa_0+\alpha_0)}\right\rfloor$. If $\alpha_j\equiv\alpha$, every block contains exactly $m\lceil\kappa_0/\alpha\rceil$ epochs and $I(T)=\left\lfloor\frac{J(T)}{m\lceil\kappa_0/\alpha\rceil}\right\rfloor$.
        \item \emph{From transitions to epochs}: set $J^-(T):=\max\left\{J\mid\sum_{j<J}L_j(\delta')\leq T\right\}$. If $\ell_j\leq L_j(\delta')$ for every $j<\min\{\sigma_R,J_{\max}(T)\}$ and $\sigma_R>J(T)$, then $J(T)\geq J^-(T)$. In particular for $\alpha_j\equiv\alpha$, $B_j\equiv B$, writing $L(\delta')$ for the common value,
        \begin{equation}
            \label{eq:ep_I_of_T}
            J(T)\geq\left\lfloor\frac{T}{L(\delta')}\right\rfloor,\qquad
            I(T)\geq\left\lfloor\frac{\alpha T}{m(\kappa_0+\alpha)L(\delta')}-\frac{\alpha}{m(\kappa_0+\alpha)}\right\rfloor.
        \end{equation}
    \end{enumerate}
\end{lemma}

\begin{proof}
    \emph{(i)}. Epoch $j$ terminates only when each of the $|\mS||\mA|$ pairs has been visited at least $B_j$ times, and each transition visits one pair, so $\ell_j\geq|\mS||\mA|B_j$; summing, $J(T)|\mS||\mA|\underline{B}\leq n_{J(T)}\leq T$. Each sub-block is non-empty, so each block has at least $m$ epochs.

    \emph{(ii)}. By definition of $I(T)$, $b_{I(T)+1}>J(T)$, and each block consumes at most $m(\kappa_0+\alpha_0)$ (\Cref{lem:sum_block}), so $\sum_{j<J(T)}\alpha_j\leq\sum_{j<b_{I(T)+1}}\alpha_j<(I(T)+1)m(\kappa_0+\alpha_0)$; the floor form follows since $I(T)$ is an integer. For $\alpha_j\equiv\alpha$ the stopping rule fires after exactly $\lceil\kappa_0/\alpha\rceil$ epochs.

    \emph{(iii)}. If $J(T)=J_{\max}(T)$ there is nothing to prove, since $J^-(T)\leq J_{\max}(T)$ by the bound $L_j(\delta')\geq m|\mS||\mA|B_j$ of \Cref{lem:ep_cover_time}. Otherwise suppose $J(T)<J^-(T)$. Every $j\leq J(T)$ satisfies $j<\sigma_R$ and $j<J_{\max}(T)$, hence $\ell_j\leq L_j(\delta')$, so $n_{J(T)+1}=\sum_{j\leq J(T)}\ell_j\leq\sum_{j<J(T)+1}L_j(\delta')\leq T$, the last inequality because $J(T)+1\leq J^-(T)$; this contradicts the maximality of $J(T)$. The constant case follows by injecting $J(T)\geq\lfloor T/L(\delta')\rfloor$ into the exact identity of (ii) and using $\lfloor x\rfloor>x-1$ along with $\lceil\kappa_0/\alpha\rceil<(\kappa_0+\alpha)/\alpha$.
\end{proof}

\begin{remark}
    \label{rem:ep_counts_comment}
    Only (i) is used for the union bounds, and it is deterministic: this is what lets $\delta'$ be fixed in advance although $J(T)$ and $I(T)$ are random. There is no circularity in (iii): \Cref{prop:ep_no_escape} establishes $\sigma_R>J(T)$ using only the almost sure finiteness of \Cref{prop:ep_wellposed}, and only then are the length bounds available.
\end{remark}

\subsubsection{The noiseless flow before the exit time}
\label{subsubsec:ep_virtual}

For each block $\mB_i$ define the \emph{virtual sequence} by $\tilde{Q}_{b_i}:=Q_{b_i}$ and $\tilde{Q}_{j+1}:=F_j(\tilde{Q}_j)$ for $j\geq b_i$, i.e.\ the full-knowledge synchronous update~\eqref{eq:sync_fk_update} started at $Q_{b_i}$; it is $\mG_{b_i}$-measurable.

\begin{lemma}[Contraction of the virtual flow at radius $R$]
    \label{lem:ep_virtual_flow}
    Let $i$ be such that $e_{b_i}\leq R$. Then, pathwise:
    \begin{enumerate}[label=(\roman*)]
        \item $\|\tilde{Q}_j-Q^*_\Omega\|_\spn\leq e_{b_i}$ for every $j\geq b_i$;
        \item $\|\tilde{Q}_{b_{i+1}}-Q^*_\Omega\|_\spn\leq(1-\bar{\beta}(R))\,e_{b_i}$.
    \end{enumerate}
\end{lemma}

\begin{proof}
    This is \Cref{lem:virtual-flow-R} with the step index replaced by the epoch index: the statement concerns the full-knowledge update~\eqref{eq:sync_fk_update} only, and $\mB_i$ is a concatenation of $m$ sub-blocks in the sense of \Cref{def:sub_blocks}.
\end{proof}

\subsubsection{Localized noise control}
\label{subsubsec:ep_noise}

The noise increments below lie in intervals whose endpoints are known one step ahead, which is the setting of the maximal Azuma--Hoeffding inequality of \Cref{lem:azuma_range}. We write $\widehat{W}_{j+1}:=\indicator_{\{j<\sigma_R\}}W_{j+1}$ and $\widehat{Z}^{(j)}_l:=\indicator_{\{j<\sigma_R\}}Z^{(j)}_l$ for the stopped noise, so that $\widehat{W}_{j+1}=\frac{1}{B_j}\sum_l\widehat{Z}^{(j)}_l$.

\begin{lemma}[Epoch noise: martingale structure and range]
    \label{lem:ep_noise}
    Fix $x=(s,a)$ and let $\mathcal{K}^{(j)}_l:=\sigma\left(\widetilde{\mathcal{H}}^{(j)}_{t_l(x)},\tilde{s}^{(j)}_{t_l(x)+1}\right)$ for $l\in\llbracket1,B_j\rrbracket$, $\mathcal{K}^{(j)}_0:=\mG_j$. Then:
    \begin{enumerate}[label=(\roman*)]
        \item $\mathcal{K}^{(j)}_0\subseteq\dots\subseteq\mathcal{K}^{(j)}_{B_j}\subseteq\mG_{j+1}\subseteq\mathcal{K}^{(j+1)}_0$, so the concatenation over $j$ is a filtration at fixed $x$; in particular $W_{j+1}$ is $\mG_{j+1}$-measurable;
        \item $\bE[\widehat{Z}^{(j)}_l(x)\mid\mathcal{K}^{(j)}_{l-1}]=0$; consequently $\bE[\widehat{W}_{j+1}\mid\mG_j]=0$;
        \item $\widehat{Z}^{(j)}_l(x)$ lies in a $\mG_j$-measurable interval of length $\indicator_{\{j<\sigma_R\}}\|\Omega^*(Q_j)\|_{\mS,\spn}\leq\bar{w}(R)$.
    \end{enumerate}
\end{lemma}

\begin{proof}
    \emph{(i)}. The $t_l(x)$ are stopping times for $(\widetilde{\mathcal{H}}^{(j)}_n)_n$, a.s.\ finite by \Cref{prop:ep_wellposed}, and $t_{l-1}(x)+1\leq t_l(x)$, so $\mathcal{K}^{(j)}_{l-1}\subseteq\widetilde{\mathcal{H}}^{(j)}_{t_l(x)}\subseteq\mathcal{K}^{(j)}_l$; and $t_{B_j}(x)+1\leq\ell_j$, the extended trajectory coinciding with the algorithm before $\ell_j$, gives $\mathcal{K}^{(j)}_{B_j}\subseteq\mG_{j+1}$.

    \emph{(ii)}. On $\{t_l(x)=n\}$ we have $(\tilde{s}^{(j)}_n,\tilde{a}^{(j)}_n)=x$ and the conditional law of $\tilde{s}^{(j)}_{n+1}$ given $\widetilde{\mathcal{H}}^{(j)}_n$ is $p(\cdot | x)$, so $\bE[Z^{(j)}_l(x)\mid\widetilde{\mathcal{H}}^{(j)}_{t_l(x)}]=0$. Since $\indicator_{\{j<\sigma_R\}}$ is $\mG_j$-, hence $\mathcal{K}^{(j)}_{l-1}$-measurable and $\mathcal{K}^{(j)}_{l-1}\subseteq\widetilde{\mathcal{H}}^{(j)}_{t_l(x)}$, the tower property gives the claim.

    \emph{(iii)}. Conditionally on $\mG_j$, $\Omega^*(Q_j)$ is fixed, so $Z^{(j)}_l(x)$ lies in the $\mG_j$-measurable interval of length exactly $\|\Omega^*(Q_j)\|_{\mS,\spn}$, translated by $-(P\Omega^*(Q_j))(x)$. On $\{j<\sigma_R\}$, \Cref{lem:omega_star_props}.(vii) and \Cref{rem:ep_exit_time} give $\|\Omega^*(Q_j)\|_{\mS,\spn}\leq\|Q_j\|_\spn+\left\|\Omega^*(0)\right\|_{\mS,\spn}\leq C_R+\left\|\Omega^*(0)\right\|_{\mS,\spn}=\bar{w}(R)$; off the event $\widehat{Z}^{(j)}_l=0$.
\end{proof}

\paragraph{Batch bound.}

\begin{lemma}[Per-epoch noise]
    \label{lem:ep_noise_epoch}
    For every $j$ and $\delta'\in(0,1)$, with probability at least $1-\delta'$,
    \begin{equation}
        \label{eq:ep_noise_bound}
        \left\|\widehat{W}_{j+1}\right\|_\infty\leq\bar{U}_j(\delta'):=\bar{w}(R)\sqrt{\frac{1}{2B_j}\log\left(\frac{2|\mS||\mA|}{\delta'}\right)},
        \qquad\text{hence}\qquad\left\|\widehat{W}_{j+1}\right\|_\spn\leq2\bar{U}_j(\delta').
    \end{equation}
\end{lemma}

\begin{proof}
    Fix $x$ and apply \Cref{lem:azuma_range} to $M_k:=\sum_{l\leq k}\widehat{Z}^{(j)}_l(x)$ with $n=B_j$ deterministic (\Cref{lem:ep_update}) and $c_l\equiv\bar{w}(R)$, legitimate by \Cref{lem:ep_noise}. With $u=B_j\bar{U}_j(\delta')$ and a union bound over the pairs this is~\eqref{eq:ep_noise_bound}; the last inequality is $\|\cdot\|_\spn\leq2\|\cdot\|_\infty$.
\end{proof}

\paragraph{Block-martingale bound.}

\begin{lemma}[Block martingale]
    \label{lem:ep_block_martingale}
    Fix $i$, let $\mathcal{K}=\llbracket b_i,n-1\rrbracket\subseteq\mB_i$ ($b_i<n\leq b_{i+1}$) be a deterministic prefix of $\mB_i$, and define $\widehat{U}^{(i)}_{b_i}:=0$, $\widehat{U}^{(i)}_{j+1}:=\widehat{U}^{(i)}_j+\alpha_j\widehat{W}_{j+1}$. Then for every $\delta'\in(0,1)$, with probability at least $1-\delta'$,
    \begin{equation}
        \label{eq:ep_block_martingale_bound}
        \max_{b_i\leq j\leq n}\left\|\widehat{U}^{(i)}_j\right\|_\infty\leq\bar{V}_{\mathcal{K}}(\delta'):=\bar{w}(R)\sqrt{\frac{1}{2}\left(\sum_{j\in\mathcal{K}}\frac{\alpha_j^2}{B_j}\right)\log\left(\frac{2|\mS||\mA|}{\delta'}\right)}.
    \end{equation}
    We write $\bar{V}_i(\delta'):=\bar{V}_{\mB_i}(\delta')$ when $\mathcal{K}=\mB_i$.
\end{lemma}

\begin{proof}
    Fix $x$ and concatenate chronologically the increments $\left(\frac{\alpha_j}{B_j}\widehat{Z}^{(j)}_l(x)\right)_{j\in\mathcal{K},\,l\leq B_j}$, adapted to the filtration of \Cref{lem:ep_noise}.(i). They are martingale differences by (ii), and by (iii) the $l$-th increment of epoch $j$ has predictable range of deterministic length $\frac{\alpha_j}{B_j}\bar{w}(R)$. Since $\mathcal{K}$ is deterministic, the total number of increments $\sum_{j\in\mathcal{K}}B_j$ is deterministic (\Cref{lem:ep_update}) and
    \begin{equation*}
        \sum_{j\in\mathcal{K}}\sum_{l=1}^{B_j}\left(\frac{\alpha_j}{B_j}\bar{w}(R)\right)^2=\bar{w}(R)^2\sum_{j\in\mathcal{K}}\frac{\alpha_j^2}{B_j}.
    \end{equation*}
    The partial sums at epoch boundaries are exactly $\widehat{U}^{(i)}_j(x)$ for $b_i\leq j\leq n$, so \Cref{lem:azuma_range} with $u=\bar{V}_{\mathcal{K}}(\delta')$ and a union bound over the pairs give~\eqref{eq:ep_block_martingale_bound}.
\end{proof}

\paragraph{The tracking lemma.}

\begin{lemma}[Tracking within a block]
    \label{lem:ep_tracking}
    Let $i$ be such that $e_{b_i}\leq R$, let $\mathcal{K}=\llbracket b_i,n-1\rrbracket\subseteq\mB_i$ ($b_i<n\leq b_{i+1}$) be a deterministic prefix of $\mB_i$, let $\tilde{Q}$ be the virtual sequence started at $Q_{b_i}$ and $E_j:=Q_j-\tilde{Q}_j$. Then for every $j\in\mathcal{K}\cup\{n\}$ with $j\leq\sigma_R$:
    \begin{enumerate}[label=(\roman*)]
        \item on $\bigcap_{k\in\mathcal{K}}\{\|\widehat{W}_{k+1}\|_\infty\leq\bar{U}_k(\delta')\}$: $\ \|E_j\|_\spn\leq\varepsilon^{\mathrm{bat}}_{\mathcal{K}}(\delta'):=2\sum_{k\in\mathcal{K}}\alpha_k\bar{U}_k(\delta')$;
        \item on $\{\max_{b_i\leq k\leq n}\|\widehat{U}^{(i)}_k\|_\infty\leq\bar{V}_{\mathcal{K}}(\delta')\}$: $\ \|E_j\|_\spn\leq\varepsilon^{\mathrm{mg}}_{\mathcal{K}}(\delta'):=2\left(1+2\sum_{k\in\mathcal{K}}\alpha_k\right)\bar{V}_{\mathcal{K}}(\delta')$.
    \end{enumerate}
    We write $\varepsilon^{\mathrm{bat}}_i(\delta'):=\varepsilon^{\mathrm{bat}}_{\mB_i}(\delta')$ and $\varepsilon^{\mathrm{mg}}_i(\delta'):=\varepsilon^{\mathrm{mg}}_{\mB_i}(\delta')$ when $\mathcal{K}=\mB_i$.
\end{lemma}

\begin{proof}
    For $k\in\mathcal{K}$ with $k<\sigma_R$, $\widehat{W}_{k+1}=W_{k+1}$ and $\ell_k<+\infty$ (\Cref{prop:ep_wellposed}), so~\eqref{eq:ep_update} applies. Subtracting the virtual recursion and using the average-path identity of \Cref{lem:mixing-properties-of-regularization}.(i), i.e.\ $\mT_\Omega(Q_k)-\mT_\Omega(\tilde{Q}_k)=\bar{K}_kE_k$ with $\bar{K}_k:=K_{\bar{\pi}_{\Omega,[Q_k,\tilde{Q}_k]}}$,
    \begin{equation}
        \label{eq:ep_err_rec}
        E_{k+1}=M_kE_k+\alpha_k\widehat{W}_{k+1},\qquad M_k:=(1-\alpha_k)I+\alpha_k\bar{K}_k,
    \end{equation}
    with $M_k$ row-stochastic, hence span non-expansive (\Cref{lem:row-stochastic-span}).

    \emph{(i)}. $\|E_{k+1}\|_\spn\leq\|E_k\|_\spn+2\alpha_k\bar{U}_k(\delta')$ by~\eqref{eq:ep_noise_bound}, and an immediate induction from $E_{b_i}=0$ concludes.

    \emph{(ii)}. The recursion~\eqref{eq:ep_err_rec} has the same form as~\eqref{eq:sync-err-rec-R}, so the decomposition $E_k=\widehat{R}_k+\widehat{U}^{(i)}_k$ in the proof of \Cref{lem:tracking-R}, with $\widehat{U}^{(i)}$ and $\bar{V}_{\mathcal{K}}(\delta')$ in place of $\widehat{U}^{(\mathcal{K})}$ and $\bar{U}_\mathcal{K}(\delta')$, gives $\|E_j\|_\spn\leq\varepsilon^{\mathrm{mg}}_{\mathcal{K}}(\delta')$.
\end{proof}

\subsubsection{No escape before the horizon}
\label{subsubsec:ep_no_escape}

Let $\delta\in(0,1)$, $T\geq 0$ and $\delta':=\frac{\delta}{3(J_{\max}(T)+1)}$. For each integer $n\in\llbracket1,J_{\max}(T)\rrbracket$, let $i(n)$ be the unique block index such that $b_{i(n)}<n\leq b_{i(n)+1}$ and let $\mathcal{K}_n:=\llbracket b_{i(n)},n-1\rrbracket$ be the corresponding deterministic prefix of $\mB_{i(n)}$. For any deterministic prefix $\mathcal{K}\subseteq\mB_i$, we set $\varepsilon_{\mathcal{K}}(\delta'):=\min\{\varepsilon^{\mathrm{bat}}_{\mathcal{K}}(\delta'),\varepsilon^{\mathrm{mg}}_{\mathcal{K}}(\delta')\}$ (with $\varepsilon_\emptyset(\delta'):=0$), and write
\begin{equation}
    \label{eq:ep_eps_i}
    \varepsilon_i(\delta'):=\varepsilon_{\mB_i}(\delta')=\min\left\{\varepsilon^{\mathrm{bat}}_i(\delta'),\varepsilon^{\mathrm{mg}}_i(\delta')\right\},\qquad
    \bar{\varepsilon}_R(\delta):=\max_{0\leq i\leq I_{\max}(T)}\varepsilon_i(\delta'),
\end{equation}
and $\varepsilon_{\mathrm{res}}(\delta'):=\varepsilon_{\mB_{\mathrm{res}}}(\delta')$. Since $J(T)\leq J_{\max}(T)$ deterministically (\Cref{lem:ep_counts}.(i)), every completed block $\mB_i$ ($i<I(T)$) equals $\mathcal{K}_{b_{i+1}}$ with $b_{i+1}\leq J_{\max}(T)$, and whenever $\mB_{\mathrm{res}}$ is non-empty we have $\mB_{\mathrm{res}}=\mathcal{K}_{J(T)}$ with $1\leq J(T)\leq J_{\max}(T)$. Both bounds of \Cref{lem:ep_tracking} are non-decreasing when a prefix is extended, so $\varepsilon_{\mathrm{res}}(\delta')\leq\varepsilon_{I(T)}(\delta')\leq\bar{\varepsilon}_R(\delta)$.

\begin{lemma}[Good event]
    \label{lem:ep_good_event}
    The event
    \begin{align}
        \mathcal{E}_R(\delta)&:=\bigcap_{j<J_{\max}(T)}\left\{\left\|\widehat{W}_{j+1}\right\|_\infty\leq\bar{U}_j(\delta')\right\}
        \cap\bigcap_{n=1}^{J_{\max}(T)}\left\{\max_{b_{i(n)}\leq j\leq n}\left\|\widehat{U}^{(i(n))}_j\right\|_\infty\leq\bar{V}_{\mathcal{K}_n}(\delta')\right\}\notag\\
        &\qquad\cap\bigcap_{j<J_{\max}(T)}\left(\{j\geq\sigma_R\}\cup\left\{\ell_j\leq L_j(\delta')\right\}\right)
        \label{eq:ep_good_event}
    \end{align}
    satisfies $\bP(\mathcal{E}_R(\delta))\geq1-\delta$.
\end{lemma}

\begin{proof}
    The three families each contain $J_{\max}(T)$ events, hence $3J_{\max}(T)\leq 3(J_{\max}(T)+1)$ in total, each of failure probability at most $\delta'$ by \Cref{lem:ep_noise_epoch,lem:ep_block_martingale} (applied to each deterministic prefix $\mathcal{K}_n$) and \Cref{lem:ep_cover_time}, the latter applied conditionally on $\mG_j$ and integrated.
\end{proof}

\begin{proposition}[No escape]
    \label{prop:ep_no_escape}
    Let $R\geq e_0$ satisfy the \emph{radius admissibility condition}
    \begin{equation}
        \label{eq:ep_R_admissible}
        \bar{\varepsilon}_R(\delta)\leq\frac{\bar{\beta}(R)}{\bar{\beta}(R)+1}\left(R-e_0\right).
    \end{equation}
    Then, on $\mathcal{E}_R(\delta)$, $\sigma_R>J(T)$; in particular $\sup_{0\leq j\leq J(T)}e_j\leq R$ and $\ell_j\leq L_j(\delta')$ for every $j<J(T)$, and for every $0\leq i\leq I(T)$,
    \begin{equation}
        \label{eq:ep_block_recursion}
        e_{b_i}\leq(1-\bar{\beta}(R))^ie_0+\sum_{\ell=0}^{i-1}(1-\bar{\beta}(R))^{i-1-\ell}\varepsilon_\ell(\delta').
    \end{equation}
\end{proposition}

\begin{proof}
    Work on $\mathcal{E}_R(\delta)$. The proof of \Cref{prop:no-escape} applies verbatim on the epoch scale, with the exit time~\eqref{eq:ep_exit_time}, \Cref{rem:ep_exit_time} for $\sigma_R\geq1$, \Cref{lem:ep_virtual_flow,lem:ep_tracking} in place of \Cref{lem:virtual-flow-R,lem:tracking-R}, and $\varepsilon_i(\delta')=\varepsilon_{\mathcal{K}_{b_{i+1}}}(\delta')$ (for $i<I(T)$, since $b_{i+1}\leq J(T)\leq J_{\max}(T)$) and $\varepsilon_{\mathrm{res}}(\delta')=\varepsilon_{\mathcal{K}_{J(T)}}(\delta')\leq\bar{\varepsilon}_R(\delta)$ (when $\mB_{\mathrm{res}}\neq\emptyset$) in place of $\varepsilon_{\mB_i}(\delta')$ and $\varepsilon_{\mB_{\mathrm{res}}}(\delta')$; it gives $\sigma_R>J(T)$ and~\eqref{eq:ep_block_recursion}. The length bound is then the third family of~\eqref{eq:ep_good_event}, since every $j<J(T)$ satisfies $j<\sigma_R$ and $j<J_{\max}(T)$.
\end{proof}

\subsubsection{Main result}
\label{subsubsec:ep_main}

\begin{theorem}[Finite-time bound for \Cref{alg:epoch_regularized_q_learning}]
    \label{thm:ep_main}
    Let $\delta\in(0,1)$ and $T\geq 0$. Under \Cref{assumption:lr}, let $\kappa_0>0$ be arbitrary and let $R\geq e_0$ satisfy the admissibility condition~\eqref{eq:ep_R_admissible}. Then, with probability at least $1-\delta$, $\sup_{0\leq j\leq J(T)}e_j\leq R$ and
    \begin{equation}
        \label{eq:ep_main_bound}
        e_{J(T)}\leq(1-\bar{\beta}(R))^{I(T)}e_0
        +\sum_{i=0}^{I(T)-1}(1-\bar{\beta}(R))^{I(T)-1-i}\varepsilon_i(\delta')+\varepsilon_{\mathrm{res}}(\delta'),
    \end{equation}
    where $J(T)$ and $I(T)$ are bounded below in \Cref{lem:ep_counts}.(iii). Explicitly, with $\iota_1:=\log\left(\frac{6|\mS||\mA|(J_{\max}(T)+1)}{\delta}\right)$,
    \begin{equation}
        \label{eq:ep_eps_explicit}
        \varepsilon_i(\delta')=\bar{w}(R)\sqrt{2\iota_1}\,\min\left\{\ \sum_{j\in\mB_i}\frac{\alpha_j}{\sqrt{B_j}},\ \ \left(1+2\sum_{j\in\mB_i}\alpha_j\right)\sqrt{\sum_{j\in\mB_i}\frac{\alpha_j^2}{B_j}}\ \right\}.
    \end{equation}
\end{theorem}

\begin{proof}
    Work on $\mathcal{E}_R(\delta)$, of probability at least $1-\delta$ (\Cref{lem:ep_good_event}). As in the proof of \Cref{thm:sync-stoch-R-app}, \Cref{prop:ep_no_escape}, \Cref{lem:ep_tracking} on $\mB_{\mathrm{res}}$ and \Cref{lem:ep_virtual_flow}.(i) give the containment and $e_{J(T)}\leq e_{b_{I(T)}}+\varepsilon_{\mathrm{res}}(\delta')$, and~\eqref{eq:ep_block_recursion} at rank $I(T)$ yields~\eqref{eq:ep_main_bound}. For~\eqref{eq:ep_eps_explicit}, $2\bar{U}_j(\delta')=\bar{w}(R)B_j^{-1/2}\sqrt{2\iota_1}$ and $2\bar{V}_i(\delta')=\bar{w}(R)\sqrt{2\iota_1\sum_{j\in\mB_i}\alpha_j^2/B_j}$.
\end{proof}

\begin{remark}[Recovering \Cref{thm:sync-stoch-R-app}]
    \label{rem:ep_recovers_sync}
    At $B_j\equiv1$, the block-martingale bound in~\eqref{eq:ep_eps_explicit} is exactly the quantity $\varepsilon_{\mB_i}(\delta')$ of~\eqref{eq:eps-K}, with a different value of $\delta'$. \Cref{thm:ep_main,thm:sync-stoch-R-app} are then the same statement up to this value; the only difference is the cost model, $|\mS||\mA|$ sampled transitions per iteration there against $L_j(\delta')$ per epoch here.
\end{remark}

\subsubsection{The case \texorpdfstring{$\alpha_j\equiv1$}{αj ≡ 1} and proof of \pdfref{Theorem}{thm:epoch-main-body}}
\label{subsubsec:ep_corners}

\begin{corollary}[$\alpha_j\equiv1$: empirical $\Omega$-regularized value iteration]
    \label{cor:ep_alpha_one}
    Take $\alpha_j\equiv1$, $B_j\equiv B$, and define the blocks directly as $\mB_i:=\llbracket im,(i+1)m-1\rrbracket$, so that $I(T)=\lfloor J(T)/m\rfloor$. Then $F_j=\mT_\Omega$, the block map is $\mT^m_\Omega$, and \Cref{lem:ep_virtual_flow} holds with $\bar{\beta}(R)$ replaced by $\tilde\beta_m(A_R)=\beta_m\gamma_\Omega(A_R)^{m-1}$, with no sub-block budget $\kappa_0$ and no probability $p_m$. With $\bar{U}:=\bar{w}(R)\sqrt{\iota_1/(2B)}$, the batch bound gives $\varepsilon_i(\delta')\leq2m\bar{U}$ and $\varepsilon_{\mathrm{res}}(\delta')\leq2m\bar{U}$, condition~\eqref{eq:ep_R_admissible} becomes
    \begin{equation}
        \label{eq:ep_admissible_alpha_one}
        2m\bar{U}\leq\frac{\tilde{\beta}_m(A_R)}{1+\tilde{\beta}_m(A_R)}\left(R-e_0\right),
    \end{equation}
    and with probability at least $1-\delta$,
    \begin{equation}
        \label{eq:ep_bound_alpha_one}
        \sup_{0\leq j\leq J(T)}e_j\leq R,\qquad
        e_{J(T)}\leq\left(1-\tilde{\beta}_m(A_R)\right)^{I(T)}e_0+\frac{(1+\tilde{\beta}_m(A_R))2m\bar{U}}{\tilde{\beta}_m(A_R)}.
    \end{equation}
\end{corollary}

\begin{proof}
    With $\alpha_j=1$, $\tilde{Q}_j=\mT^{j-b_i}_\Omega(Q_{b_i})$. By \Cref{cor:T_omega_props}, $\mT^k_\Omega(Q^*_\Omega)=Q^*_\Omega+k\rho^*_\Omega\ones$, of action-span $\|Q^*_\Omega\|_\asp$. If $e_{b_i}\leq R$, span non-expansiveness gives $\|\tilde{Q}_j-Q^*_\Omega\|_\spn\leq e_{b_i}\leq R$, which is \Cref{lem:ep_virtual_flow}.(i), hence $\|\tilde{Q}_j\|_\asp\leq A_R$, i.e.\ $R[\mT_\Omega^k(Q_{b_i}),\mT_\Omega^k(Q^*_\Omega)]\leq A_R$ for $k\in\llbracket0,m-1\rrbracket$ (\Cref{def:mixing_objects}.(ii)). Since $\gamma_\Omega$ is non-increasing, \Cref{lem:mixing-properties-of-regularization}.(iv) gives
    \begin{equation*}
        \left\|\mT^m_\Omega(Q_{b_i})-\mT^m_\Omega(Q^*_\Omega)\right\|_\spn\leq\left(1-\beta_m\gamma_\Omega(A_R)^{m-1}\right)e_{b_i},
    \end{equation*}
    which is \Cref{lem:ep_virtual_flow}.(ii) with $\tilde{\beta}_m(A_R)$ in place of $\bar{\beta}(R)$; neither \Cref{lem:pickup_synchronous} nor the probabilistic argument of \Cref{lem:jacobian_bound_synchronous} is needed. \Cref{prop:ep_no_escape} and \Cref{thm:ep_main} then apply verbatim with $\tilde{\beta}_m(A_R)$ in place of $\bar{\beta}(R)$, the geometric sum being bounded by $1/\tilde{\beta}_m(A_R)$ and the residual block adding one copy of $2m\bar{U}$.
\end{proof}

\begin{theorem}[Epoch $\Omega$-regularized $Q$-learning with $\alpha_j\equiv1$]
    \label{thm:epoch-main-body-detailed}
    Let \Cref{assumption:m-step-mixing,assumption:convex-analytic-regularity,assumption:reference_policy,assumption:positive-domination} hold, and set $B(u):=8m^2\bar w^2\iota/(\tilde\beta u)^2$ for $u>0$. Run \Cref{alg:epoch_regularized_q_learning} with $\alpha_j\equiv1$ and a constant batch size $B\geq B(e_0)$, and let $J(T)$ be the number of epochs completed within $T$ transitions. 
    
    Then, with probability at least $1-\delta$, $\sup_{j\leq J(T)}\|Q_j-Q^*_\Omega\|_\spn\leq2e_0$, each epoch completed within the budget lasts at most $L:=m\big\lceil(\sqrt B+\sqrt\iota)^2/(\bar\gamma\tilde\beta\eta_{\min})\big\rceil+1$ transitions, so that $J(T)\geq\lfloor T/L\rfloor$, and
    \begin{equation}
        \label{eq:epoch-main-bound-detailed}
        \|Q_{J(T)}-Q^*_\Omega\|_\spn\leq(1-\tilde\beta)^{\lfloor J(T)/m\rfloor}\|Q_0-Q_\Omega^*\|_\spn+\frac{2\sqrt2\,m\bar w}{\tilde\beta}\sqrt{\frac{\iota}{B}}.
    \end{equation}
    In particular, for $\varepsilon\in(0,e_0)$, the choices $B=\lceil B(\varepsilon/2)\rceil$ and $T\geq mL\lceil\tilde\beta^{-1}\log(2e_0/\varepsilon)\rceil$ give $\|Q_{J(T)}-Q^*_\Omega\|_\spn\leq\varepsilon$, i.e.\ $\tilde O\big(m^4\bar w^2/(\tilde\beta^4\bar\gamma\eta_{\min}\varepsilon^2)\big)$ sampled transitions.
\end{theorem}

\begin{proof}[Proof of \Cref{thm:epoch-main-body-detailed}]
    We apply \Cref{cor:ep_alpha_one} with $R=2e_0$. Then $A_R=\|Q^*_\Omega\|_\asp+2e_0$, so that $\tilde{\beta}_m(A_R)=\tilde\beta$, $\bar{w}(R)=\bar w$ and $q_R=\bar\gamma\tilde\beta\eta_{\min}$; moreover $\iota(\delta')\leq\iota_1\leq\iota$ since $J_{\max}(T)\leq T$.

    \emph{Admissibility.} We have $2m\bar{U}\leq\sqrt2\,m\bar w\sqrt{\iota/B}$ and $1+\tilde\beta\leq2$, so~\eqref{eq:ep_admissible_alpha_one} holds as soon as $\sqrt2\,m\bar w\sqrt{\iota/B}\leq\tilde\beta e_0/2$, i.e.\ $B\geq B(e_0)$.

    \emph{Bound.} With probability at least $1-\delta$, \Cref{cor:ep_alpha_one} gives $\sup_{j\leq J(T)}e_j\leq2e_0$, and its residual term is at most $\frac{2\cdot\sqrt2\,m\bar w\sqrt{\iota/B}}{\tilde\beta}$, which gives~\eqref{eq:epoch-main-bound-detailed}. On the same event, \Cref{prop:ep_no_escape} gives $\sigma_R>J(T)$, and the epoch-length events of \Cref{lem:ep_good_event} give $\ell_j\leq L_j(\delta')\leq L$ for every $j<\min\{\sigma_R,J_{\max}(T)\}$, since $\iota(\delta')\leq\iota$. Thus \Cref{lem:ep_counts}.(iii) gives $J(T)\geq\lfloor T/L\rfloor$, and $I(T)=\lfloor J(T)/m\rfloor$.

    \emph{Sample complexity.} For $\varepsilon\in(0,e_0)$, the choice $B=\lceil B(\varepsilon/2)\rceil\geq B(e_0)$ makes the residual term at most $\varepsilon/2$. For $T\geq mL\lceil\tilde\beta^{-1}\log(2e_0/\varepsilon)\rceil$, which is possible since $\iota$ depends on $T$ only logarithmically, we get $I(T)\geq\tilde\beta^{-1}\log(2e_0/\varepsilon)$, and $1-u\leq e^{-u}$ makes the contraction term at most $\varepsilon/2$. Since $L=\tilde\Theta\big(mB/(\bar\gamma\tilde\beta\eta_{\min})\big)$ and $B=\tilde\Theta\big(m^2\bar w^2/(\tilde\beta\varepsilon)^2\big)$, this gives $\tilde O\big(m^4\bar w^2/(\tilde\beta^3\cdot\bar\gamma\tilde\beta\eta_{\min}\cdot\varepsilon^2)\big)$ sampled transitions.
\end{proof}

\begin{proof}[Proof of \Cref{thm:epoch-main-body}]
    Apply \Cref{thm:epoch-main-body-detailed} with its batch-size condition. Substituting the definition of $\iota$ into~\eqref{eq:epoch-main-bound-detailed} gives~\eqref{eq:epoch-main-bound} with $C=2\sqrt2$. Set $q:=\bar\gamma\tilde\beta\eta_{\min}$. The batch-size condition implies $B\geq\iota$, since $\bar w\geq e_0$, $m\geq1$ and $\tilde\beta\leq1$. Therefore, using $B\geq1$ and $q\leq1$, the epoch-length bound in \Cref{thm:epoch-main-body-detailed} satisfies $L\leq m(4B/q+1)+1\leq6mB/q$. Hence $J(T)\geq\lfloor T/L\rfloor\geq\lfloor Tq/(6mB)\rfloor$, which gives the claimed epoch-count bound with $c=1/6$. The containment and accuracy guarantees use the batch size and transition budget specified in that theorem.
\end{proof}

\begin{corollary}[Epoch $\Omega$-regularized $Q$-learning for $m=1$]
    \label{cor:ep-m-1}
    Let $m=1$. Then $\tilde\beta=\beta_1$, so that in \Cref{thm:epoch-main-body-detailed} the domination profile only enters through the coverage factor $\bar\gamma\beta_1\eta_{\min}$ of the epoch length $L$, and the sample complexity is $\tilde O\big(\bar w^2/(\beta_1^3\cdot\bar\gamma\beta_1\eta_{\min}\cdot\varepsilon^2)\big)$.
\end{corollary}

\begin{proof}
    For $m=1$, $\tilde\beta_1(A_R)=\beta_1$ for every $R$ (\Cref{def:localized-radii}), so the contraction constants of the proof of \Cref{thm:epoch-main-body-detailed} do not depend on the radius, whereas the coverage constant $q_R=\gamma_\Omega(A_R)\,\tilde\beta_1(A_R)\,\eta_{\min}$ of \Cref{lem:ep_cover_time} does. At $R=2e_0$, it equals $\bar\gamma\beta_1\eta_{\min}$, and the claim is \Cref{thm:epoch-main-body-detailed} with $m=1$.
\end{proof}

\begin{remark}[Acting with $\mu$ or with the greedy policies]
    \label{rem:ep_behavior_mu}
    The update of \Cref{alg:epoch_regularized_q_learning} is off-policy: its target $r(s,a)+\Omega^*(Q_j)(s')$ is an expectation under the greedy policy and does not sample the next action, so the noise analysis only uses that the successor of a visited pair $(s,a)$ is distributed as $p(\cdot|s,a)$ (\Cref{lem:ep_noise}.(ii)), whatever the law of the action. Running the algorithm with actions $a_n\sim\mu(\cdot|s_n)$ is therefore covered by the same analysis, with $\gamma_\Omega$ replaced by $1$ in \Cref{lem:ep_cover_time}: $q_R$ becomes $\beta_m\eta_{\min}$, $L$ becomes $m\lceil(\sqrt B+\sqrt\iota)^2/(\beta_m\eta_{\min})\rceil+1$ in \Cref{thm:epoch-main-body-detailed}, and the sample complexity becomes $\tilde O\big(m^4\bar w^2/(\tilde\beta^3\cdot\beta_m\eta_{\min}\cdot\varepsilon^2)\big)$, the factor $\tilde\beta^3$ of the contraction being unchanged. Thus the coverage lower bound under the greedy behavior is $\bar\gamma^m$ times that under $\mu$, and the corresponding leading sample-complexity upper bound is inflated by a factor $\bar\gamma^{-m}$.

    Our criterion, uniform accuracy on all pairs of a tabular MDP, favors a fixed well-mixing behavior, and cannot credit what acting with the learned policy brings: the reward collected during learning, data concentrated where a good policy goes and, beyond the tabular setting, the stability of on-policy data distributions under function approximation \citep[\S11.3]{sutton2018reinforcement}. What the analysis does show is that regularization makes this choice safe: with a hard maximum, the greedy policy is deterministic and needs an added exploration mechanism, whereas under \Cref{assumption:positive-domination} it explores by itself, at a quantified rate.
\end{remark}

\section{Planning: Setting and Main Results}
\label{app:planning_results}
 
This section sets up the objects of \Cref{subsec:planning_results} in detail and states the full planning theorems, which are proved in \Cref{app:planning}. From a reset distribution $\xi\in\Delta(\mS)$, planning computes the regularized $H$-step solution bootstrapped at depth $H$ by a boundary value $Q_B\in\bR^{\mS'\times\mA}$, e.g.\ a value estimate or a heuristic, and uses its root level to approximate $Q^*_\Omega$. It is learning localized in space, since only the states reachable from $\xi$ are visited, and in time, since only the levels closest to the root matter.

\paragraph{Depth-dependent objects.}
Let $F(X):=\{s'\in\mS\mid\exists(s,a)\in X\times\mA,\ p(s'|s,a)>0\}$ for $X\subseteq\mS$. The \emph{reachable layers} are $\mS_0:=\mathrm{supp}(\xi)$ and $\mS_{h+1}:=F(\mS_h)$; we write $\mX_h:=\mS_h\times\mA$. We define $\mS':=\bigcup_{h\geq0}\mS_h$ for the reachable set, which satisfies $F(\mS')\subseteq\mS'$, so that the MDP restricted to $\mS'$ is well defined and $\mT_\Omega$ acts on $\bR^{\mS'\times\mA}$, and $\bar S_H:=\max_{h\leq H}|\mS_h|$. A \emph{profile} is a family $(Q_h)_{h=0}^H$ of depth-dependent $Q$-functions $Q_h\in\bR^{\mX_h}$. Since the successors of $\mS_h$ lie in $\mS_{h+1}$, the \emph{depth-dependent} operator $\mT_{\Omega,h}:\bR^{\mX_{h+1}}\to\bR^{\mX_h}$, $\mT_{\Omega,h}(Q)(s,a):=r(s,a)+\sum_{s'\in\mS_{h+1}}p(s'|s,a)\,\Omega^*(Q)(s')$, is well defined. The \emph{finite-horizon target} (or \emph{planning solution}) is the unique profile defined backward by $Q^*_{\Omega,H}:=Q_B|_{\mX_H}$ and $Q^*_{\Omega,h}:=\mT_{\Omega,h}(Q^*_{\Omega,h+1})$. It solves this triangular system exactly, whereas $Q^*_\Omega$ only solves $\mT_\Omega(Q^*_\Omega)=Q^*_\Omega+\rho^*_\Omega\ones$, so the planners need no centering projection.
 
\paragraph{Mixing and target.}
Throughout the planning appendices, \Cref{assumption:convex-analytic-regularity,assumption:reference_policy,assumption:pl_local_mixing,assumption:positive-domination} hold, \Cref{assumption:m-step-mixing} is not assumed, and $\mu$, $m$, $\beta_m$, $\nu$ denote the objects of \Cref{assumption:pl_local_mixing}, with respect to which \Cref{assumption:reference_policy} and \Cref{assumption:positive-domination} are understood; we write $\eta(s,a):=\nu(s)\mu(a|s)$ and $\eta_{\min}:=\min_{\mS_\infty\times\mA}\eta$.
 
\begin{remark}[Reading \Cref{assumption:pl_local_mixing}]
    \label{rem:pl_local_mixing}
    \leavevmode
    \begin{enumerate}[label=(\roman*)]
        \item \emph{The core $\mS_\infty$ is closed}: by (i), $\mS_\infty=\mS_{h_1+1}=F(\mS_{h_1})=F(\mS_\infty)$. On the MDP restricted to $\mS_\infty$, (ii) is \Cref{assumption:m-step-mixing} with a full-support $\nu$, so \Cref{thm:regularized_fixed_point} gives $(Q^*_\Omega,\rho^*_\Omega)$ on $\mS_\infty\times\mA$, with $Q^*_\Omega$ unique up to an additive constant under \Cref{assumption:positive-domination}.
        \item \emph{Global case}: if \Cref{assumption:m-step-mixing} holds, then $K^m_\mu(s'|s,a)>0$ for all $(s,a,s')$, so $\mS_h=\mS$ for $h\geq m$ and \Cref{assumption:pl_local_mixing} holds with $\mS_\infty=\mS$ and $h_1=m$. The spatial localization is then vacuous, and only the temporal one remains.
        \item \emph{Transient layers}: no mixing condition is imposed on $\mS_0,\dots,\mS_{h_1-1}$. These layers affect the bounds through the transient depth $h_1$ and the constants on the reachable layers, including the shallow-level coverage constants (\Cref{lem:pl_coverage}).
        \item \emph{When (i) holds}: $\mS_h$ is the set of states reachable in exactly $h$ steps from $\mathrm{supp}(\xi)$ in the union graph $G$ (\Cref{def:union_graph}), so $(\mS_h)_h$ is eventually periodic, and (i) holds for some finite $h_1$ as soon as every strongly connected component of $G$ that contains a cycle and is reachable from $\mathrm{supp}(\xi)$ is aperiodic, e.g.\ contains a state $s$ with $p(s|s,a)>0$ for some $a$: closed walks in such a component have all large enough lengths, so a state reached at arbitrarily large depths is reached at every large enough depth.
    \end{enumerate}
\end{remark}
 
By (i), $Q^*_\Omega$ is defined on $\mS_\infty\times\mA$; it extends uniquely to $\mS'\times\mA$ through the fixed-point equation (\Cref{lem:pl_extension}), and the norms of $Q^*_\Omega$ are taken over $\mS'\times\mA$.
 
\paragraph{Errors and constants.}
We measure the quality of the boundary by $\Delta_B:=\|Q_B-Q^*_\Omega\|_{\mX_H,\spn}$ and that of the initial profile by $D_0:=\max_{h<H}\|Q_{0,h}-Q^*_{\Omega,h}\|_\spn$, and use $\tilde\beta$, $\bar w$ and $H_{\mathrm{eff}}$ of \Cref{sec:results} with $e_0:=D_0+\Delta_B$. We assume $H\geq h_1+m$. 

The bounds below also hold when $e_0$ is any positive upper bound on $D_0+\Delta_B$, with all localization constants computed using this bound. For example, if $Q_{0,h}=0$ for $h<H$ and $Q_B$ is fixed, \Cref{lem:pl_uniform_bounds} gives
\begin{equation*}
    D_0+\Delta_B \leq \|Q^*_\Omega\|_\spn+2\Delta_B.
\end{equation*}
Since $\mS_H=\mS_\infty$ for $H\geq h_1$, $\Delta_B$ and this upper bound are independent of $H$. Thus $\tilde{\beta}$, $\bar{w}$ and $H_{\mathrm{eff}}$ can be fixed before choosing the planning horizon.

With $\eta^{\mathrm{res}}_h(s,a):=(\xi P^h_\mu)(s)\,\mu(a|s)$, the probability that a rollout started from $\xi$ and following $\mu$ visits $(s,a)$ at depth $h$, which is positive exactly on $\mX_h$ (\Cref{lem:pl_reachability}), the coverage constants under $\mu$ and under a behavior dominating $\gamma_\Omega(A)\mu$ are
\begin{equation}
    \label{eq:pl_q}
    q_\mu:=\min\Big\{\min_{h<h_1+m}\ \min_{(s,a)\in\mX_h}\eta^{\mathrm{res}}_h(s,a),\ \beta_m\eta_{\min}\Big\}, \qquad q_\pi(A):=\gamma_\Omega(A)^{h_1+m}\,q_\mu \quad(A\geq0),
\end{equation}
both positive, $q_\pi(A)$ by \Cref{assumption:positive-domination}.
 
\paragraph{Algorithms and results.}
With a model, the backward sweep $Q_h\leftarrow\mT_{\Omega,h}(Q_{h+1})$ (\Cref{alg:pl_backward_sync}) returns the planning solution exactly, and \Cref{prop:pl_horizon_main} bounds its horizon bias and its cost. Without a model, \Cref{alg:epoch_regularized_q_planning} estimates each depth operator from rollouts started from $\xi$. Whereas the learning algorithms iterate the stationary update $Q_{j+1}=\mT_\Omega(Q_j)$, it applies, up to sampling noise, the depth-shifted update $Q_{j+1,h}=\mT_{\Omega,h}(Q_{j,h+1})$: level $h$ at epoch $j+1$ is computed from level $h+1$ at epoch $j$, which is frozen during epoch $j$. Without noise, this recursion reaches the planning solution after $H-h$ epochs at level $h$ (\Cref{lem:pl_sync_nilpotent}), so the initialization is forgotten exactly rather than geometrically. As in \Cref{alg:epoch_regularized_q_learning}, a cap makes the number of samples per pair and per level deterministic.
 
\begin{theorem}[Epoch $\Omega$-regularized $Q$-planning, general bound]
    \label{thm:pl-epoch-main}
    Let $q:=q_\pi(\|Q^*_\Omega\|_\asp+2e_0)$, or $q:=q_\mu$ if the actions are drawn from $\mu$ instead of the greedy policies (\Cref{rem:pl_behavior}). Let $J\geq H$, $\iota_H:=\log(4JH\bar S_H|\mA|/\delta)$, $M:=h_1+H_{\mathrm{eff}}$ and $B(u):=2M^2\bar w^2\iota_H/u^2$. Run \Cref{alg:epoch_regularized_q_planning} for $J$ epochs with a constant batch size $B\geq B(e_0)$. Then, with probability at least $1-\delta$, each epoch lasts at most $H\lceil(\sqrt B+\sqrt{\iota_H})^2/q\rceil$ transitions and
    \begin{equation}
        \label{eq:pl-epoch-main}
        \big\|Q_{J,0}-Q^*_\Omega\big\|_{\mX_0,\spn}\leq(1-\tilde\beta)^{\lfloor(H-h_1)/m\rfloor}\Delta_B+2M\bar w\sqrt{\frac{\iota_H}{2B}}.
    \end{equation}
    In particular, for $\varepsilon\in(0,e_0)$, the choices $H=\lceil h_1+m+H_{\mathrm{eff}}\log_+(2\Delta_B/\varepsilon)\rceil$, $J=H$, and $B=\lceil B(\varepsilon/2)\rceil$ give $\|Q_{J,0}-Q^*_\Omega\|_{\mX_0,\spn}\leq\varepsilon$, matching the sample complexity of \Cref{thm:pl-epoch-main-body}.
\end{theorem}
 
The bound~\eqref{eq:pl-epoch-main} has the structure of~\eqref{eq:epoch-main-bound-detailed}, with the horizon bias in place of the initialization term. Finally, \Cref{alg:pl_reset_learning} merges the $H$ tables into a single table learned from resets, which removes the extra factor $1/\tilde\beta$ (\Cref{thm:pl_reset_learning}).
    
\section{Analysis of the Planning Algorithms}
\label{app:planning}

This section proves the results of \Cref{app:planning_results} and analyses \Cref{alg:pl_reset_learning}, which merges the depth-indexed tables of the epoch planner into a single one.

The analysis rests on three facts. First, the planning solution solves its equations exactly (\Cref{app:planning_results}), so the planners forget their initialization after $H$ epochs. Second, a single lemma (\Cref{lem:pl_master}) bounds the error at depth $h$ by the Bellman residuals of the deeper levels, damped at the rate of \Cref{lem:mixing-properties-of-regularization}.(iv) every $m$ levels: it gives the planning depth and shows that noise does not accumulate across the horizon. Third, resets provide the exploration of the shallow levels, and the mixing of \Cref{assumption:pl_local_mixing} that of the deep ones (\Cref{lem:pl_coverage}). The model-based sweep is treated in \Cref{subsec:pl_sync}, epoch planning in \Cref{subsec:pl_epoch} and epoch learning with resets in \Cref{subsec:pl_reset_learning}.

\subsection{Reachability and the extended target}
\label{subsec:pl_setting}

\begin{lemma}[Reachability under $\mu$]
    \label{lem:pl_reachability}
    For every $h\geq0$ and $s\in\mS$, $(\xi P^h_\mu)(s)>0$ if and only if $s\in\mS_h$.
\end{lemma}

\begin{proof}
    Induction on $h$, the case $h=0$ being the definition of $\mS_0$. Since $\mu$ has full support, $P_\mu(s'|s)>0$ if and only if $p(s'|s,a)>0$ for some $a$, so $(\xi P^{h+1}_\mu)(s')=\sum_s(\xi P^h_\mu)(s)P_\mu(s'|s)>0$ if and only if $s'\in F(\mS_h)=\mS_{h+1}$.
\end{proof}

\begin{lemma}[Extension of the target to the transient states]
    \label{lem:pl_extension}
    Let $\mathcal{V}:=\mS'\setminus\mS_\infty$. The accessibility relation restricted to $\mathcal{V}$ has no directed cycle, and $Q^*_\Omega$ has a unique extension to $\mS'\times\mA$ such that
    \begin{equation}
        \label{eq:pl_extension}
        \mT_\Omega(Q^*_\Omega)=Q^*_\Omega+\rho^*_\Omega\ones \qquad\text{on }\mS'\times\mA.
    \end{equation}
    Shifting $Q^*_\Omega$ by a constant on $\mS_\infty$ shifts its extension by the same constant.
\end{lemma}

\begin{proof}
    If $s\in\mathcal{V}\cap\mS_j$ lay on a cycle of length $\ell\geq1$ within $\mathcal{V}$, then $s\in\mS_{j+\ell i}$ for every $i\geq0$, hence $s\in\mS_h=\mS_\infty$ for some $h\geq h_1$, a contradiction. The subgraph induced on $\mathcal{V}$ is therefore acyclic, and its vertices admit a topological order in which every successor of $s\in\mathcal{V}$ lies either in $\mS_\infty$ or later in the order. Defining, along the reverse of that order,
    \begin{equation*}
        Q^*_\Omega(s,a):=r(s,a)-\rho^*_\Omega+\sum_{s'}p(s'|s,a)\,\Omega^*(Q^*_\Omega)(s'), \qquad s\in\mathcal{V},
    \end{equation*}
    each right-hand side only involves values already defined, which gives existence and uniqueness. The last claim follows from the translation equivariance of $\Omega^*$ (\Cref{lem:omega_star_props}.(iii)).
\end{proof}

We keep the notation $Q^*_\Omega$ for the extension. By~\eqref{eq:pl_extension} and \Cref{cor:T_omega_props}, $\mT^k_\Omega(Q^*_\Omega)=Q^*_\Omega+k\rho^*_\Omega\ones$ on $\mS'\times\mA$ for every $k\geq0$; this is the only property of the extension used below, and it allows us to compare a profile to $Q^*_\Omega$ at the depths $h<h_1$.

\subsection{Depth-dependent objects}
\label{subsec:pl_objects}

\begin{lemma}[Depth-indexed average-path representation]
    \label{lem:pl_avg_path_depth}
    Let $Q,Q'\in\bR^{\mX_{h+1}}$ and let $\bar{\pi}:=\bar{\pi}_{\Omega,[Q,Q']}$ be the average-path policy of \Cref{def:mixing_objects}.(i). Then
    \begin{equation*}
        \mT_{\Omega,h}(Q)-\mT_{\Omega,h}(Q')=K^{(h)}_{\bar{\pi}}\left(Q-Q'\right), \qquad K^{(h)}_{\bar{\pi}}(s',a'|s,a):=p(s'|s,a)\,\bar{\pi}(a'|s'),
    \end{equation*}
     where $K^{(h)}_{\bar{\pi}}$, the restriction of the state--action kernel $K_{\bar\pi}$ to $\mX_h\times\mX_{h+1}$, is row-stochastic from $\mX_h$ to $\mX_{h+1}$, and $\bar{\pi}\geq\gamma_\Omega(R[Q,Q'])\,\mu$ on $\mX_{h+1}$.
\end{lemma}

\begin{proof}
    This is the computation of \Cref{lem:mixing-properties-of-regularization}.(i) restricted to $(s,a)\in\mX_h$, all successors lying in $\mS_{h+1}$; domination is \Cref{lem:mixing-properties-of-regularization}.(ii), the greedy policy at $s'$ depending on $Q$ only through $Q_{s'}$.
\end{proof}

For policies $\pi^{(h)},\dots,\pi^{(h+k-1)}$, we write $K_{h:h+k}:=K^{(h)}_{\pi^{(h)}}K^{(h+1)}_{\pi^{(h+1)}}\cdots K^{(h+k-1)}_{\pi^{(h+k-1)}}$, with $K_{h:h}:=I$, for the product of depth kernels from $\mX_h$ to $\mX_{h+k}$ in the sense of \Cref{convention:product_order}; it is row-stochastic.

\begin{lemma}[The planning solution is a restriction]
    \label{lem:pl_restriction}
    For every $Q\in\bR^{\mS'\times\mA}$ and $h\in\llbracket0,H\rrbracket$, $\big(\mT^{H-h}_\Omega Q\big)\big|_{\mX_h}=\mT_{\Omega,h}\circ\cdots\circ \mT_{\Omega,H-1}\big(Q|_{\mX_H}\big)$. In particular, $Q^*_{\Omega,h}=\big(\mT^{H-h}_\Omega Q_B\big)\big|_{\mX_h}$ and $\big(\mT^{H-h}_\Omega Q^*_\Omega\big)\big|_{\mX_h}=Q^*_\Omega|_{\mX_h}+(H-h)\rho^*_\Omega\ones$.
\end{lemma}

\begin{proof}
    Backward induction on $h$: at $(s,a)\in\mX_h$, $\mT_\Omega$ only involves the values of $Q$ at the states $s'\in\mS_{h+1}$, which coincide with the depth-$(h+1)$ restriction by the induction hypothesis. The last identity is~\eqref{eq:pl_extension} iterated.
\end{proof}

\paragraph{Constants.}
For a radius $R\geq0$, we use the localized constants $A_R=\|Q^*_\Omega\|_\asp+R$, $\tilde{\beta}_m(A_R)=\beta_m\gamma_\Omega(A_R)^{m-1}$ and $\bar{w}(R)=\|Q^*_\Omega\|_\spn+R+\|\Omega^*(0)\|_{\mS,\spn}$ of \Cref{def:localized-radii}; here $R$ bounds the span distance to $Q^*_\Omega$ at every depth. With $e_0=D_0+\Delta_B$ as in \Cref{app:planning_results}, the constants of the main text are $\tilde\beta=\tilde\beta_m(A_{2e_0})$ and $\bar w=\bar w(2e_0)$. We also write
\begin{equation}
    \label{eq:pl_Heff}
    H_{\mathrm{eff}}(A):=\frac{m}{\tilde{\beta}_m(A)}, \qquad M(h,A):=(h_1-h)_++H_{\mathrm{eff}}(A), \qquad M(A):=M(0,A)=h_1+H_{\mathrm{eff}}(A),
\end{equation}
so that $H_{\mathrm{eff}}=H_{\mathrm{eff}}(A_{2e_0})$. By \Cref{assumption:positive-domination}, $\tilde{\beta}_m(A)>0$ for every finite $A$.

\begin{lemma}[Distance between the planning solution and the target]
    \label{lem:pl_uniform_bounds}
    For every $h\in\llbracket0,H\rrbracket$, $\|Q^*_{\Omega,h}-Q^*_\Omega\|_{\mX_h,\spn}\leq\Delta_B$. Consequently $\|Q^*_{\Omega,h}\|_\asp\leq A_{\Delta_B}$ and $\|Q^*_{\Omega,h}\|_\spn\leq\|Q^*_\Omega\|_\spn+\Delta_B$, uniformly in $h$ and $H$.
\end{lemma}

\begin{proof}
    By \Cref{lem:pl_restriction}, $Q^*_{\Omega,l}=\mT_{\Omega,l}(Q^*_{\Omega,l+1})$ and $V_l:=Q^*_\Omega|_{\mX_l}+(H-l)\rho^*_\Omega\ones$ satisfies $V_l=\mT_{\Omega,l}(V_{l+1})$ for every $l<H$, with $V_H=Q^*_\Omega|_{\mX_H}$. \Cref{lem:pl_avg_path_depth} gives $Q^*_{\Omega,l}-V_l=K^{(l)}_{\pi^{(l)}}(Q^*_{\Omega,l+1}-V_{l+1})$ with $\pi^{(l)}:=\bar{\pi}_{\Omega,\left[Q^*_{\Omega,l+1},Q^*_\Omega|_{\mX_{l+1}}\right]}$, the greedy policies being invariant under constants. Unrolling from $l=h$ to $l=H$,
    \begin{equation*}
        Q^*_{\Omega,h}-Q^*_\Omega|_{\mX_h}-(H-h)\rho^*_\Omega\ones=K_{h:H}\big(Q_B-Q^*_\Omega\big)\big|_{\mX_H}.
    \end{equation*}
    $K_{h:H}$ is row-stochastic, hence non-expansive in span (\Cref{lem:row-stochastic-span}), and the span annihilates constants, whence the first claim. The other two follow from $\|\cdot\|_\asp\leq\|\cdot\|_\spn$, the invariance of $\|\cdot\|_\asp$ under constants and the triangle inequality.
\end{proof}

\subsection{The master lemma}
\label{subsec:pl_master}

For $0\leq h\leq h+k\leq H$, the \emph{contracting block count}
\begin{equation}
    \label{eq:pl_block_count}
    n(h,k):=\left\lfloor\frac{\big(h+k-\max\{h,h_1\}\big)_+}{m}\right\rfloor
\end{equation}
is the number of disjoint blocks of $m$ consecutive depths inside $\llbracket\max\{h,h_1\},h+k-1\rrbracket$; when $h\geq h_1$, $n(h,k)=\lfloor k/m\rfloor$.

\begin{lemma}[Contraction of a product of depth kernels]
    \label{lem:pl_product}
    Let $A\geq0$, $0\leq h\leq h+k\leq H$, and let $\pi^{(h)},\dots,\pi^{(h+k-1)}$ be policies with $\pi^{(l)}\geq\gamma_\Omega(A)\,\mu$ for every $l$. Then $\|K_{h:h+k}\|_{\spn,\op}\leq(1-\tilde{\beta}_m(A))^{n(h,k)}$.
\end{lemma}

\begin{proof}
    Group the factors of depths $l\geq\max\{h,h_1\}$ into $n(h,k)$ consecutive blocks of exactly $m$ factors, leaving ungrouped the factors of depths $l<h_1$ and the fewer than $m$ remaining ones. A block starting at depth $l\geq h_1$ is the product of $m$ depth kernels from $\mX_l=\mS_\infty\times\mA$ to $\mX_{l+m}=\mS_\infty\times\mA$, whose policies dominate $\gamma_\Omega(A)\mu$; as $\mS_\infty$ is closed, these kernels coincide with the unrestricted ones on $\mS_\infty\times\mA$. The proof of \Cref{lem:terminal_action} only uses the minorization $K^m_\mu(\cdot|s,a)\geq\beta_m\nu$ at the starting pairs, which \Cref{assumption:pl_local_mixing}.(ii) provides on $\mS_\infty\times\mA$, so it bounds the Dobrushin coefficient of each block by $1-\beta_m\gamma_\Omega(A)^{m-1}=1-\tilde{\beta}_m(A)$, and \Cref{lem:row-stochastic-span} bounds its operator span seminorm by the same quantity. The ungrouped factors are row-stochastic, hence of operator span seminorm at most $1$, and we conclude by submultiplicativity~\eqref{eq:sub-multiplicativity}.
\end{proof}

\begin{lemma}[Geometric sum]
    \label{lem:pl_geometric_sum}
    For every $h\geq0$ and $A\geq0$, $\sum_{k\geq0}(1-\tilde{\beta}_m(A))^{n(h,k)}\leq M(h,A)$.
\end{lemma}

\begin{proof}
    For $k<(h_1-h)_+$, $n(h,k)=0$, and these terms contribute $(h_1-h)_+$. For $k\geq(h_1-h)_+$, $n(h,k)=\lfloor(k-(h_1-h)_+)/m\rfloor$ takes each value $q\geq0$ exactly $m$ times, and these terms contribute $m\sum_{q\geq0}(1-\tilde{\beta}_m(A))^q=H_{\mathrm{eff}}(A)$.
\end{proof}

The next lemma drives every planning bound. It compares a partial profile to the planning solution through its Bellman residuals: the model-based sweep has no residual, and the epoch planner is a profile whose residuals are the noise of successive epochs (\Cref{subsec:pl_epoch}).

\begin{lemma}[Master lemma]
    \label{lem:pl_master}
    Let $0\leq h_0\leq h_e\leq H$, let $Q_l\in\bR^{\mX_l}$ for $l\in\llbracket h_0,h_e\rrbracket$, with residuals
    \begin{equation}
        \label{eq:pl_residual}
        r_l:=Q_l-\mT_{\Omega,l}(Q_{l+1}), \qquad l\in\llbracket h_0,h_e-1\rrbracket,
    \end{equation}
    and assume $\|Q_l\|_\asp\leq A$ for $l\in\llbracket h_0+1,h_e\rrbracket$, with $A\geq A_{\Delta_B}$. Then, writing $\tilde{\beta}:=\tilde{\beta}_m(A)$,
    \begin{equation}
        \label{eq:pl_master}
        \begin{aligned}
            \big\|Q_{h_0}-Q^*_{\Omega,h_0}\big\|_\spn
            &\leq(1-\tilde{\beta})^{n(h_0,h_e-h_0)}\big\|Q_{h_e}-Q^*_{\Omega,h_e}\big\|_\spn+\sum_{l=h_0}^{h_e-1}(1-\tilde{\beta})^{n(h_0,l-h_0)}\|r_l\|_\spn\\
            &\leq(1-\tilde{\beta})^{n(h_0,h_e-h_0)}\big\|Q_{h_e}-Q^*_{\Omega,h_e}\big\|_\spn+2M(h_0,A)\max_{h_0\leq l<h_e}\|r_l\|_\infty.
        \end{aligned}
    \end{equation}
\end{lemma}

\begin{proof}
    Let $e_l:=Q_l-Q^*_{\Omega,l}$. By~\eqref{eq:pl_residual}, the definition of the planning solution and \Cref{lem:pl_avg_path_depth},
    \begin{equation*}
        e_l=r_l+\mT_{\Omega,l}(Q_{l+1})-\mT_{\Omega,l}(Q^*_{\Omega,l+1})=r_l+K^{(l)}_{\pi^{(l)}}e_{l+1}, \qquad \pi^{(l)}:=\bar{\pi}_{\Omega,[Q_{l+1},Q^*_{\Omega,l+1}]}.
    \end{equation*}
    Unrolling from $h_0$ to $h_e$ gives $e_{h_0}=K_{h_0:h_e}e_{h_e}+\sum_{l=h_0}^{h_e-1}K_{h_0:l}\,r_l$. By \Cref{lem:pl_uniform_bounds}, $R[Q_{l+1},Q^*_{\Omega,l+1}]\leq\max\{A,A_{\Delta_B}\}=A$, so every $\pi^{(l)}$ dominates $\gamma_\Omega(A)\mu$ (\Cref{lem:pl_avg_path_depth}) and \Cref{lem:pl_product} applies to each product. The triangle inequality gives the first bound, and $\|\cdot\|_\spn\leq2\|\cdot\|_\infty$ together with \Cref{lem:pl_geometric_sum} gives the second.
\end{proof}

\begin{corollary}[Horizon bias and planning depth]
    \label{cor:pl_horizon_bias}
    For every $h\in\llbracket0,H\rrbracket$,
    \begin{equation}
        \label{eq:pl_horizon_bias}
        \big\|Q^*_{\Omega,h}-Q^*_\Omega\big\|_{\mX_h,\spn}\leq\left(1-\tilde{\beta}_m(A_{\Delta_B})\right)^{n(h,H-h)}\Delta_B,
    \end{equation}
    and the right-hand side is at most $\varepsilon$ as soon as
    \begin{equation}
        \label{eq:pl_depth}
        H\geq\max\{h,h_1\}+m+H_{\mathrm{eff}}(A_{\Delta_B})\log_+\left(\frac{\Delta_B}{\varepsilon}\right).
    \end{equation}
\end{corollary}

\begin{proof}
    By \Cref{lem:pl_restriction}, the profile $V_l:=Q^*_\Omega|_{\mX_l}+(H-l)\rho^*_\Omega\ones$ satisfies $V_l=\mT_{\Omega,l}(V_{l+1})$ for $l<H$, and $\|V_l\|_\asp=\|Q^*_\Omega|_{\mX_l}\|_\asp \leq \|Q^*_\Omega\|_\asp\leq A_{\Delta_B}$. \Cref{lem:pl_master} with $h_0=h$, $h_e=H$, zero residuals and $A=A_{\Delta_B}$ gives~\eqref{eq:pl_horizon_bias}, since $\|V_H-Q^*_{\Omega,H}\|_\spn=\Delta_B$ and the span annihilates constants. Finally, $(1-\tilde{\beta})^n\leq e^{-\tilde{\beta}n}$ and $n(h,H-h)\geq(H-\max\{h,h_1\})/m-1$ give~\eqref{eq:pl_depth}.
\end{proof}

\Cref{lem:pl_master} has two consequences. First, noise does not accumulate across the horizon. Without contraction, the errors of the $H$ levels of a backward induction add up. Here, the residual at depth $l$ is damped by $(1-\tilde{\beta})^{n(h_0,l-h_0)}$ before it reaches depth $h_0$, so only the $M(h_0,A)$ levels closest to $h_0$ matter. Second, $H_{\mathrm{eff}}(A)$ plays the role of the discounted horizon $1/(1-\gamma)$ of discounted planning, but it is set by the regularizer through $\gamma_\Omega$. The boundary mismatch enters the sufficient depth~\eqref{eq:pl_depth} through both $\log_+(\Delta_B/\varepsilon)$ and the effective horizon $H_{\mathrm{eff}}(A_{\Delta_B})$, since $A_{\Delta_B}=\|Q^*_\Omega\|_\asp+\Delta_B$. A better bootstrap can therefore allow a shallower lookahead by reducing the logarithmic term and potentially improving the contraction bound.

\subsection{Coverage}
\label{subsec:pl_coverage}

Two mechanisms explore in planning: resets cover the shallow levels, and the mixing of \Cref{assumption:pl_local_mixing}.(ii) covers the deep ones. The following lemma gives the coverage constants~\eqref{eq:pl_q}.

\begin{lemma}[Level coverage]
    \label{lem:pl_coverage}
    Let $\mathcal{G}$ be a $\sigma$-algebra and let $(s_h,a_h)_{h<H}$ be a trajectory with $s_0\sim\xi$, $a_h\sim\pi_h(\cdot|s_h)$ and $s_{h+1}\sim p(\cdot|s_h,a_h)$, where the policies $\pi_h$ are $\mathcal{G}$-measurable and the trajectory is otherwise independent of $\mathcal{G}$. Let $E\in\mathcal{G}$ be an event on which $\pi_h\geq\gamma_h\mu$ for every $h$, with deterministic $\gamma_h\in[0,1]$, and write $\mF_n:=\sigma(\mathcal{G},s_0,a_0,\dots,s_n,a_n)$. Then, almost surely on $E$, for every $h<H$ and $(s,a)\in\mX_h$:
    \begin{enumerate}[label=(\roman*)]
        \item if $h\geq h_1+m$, $\bP\big((s_h,a_h)=(s,a)\mid\mF_{h-m}\big)\geq\big(\prod_{l=h-m+1}^{h}\gamma_l\big)\beta_m\,\eta(s,a)$;
        \item if $h<h_1+m$, $\bP\big((s_h,a_h)=(s,a)\mid\mathcal{G}\big)\geq\big(\prod_{l=0}^{h}\gamma_l\big)\eta^{\mathrm{res}}_h(s,a)$.
    \end{enumerate}
    In particular, almost surely on $E$, $\bP((s_h,a_h)=(s,a)\mid\mathcal{G})\geq q_\mu$ if $\gamma_h\equiv1$, and $\geq q_\pi(A)$ if $\gamma_h\equiv\gamma_\Omega(A)$.
\end{lemma}

\begin{proof}
    \emph{(i)}. Since $h-m\geq h_1$, $s_{h-m}\in\mS_{h-m}=\mS_\infty$, so $K^m_\mu(\cdot|s_{h-m},a_{h-m})\geq\beta_m\nu$ by \Cref{assumption:pl_local_mixing}.(ii). The proof of \Cref{lem:mixing-properties-of-regularization}.(v) only uses this minorization at the conditioning pair and the domination of the $m$ following policies, and applies verbatim.

    \emph{(ii)}. Conditionally on $\mathcal{G}$, $(s_h)$ is a Markov chain with kernels $P_{\pi_l}$, and the transport inequality $P_{\pi_l}\geq\gamma_lP_\mu$ of the proof of \Cref{lem:mixing-properties-of-regularization}.(iii) holds entrywise on $E$. Hence $\bP(s_h=s\mid\mathcal{G})\geq\big(\prod_{l<h}\gamma_l\big)(\xi P^h_\mu)(s)$, and $\pi_h(a|s)\geq\gamma_h\mu(a|s)$ concludes.

    For the last claim, integrate (i) and use $\gamma_\Omega(A)^{h+1}\geq\gamma_\Omega(A)^{h_1+m}$ in (ii) and $\gamma_\Omega(A)^m\geq\gamma_\Omega(A)^{h_1+m}$ in (i).
\end{proof}

\begin{remark}[Which behavior to use]
    \label{rem:pl_behavior}
    The sample $r(s,a)+\Omega^*(Q)(s')$ is unbiased for $\mT_{\Omega,h}(Q)(s,a)$ whatever the law of the action $a$, so the behavior policy only affects which pairs are visited. The factor $\gamma_\Omega^{m-1}$ of the contraction rate comes from the average-path policy of \Cref{lem:pl_avg_path_depth}, internal to the operator, whereas the factor $\gamma_\Omega(A)^{h_1+m}$ of $q_\pi(A)$ is the price of exploring with the greedy policies, and disappears under the behavior $\mu$, exactly as on a single trajectory (\Cref{rem:ep_behavior_mu}). The shallow levels may however be covered slowly, since $(\xi P^h_\mu)(s)$ is small for states reachable only through a narrow path; if $\xi$ and $p$ are bounded below by $\xi_{\min}$ and $p_{\min}$ on their supports, then $\eta^{\mathrm{res}}_h\geq\xi_{\min}(\mu_{\min}p_{\min})^{h}\mu_{\min}$ on $\mX_h$.
\end{remark}

\subsection{Model-based planning}
\label{subsec:pl_sync}

\begin{proposition}[Horizon bias and model-based planning]
    \label{prop:pl_horizon_main}
    For every $h\leq H$,
    \begin{equation}
        \label{eq:pl-horizon-main}
        \big\|Q^*_{\Omega,h}-Q^*_\Omega\big\|_{\mX_h,\spn}\leq(1-\tilde\beta)^{\lfloor(H-\max\{h,h_1\})/m\rfloor}\,\Delta_B .
    \end{equation}
    In particular, \Cref{alg:pl_backward_sync} returns $(Q^*_{\Omega,h})_{h=0}^H$ exactly after $\sum_{h<H}|\mS_h||\mA||\mS_{h+1}|\leq H\bar S_H^2|\mA|$ model queries, one query being an evaluation of $p(s'|s,a)$, and $Q^*_{\Omega,0}$ is $\varepsilon$-accurate for $Q^*_\Omega$ on $\mX_0$ as soon as $H\geq h_1+m+H_{\mathrm{eff}}\log_+(\Delta_B/\varepsilon)$.
\end{proposition}

\begin{proof}
    With a model, $D_0=0$, so $e_0=\Delta_B$ and $\tilde\beta=\tilde\beta_m(A_{2\Delta_B})\leq\tilde\beta_m(A_{\Delta_B})$, $\gamma_\Omega$ being non-increasing. Since $n(h,H-h)=\lfloor(H-\max\{h,h_1\})/m\rfloor$, \eqref{eq:pl-horizon-main} follows from \Cref{cor:pl_horizon_bias}. The sweep computes $Q_h=\mT_{\Omega,h}(Q_{h+1})$ from $Q_H=Q_B|_{\mX_H}$, which is the definition of the planning solution, and each value $\mT_{\Omega,h}(Q_{h+1})(s,a)$ uses the $|\mS_{h+1}|$ probabilities $p(s'|s,a)$, $s'\in\mS_{h+1}$, whence the count. Finally, $H_{\mathrm{eff}}\geq H_{\mathrm{eff}}(A_{\Delta_B})$, so the condition on $H$ implies~\eqref{eq:pl_depth} at $h=0$.
\end{proof}

Compared with \Cref{cor:q-iteration}, model-based planning is $Q$-iteration restricted to the reachable set, with the boundary mismatch $\Delta_B$ in place of the initial error: it saves a factor $(\bar S_H/|\mS|)^2$ in model queries, which is non-trivial only when $\mS_\infty\subsetneq\mS$ (\Cref{rem:pl_local_mixing}.(ii)).

The same exactness holds for the depth-shifted recursion $Q_{j+1,h}=\mT_{\Omega,h}(Q_{j,h+1})$, which is the noiseless version of epoch planning (\Cref{alg:epoch_regularized_q_planning}): each level is computed from the previous iterate of the level below, and the initialization is forgotten after at most $H$ iterations.

\begin{lemma}[Exact forgetting of the initialization]
    \label{lem:pl_sync_nilpotent}
    Let $(Q_{j,h})_{j\geq0,\,h\leq H}$ satisfy $Q_{j+1,h}=\mT_{\Omega,h}(Q_{j,h+1})$ for $h<H$ and $Q_{j,H}=Q_B|_{\mX_H}$ for every $j$. Then $Q_{j,h}=Q^*_{\Omega,h}$ for every $h\leq H$ and $j\geq H-h$.
\end{lemma}

\begin{proof}
    Backward induction on $h$. The case $h=H$ holds by assumption. If $Q_{j,h+1}=Q^*_{\Omega,h+1}$ for every $j\geq H-h-1$, then $Q_{j+1,h}=\mT_{\Omega,h}(Q^*_{\Omega,h+1})=Q^*_{\Omega,h}$ for every such $j$, i.e.\ $Q_{j,h}=Q^*_{\Omega,h}$ for every $j\geq H-h$.
\end{proof}

\subsection{Epoch planning}
\label{subsec:pl_epoch}
 
\Cref{alg:epoch_regularized_q_planning} is the depth-indexed analogue of \Cref{cor:ep_alpha_one}: its step size is $\alpha_j\equiv1$, which makes the recursion over epochs nilpotent (\Cref{lem:pl_sync_nilpotent}), and its batch sizes $B_j$ are deterministic. Let $\mG_j$ be the $\sigma$-algebra generated by $Q_0$ and the rollouts of epochs $0,\dots,j-1$; the analysis covers \Cref{alg:epoch_regularized_q_planning} as written, with actions drawn from $\pi_{\Omega,Q_{j,h}}$ at level $h$, and its variant with actions drawn from $\mu$ (\Cref{rem:pl_behavior}); in both cases the behavior of epoch $j$ is $\mG_j$-measurable. As in \Cref{lem:ep_update}, on the event that epoch $j$ terminates, for $h<H$,
\begin{equation}
    \label{eq:pl_epoch_update}
    \begin{aligned}
        Q_{j+1,h}&=\mT_{\Omega,h}(Q_{j,h+1})+w_{j,h},\\
        w_{j,h}(s,a)&:=\frac{1}{B_j}\sum_{l=1}^{B_j}\Big(\Omega^*(Q_{j,h+1})(s'_l)-\bE_{s'\sim p(\cdot|s,a)}\big[\Omega^*(Q_{j,h+1})(s')\big]\Big),
    \end{aligned}
\end{equation}
where $s'_l$ is the successor of the $l$-th level-$h$ visit of $(s,a)$ during epoch $j$, and $Q_{j,H}=Q_B|_{\mX_H}$. We set $Q_{j+1,\cdot}:=Q_{j,\cdot}$ if epoch $j$ does not terminate, which happens with probability zero (\Cref{prop:pl_epoch_wellposed}).
 
\begin{remark}[Role of the outer epochs and single-epoch backward sweep under $\mu$]
    \label{rem:pl_no_backward_sweep}
    In \Cref{alg:epoch_regularized_q_planning}, level $h$ at epoch $j+1$ is computed from level $h+1$ at epoch $j$, which is $\mG_j$-measurable. Conditionally on $\mG_j$, the target $\Omega^*(Q_{j,h+1})$ is therefore fixed while the retained transitions of epoch $j$ are collected, which allows the algorithm to accumulate empirical Bellman backups online without storing transition tables and to update the behavior policies $\pi_{\Omega,Q_{j,h}}$ on-policy across epochs. Under a fixed behavior policy such as $\mu$, if one instead stores the first $B$ sampled next states for every $(h,s,a)$ in a single epoch and performs a backward sweep $\widehat Q_h(s,a)=r(s,a)+\frac{1}{B}\sum_{l=1}^B\Omega^*(\widehat Q_{h+1})(s'_l)$ from $\widehat Q_H=Q_B|_{\mX_H}$, a standard Markov coupling with independent successor tapes shows that the retained successor tables are mutually independent across depths, yielding the same root error guarantee in a single epoch of $\tilde O\big((h_1+H_{\mathrm{eff}})^3\bar w^2/(q_\mu\varepsilon^2)\big)$ transitions.
\end{remark}
 
\begin{proposition}[Well-posedness]
    \label{prop:pl_epoch_wellposed}
    Every epoch terminates almost surely, and every $Q_{j,h}$ is a finite, $\mG_j$-measurable random vector.
\end{proposition}
 
\begin{proof}
    Finiteness and measurability follow by induction on $j$ as in \Cref{lem:ep_update}.(ii), each coordinate of $w_{j,h}$ being bounded by $\|Q_{j,h+1}\|_\spn+\|\Omega^*(0)\|_{\mS,\spn}$ (\Cref{lem:omega_star_props}.(vii)). For termination, fix $j$ and a deterministic $u\geq0$. Conditionally on $\mG_j$, the rollouts of epoch $j$ are i.i.d.\ and, on $\{\max_{h<H}\|Q_{j,h}\|_\asp\leq u\}\in\mG_j$, each of them visits a given $(h,x)$, $x\in\mX_h$, with probability at least $\min\{q_\mu,q_\pi(u)\}>0$ by \Cref{lem:pl_coverage}. Each pair is thus visited $B_j$ times after finitely many rollouts, almost surely on that event, and letting $u\uparrow\infty$ concludes.
\end{proof}
 
For $R\geq e_0$, let
\begin{equation*}
    \sigma_R:=\inf\Big\{j\geq0\ \Big|\ \max_{h<H}\|Q_{j,h}-Q^*_\Omega\|_{\mX_h,\spn}>R\Big\},
\end{equation*}
a $(\mG_j)$-stopping time with $\sigma_R\geq1$, since $\|Q_{0,h}-Q^*_\Omega\|_{\mX_h,\spn}\leq D_0+\Delta_B=e_0$ by \Cref{lem:pl_uniform_bounds}, and write $\widehat{w}_{j,h}:=\indicator_{\{j<\sigma_R\}}w_{j,h}$. On $\{j<\sigma_R\}$, $\|Q_{j,h}\|_\asp\leq A_R$ and $\|\Omega^*(Q_{j,h})\|_{\mS_h,\spn}\leq\|Q_{j,h}\|_\spn+\|\Omega^*(0)\|_{\mS,\spn}\leq\bar{w}(R)$ for every $h\leq H$, the level $H$ being covered by $\Delta_B\leq e_0\leq R$. We write $q:=q_\pi(A_R)$, or $q:=q_\mu$ for the variant with actions drawn from $\mu$, and, for $J\geq1$ and $\delta\in(0,1)$, $\iota_H:=\log(4JH\bar S_H|\mA|/\delta)$.
 
\begin{lemma}[Noise]
    \label{lem:pl_epoch_noise}
    With probability at least $1-\delta/2$, $\|\widehat{w}_{j,h}\|_\infty\leq\bar{u}_j(R):=\bar{w}(R)\sqrt{\iota_H/(2B_j)}$ for all $j<J$ and $h<H$.
\end{lemma}
 
\begin{proof}
    Fix $(j,h,x)$ with $x\in\mX_h$, and order the $B_j$ terms of $\widehat{w}_{j,h}(x)$ by the successive level-$h$ visits of $x$ during epoch $j$, the rollouts being taken in chronological order. Given $\mG_j$ and the past of these visits, the successor of a visit is distributed as $p(\cdot|x)$, while $\Omega^*(Q_{j,h+1})$ is $\mG_j$-measurable. The terms are therefore martingale differences, lying in $\mG_j$-measurable intervals of length at most $\bar{w}(R)/B_j$ on $\{j<\sigma_R\}$ and vanishing off it, and their number is $B_j$ by the cap (\Cref{lem:ep_update}). \Cref{lem:azuma_range} with $u=\bar{u}_j(R)$ gives a failure probability $2e^{-\iota_H}$, and a union bound over the at most $JH\bar S_H|\mA|$ triples concludes.
\end{proof}
 
\begin{lemma}[Epoch length]
    \label{lem:pl_epoch_length}
    With probability at least $1-\delta/4$, every epoch $j<J$ with $j<\sigma_R$ lasts at most $N_jH$ transitions, where $N_j:=\lceil(\sqrt{B_j}+\sqrt{\iota_H})^2/q\rceil$.
\end{lemma}
 
\begin{proof}
    Conditionally on $\mG_j$, the rollouts of epoch $j$, extended beyond its end if needed, are i.i.d., and on $\{j<\sigma_R\}\in\mG_j$, \Cref{lem:pl_coverage} with $\mathcal{G}=\mG_j$ shows that each of them visits a given $(h,x)$, $x\in\mX_h$, with probability at least $q$. As in the proof of \Cref{lem:ep_cover_time}, the exact Chernoff bound~\eqref{eq:exact-chernoff} with $b=1$, $a=N_jq$ and $c=B_j$ bounds by $e^{-\iota_H}$ the conditional probability that fewer than $B_j$ of the first $N_j$ rollouts visit $(h,x)$, and a union bound over the at most $JH\bar S_H|\mA|$ triples concludes. Independence across resets replaces the $m$-step sub-sampling of \Cref{lem:ep_cover_time}.
\end{proof}
 
\begin{theorem}[Finite-time bound for epoch planning]
    \label{thm:pl_epoch}
    Let $R\geq e_0$ satisfy the \emph{radius admissibility condition}
    \begin{equation}
        \label{eq:pl_admissible}
        e_0+2M(A_R)\max_{j<J}\bar{u}_j(R)\leq R.
    \end{equation}
    Then, with probability at least $1-\delta$, every epoch $j<J$ lasts at most $N_jH$ transitions, $\max_{h<H}\|Q_{j,h}-Q^*_\Omega\|_{\mX_h,\spn}\leq R$ for every $j\leq J$, and, for every $j\leq J$ and $h<H$,
    \begin{equation}
        \label{eq:pl_epoch_bound}
        \begin{aligned}
            \big\|Q_{j,h}-Q^*_{\Omega,h}\big\|_\spn
            &\leq\indicator_{\{j<H-h\}}\left(1-\tilde{\beta}_m(A_R)\right)^{n(h,j)}D_0\\
            &\quad+2\sum_{k=0}^{\min\{j,H-h\}-1}\left(1-\tilde{\beta}_m(A_R)\right)^{n(h,k)}\bar{u}_{j-1-k}(R).
        \end{aligned}
    \end{equation}
    In particular, the initialization has vanished exactly for $j\geq H-h$.
\end{theorem}

\begin{proof}
    Work on the intersection of the events of \Cref{lem:pl_epoch_noise,lem:pl_epoch_length}, of probability at least $1-\delta$, on which all epochs terminate (\Cref{prop:pl_epoch_wellposed}). We prove by induction on $j\leq J$ that $\sigma_R>j$ and that~\eqref{eq:pl_epoch_bound} holds at rank $j$; the length bound then follows from \Cref{lem:pl_epoch_length}. At $j=0$, both claims were checked above. Assume $\sigma_R>j-1$, fix $h<H$, let $\ell:=\min\{j,H-h\}$ and consider the \emph{diagonal} profile $P_l:=Q_{j-(l-h),l}$, $l\in\llbracket h,h+\ell\rrbracket$. For $l<h+\ell$, the epoch $i:=j-(l-h)-1$ satisfies $0\leq i<j\leq\sigma_R$, so~\eqref{eq:pl_epoch_update} gives $P_l=\mT_{\Omega,l}(P_{l+1})+\widehat{w}_{i,l}$: the residuals of the diagonal are the stopped noises, and \Cref{lem:pl_epoch_noise} and $\|\cdot\|_\spn\leq2\|\cdot\|_\infty$ bound their span by $2\bar{u}_i(R)$. For $l>h$, $P_l$ is either computed at epoch $j-(l-h)<j\leq\sigma_R$ or equal to $Q_B|_{\mX_H}$, whose action span is at most $A_{\Delta_B}$ by \Cref{lem:pl_uniform_bounds}, so $\|P_l\|_\asp\leq A_R$, and $A_R\geq A_{\Delta_B}$. The last element is $P_{h+\ell}=Q_B|_{\mX_H}=Q^*_{\Omega,H}$ if $\ell=H-h$, and $P_{h+\ell}=Q_{0,h+j}$, within $D_0$ of $Q^*_{\Omega,h+j}$, if $\ell=j<H-h$. The first inequality of \Cref{lem:pl_master} with $h_0=h$ and $h_e=h+\ell$ gives~\eqref{eq:pl_epoch_bound} at rank $j$, since $n(h,\ell)=n(h,j)$ when $\ell=j$. Finally, \Cref{lem:pl_geometric_sum} and $M(h,A_R)\leq M(A_R)$ bound the sum in~\eqref{eq:pl_epoch_bound} by $2M(A_R)\max_{i<J}\bar{u}_i(R)$, so \Cref{lem:pl_uniform_bounds} and~\eqref{eq:pl_admissible} give $\|Q_{j,h}-Q^*_\Omega\|_{\mX_h,\spn}\leq D_0+2M(A_R)\max_{i<J}\bar{u}_i(R)+\Delta_B\leq R$ for every $h<H$, i.e.\ $\sigma_R>j$.
\end{proof}
 
Unlike~\eqref{eq:R-admissible} and~\eqref{eq:ep_R_admissible}, condition~\eqref{eq:pl_admissible} has no factor $\bar\beta(R)/(\bar\beta(R)+1)$: the profile at epoch $j$ is an explicit function of the data of epochs $i<j$ along finite diagonals, and the initialization is forgotten exactly, not geometrically.

\begin{proof}[Proof of \Cref{thm:pl-epoch-main} and \Cref{thm:pl-epoch-main-body}]
    We apply \Cref{thm:pl_epoch} with $B_j\equiv B$ and $R=2e_0$, for which~\eqref{eq:pl_epoch_bound} and \Cref{lem:pl_geometric_sum} give the upper bound $2M(h,A_R)\bar u(R)$. Then $\tilde{\beta}_m(A_R)=\tilde\beta$, $\bar{w}(R)=\bar w$, $M(A_R)=M$ and $q_\pi(A_R)$ is the constant $q$ of the greedy behavior, so that $2M\bar{u}(R)=2M\bar w\sqrt{\iota_H/(2B)}$, which is at most $e_0$ if and only if $B\geq B(e_0)$: condition~\eqref{eq:pl_admissible} holds. With probability at least $1-\delta$, every epoch lasts at most $NH=H\lceil(\sqrt B+\sqrt{\iota_H})^2/q\rceil$ transitions and, since $J\geq H$, \eqref{eq:pl_epoch_bound} at $h=0$ gives $\|Q_{J,0}-Q^*_{\Omega,0}\|_\spn\leq2M\bar w\sqrt{\iota_H/(2B)}$. \Cref{cor:pl_horizon_bias} at $h=0$, with $n(0,H)=\lfloor(H-h_1)/m\rfloor$ and $\tilde\beta\leq\tilde\beta_m(A_{\Delta_B})$, bounds $\|Q^*_{\Omega,0}-Q^*_\Omega\|_{\mX_0,\spn}$ by $(1-\tilde\beta)^{\lfloor(H-h_1)/m\rfloor}\Delta_B$, and the triangle inequality gives~\eqref{eq:pl-epoch-main} (hence~\eqref{eq:pl-epoch-main-body} with $C=\sqrt{2}$).
 
    For $\varepsilon\in(0,e_0)$, the choice $B=\lceil B(\varepsilon/2)\rceil\geq B(e_0)$ makes the noise term at most $\varepsilon/2$. If $\Delta_B\leq\varepsilon/2$, the horizon bias is already at most $\varepsilon/2$. Otherwise, the choice of $H$ gives $\lfloor(H-h_1)/m\rfloor\geq \tilde\beta^{-1}\log(2\Delta_B/\varepsilon)$, so that $1-u\leq e^{-u}$ makes the bias term at most $\varepsilon/2$. The $J=H$ epochs use at most $H\cdot NH$ transitions, with $N=\tilde\Theta(B/q)$, $B=\tilde\Theta(M^2\bar w^2/\varepsilon^2)$ and $H=\tilde O(h_1+H_{\mathrm{eff}})$, which gives $\tilde O\big((h_1+H_{\mathrm{eff}})^4\bar w^2/ (q\varepsilon^2)\big)$; the dependence of $\iota_H$ on $J=H$ is logarithmic.
\end{proof}
 
\begin{remark}[Planning and learning]
    \label{rem:pl_epoch_vs_learning}
    For $h_1=0$ (so that $\mX_0=\mS_\infty\times\mA$), \Cref{thm:pl-epoch-main} costs $\tilde O(m^4\bar w^2/(\tilde\beta^4q\varepsilon^2))$ transitions, against $\tilde O(m^4\bar w^2/(\tilde\beta^3q_R\varepsilon^2))$ for \Cref{thm:epoch-main-body}: at equal coverage, radius and noise constants, online epoch planning is a factor $1/\tilde\beta$ more expensive. In \Cref{alg:epoch_regularized_q_planning}, propagation is paid twice, once in depth across the $H$ distinct tables, and once across the $J=H$ epochs needed to propagate the boundary to the root. Per table, planning is $m$ times more efficient, the independence across resets removing the sub-sampling of \Cref{lem:ep_cover_time}. Pooling the transitions on $\mS_\infty$ into a single stationary table removes the extra factor $1/\tilde\beta$ (\Cref{subsec:pl_reset_learning}).
\end{remark}

\subsection{Merging the tables: epoch \texorpdfstring{$Q$}{Q}-learning with resets}
\label{subsec:pl_reset_learning}

The factor $1/\tilde{\beta}$ of \Cref{rem:pl_epoch_vs_learning} comes from maintaining $H$ separate depth-indexed tables. Yet all depths $h\geq h_1$ share the same state--action set $\mS_\infty\times\mA$ and the same stationary Bellman operator $\mT_\Omega$ (and by \Cref{cor:pl_horizon_bias}, $Q^*_{\Omega,h}$ is within $\varepsilon$ of $Q^*_\Omega$ in span whenever $h_1\leq h\leq H-m-H_{\mathrm{eff}}(A_{\Delta_B})\log_+(\Delta_B/\varepsilon)$). Pooling all depths $h\geq h_1$ into a single table on $\mS_\infty\times\mA$ gives \Cref{alg:pl_reset_learning}, which is \Cref{alg:epoch_regularized_q_learning} run on a reset process: the transient prefix of each rollout ($h<h_1$) is played with $\mu$ and discarded, and the transitions retained at depths $h\geq h_1$ are transitions of the MDP restricted to $\mS_\infty$, since $s_h,s_{h+1}\in\mS_\infty$. This MDP satisfies \Cref{assumption:m-step-mixing} (\Cref{rem:pl_local_mixing}.(i)). We keep the notation of \Cref{subsec:proofs_epoch_q_learning}, every constant being computed on $\mS_\infty$: $e_0=\|Q_0-Q^*_\Omega\|_\spn$, the localized constants of \Cref{def:localized-radii}, the exit time~\eqref{eq:ep_exit_time}, the blocks $\mB_i$ and the $\sigma$-algebras $\mG_j$. Only the coverage constant and the epoch length change.

\begin{lemma}[Cover time with resets]
    \label{lem:pl_reset_cover_time}
    Assume $H\geq h_1+m+1$ and let $K_0:=\lfloor(H-1-h_1)/m\rfloor\geq1$. For $j\geq0$, a deterministic radius $u\geq0$ and $\delta'\in(0,1)$, let $\mathcal{D}_j(u):=\{\|Q_j\|_\asp\leq u\}$, let $q^{\mathrm{res}}(u):=\beta_m\gamma_\Omega(u)^m\eta_{\min}$, or $q^{\mathrm{res}}(u):=\beta_m\eta_{\min}$ if the actions at depths $h\geq h_1$ are drawn from $\mu$ instead of $\pi_{\Omega,Q_j}$, let $\iota(\delta'):=\log(|\mS_\infty||\mA|/\delta')$ and
    \begin{equation}
        \label{eq:pl_reset_length}
        L^{\mathrm{res}}_j(u,\delta'):=H\left\lceil\frac{1}{K_0}\left\lceil\frac{\big(\sqrt{B_j}+\sqrt{\iota(\delta')}\big)^2}{q^{\mathrm{res}}(u)}\right\rceil\right\rceil.
    \end{equation}
    Then, almost surely, $\bP\big(\{\ell_j>L^{\mathrm{res}}_j(u,\delta')\}\cap\mathcal{D}_j(u)\mid\mG_j\big)\leq\delta'\,\indicator_{\mathcal{D}_j(u)}$. If moreover $H\geq2(h_1+m+1)$, then $L^{\mathrm{res}}_j(u,\delta')\leq2m\big\lceil(\sqrt{B_j}+\sqrt{\iota(\delta')})^2/q^{\mathrm{res}}(u)\big\rceil+H$.
\end{lemma}

\begin{proof}
    Conditionally on $\mG_j$, the rollouts of epoch $j$, extended beyond its end if needed, are i.i.d. Consider the instants of depths $h_1+km$, $k\in\llbracket1,K_0\rrbracket$, of the successive rollouts, in chronological order, fix $(s,a)\in\mS_\infty\times\mA$, and let $Z_l$ indicate that the $l$-th instant visits $(s,a)$. On $\mathcal{D}_j(u)\in\mG_j$, the actions at depths $h\geq h_1$ are drawn from $\mu$ or from a policy dominating $\gamma_\Omega(u)\mu$, and the $m$ steps before each instant start at a depth $h\geq h_1$, hence in $\mS_\infty$. \Cref{lem:pl_coverage}.(i) with $\mathcal{G}=\mG_j$, conditional on the previous instant of the same rollout or integrated over a fresh reset, therefore gives $\bE[Z_l\mid Z_1,\dots,Z_{l-1},\mG_j]\geq q^{\mathrm{res}}(u)$. The rest is the proof of \Cref{lem:ep_cover_time}: the exact Chernoff bound~\eqref{eq:exact-chernoff} with $b=1$, $a=Kq^{\mathrm{res}}(u)$ and $c=B_j$, where $K:=\lceil(\sqrt{B_j}+\sqrt{\iota(\delta')})^2/q^{\mathrm{res}}(u)\rceil$, and a union bound over $\mS_\infty\times\mA$ show that, with conditional probability at least $1-\delta'$, every pair is visited $B_j$ times within $K$ instants, which $\lceil K/K_0\rceil$ rollouts of $H$ transitions provide. Finally, $K_0\geq(H-1-h_1-m)/m\geq H/(2m)$ when $H\geq2(h_1+m+1)$.
\end{proof}

\begin{theorem}[Epoch $Q$-learning with resets]
    \label{thm:pl_reset_learning}
    Let $H\geq h_1+m+1$. \Cref{thm:ep_main}, under \Cref{assumption:lr}, and \Cref{cor:ep_alpha_one} hold for \Cref{alg:pl_reset_learning}, with every constant computed on the MDP restricted to $\mS_\infty$ and the epoch length $L_j(\delta')$ of \Cref{lem:ep_cover_time} replaced by $L^{\mathrm{res}}_j(A_R,\delta')$. In particular, \Cref{thm:epoch-main-body-detailed} holds with $L:=H\lceil\frac{1}{K_0}\lceil(\sqrt B+\sqrt\iota)^2/q^{\mathrm{res}}\rceil\rceil$, where $q^{\mathrm{res}}:=\bar\gamma\tilde\beta\eta_{\min}$, or $q^{\mathrm{res}}:=\beta_m\eta_{\min}$ if the actions at depths $h\geq h_1$ are drawn from $\mu$. For $H\geq2(h_1+m+1)$, this gives the sample complexity
    \begin{equation}
        \label{eq:pl_reset_complexity}
        \tilde{O}\left(\frac{m^4\,\bar{w}^2}{\tilde{\beta}^3\,q^{\mathrm{res}}\,\varepsilon^2}+\frac{mH}{\tilde{\beta}}\right),
    \end{equation}
    whose first term is $\tilde O\big(m^4\bar w^2/(\tilde\beta^3\cdot\beta_m\eta_{\min}\cdot\varepsilon^2)\big)$ under the behavior $\mu$ and $\tilde O\big(m^4\bar w^2/(\tilde\beta^3\cdot\bar\gamma\tilde\beta\eta_{\min}\cdot\varepsilon^2)\big)$ under the greedy behavior.
\end{theorem}

\begin{proof}
    The proofs of \Cref{subsec:proofs_epoch_q_learning} use the trajectory through three properties only: the behavior during epoch $j$ is $\mG_j$-measurable; given the past, the successor of a retained visit of $(s,a)$ is distributed as $p(\cdot|s,a)$ (\Cref{lem:ep_noise}.(ii)); and the cap makes the number of retained samples deterministic (\Cref{lem:ep_update}). The reset process satisfies all three, by the Markov property within each rollout and the independence of the rollouts given $\mG_j$. The virtual flow, tracking and no-escape arguments (\Cref{lem:ep_virtual_flow,lem:ep_tracking}, \Cref{prop:ep_no_escape}) only concern the update~\eqref{eq:ep_update}, and the trajectory structure only enters through \Cref{lem:ep_cover_time}, which \Cref{lem:pl_reset_cover_time} replaces, and \Cref{prop:ep_wellposed}, whose proof applies verbatim with it. The count follows as in the proof of \Cref{thm:epoch-main-body-detailed}, with $L\leq2m\lceil(\sqrt B+\sqrt\iota)^2/q^{\mathrm{res}}\rceil+H$ by \Cref{lem:pl_reset_cover_time}.
\end{proof}

\begin{remark}[What the merge buys and what it loses]
    \label{rem:pl_merge}
    \leavevmode
    \begin{enumerate}[label=(\roman*)]
        \item \emph{What the resets bring.} Relative to \Cref{thm:epoch-main-body-detailed}, \eqref{eq:pl_reset_complexity} takes $\eta_{\min}$ over $\mS_\infty\times\mA$, with the minorizing measure of \Cref{assumption:pl_local_mixing}, which can be much larger when $\mS_\infty\subsetneq\mS$. The factor $\gamma_\Omega^m$ removed under the behavior $\mu$ is not due to the resets: it is removed on a single trajectory as well (\Cref{rem:ep_behavior_mu}). Relative to epoch planning, merging the tables saves the factor $1/\tilde{\beta}$ of \Cref{rem:pl_epoch_vs_learning}, and the rollout length only enters through $H\geq2(h_1+m+1)$ and the second term of~\eqref{eq:pl_reset_complexity}.
        \item \emph{The boundary becomes the initialization.} There is no horizon bias: the target is $Q^*_\Omega$ itself, and the boundary value $Q_B$ of the planners plays the role of $Q_0$. The horizon bias $(1-\tilde\beta)^{\lfloor(H-h_1)/m\rfloor}\Delta_B$ of \Cref{thm:pl-epoch-main} and the initialization term $(1-\tilde\beta)^{\lfloor J(T)/m\rfloor}e_0$ of \Cref{thm:epoch-main-body-detailed} are the same term under two names.
        \item \emph{Target domain and centering.} The single table approximates a solution of $\mT_\Omega(Q^*_\Omega)=Q^*_\Omega+\rho^*_\Omega\ones$ on the core $\mS_\infty\times\mA$, with the centering projection $\Pi$ fixing an additive representative without affecting the span analysis. In contrast, \Cref{alg:epoch_regularized_q_planning} estimates the value on all reachable layers, including $\mX_0$ when $\mS_0\not\subseteq\mS_\infty$.
        \item \emph{A single parameterized algorithm.} Without resets, \Cref{alg:pl_reset_learning} is \Cref{alg:epoch_regularized_q_learning}, and \Cref{thm:pl_reset_learning,thm:epoch-main-body-detailed} are one statement in which $(H,\xi)$ only enters through the coverage constant and the epoch length.
    \end{enumerate}
\end{remark}

\section{The \texorpdfstring{$\alpha$}{α}-divergence regularizers}
\label{app:proofs_alpha}

\Cref{tab:alpha_family} summarizes \Cref{thm:alpha,cor:landmarks} for $\Omega_{\alpha,s}:=D_\alpha(\cdot\|\mu_s)$. In it, $Q_a:=Q(s,a)$, $\mu_a:=\mu(a|s)$, $\exp_q(x):=[1+(1-q)x]_+^{1/(1-q)}$ with $\exp_1=\exp$, and $\zeta_s$ is the normalizer making $\pi_{\Omega,Q}(\cdot|s)$ sum to one; strong convexity is with respect to $\|\cdot\|_1$ (for $\alpha>2$, it holds with modulus $2^{-\alpha}\mu_{\max}^{1-\alpha}$ when $|\mA|=2$ and fails when $|\mA|\geq3$), and the profile bounds are tight as $\mu_{\min}\to0$ (\Cref{rem:alpha_exact}).

On the simplex $D_\alpha(p\|q)=D_{1-\alpha}(q\|p)$ \citep{cichocki2010families}, so divergences to $\mu$ in either argument belong to the family. For a general reference $\mu_s$, the greedy map of $\Omega_{\alpha,s}$ is the $f$-softargmax of \citet{roulet2025loss} with reference measure $\mu_s$, and \citet{belousov2019entropic} give its closed form in MDPs. For $\mu$ uniform, $\Omega_{\alpha,s}$ is the Tsallis negentropy of index $\alpha$ up to the factor $|\mA|^{\alpha-1}$ on $\tau$ and an additive constant, and for $\alpha\geq1$ its greedy map is the $\alpha$-entmax of \citet{peters2019sparse}. Tsallis-KL regularization toward a reference policy, which is this family up to a reparametrization of the index and a rescaling of $\tau$, is used in reinforcement learning by \citet{zhu2023general}, and \citet{labbi2026beyond} use the same divergences for policy gradient. Parts (i)--(iii) of \Cref{thm:alpha} collect known facts in our notation; our contribution is the domination profile (iv) and the regularization modulus (v). Since $\Omega_{\alpha,s}\geq0$ vanishes at $\mu_s$, \Cref{lem:omega_star_props}.(ii) gives $g_Q\leq0$, while $g_Q\geq-\Lambda^{\mathrm{pt}}_{\Omega_\alpha}(R)$ for $\|Q\|_\asp\leq R$ by definition; hence $\Lambda_{\Omega_\alpha}\leq\Lambda^{\mathrm{pt}}_{\Omega_\alpha}$, and we bound the latter.

\begin{theorem}[The $\alpha$-divergence regularizer]
    \label{thm:alpha}
    Let $\alpha\in\bR$, $\tau>0$ and $\mu$ have full support.
    \begin{enumerate}[label=(\roman*)]
        \item \emph{Domain.} $\dom(\Omega_{\alpha,s})=\Delta(\mA)$ for $\alpha>0$ and $\{\pi_s>0\}$ for $\alpha\leq0$; $\Omega_{\alpha,s}$ is proper and closed, and \Cref{assumption:reference_policy_general} holds with $\tilde{\pi}=\mu$.
        \item \emph{Strong convexity} (Pinsker's inequality at $\alpha=1$). $\Omega_{\alpha,s}$ is $\mu_{\min}^{(1-\alpha)_+}$-strongly convex with respect to $\|\cdot\|_1$ for $\alpha\leq2$. For $\alpha>2$, it is $2^{-\alpha}\mu_{\max}^{1-\alpha}$-strongly convex if $|\mA|=2$, and strictly but not strongly convex if $|\mA|\geq3$.
        \item \emph{Greedy policy and conjugate} \citep[Proposition~1]{roulet2025loss}. For $\alpha\neq1$,
        \begin{align}
            \label{eq:alpha_policy_1}
            \pi_{\Omega_\alpha,Q}(a|s)&=\mu(a|s)\left[1+\frac{(\alpha-1)\big(Q(s,a)-\zeta_s(Q_s)\big)}{\tau}\right]_+^{\frac{1}{\alpha-1}},\\
            \label{eq:alpha_policy_2}
            \Omega^*_\alpha(Q)(s)&=\frac{(\alpha-1)\langle\pi_{\Omega_\alpha,Q}(\cdot|s),Q_s\rangle+\zeta_s(Q_s)}{\alpha}\quad(\alpha\neq0),
        \end{align}
        where $\zeta_s(Q_s)$ is the unique normalizer making~\eqref{eq:alpha_policy_1} sum to one; the truncation is active only for $\alpha>1$ (for $\alpha<1$, the bracket in~\eqref{eq:alpha_policy_1} is strictly positive). For any $a^*\in\argmax_aQ(s,a)$,
        \begin{equation*}
            \max_aQ(s,a)-\tau\,\frac{\mu(a^*|s)^{1-\alpha}-1}{\alpha-1}\ \leq\ \zeta_s(Q_s)\ \leq\ \max_aQ(s,a),
        \end{equation*}
        so that $\zeta_s$ is computable by bisection.
        \item \emph{Domination profile.} $\gamma_{\Omega_\alpha}(R)\geq\exp_{2-\alpha}(-R/\tau)$. If $\alpha>1$ and $|\mA|\geq2$, then $\gamma_{\Omega_\alpha}(R)=0$ for $R\geq\tau\mu_{\max}^{1-\alpha}/(\alpha-1)$, with $\mu_{\max}:=\max_{s,a}\mu(a|s)$. Hence, for $|\mA|\geq2$, \Cref{assumption:positive-domination} holds if and only if $\alpha\leq1$.
        \item \emph{Regularization modulus.} For $\alpha>0$, $\Lambda^{\mathrm{pt}}_{\Omega_\alpha}(R)\leq\tau\Delta_\alpha$ with $\Delta_\alpha:=\frac{\mu_{\min}^{1-\alpha}-1}{\alpha(\alpha-1)}$ ($\Delta_1:=\log(1/\mu_{\min})$ by continuity), and $\Delta_\alpha\leq\frac{1}{\alpha(1-\alpha)}$ for $0<\alpha<1$. For $\alpha\leq0$,
        \begin{equation}
            \label{eq:alpha_modulus}
            \Lambda^{\mathrm{pt}}_{\Omega_\alpha}(R)\ \leq\ \inf_{0<\varepsilon\leq1}\Big\{\varepsilon R+\tau\,\frac{1-\varepsilon^\alpha}{\alpha(1-\alpha)}\Big\}
            =\begin{cases}\Theta\big(\tau^{\frac{1}{1-\alpha}}R^{\frac{-\alpha}{1-\alpha}}\big)&\alpha<0,\\ \Theta\big(\tau\log(R/\tau)\big)&\alpha=0,\end{cases}
        \end{equation}
        as $R/\tau\to\infty$, the fraction read as $\log(1/\varepsilon)$ at $\alpha=0$, and these rates are attained when $|\mA|\geq2$. Hence, for $|\mA|\geq2$, $\Lambda^{\mathrm{pt}}_{\Omega_\alpha}(R)$ is bounded in $R$ if and only if $\alpha>0$ (and the same holds for $\Lambda_{\Omega_\alpha}(R)$ when $|\mS|\geq2$).
    \end{enumerate}
\end{theorem}

\begin{proof}
    Fix $s$ and write $\pi_a$, $\mu_a$, $Q_a$ for $\pi_s(a)$, $\mu(a|s)$, $Q(s,a)$.

    \emph{(i)} The term $\pi_a^\alpha\mu_a^{1-\alpha}$ is finite for every $\pi_a\geq0$ when $\alpha>0$, and infinite at $\pi_a=0$ when $\alpha<0$, as is $\log(\mu_a/\pi_a)$ at $\alpha=0$. $\Omega_{\alpha,s}$ is convex and continuous on its domain and tends to $+\infty$ at the boundary when $\alpha\leq0$, hence proper and closed; $D_\alpha(\mu_s\|\mu_s)=0$ gives \Cref{assumption:reference_policy_general} with $\tilde{\pi}=\mu$.

    \emph{(ii)} On the relative interior $\nabla^2\Omega_{\alpha,s}(\pi)=\operatorname{diag}(\pi_a^{\alpha-2}\mu_a^{1-\alpha})$, and Cauchy--Schwarz gives $\langle\nabla^2\Omega_{\alpha,s}(\pi)h,h\rangle\geq\|h\|_1^2\big/\sum_a\pi_a^{2-\alpha}\mu_a^{\alpha-1}$. For $1\leq\alpha\leq2$, H\"older bounds the sum by $(\sum_a\pi_a)^{2-\alpha}(\sum_a\mu_a)^{\alpha-1}=1$; for $\alpha<1$, $\pi_a^{2-\alpha}\leq\pi_a$ bounds it by $\mu_{\min}^{\alpha-1}$. Integrating along segments and closing by lower semicontinuity gives the modulus. For $\alpha>2$, on the interior we have $\max_a\pi_a\geq1/2$ when $|\mA|=2$, so $\langle\nabla^2\Omega_{\alpha,s}(\pi)h,h\rangle\geq(1/2)^{\alpha-2}\mu_{\max}^{1-\alpha}(\|h\|_1/2)^2=2^{-\alpha}\mu_{\max}^{1-\alpha}\|h\|_1^2$, whereas for $|\mA|\geq3$ the curvature along $e_{a_1}-e_{a_2}$ vanishes as $\pi_{a_1},\pi_{a_2}\to0$.

    \emph{(iii)} The form~\eqref{eq:alpha_policy_1} of the maximizer, with a unique normalizer, is \citet[Proposition~1]{roulet2025loss} with reference measure $\mu_s$ and logits $Q_s/\tau$. For the bracket, $S(\zeta):=\sum_a\mu_a[1+(\alpha-1)(Q_a-\zeta)/\tau]_+^{1/(\alpha-1)}$ is continuous and non-increasing on its domain, $\bR$ if $\alpha>1$ and $(\max_aQ_a-\tau/(1-\alpha),+\infty)$ if $\alpha<1$; $S(\max_aQ_a)\leq1$, every ratio $\pi_a/\mu_a$ being then at most $1$; and $S\geq1$ at $\zeta=\max_aQ_a-\tau(\mu_{a^*}^{1-\alpha}-1)/(\alpha-1)$, which lies in the domain and where the term of $a^*$ alone equals $1$. Finally, for $\alpha\neq0$, with $u_a:=1+(\alpha-1)(Q_a-\zeta_s)/\tau$, we have $\pi_a^\alpha\mu_a^{1-\alpha}=\pi_au_a$ and $\sum_a\pi_au_a=1+(\alpha-1)(\langle\pi,Q_s\rangle-\zeta_s)/\tau$, hence $\tau\Omega_{\alpha,s}(\pi)=(\langle\pi,Q_s\rangle-\zeta_s)/\alpha$, and $\Omega^*_\alpha(Q)(s)=\langle\pi,Q_s\rangle-\tau\Omega_{\alpha,s}(\pi)$ is~\eqref{eq:alpha_policy_2}.

    \emph{(iv)} If $\|Q\|_\asp\leq R$ then $Q_a-\zeta_s\geq Q_a-\max_bQ_b\geq-R$ by (iii), and $x\mapsto[1+(\alpha-1)x/\tau]_+^{1/(\alpha-1)}$ is non-decreasing for every $\alpha\neq1$, which gives the bound; $\alpha=1$ is the limit. For $\alpha>1$, take $a^*$ with $\mu_{a^*}=\mu_{\max}$, $Q_{a^*}=0$ and $Q_b=-R$ otherwise. The candidate $\pi=\delta_{a^*}$ has normalizer $\zeta=\tau(1-\mu_{\max}^{1-\alpha})/(\alpha-1)$ and $u_b=\mu_{\max}^{1-\alpha}-(\alpha-1)R/\tau\leq0$ for $b\neq a^*$ as soon as $R\geq\tau\mu_{\max}^{1-\alpha}/(\alpha-1)$; by uniqueness it is the greedy policy, so $\gamma_{\Omega_\alpha}(R)=0$. That the greedy map can be sparse for $\alpha>1$ is known \citep{peters2019sparse,roulet2025loss}; what \Cref{assumption:positive-domination} requires, and what (iv) provides, is the uniform lower bound and the explicit threshold.

    \emph{(v)} Let $a^*\in\argmax_aQ_a$. For $\alpha>0$, $\Omega_{\alpha,s}$ is bounded on $\Delta(\mA)$ by $0\leq\Omega_{\alpha,s}\leq\max_a\Omega_{\alpha,s}(\delta_a)\leq\Delta_\alpha$, the maximum of a convex function on the simplex being attained at a vertex, and the boundedness property of the conjugate \citep[Proposition~1(ii)]{geist2019theory} gives $\max_aQ_a-\tau\Delta_\alpha\leq\Omega^*_\alpha(Q)(s)\leq\max_aQ_a$; testing the maximum defining $\Omega^*_\alpha(Q)(s)$ against $\delta_{a^*}$ gives the same bound with $\mu_{a^*}$ in place of $\mu_{\min}$. For $0<\alpha<1$, $\Delta_\alpha=\frac{1-\mu_{\min}^{1-\alpha}}{\alpha(1-\alpha)}\leq\frac{1}{\alpha(1-\alpha)}$. For $\alpha\leq0$, $\Omega_{\alpha,s}$ is unbounded and we test instead $\pi^\varepsilon:=(1-\varepsilon)\delta_{a^*}+\varepsilon\mu_s$: $\langle\pi^\varepsilon,Q_s\rangle\geq\max_aQ_a-\varepsilon R$, and $\pi^\varepsilon\geq\varepsilon\mu_s$ with $x\mapsto x^\alpha$ non-increasing gives $\Omega_{\alpha,s}(\pi^\varepsilon)\leq\frac{\varepsilon^\alpha-1}{\alpha(\alpha-1)}$, with limit $\log(1/\varepsilon)$ at $\alpha=0$. Minimizing over $\varepsilon$, at $\varepsilon=\min\{1,(\tau/((1-\alpha)R))^{1/(1-\alpha)}\}$, gives~\eqref{eq:alpha_modulus}. The rates are attained for $|\mA|\geq2$ with $Q_{a^*}=R$ and $Q_b=0$ for $b\neq a^*$: writing $\delta=\sum_{b\neq a^*}\pi_b$ (which by convexity is optimally split as $\pi_b=\delta\mu_b/(1-\mu_{a^*})$), the gap is $\min_\delta\{R\delta+\tau\Omega_{\alpha,s}(\pi)\}$ with $\Omega_{\alpha,s}(\pi)\sim(1-\mu_{a^*})^{1-\alpha}\delta^\alpha/(\alpha(\alpha-1))$ for $\alpha<0$ and $\sim(1-\mu_{a^*})\log(1/\delta)$ for $\alpha=0$, whose minimizer gives the same powers of $R$. When $|\mS|\geq2$, $\Lambda_{\Omega_\alpha}(R)=\Lambda^{\mathrm{pt}}_{\Omega_\alpha}(R)$ by the proof of \Cref{prop:gap_optimal_q}.(ii).
\end{proof}

\begin{corollary}[Four landmarks]
    \label{cor:landmarks}
    \Cref{thm:alpha} at $\alpha\in\{1,0,\tfrac12,2\}$ gives the columns of \Cref{tab:alpha_family}, and the conjugates:
    \begin{center}
        \renewcommand{\arraystretch}{1.4}
        \begin{tabular}{@{}lll@{}}
            \toprule
            $\alpha$ & $\Omega^*(Q)(s)$ & $\zeta_s\in$ \\
            \midrule
            $1$ & $\tau\log\sum_a\mu_ae^{Q_a/\tau}$ & -- \\
            $0$ & $\zeta_s-\tau\sum_a\mu_a\log\big(1+\frac{\zeta_s-Q_a}{\tau}\big)$ & $(\max_aQ_a-\tau,\max_aQ_a]$ \\
            $\tfrac12$ & $2\zeta_s-\langle\pi_{\Omega,Q}(\cdot|s),Q_s\rangle$ & $(\max_aQ_a-2\tau,\max_aQ_a]$ \\
            $2$ & $\tfrac12\big(\langle\pi_{\Omega,Q}(\cdot|s),Q_s\rangle+\zeta_s\big)$ & $[\min_aQ_a,\max_aQ_a]$ \\
            \bottomrule
        \end{tabular}
    \end{center}
\end{corollary}

\begin{proof}
    At $\alpha=1$ the normalizer factors out of the exponential. At $\alpha=0$, \eqref{eq:alpha_policy_1} reads $\pi_a=\tau\mu_a/(\tau+\zeta_s-Q_a)$, so $\sum_a\pi_a(\tau+\zeta_s-Q_a)=\tau$ gives $\langle\pi,Q_s\rangle=\zeta_s$, and $\Omega^*(Q)(s)=\langle\pi,Q_s\rangle-\tau\KL(\mu_s\|\pi)$ gives the conjugate; the modulus bound is~\eqref{eq:alpha_modulus} at $\varepsilon=\min\{1,\tau/R\}$. At $\alpha=\tfrac12$ and $\alpha=2$ the conjugates are those of \Cref{thm:alpha}.(iii), and $\Omega_s(\delta_a)=4(1-\sqrt{\mu_a})$ and $\tfrac12(1/\mu_a-1)$ give the modulus bounds. The brackets at $\alpha\in\{0,\tfrac{1}{2}\}$ follow from \Cref{thm:alpha}.(iii); at $\alpha=2$, the sum in~\eqref{eq:alpha_policy_1} is at least $1$ at $\zeta=\min_aQ_a$ and at most $1$ at $\zeta=\max_aQ_a$.
\end{proof}

\begin{remark}[Sharpness]
    \label{rem:alpha_exact}
    For a fixed $\mu$, $\zeta_s$ is non-decreasing in each $Q_a$ while $\pi_b/\mu_b$ is non-increasing in $\zeta_s-Q_b$, so the infimum defining $\gamma_\Omega(R)$ is attained with the least likely action at the bottom and every other one at $+R$; symmetrically, $\Lambda_\Omega^{\mathrm{pt}}(R)$ is attained with the least likely action alone at the top and every other one at $-R$. Both depend on $\mu$ only through $\mu_{\min}$, and the explicit normalizers give, for instance, $\gamma_\Omega(R)=\big(\mu_{\min}+(1-\mu_{\min})e^{R/\tau}\big)^{-1}$ for the reverse KL and $\gamma_\Omega(R)=[1-(1-\mu_{\min})R/\tau]_+$ for the $\chi^2$, whose limits as $\mu_{\min}\to0$ are the bounds of \Cref{tab:alpha_family}. Inside the window $0<\alpha\leq1$, the comparison is quantitative and depends on how $\tau$ is matched across $\alpha$. For $0<\alpha<1$, the profile decays polynomially as $(R/\tau)^{-1/(1-\alpha)}$ when $R/\tau\to\infty$, more slowly for smaller $\alpha$; at $\alpha=1$, it decays exponentially. For $0<\alpha<1$, the modulus constant $\Delta_\alpha=(1-\mu_{\min}^{1-\alpha})/(\alpha(1-\alpha))$ tends, as $\mu_{\min}\to0$, to $1/(\alpha(1-\alpha))$, whose minimum is $4$ at the Hellinger value $\alpha=1/2$. In contrast, $\Delta_1=\log(1/\mu_{\min})\to+\infty$. Thus, among fixed indices in this window, Hellinger minimizes the limiting modulus constant as $\mu_{\min}\to0$, whereas for a balanced $\mu$ the reverse KL can have the smaller constant: for a uniform $\mu$, $\log|\mA|<4(1-|\mA|^{-1/2})$ whenever $2\leq|\mA|\leq24$. The lower bound of \Cref{thm:alpha}.(iv) equals $(f')^{-1}(-R/\tau)$ for the generator $f(u)=(u^\alpha-\alpha u+\alpha-1)/(\alpha(\alpha-1))$ of $D_\alpha$, the inverse link that also drives the policy-gradient analysis of \citet{labbi2026beyond}: their probability floor \citep[Theorem~3.5]{labbi2026beyond} is proportional to $(f')^{-1}$ evaluated at an argument of order $-1/\lambda$, where $\lambda$ is their temperature, up to factors depending on the discount and the initial distribution, and their exponential-versus-polynomial dichotomy between the KL and $\alpha<1$ is the one displayed by our profile bounds. They also find that the best $\alpha$ lies strictly inside $(0,1)$ \citep[Corollary~4.8]{labbi2026beyond}, for a different criterion.
\end{remark}

\end{document}